%% file: main_onecolumn.tex
\documentclass[11pt]{article}
\usepackage{fullpage}
\usepackage{natbib}

\usepackage{microtype}
\usepackage{graphicx}
\usepackage{subfigure}
\usepackage{booktabs} % for professional tables
\usepackage{amsmath,amsthm,amssymb}
\usepackage{graphicx}
\usepackage{enumitem}
\usepackage{hyperref}
\usepackage{url}
\usepackage[ruled]{algorithm2e}

\usepackage{hyperref}

\newcommand{\E}{\mathbf{E}}

\newcommand{\minimize}{\mathop{\mathrm{minimize}{}}}

\newcommand{\R}{\mathbf{R}}

\newcommand{\dk}{d^k}
\newcommand{\dkm}{d^{k-1}}

\newcommand{\hk}{h^k}
\newcommand{\hkm}{h^{k-1}}

\newcommand{\xk}{x^k}
\newcommand{\xkp}{x^{k+1}}

\newcommand{\xik}{{\xi_k}}
\newcommand{\gk}{g^k}

\newtheorem{prop}{Proposition}

\title{\huge Statistical Adaptive Stochastic Gradient Methods}

\author{
Pengchuan Zhang\thanks{
Microsoft Research AI, Redmond, WA 98052, USA (\texttt{penzhan@microsoft.com}).}
\and
Hunter Lang\thanks{
Microsoft Research NExT, Redmond, WA 98052, USA (\texttt{hjl@mit.edu}).} 
\and 
Qiang Liu\thanks{
University of Texas at Austin, Austin, TX 78712, USA (\texttt{lqiang@cs.utexas.edu}).}
\and
Lin Xiao\thanks{
Microsoft Research AI, Redmond, WA 98052, USA (\texttt{lin.xiao@microsoft.com}).}
\vspace{1ex}
}

\date{February 24, 2020}

\begin{document}
\maketitle

\begin{abstract}
We propose a statistical adaptive procedure called SALSA for automatically scheduling the learning rate (step size) in stochastic gradient methods. 
SALSA first uses a smoothed stochastic line-search procedure to gradually increase the learning rate, then automatically switches to a statistical method to decrease the learning rate.
The line search procedure ``warms up'' the optimization process, 
reducing the need for expensive trial and error in setting an initial learning rate.
The method for decreasing the learning rate is based on a new statistical test for detecting stationarity when using a constant step size. 
Unlike in prior work, our test applies to a broad class of stochastic gradient algorithms without modification.
The combined method is highly robust and autonomous, and it matches the performance of the best hand-tuned learning rate schedules in our experiments on several deep learning tasks.
\end{abstract}

\input{introduction_1c.tex}

\input{sasa_plus_1c.tex}
\input{ssls_1c.tex}
\input{salsa_experiments_1c.tex}

\input{conclusions_1c.tex}

\bibliographystyle{plainnat}
\bibliography{salsa}

\clearpage

\appendix
\input{morestatistics_1c.tex}
\input{sasavssasaplus_1c.tex}
\input{more_experiments_1c.tex}

\end{document}

%% file: introduction_1c.tex
\section{Introduction}

We study adaptive stochastic optimization methods in the context of 
large-scale machine learning. 
Specifically, we consider the stochastic optimization problem
\begin{equation}\label{eqn:stoch-opt}
\minimize_{x\in\R^p}\quad F(x)\triangleq \E_\xi\bigl[f_\xi(x)\bigr],
\end{equation}
where $\xi$ is a random variable representing data sampled from some
(unknown) probability distribution, $x\in\R^p$ represents the
parameters of the machine learning model 
(e.g., the weight matrices in a neural network), 
and $f_\xi$ is the loss function associated with a random sample~$\xi$ (which can also be a mini-batch of samples).
%The objective function~$F$ is the expectation of $f_\xi$ over the distribution of the data.

Many stochastic optimization methods for solving 
problem~\eqref{eqn:stoch-opt} can be written in the form of
\begin{equation}\label{eqn:sgm-general}
    \xkp = \xk - \alpha_k \dk,
\end{equation}
where $\dk$ is a (stochastic) search direction 
and $\alpha_k>0$ is the step size or learning rate.  
In the classical stochastic gradient descent (SGD) method, 
\begin{equation}\label{eqn:d-sgd}
    \dk=\gk\triangleq\nabla f_{\xi^k}(\xk),
\vspace{1ex}
\end{equation}
where $\xi^k$ is a training example (or a mini-batch of examples)
randomly sampled at iteration~$k$.
This method traces back to the seminal work of
\citet{RobbinsMonro51}, and it has become very popular in machine
learning \citep[e.g.,][]{Bottou98,GoodfellowBengioCourville16book}.
Many modifications of SGD aim to improve its theoretical and practical performance. 
For example, \citet{GupalBazhenov72} studied a stochastic analog (SHB) of the 
heavy-ball method \citep{Polyak64heavyball}, where
%\vspace{-1ex}
\begin{equation}\label{eqn:d-shb}
    \dk = (1-\beta_k)\gk + \beta_k \dkm,
\end{equation}
and $\beta_k\in[0, 1)$ is the momentum coefficient. 
\citet{SutskeverMartensDahlHinton13} proposed to use a stochastic variant of
Nesterov's accelerated gradient (NAG) method \citep{Nesterov04book}, where
\begin{equation}\label{eqn:d-nes}
    \dk = \nabla f_{\xik}(\xk-\alpha_k\beta_k\dkm) + \beta_k\dkm.
\end{equation}
Other recent variants include, e.g., \citet{Jain2018accelerating}, \citet{kidambi2018insufficiency}
and \citet{ma2019qh}.

Theoretical conditions for the asymptotic convergence of SGD
are well studied, and they mostly focus on polynomially decaying learning
rates of the form $\alpha_k=a/(b+k)^c$ for some $a,b>0$ and $1/2<c\leq 1$
\citep[e.g.,][]{RobbinsMonro51,PolyakJuditsky92acceleration}.
%\citep[e.g.,][]{KushnerYin03book}.
Similar conditions for the stochastic heavy-ball methods are also established
\citep[e.g.,][]{GupalBazhenov72,Polyak77comparison}.
However, these learning rate schedules still require significant hyperparameter tuning efforts in modern machine learning practice. 

Adaptive rules for adjusting the learning rate and other parameters on the fly 
have been developed in both the stochastic optimization literature
\citep[e.g.,][]{Kesten58,MirzoakhmedovUryasev83,RuszczynskiSyski83stats,
RuszczynskiSyski84IFAC,RuszczynskiSyski86Cesaro,RuszczynskiSyski86forcing,DelyonJuditsky93} 
and by the machine learning community
\citep[e.g.,][]{Jacobs88DBD,Sutton92IDBD,Schraudolph1999,
MahmoodSutton12TuningFree,Baydin18hypergradient}.  
Recently, adaptive algorithms that use diagonal scaling --- replacing
$\alpha_k$ in~\eqref{eqn:sgm-general} with an adjustable diagonal
matrix --- have become very popular
\citep[e.g.,]{DuchiHazanSinger2011adagrad,
TielemanHinton2012lecture,KingmaBa2014adam}.
Despite these advances, costly hand-tuning efforts are still needed to obtain good 
(generalization) performance \citep{wilson2017marginal}.
%performance in practice \citep{wilson2017marginal}.

A typical procedure for hand-tuning the learning rate is as follows.
First, trial and error is required to set an initial learning rate, which cannot be too small (leading to very slow training) or too large (causing instability or bad train and test performance).
The choice of a good initial rate often depends on the particular model and dataset used in training.
Then the learning rate must be gradually decreased, which is critical for the convergence of stochastic gradient methods and also for obtaining good testing performance in practice. 
A very popular scheme is to use a ``constant-and-cut'' learning-rate schedule, which decreases the learning rate by a constant factor after some fixed number of epochs. Both the factor of reduction and number of epochs between reductions also require trial and error to set.

\begin{figure}[t]
    \centering
    \includegraphics[width=0.45\linewidth]{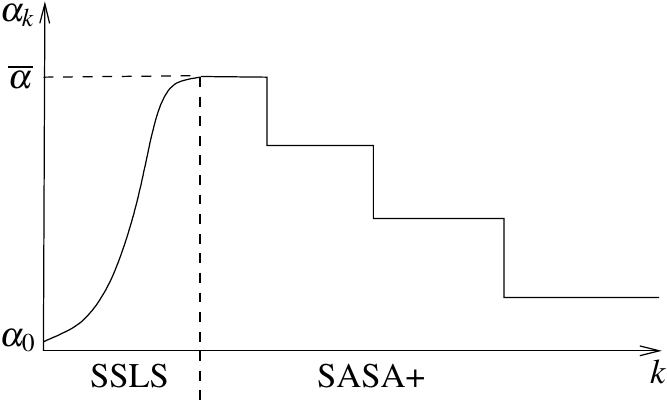}
%    \vspace{-1em}
    \caption{Typical learning rate profile of SALSA (log scale).}
%    \vspace{-em}
    \label{fig:salsa_profile}
\end{figure}

In this paper, we study statistical methods for automatically
scheduling the learning rate, and make the following contributions.
%\begin{itemize}[leftmargin=1em,topsep=0pt,itemsep=0pt]
\begin{itemize}
\item 
    We propose a smoothed stochastic line-search (SSLS) method that ``warms up'' the optimization process.
    Starting from a small, arbitrary learning rate, this method  
    increases it to a larger value to enable fast initial progress of the stochastic gradient method. 
    Unlike most line-search methods in optimization \citep[e.g.,][Chapter~3]{NocedalWright2006book}, 
    the purpose of SSLS is \emph{not} asymptotic convergence, but to automatically reach a stable learning rate that is well-suited to the training task.
    We demonstrate empirically that it significantly reduces the need for trial and error in setting an appropriate initial learning rate.
\item  
    Given a good initial learning rate, %(found by SSLS or trial and error), 
    we propose a statistical procedure called SASA+ that can automatically decrease the learning rate to obtain good training and test performance.
    SASA+ uses a statistical test to detect if the optimization process with a constant learning rate has reached stationarity, signaling slow training progress. 
    Whenever the test fires, the learning rate is decreased by a fixed factor. 
    SASA+ is an extension of SASA (Statistical Adaptive Stochastic Approximation) proposed by
    \citet{LangZhangXiao2019}, whose statistical test only works for 
    the stochastic heavy ball (SHB) method in~\eqref{eqn:d-shb}. 
    In SASA+, we derive a new stationarity test that applies to a broad class of stochastic methods without modification.
    %Moreover, we simplify the statistical test itself!
    %Because of its conceptual and analytical simplicity, it greatly simplifies implementation and deployment in software packages.
% \iffalse
%     We consider a broad family of stochastic optimization methods
%     with constant hyperparameters (including the learning rate and 
%     various forms of momentum) and derive a more general necessary condition 
%     (than Yaida's) for the associated learning process to be stationary.
%     \citet{yaida2018fluctuation} established a ``master equation,'' from which
%     different stationarity conditions can be derived for different variants 
%     of stochastic gradient methods 
%     (some of them may not admit a usable analytical form).
%     We derive a simple ``master condition'' that works for different methods 
%     without any change and holds \emph{exactly} for any stationary process.
%     Because of its conceptual and analytical simplicity, 
%     it greatly simplifies implementation and deployment in software packages.
% \item 
%     We develop a simple statistical test for checking stationarity based on
%     our ``master condition.'' 
%     It is a simple confidence interval test (checking if an interval contains zero),
%     as opposed to the ``equivalence test'' for Yaida's 
%     condition used by \citet{LangZhangXiao2019}.
%     This simple test is as robust as the equivalence test 
%     %of \citet{LangZhangXiao2019} 
%     and avoids the use of an additional hyperparameter.
% \fi
\end{itemize}

\begin{figure}[t]
    \centering
    \includegraphics[width=0.33\linewidth]{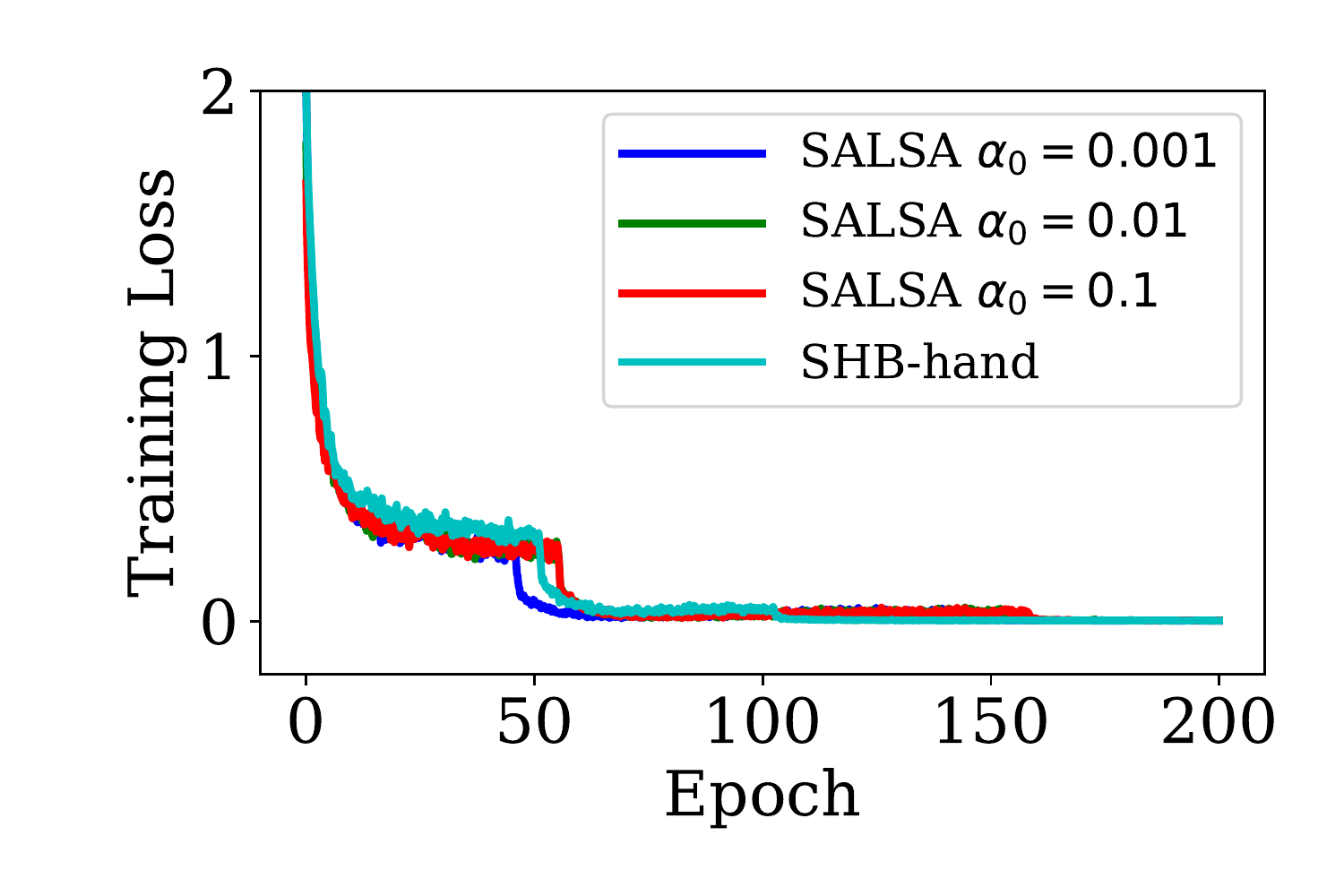}\quad
    \includegraphics[width=0.33\linewidth]{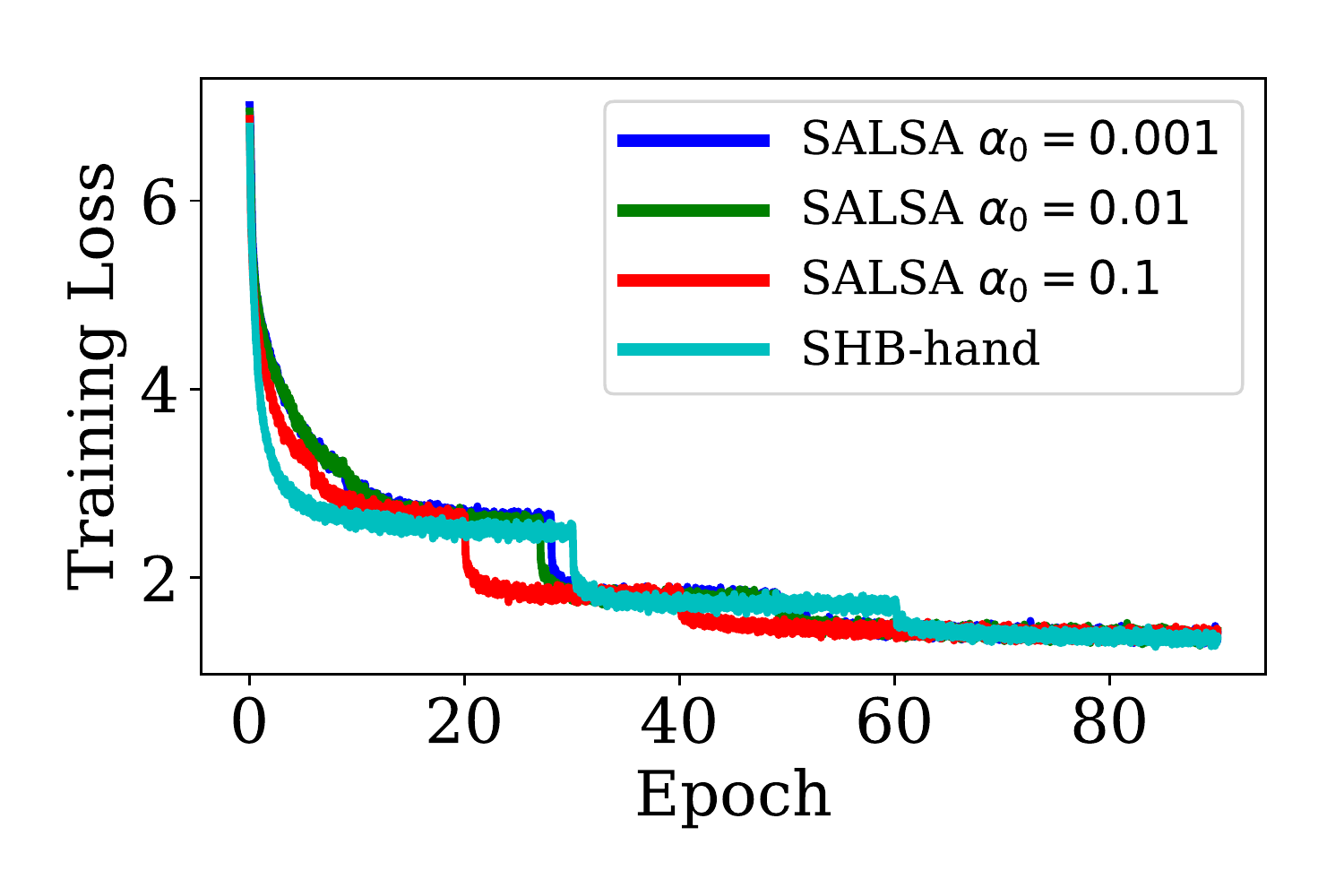}\\
    \includegraphics[width=0.33\linewidth]{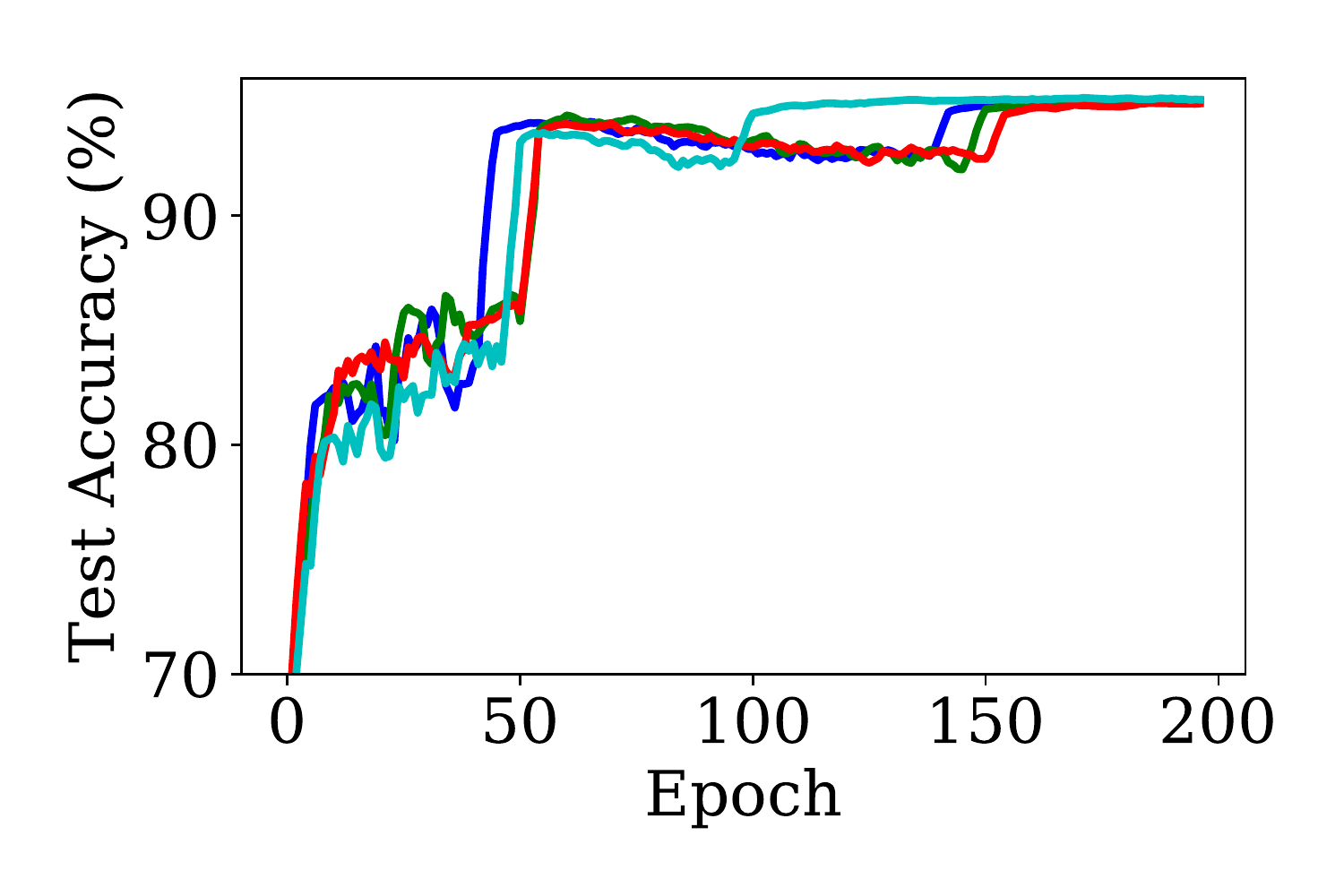}\quad
    \includegraphics[width=0.33\linewidth]{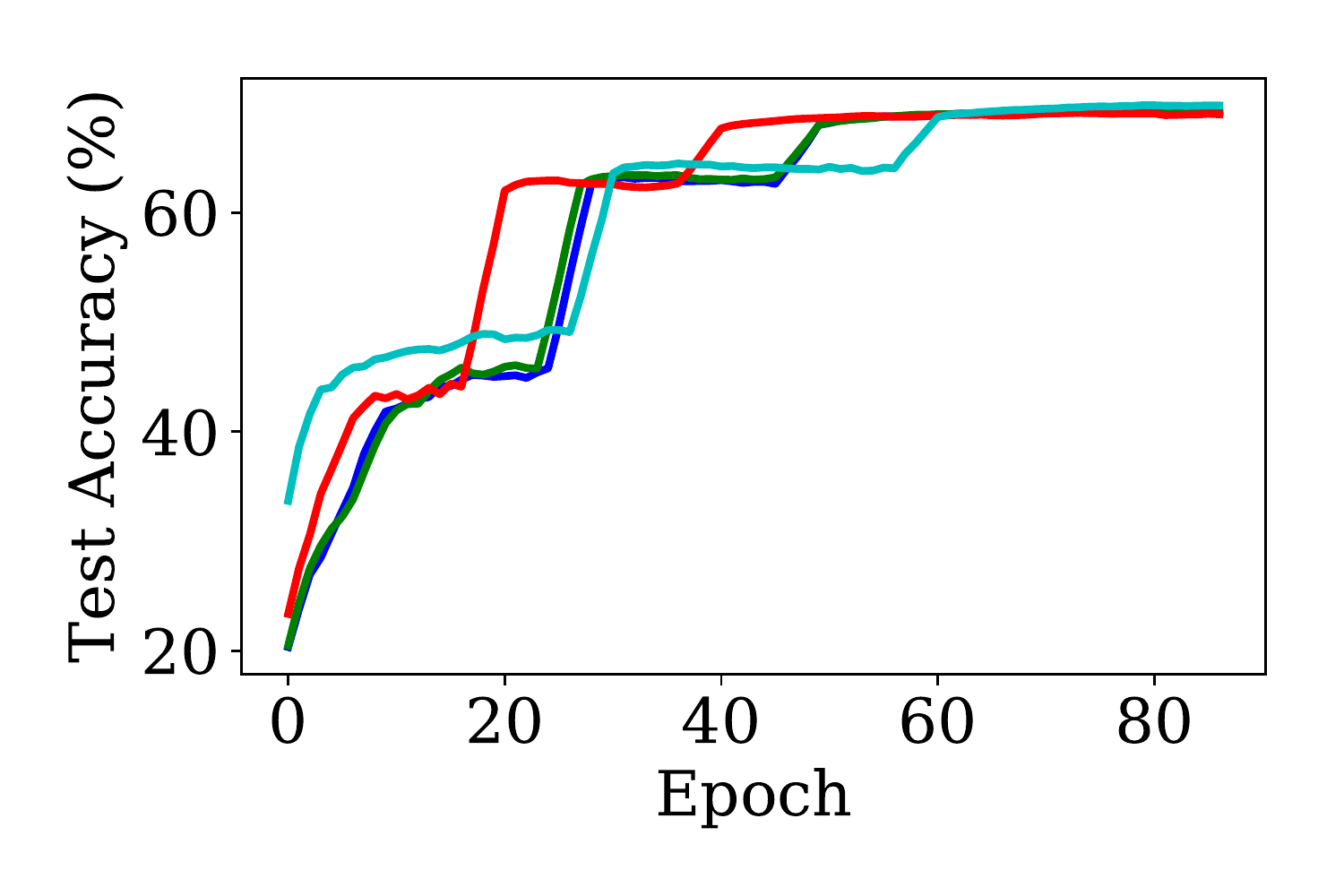}\\
    \includegraphics[width=0.33\linewidth]{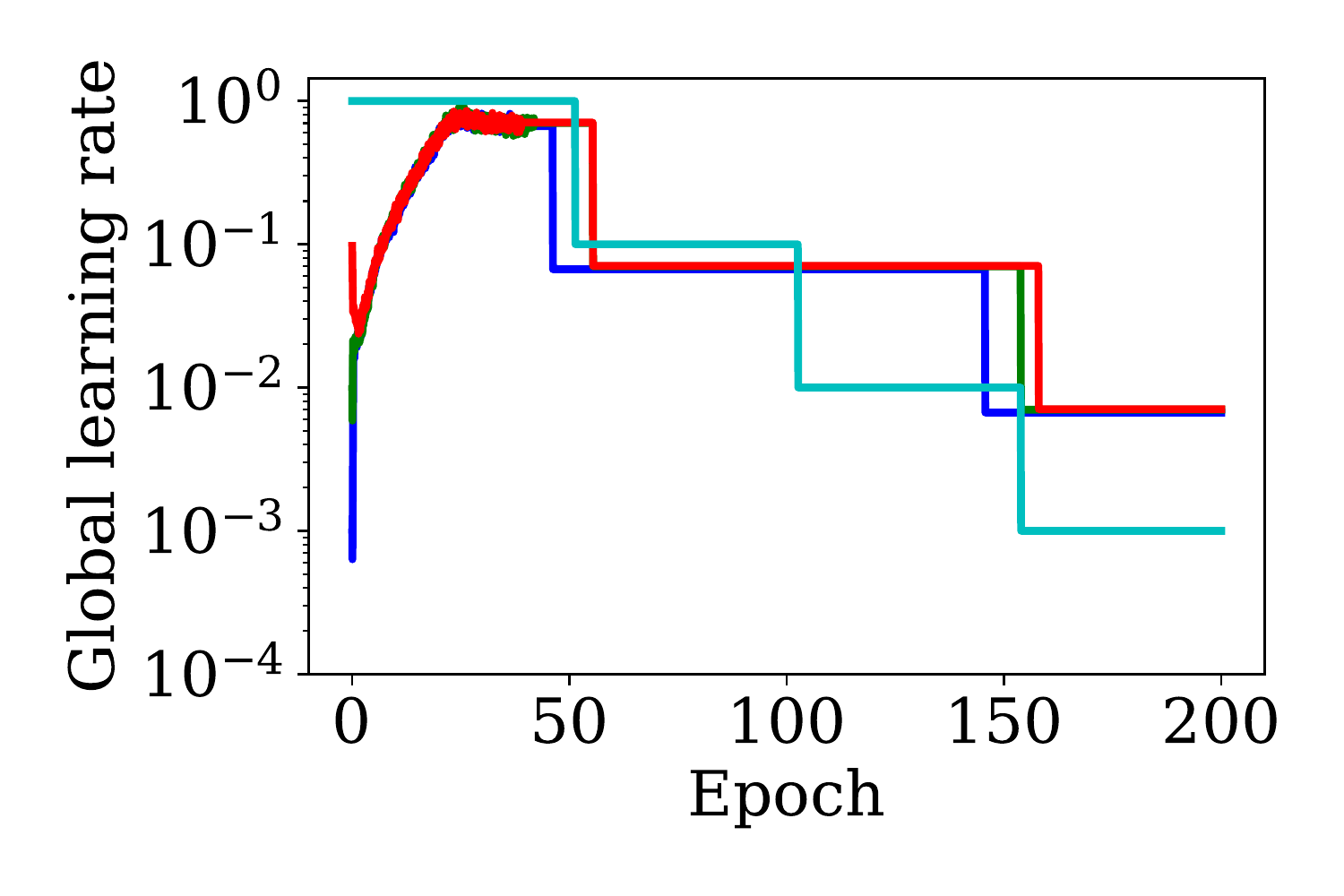}\quad
    \includegraphics[width=0.33\linewidth]{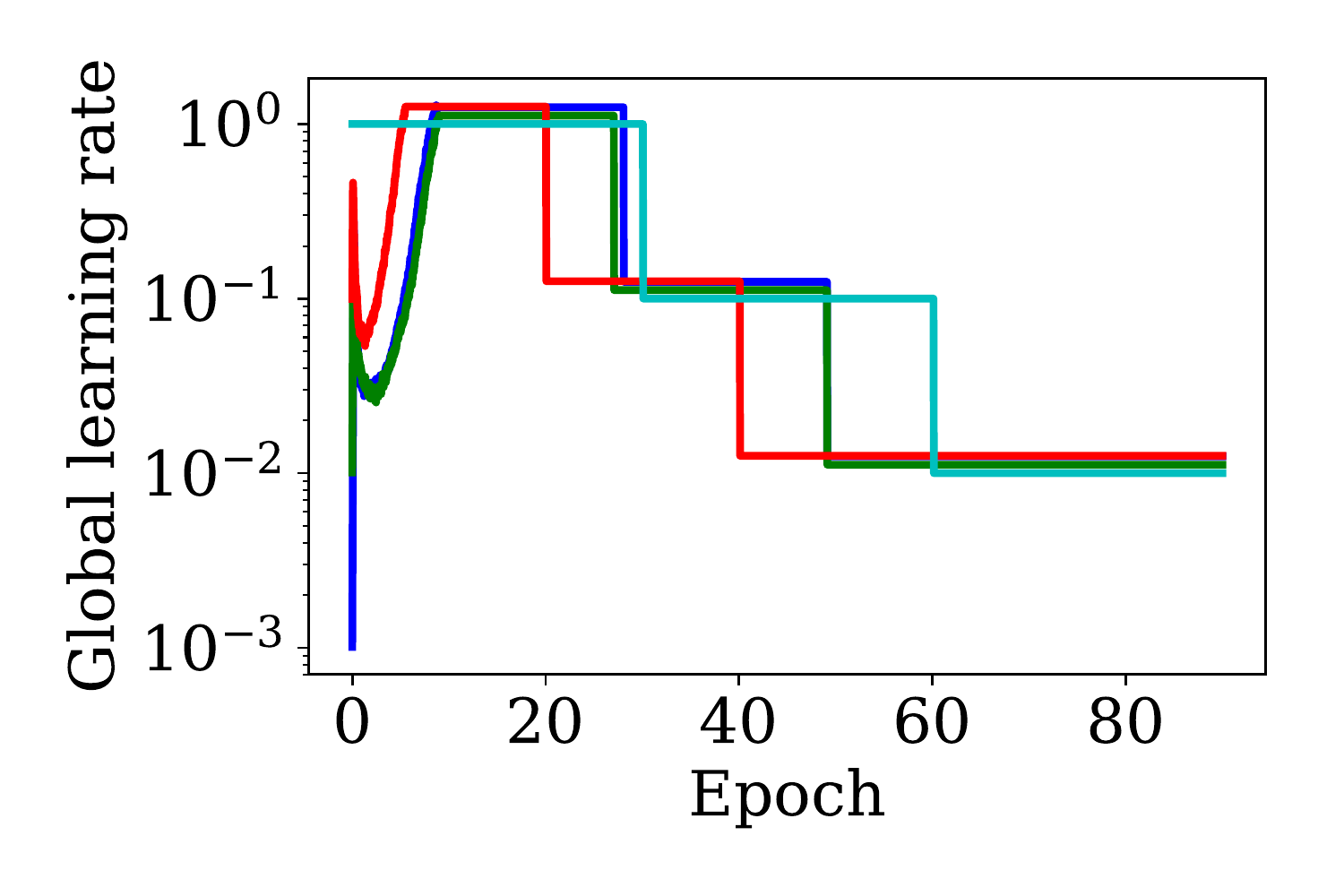}
%    \vspace{-2mm}
    \caption{SALSA on CIFAR-10 (left column) and ImageNet (right column) with default parameters, starting from three different initial learning rates. The performance of SALSA matches that of a hand-tuned SHB method (SHB-hand).
    %It matches the training and testing performance of hand-tuned SHB method (SHB-hand), which chooses $\alpha_0=1.0$ after trial and error and drops the leaning rate by a factor of 10 after every 50 epochs for CIFAR-10 and every 30 epochs for ImageNet.
    }
    \label{fig:salsa_cifar10_imagenet}
\end{figure}

We combine SSLS and SASA+ to form an autonomous algorithm called SALSA 
(Stochastic Approximation with Line-search and Statistical Adaptation).
Figure~\ref{fig:salsa_profile} shows the typical learning rate profile generated by SALSA. 
It starts with a small, arbitrary learning rate $\alpha_0$ and uses SSLS to warm up the optimization process; once a stable learning rate $\bar{\alpha}$ is reached, it automatically switches to SASA+, which relies on online statistical tests to decrease the learning rate in a constant-and-cut (staircase) fashion.  
%It reduces the need for trial and error with different machine learning models and datasets.
%Because SALSA automatically sets the stable learning rate~$\bar{\alpha}$ and the times to drop the learning rate, it reduces the need for trial and error when doing optimization with different machine learning models and datasets.
%The initial learning rate $\alpha_0$ can be arbitrary, but we usually choose it to be very small to be safe. 

Figure~\ref{fig:salsa_cifar10_imagenet} shows the performance of SALSA on training ResNet18~\cite{HeZhangRenSun2016ResNet} on two different datasets: 
CIFAR-10 \citep{krizhevsky2009learning} and ImageNet \citep{imagenet_cvpr09}. 
We tried three small initial learning rates $\alpha_0\in\{0.1,0.01,0.001\}$ with SALSA. In all three cases, the SSLS phase automatically settled to a stable learning rate close to~$1.0$ and switched to SASA+, which then automatically decreased the learning rate twice (each by a factor of 10) based on the online statistical test. 
SALSA matches the train and test performance of a hand-tuned SHB method, for which the best schedule (obtained after significant tuning) sets $\alpha_0=1.0$ and drops the leaning rate by a factor of 10 after every 50 and 30 epochs for CIFAR-10 and ImageNet, respectively.
The learning rate profiles in the bottom row of Figure~\ref{fig:salsa_cifar10_imagenet} are all similar to the sketch in Figure~\ref{fig:salsa_profile}.

In the rest of this paper, we first present the SASA+ method in Section~\ref{sec:sasa+} and then the SSLS procedure in Section~\ref{sec:ssls}. Finally, we describe their combination into SALSA in Section~\ref{sec:salsa} and conclude in Section~\ref{sec:conclusion}.

%% file: sasa_plus_1c.tex
\section{The SASA+ Method}
\label{sec:sasa+}

In this section, we focus on statistical methods for automatically
decreasing the learning rate. This line of work goes back to
\citet{Kesten58}, who used the change of sign of the inner products of
consecutive stochastic gradients as a statistical indicator of slow
progress and as a trigger to decrease the learning rate \citep[see
  extensions in][]{DelyonJuditsky93}.  \citet{Pflug83,Pflug88IFAC}
considered the dynamics of SGD with constant step size for minimizing
convex quadratic functions and derived necessary conditions for the
resulting dynamic process to be stationary.
Most recently, \citet{yaida2018fluctuation} derived fluctuation-dissipation
relations that characterized the stationary behavior of SGD and SHB
with constant learning rate and momentum.
We derive a more general condition for testing stationarity that works for a much broader family of stochastic optimization methods. 

%These relations hold \emph{exactly} for any stationary state (regardless of the loss function or the noise model), so they can be very effective at detecting stationarity and can serve as a reliable indicator of when to decrease the learning rate.

We consider general stochastic optimization methods 
in the form of~\eqref{eqn:sgm-general} with a constant learning rate:
\begin{equation}\label{eqn:sgm-const-lr}
    \xkp = \xk - \alpha \dk.
\end{equation}
We assume that the stochastic search direction $\dk$ is generated with 
time-homogeneous dynamics; in particular, any additional hyperparameters
involved must be constant (not depending on~$k$).
As an example, we consider the search directions generated by the family
of Quasi-Hyperbolic Momentum (QHM) methods \citep{ma2019qh}:
\begin{equation}\label{eqn:d-qhm}
    \begin{split}
        \hk &= (1-\beta)\gk + \beta\hkm,  \\
        \dk &= (1-\nu)\gk + \nu\hk, 
   \end{split}
\end{equation}
where $0\leq\beta<1$ and $0\leq\nu\leq 1$.
With $\beta=0$ or $\nu=0$, it recovers the SGD direction~\eqref{eqn:d-sgd}.
With $\nu=1$,~\eqref{eqn:d-qhm} recovers the SHB direction~\eqref{eqn:d-shb}.
With $0<\beta=\nu<1$, it is equivalent to the direction of
Nesterov momentum~\eqref{eqn:d-nes} 
\citep{SutskeverMartensDahlHinton13,GitmanLangZhangXiao2019}.
We assume that the dynamics of~\eqref{eqn:sgm-const-lr}, driven by 
the stochastic gradients $\gk$ through \eqref{eqn:d-qhm}, are stable. 
Stability regions of the hyperparameters in QHM are characterized
by \citet{GitmanLangZhangXiao2019}.

\subsection{Necessary Conditions for Stationarity}

%In addition, we assume that the stochastic process $\{\xk\}$, as $k\to\infty$, becomes \emph{stationary}. 
%and let $\pi$ denote the stationary distribution of $\xk$.
The stochastic process $\{\xk\}$ is (strongly) stationary if the joint
distribution of any subset of the sequence is invariant with respect to
simultaneous shifts in the time index 
\citep[see, e.g.,][Chapter~9]{GrimmettStirzaker2001}.
As a direct consequence, for any test function $\phi:\R^p\to\R$, we have
\begin{equation}\label{eqn:stationary-test}
\E_{x^k \sim \pi}\bigl[\phi(x^{k+1})\bigr] 
= \E_{x^k \sim \pi}\bigl[\phi(x^k)\bigr], 
\end{equation}
where $\pi$ denotes the stationary distribution of $\{\xk\}$.
If~$\phi$ is smooth, then we use Taylor expansion to obtain
\begin{align}
\phi(\xkp)
&=\phi(\xk) \!-\! \alpha\bigl\langle\nabla\phi(\xk),\dk\bigr\rangle
\textstyle \!+\!\frac{\alpha^2}{2}\bigl\langle\nabla^2\phi(\xk)\dk,\dk\bigr\rangle + O(\alpha^3) .
\label{eqn:taylor-expansion}
\end{align}
Taking expectations on both sides of the above equality and 
applying~\eqref{eqn:stationary-test}, we obtain (after cancleing a common factor~$\alpha$)
\begin{equation}\label{eqn:stationary-2nd}
\E_{x^k \sim \pi} \Bigl[\bigl\langle\nabla\phi(\xk),\dk\bigr\rangle
\!-\textstyle\frac{\alpha}{2}\bigl\langle\nabla^2\phi(\xk)\dk,\dk\bigr\rangle
\Bigr] \!= O(\alpha^2) .
\end{equation}
For an arbitrary test function~$\phi$, it is very hard in practice to compute or approximate the $O(\alpha^2)$ term on the right-hand side. 
In addition, computing the Hessian-vector product $\nabla^2\phi(\xk)\dk$ can be very costly. 
%Therefore, we only consider simple quadratic functions for which the
%$O(\alpha^2)$ term in~\eqref{eqn:stationary-2nd} vanishes. 
%In particular, the choice of $\phi(x) = (1/2)\|x\|^2$ results in
%\begin{equation}\label{eqn:master-condition}
%\E_\pi \left[\bigl\langle \xk,\dk\bigr\rangle
%-\textstyle\frac{\alpha}{2}\|\dk\|^2\right] = 0.
%\end{equation}
%
%For an arbitrary test function~$\phi$, it is very hard in practice to compute/approximate the expectations in~\eqref{eqn:stationary-test}. 
Therefore, we choose\footnote{Note that the choice of the test function~$\phi$
has no implications for the loss function~$F$: we do \emph{not} assume that~$F$ is quadratic, or even that~$F$ is convex. The choice of $\phi$ is arbitrary.}
%and so can be done to enable easy computation of \eqref{eqn:stationary-test}.} 
the simple quadratic function $\phi(x) = \frac{1}{2}\|x\|^2$, which results in
\begin{equation}\label{eqn:master-condition}
\E_{x^k \sim \pi} \left[\bigl\langle \xk,\dk\bigr\rangle
-\textstyle\frac{\alpha}{2}\|\dk\|^2\right] = 0.
\end{equation}
This condition holds \emph{exactly} for any stochastic optimization method
of the form~\eqref{eqn:sgm-const-lr} if it reaches stationarity.
% Indeed, \emph{weak} stationarity is sufficient since $\phi$ is a quadratic function of the state \citep[e.g.,][Section~3.2.3]{Dembo2013LectureNote}.
Beyond stationarity, it requires no specific assumption on the loss function
or noise model for the stochastic gradients. 
%Moreover, it can be applied to different stochactic optimization methods without any change.
%We call this the \emph{master condition} for stationarity.

\citet{yaida2018fluctuation} focused on the SHB method with direction 
$\dk=(1-\beta)\gk+\beta\dkm$ and proposed the condition 
\begin{equation}\label{eqn:yaida-hb-test}
\E_{x^k \sim \pi} \left[\bigl\langle \xk,\gk\bigr\rangle \textstyle
-\frac{\alpha}{2}\frac{1+\beta}{1-\beta}\|\dk\|^2\right] = 0 .
\end{equation}
%by combining the stationarity and the heavy-ball update rule.
It can be shown that this is equivalent to a special case of~\eqref{eqn:master-condition}.
%In Appendix~\ref{apd:batch-means}, we show that this is equivalent to our master condition~\eqref{eqn:master-condition}.
%For the QHM update~\eqref{eqn:d-qhm}, a more complex stationarity condition may also be derived, but it will still be equivalent to~\eqref{eqn:master-condition}.
%In practice, the single, simple master condition~\eqref{eqn:master-condition} is much more preferred, since it applies no matter how the direction $\dk$ is generated.
Condition~\eqref{eqn:master-condition} can be applied \emph{without modification} to the more general QHM family \eqref{eqn:d-qhm} and other algorithms of the form \eqref{eqn:sgm-const-lr} with time-homogeneous dynamics.

If $\{\xk\}$ starts with a nonstationary distribution and converges
to a stationary state, we have
\begin{equation}\label{eqn:asymptotic-test}
    \lim_{k\to\infty} 
\E \left[\bigl\langle \xk,\dk\bigr\rangle
-\textstyle\frac{\alpha}{2}\|\dk\|^2\right] = 0.
\end{equation}
In the next section, we devise a simple statistical test to determine
if condition~\eqref{eqn:master-condition} fails to hold. In this case, the learning
process has not stalled, and we can continue with the same step size.
If we fail to detect non-stationarity (i.e., the dynamics may be
approximately stationary), we will reduce the learning
rate~$\alpha$ to allow for finer convergence.

%In general can use $\phi(x)=\frac{1}{2}\langle Ax, x\rangle$, especially $A$
%being a diagonal matrix.

%$\phi\equiv F$ the objective function of the optimization problem.
%Need approximations due to the $O(\alpha^2)$ term.
%The same problem exists for any non-quadratic functions.

%A (discrete-time) stochastic process is \emph{weak stationary} 
%if its means $\mu(k)$ is constant and its auto-covariance functions 
%satisfy $\rho(t,s)=r(|t-s|)$ for all $t,s\geq 0$.
%(The weak sense stationarity is sufficient for quadratic test functions!)

\textbf{Stationarity of the Loss Function.}
Another obvious test function one may consider is the loss function~$F$ itself.
Since $F(x^k)=\E_{\xi}[f_{\xi}(x^k)]$, we can test (for any $i>0$)
\begin{equation}
%\E_{x^k\sim\pi}[F(x^k)] = \E_{x^{k+i}\sim\pi}[F(x^{k+i})], \quad \forall\,i>0.
\E_{x^k, \xi^k}\bigl[f_{\xi^k}(x^k)\bigr] = \E_{x^{k+i},\xi^{k+i}}\bigl[f_{\xi^{k+i}}(x^{k+i})\bigr], 
\end{equation}
where $\xi^k$ is independent of~$x^k$. 
We cannot use $\phi(x)=F(x)$ in~\eqref{eqn:taylor-expansion} to derive a condition like~\eqref{eqn:master-condition}, since both the Hessian and the higher order terms are hard to estimate in practice. 
Instead, we will derive a simple SLOPE test in the next section to detect if the training loss has a decreasing trend. 
%XXX this feels out of place, but no space for transition.
As we will show in experiments, the SLOPE test does not work well for detecting stationarity for the purpose of decreasing the learning rate. However, it can be very useful to automatically switch from SSLS to SASA+ (Figure~\ref{fig:salsa_profile}).
%(see Figure~\ref{fig:salsa_profile}).

% DUE TO SPACE LIMIT, WE WILL NOT DISCUSS THE FOLLOWING POINTS FOR NOW, MAYBE IN FINAL VERSION WITH ONE MORE PAGE ALLOWED.
%the gradient test at stationary point: $\E_{x^k\sim\pi}[\nabla F(x^k)] = 0$, good to mention, but hard to estimate variance for statistical testing.
%We can even mention using validation loss for generalization purposes. This is what has been done in practice
%\paragraph{Stationarity from Different Perspectives} It is often observed in deep learning practice, especially for overparametrized models, training loss goes to zero but still good to continue training in order to obtain better test performance. The overall dynamics has not settled even though the training loss is not changing!

\subsection{Statistical tests of (non-)stationarity}
\label{sec:stats-tests}

%\citet{Pflug83,Pflug90nonasymptotic} also devised a sequential statistical test to detect stationarity, and he proposed to decrease the step size by a constant factor whenever the test fires. \citet{RuszczynskiSyski83stats} used online statistical tests to check if the present step size and momentum constants satisfy certain optimality conditions and adjust them if the conditions are not satisfied. However, \citet{LangZhangXiao2019} found that these methods have limited success in machine learning applications due to their reliance on using a quadratic approximation of the loss function and/or strong assumptions on the noise models.

%\citet{LangZhangXiao2019} devised a more robust statistical test for Yaida's conditions and obtained performance on common deep learning tasks that is competitive with the best hand-tuned schedules.

\SetAlgoHangIndent{3em}  % for lines that are too long use second line
\begin{algorithm}[t]
	\DontPrintSemicolon
	%\caption{SASA+: SASA with master condition and simple statistical testing}
	\caption{SASA+}
	\label{alg:sasa}
	\textbf{input:} 
	$x^0$, $\alpha_0$~
	(default parameters: $N_\mathrm{min}$, $K_\mathrm{test}$, $\delta, \theta, \tau$)\\
%    $\delta\in(0,1)$, $\theta\in(0,1)$, $\tau\in(0,1)$\\
    $\alpha\gets\alpha_0$\\
    $k_o\gets 0$ \\
	\For{$k = 0,...,T-1$}{
        Randomly sample $\xik$ and compute $\dk$ (e.g., using QHM)\\
        $\xkp \gets \xk - \alpha \dk$ \\
        %\rule[0.5ex]{0.9\linewidth}{0.4pt} \\
        $\Delta_k \gets \bigl\langle \xk, \dk\bigr\rangle-\frac{\alpha}{2}\|\dk\|^2$\\
        $N \gets \lceil \theta (k -k_o) \rceil$  \\
        \If{$N > N_\mathrm{min}$ \textbf{and} $k\!\! \mod K_\mathrm{test} == 0$}{
            $(\hat\mu_N, \hat{\sigma}_N)\gets$ statistics of
            %sample mean and BM/OLBM variance of
            $\{\Delta_{k-N+1},\ldots,\Delta_k\}$ \\[0.5ex]
            \If{$0\in\hat\mu_N\pm t^*_{1-\delta/2}\frac{\hat{\sigma}_N}{\sqrt{N}}$}{
            \vspace{0.5ex}
                $\alpha \gets \tau \alpha$ \\
                $k_o \gets k$
            }
            \vspace{-0.3ex}
        }
        \vspace{-0.3ex}
    }
    \vspace{-0.3ex}
    \textbf{output:} $x^T$ (or the average of last epoch)
\end{algorithm}

In order to test if the stationarity condition~\eqref{eqn:master-condition} holds
approximately, we collect 
%at each iteration 
the simple statistics 
\[
\Delta_k~\triangleq~ \textstyle
\bigl\langle \xk,\dk\bigr\rangle-\frac{\alpha}{2}\|\dk\|^2.
%~=~\bigl\langle \xkp,\dk\bigr\rangle+\frac{\alpha}{2}\|\dk\|^2.
\]
%Here the second expression for $\Delta_k$ is obtained using the direct substitution $\xkp=\xk-\alpha\dk$, which can be more convenient to implement if $\Delta_k$ is collected after the state updates to $\xkp$.
%
In the language of hypothesis testing
\citep[e.g.,][]{LehmannRomano2005book}, we make as our
\emph{null hypothesis} that the dynamics~\eqref{eqn:sgm-const-lr} have
reached a stationary distribution $\pi$. If we have $N$ samples
$\{\Delta_k\}$, we know from equation~\eqref{eqn:master-condition} and
the Markov chain CLT (see its application in \citet{jones2006fixed})
that as $N\to\infty$, under the null hypothesis, the mean statistic
$\bar{\Delta}$ follows a normal distribution with mean~0 and variance
$\sigma_\Delta^2/\sqrt{N}$. Our \emph{alternative hypothesis} is that
the dynamics \eqref{eqn:sgm-const-lr} have \emph{not} reached
stationarity. 

To test these hypotheses, we adopt the classical
confidence interval test.  We use the most
recent~$N$ samples $\{\Delta_{k-N+1},\ldots,\Delta_k\}$ to compute the
sample mean $\hat\mu_N$ and a variance estimator $\hat\sigma_N^2$ for
$\sigma_\Delta^2$.  Then we form the $(1\!-\!\delta)$-confidence interval
$(\hat\mu_N-\omega_N, ~\hat\mu_N+\omega_N)$ with half width
\begin{equation}\label{eqn:confi-interval}
%\hat\mu_N \pm t^*_{1-\delta/2}\frac{\hat{\sigma}_N}{\sqrt{N}},
%(\hat\mu_N-\omega_N, ~\hat\mu_N+\omega_N)
%\mbox{ with half width }
\omega_N = t^*_{1-\delta/2}\frac{\hat{\sigma}_N}{\sqrt{N}},
\end{equation}
where $t_{1-\delta/2}^*$ is the $(1-\delta/2)$ quantile of the Student's
$t$-distribution with degrees of freedom corresponding to that
%the degrees of freedom 
in the variance estimator $\hat\sigma_N^2$. Because the sequence 
$\{\Delta_{k-N+1},\ldots,\Delta_k\}$
are highly correlated due to the underlying Markov dynamics,
the classical formula for the sample variance 
(obtained by assuming i.i.d.\ samples) does not work for $\hat\sigma_N^2$
(causing too small confidence intervals).
We need to use more sophisticated \emph{batch mean} (BM) or 
\emph{overlapping batch mean} (OLBM) variance estimators developed in the 
Markov chain Monte Carlo literature
\citep[e.g.,][]{jones2006fixed,flegal2010batch}.
See \citet{LangZhangXiao2019} for detailed explanation.
We also list the formulas for computing BM and OLBM in 
Appendix~\ref{apd:batch-means}.

If the confidence interval in~\eqref{eqn:confi-interval} contains~$0$, 
we fail to reject the null hypothesis that the learning process is at stationarity, which
means the learning rate should be decreased.  Otherwise, we accept the
alternative hypothesis of non-stationarity and keep using the current
constant learning rate.
Algorithm~\ref{alg:sasa} summarizes our new method called SASA+,
and Table~\ref{tab:sasa-params} lists its hyperparameters and their default values.
%(where $n$ is total number of training examples and $b$ is the mini-batch size).

\begin{table}[t]
%\vspace{-1ex}
\caption{List of hyperparameters of Algorithm~\ref{alg:sasa}}
\label{tab:sasa-params}
\begin{center}
%\resizebox{1.\linewidth}{!}{
\begin{tabular}{lll}
    \textbf{Parameter} & \textbf{Explanation} & \textbf{Default value}
    \\ \hline \\[-1ex]
$N_\textrm{min} \in \mathbb{Z}_+$  & minimum \# of samples for testing & $\min\{1000, \lceil n/b\rceil\}$ \\
$K_\textrm{test} \in \mathbb{Z}_+$ & period to perform statistical test & $\min\{100, \lceil n/b\rceil\}$\\
$\delta \in (0,1)$    & $(1\!-\!\delta)$-confidence interval & 0.05 \\
%$\delta \in (0,1)$    & $(1\!-\!\delta)$-confidence interval & 0.05 (95\% confidence)\\
%$\theta \in (0,1)$    & fraction of recent samples to keep (after reset) & $1/8$ \\
$\theta \in (0,1)$    & fraction of recent samples to keep & $1/8$ \\
$\tau \in (0,1)$     & learning rate drop factor   & $1/10$ \\
\hline\\[-2ex]
\multicolumn{3}{l}{where $n$ is total number of training examples and $b$ is the mini-batch size.}
\end{tabular}
%}
\end{center}
\end{table}

Our method is an extension of SASA (Statistical Adaptive Stochastic Approximation) proposed by
\citet{LangZhangXiao2019}, which is based on testing the condition~\eqref{eqn:yaida-hb-test} for the SHB dynamics only.
In addition to using a much more general stationarity condition~\eqref{eqn:master-condition},
%\citet{LangZhangXiao2019} focused on the stochastic heavy-ball method and
%devised a statistical test for the condition~\eqref{eqn:yaida-hb-test}.
another major difference is that they set non-stationarity as the null 
hypothesis and stationarity as the alternative hypothesis (opposite to ours).
This leads to a more complex ``equivalence test''
that requires an additional hyperparameter (see Appendix~\ref{apd:sasavssasaplus}).
Our test is simpler, 
%more rigorous, 
more intuitive, and computationally as robust.

\begin{figure}[t]
\centering
    \includegraphics[width=0.33\linewidth]{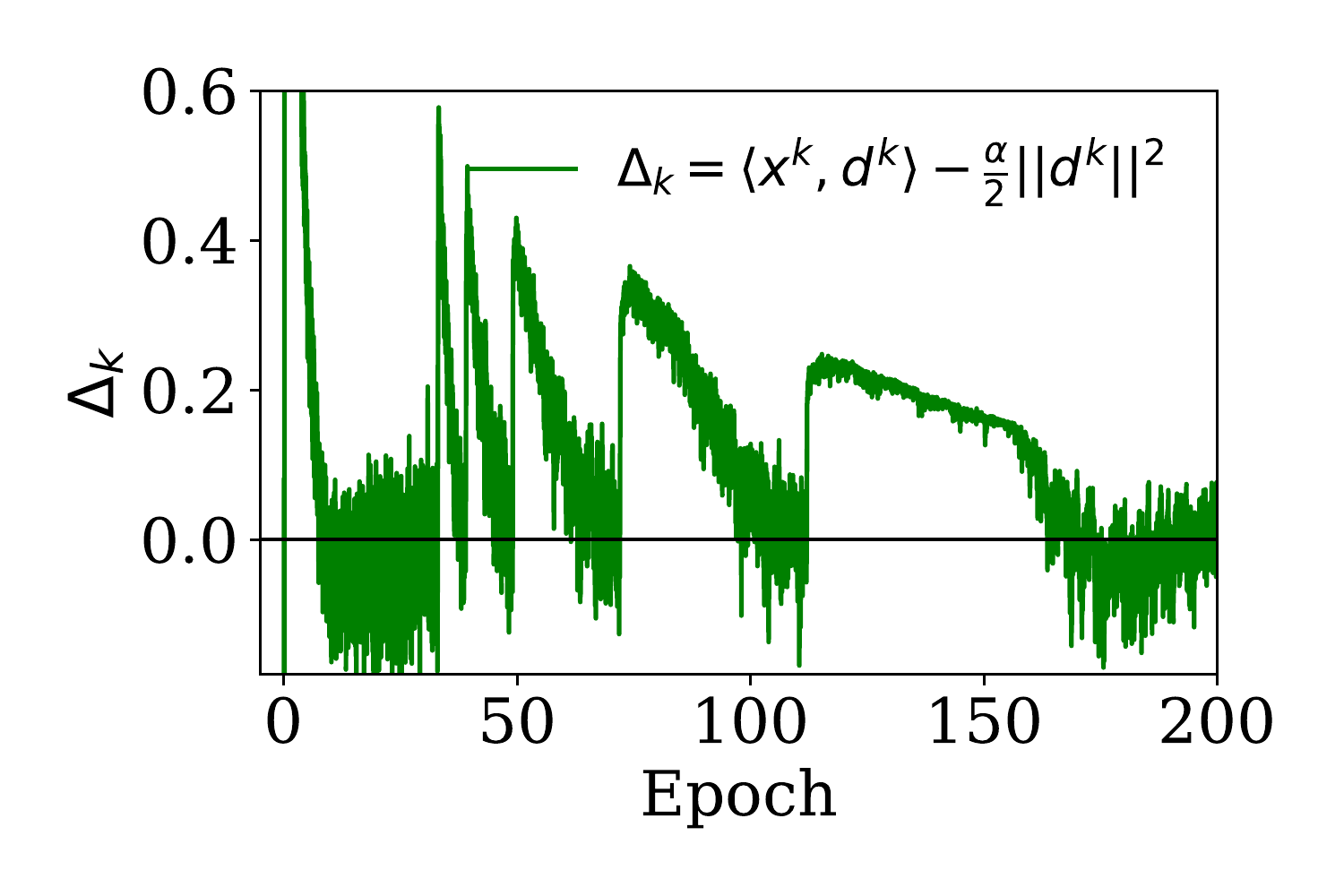}\quad
    \includegraphics[width=0.33\linewidth]{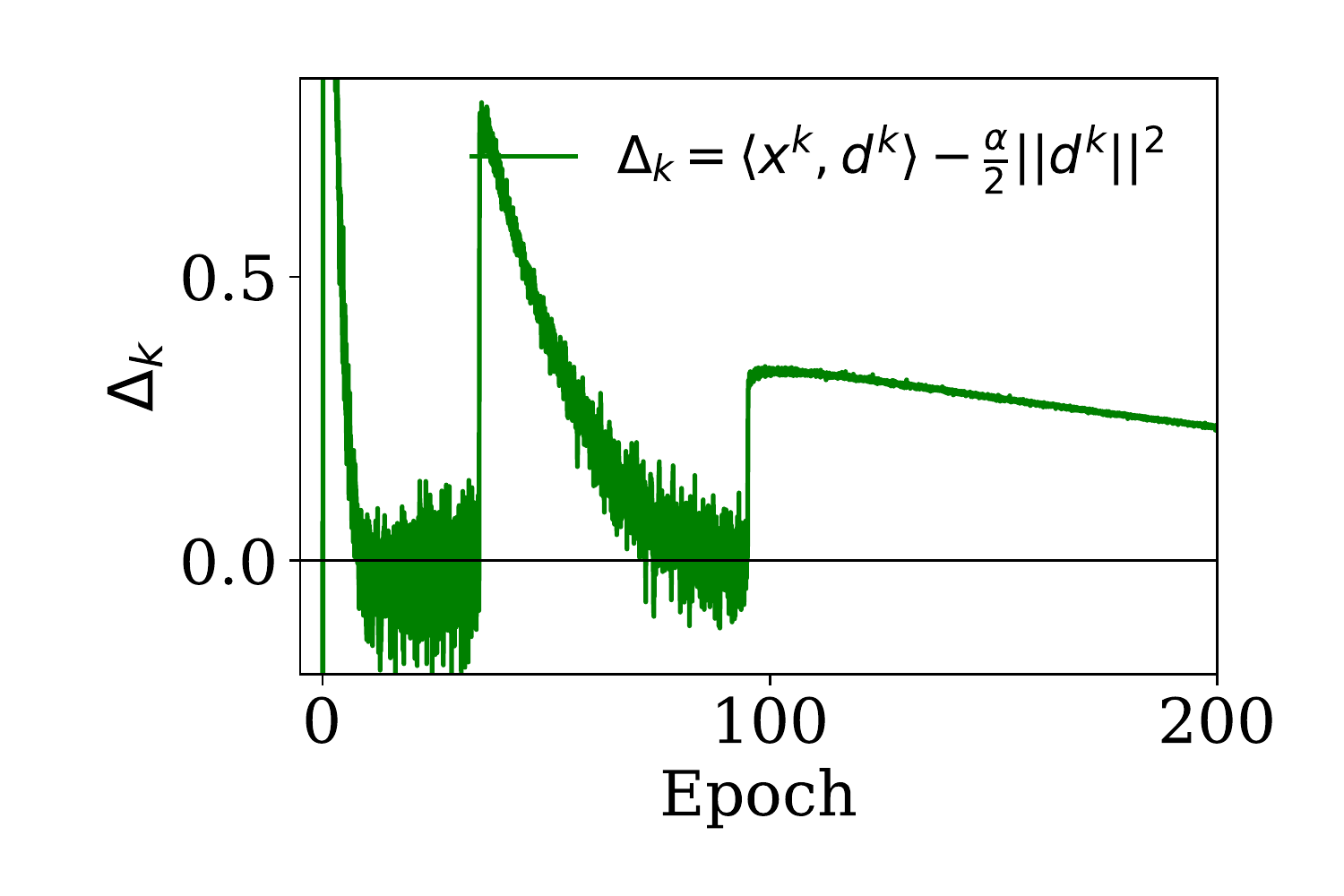}\\%[-1ex]
    \includegraphics[width=0.33\linewidth]{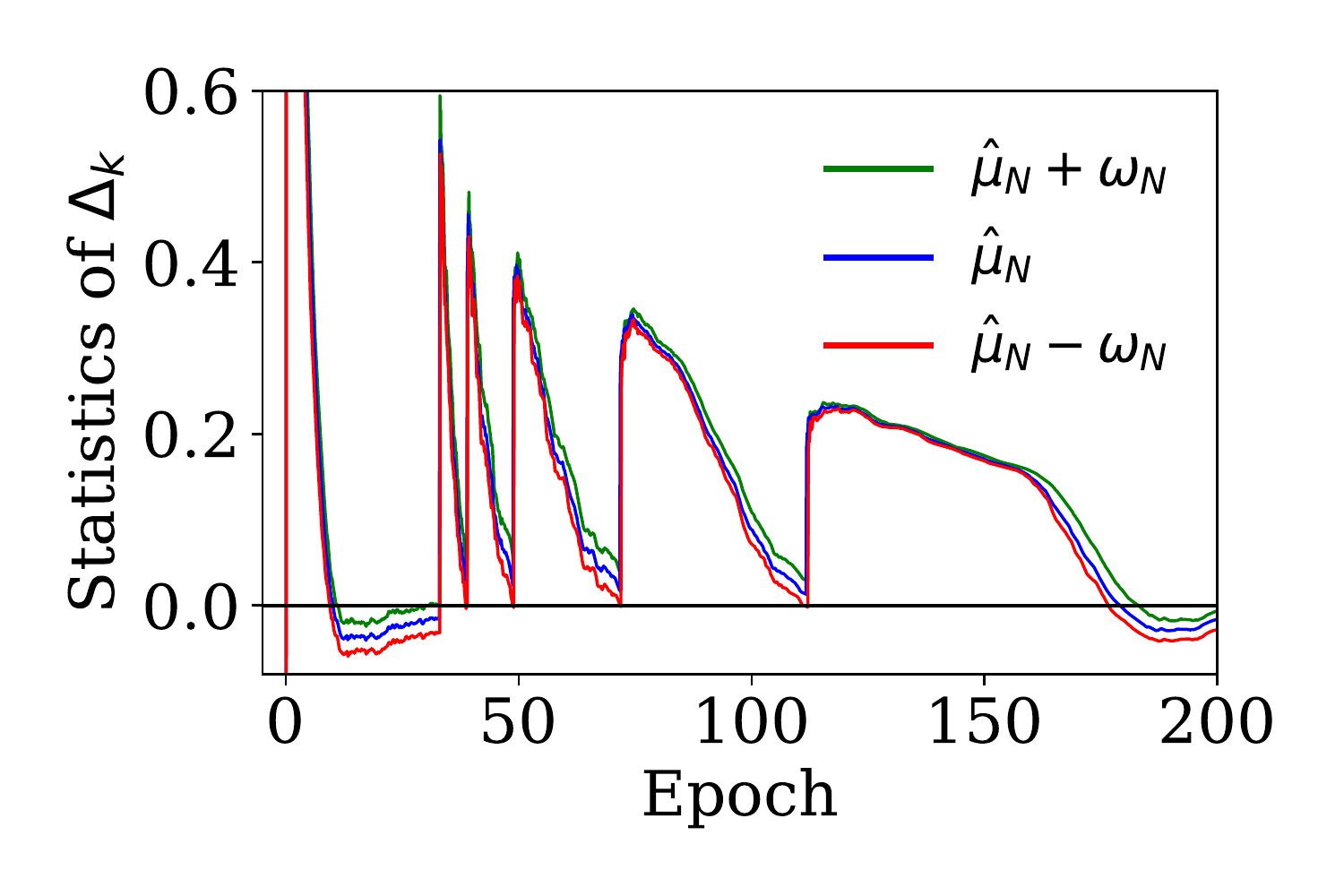}\quad
    \includegraphics[width=0.33\linewidth]{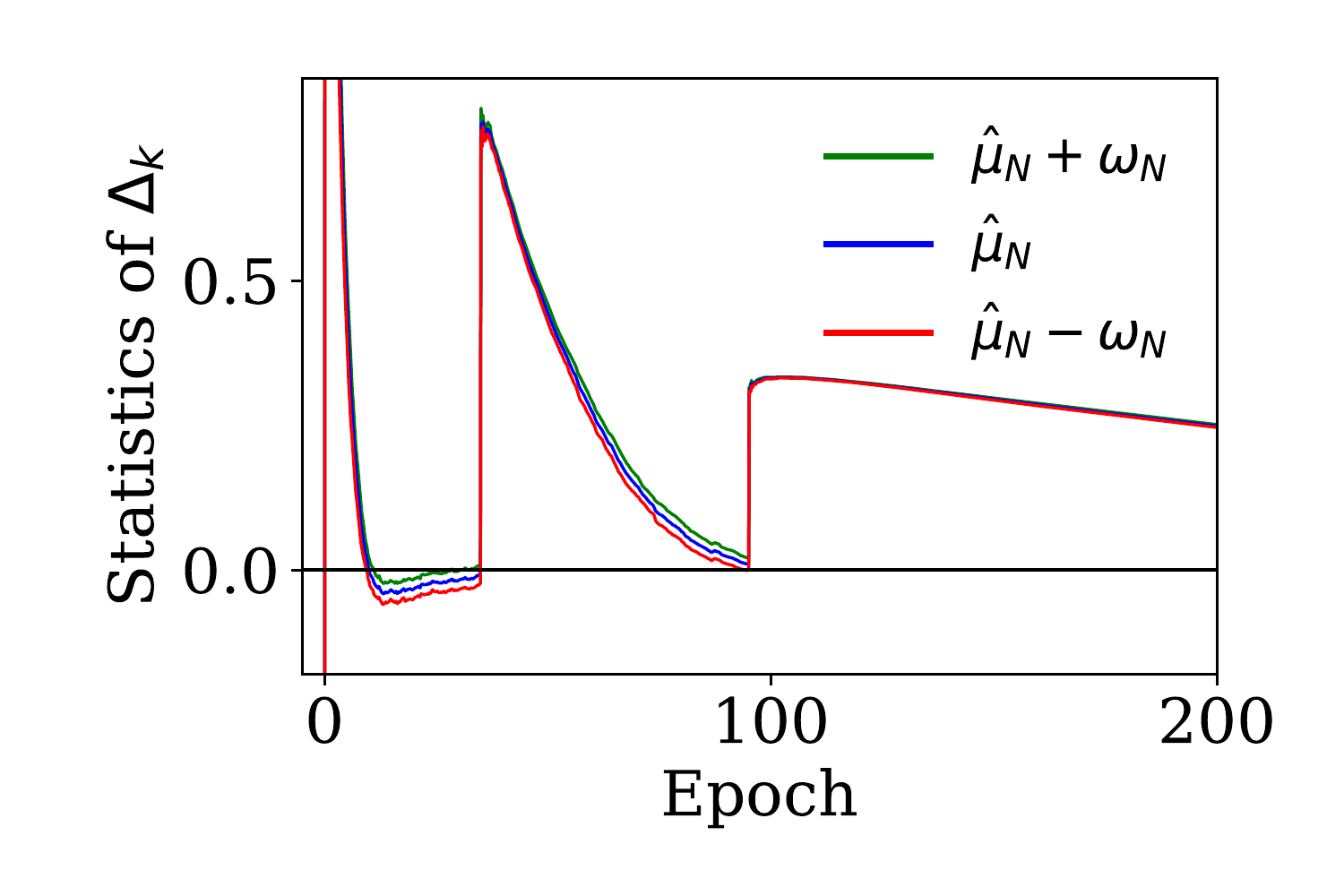}\\%[-1ex]
    \includegraphics[width=0.33\linewidth]{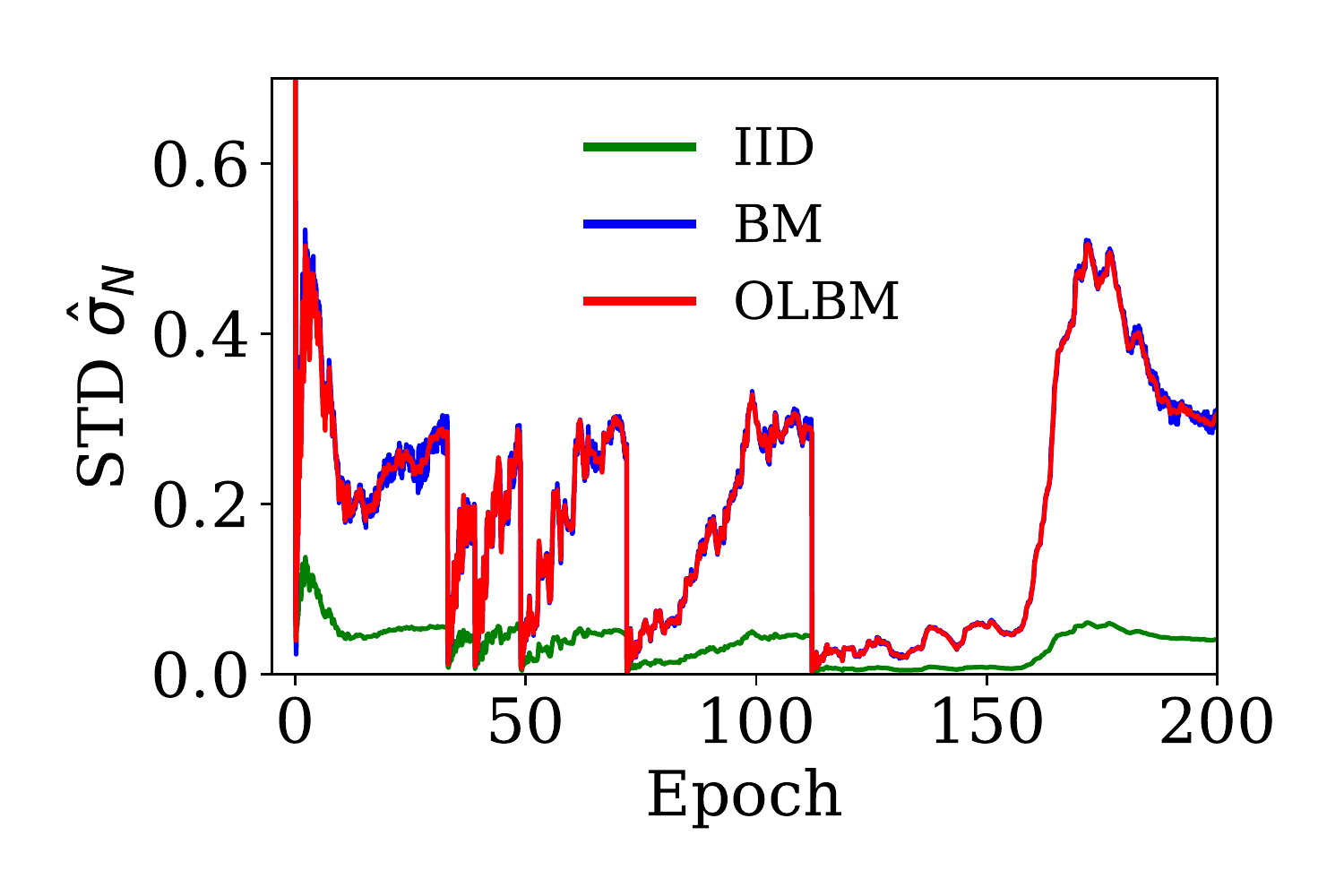}\quad
    \includegraphics[width=0.33\linewidth]{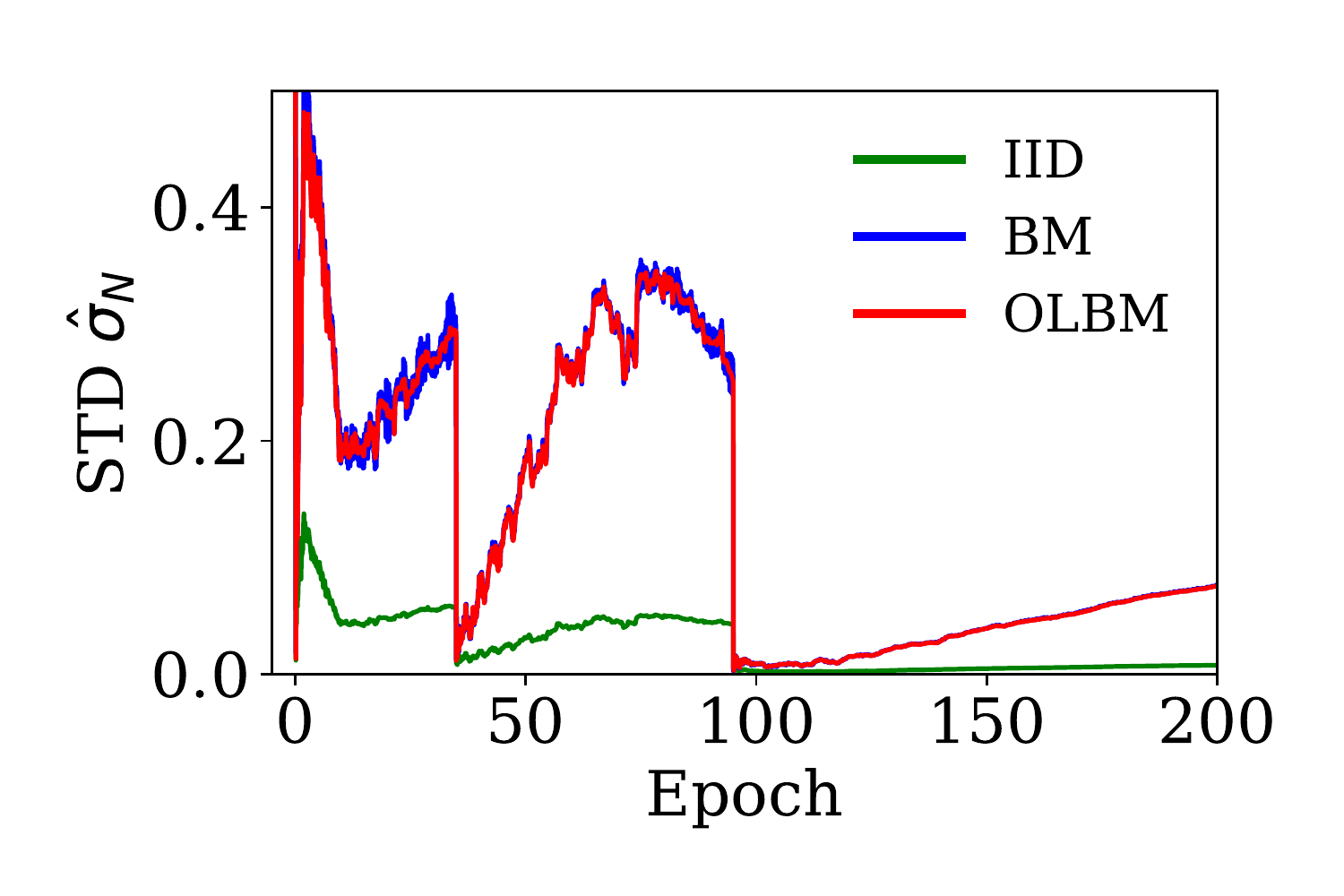}
%\vspace{-2ex}
    \caption{Statistics in SASA+. 
    The two columns are from runs with drop ratio $\tau=1/2$ and $1/10$ respectively.
    The top row shows the instantaneous value of $\Delta_k$. The middle row shows the confidence interval of $\E[\Delta_k]$, which contains~0 with high probability if the process is stationary. The bottom row shows the variance estimated by different methods:
    i.i.d.\ formula, BM and OLBM estimators.}
    %where BM and OLBM estimators take into account of correlation in Markov chains and are more accurate.}
    \label{fig:cifar_sasa_statistics}
\end{figure}

\begin{figure}[t]
\centering
    \includegraphics[width=0.32\linewidth]{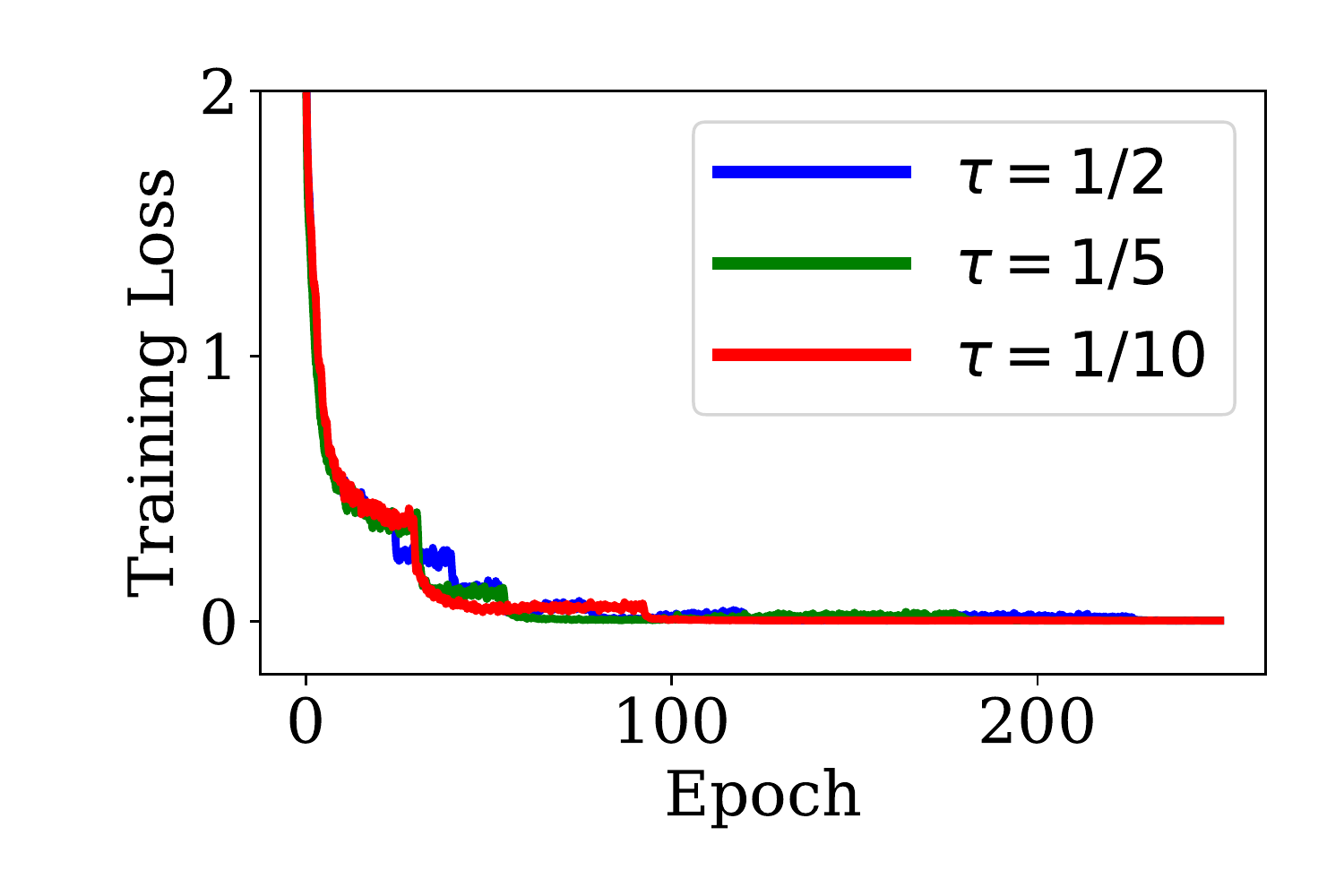}
    \includegraphics[width=0.33\linewidth]{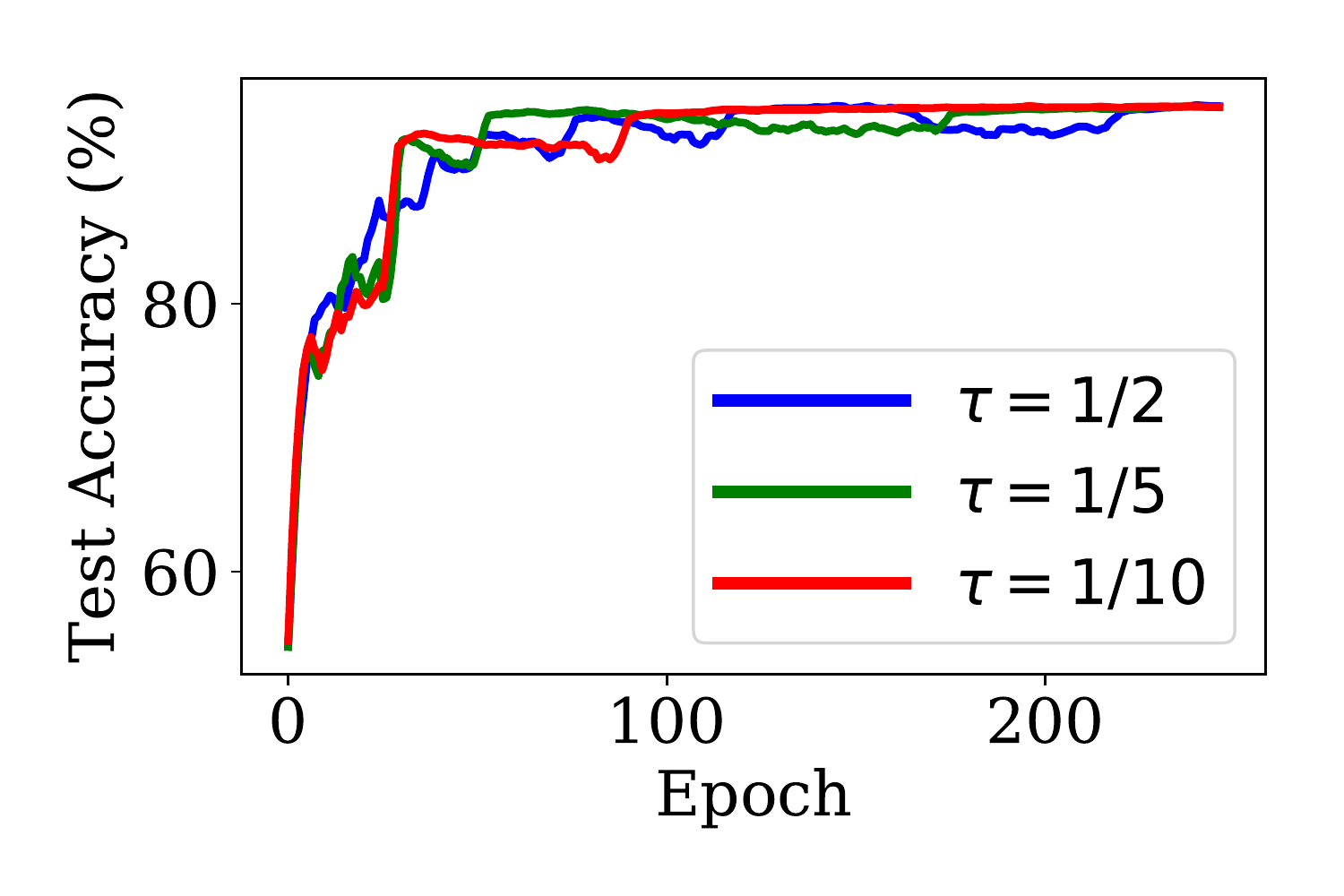}
    \includegraphics[width=0.33\linewidth]{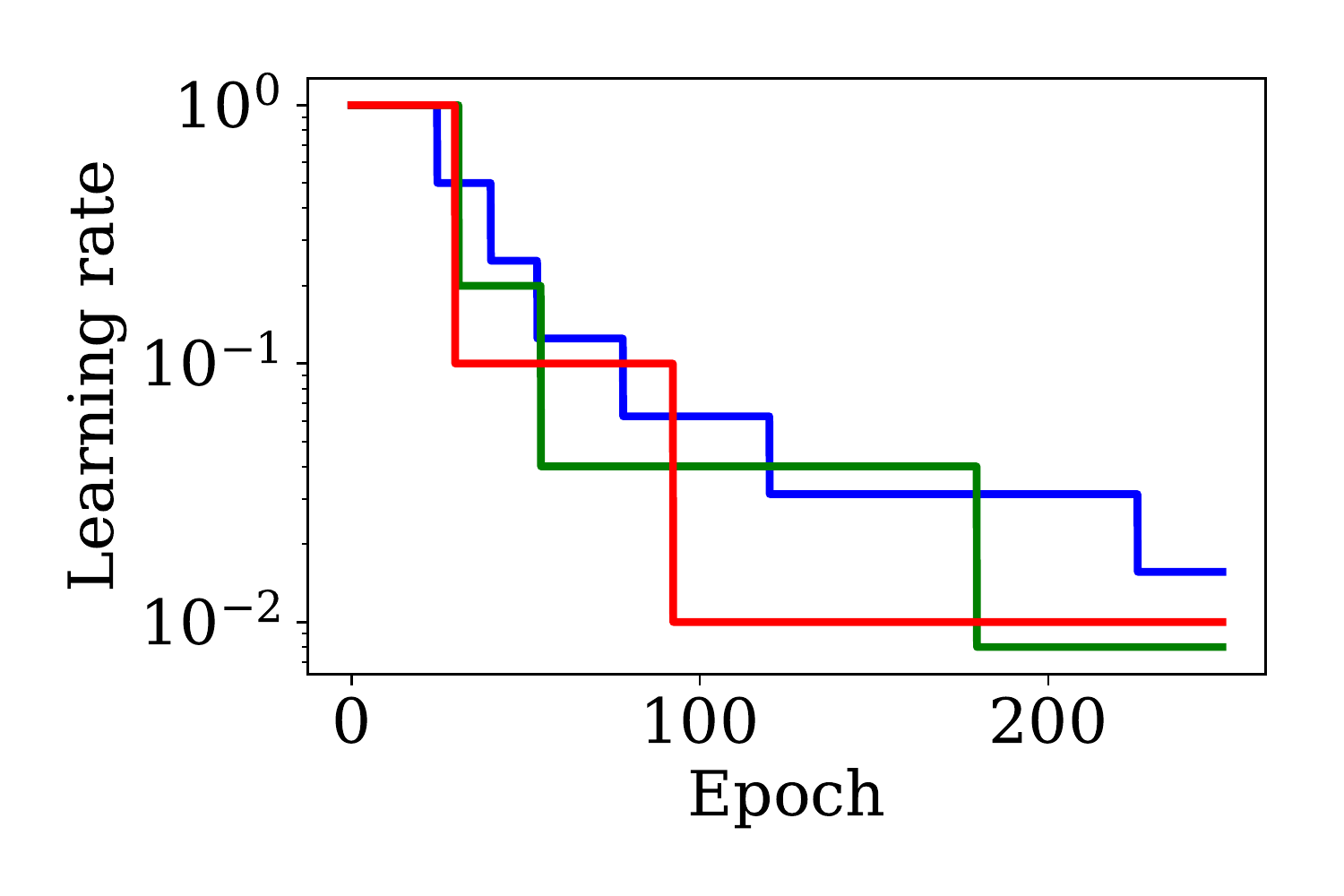} \\%[-1ex]
    \includegraphics[width=0.32\linewidth]{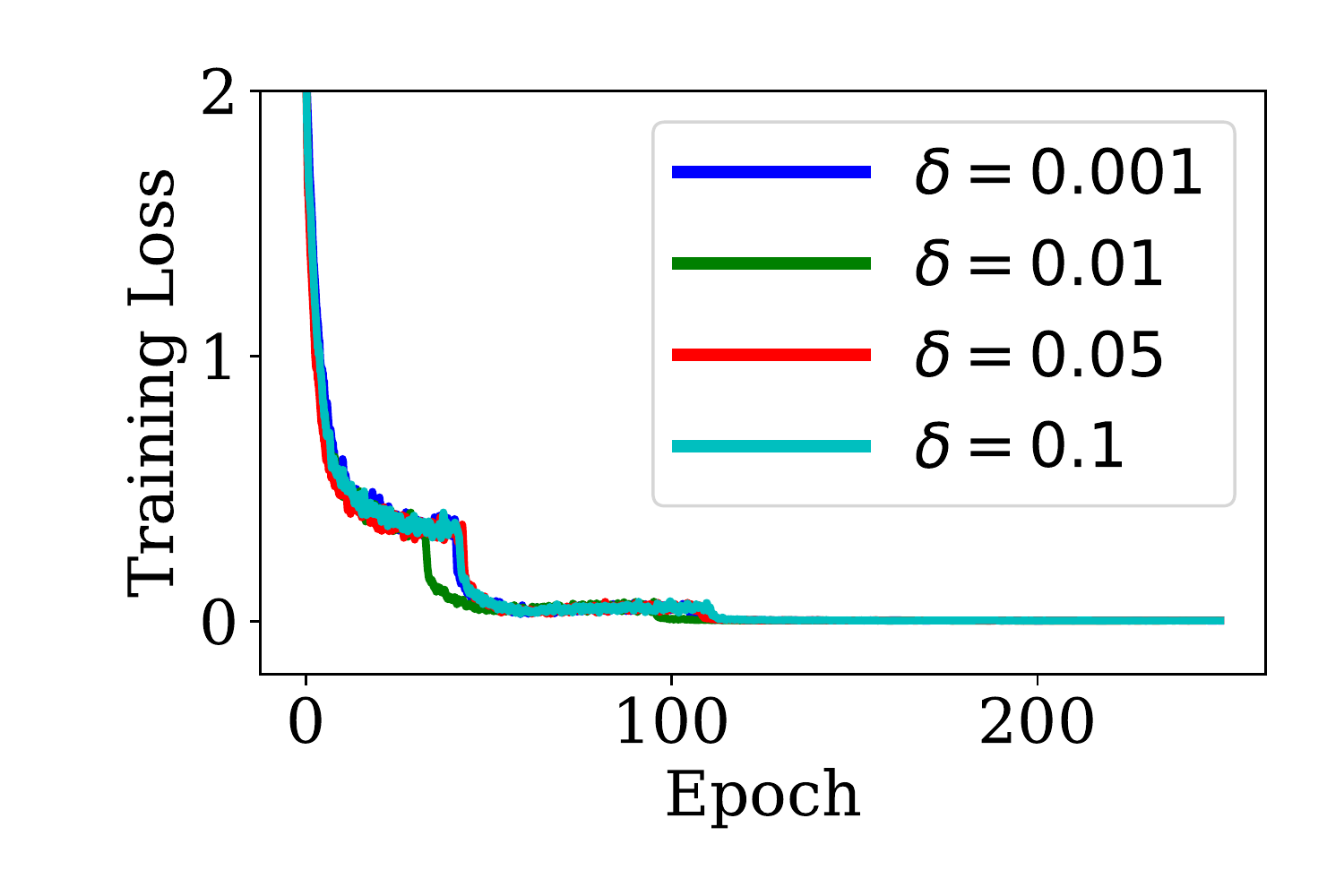}
    \includegraphics[width=0.33\linewidth]{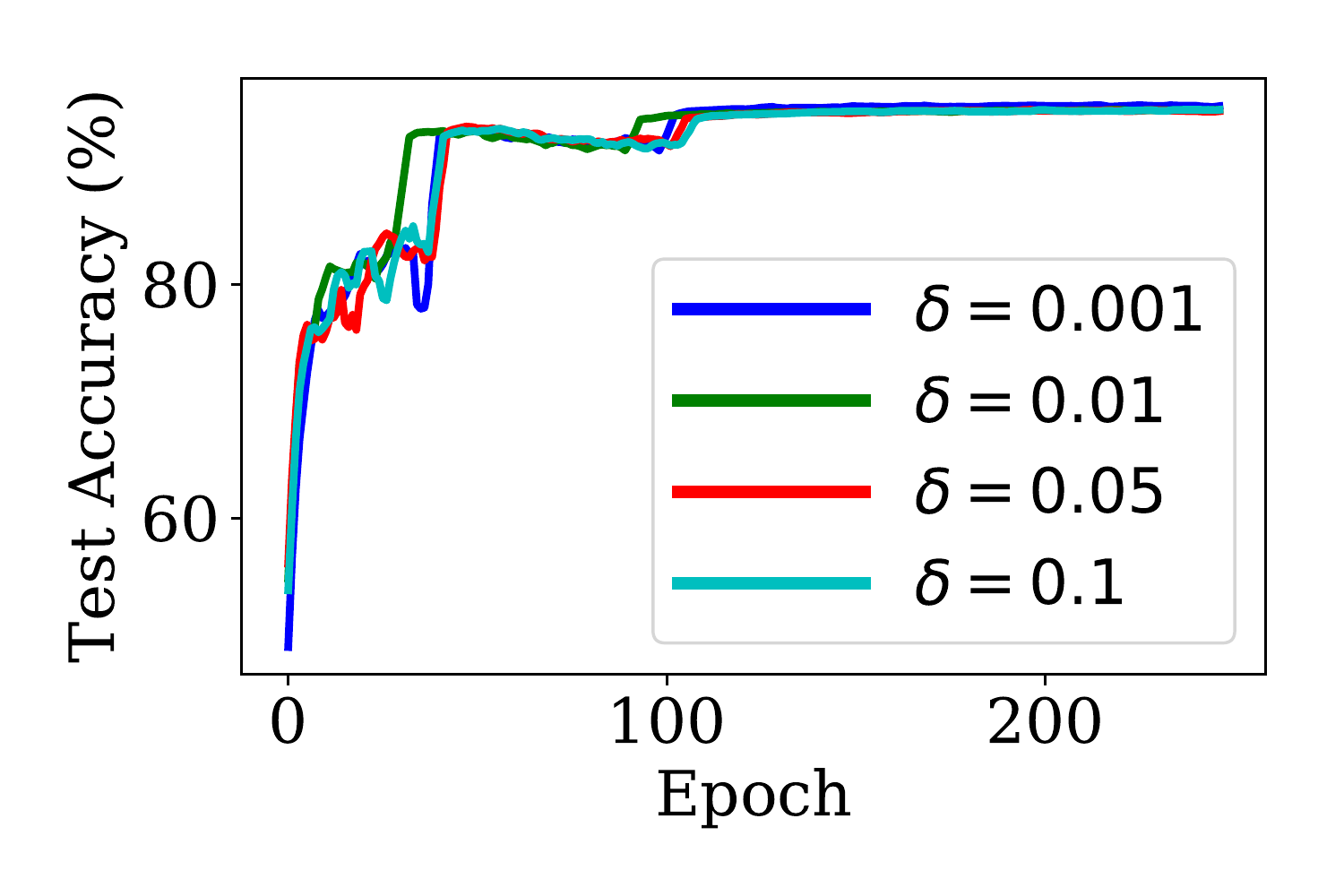}
    \includegraphics[width=0.33\linewidth]{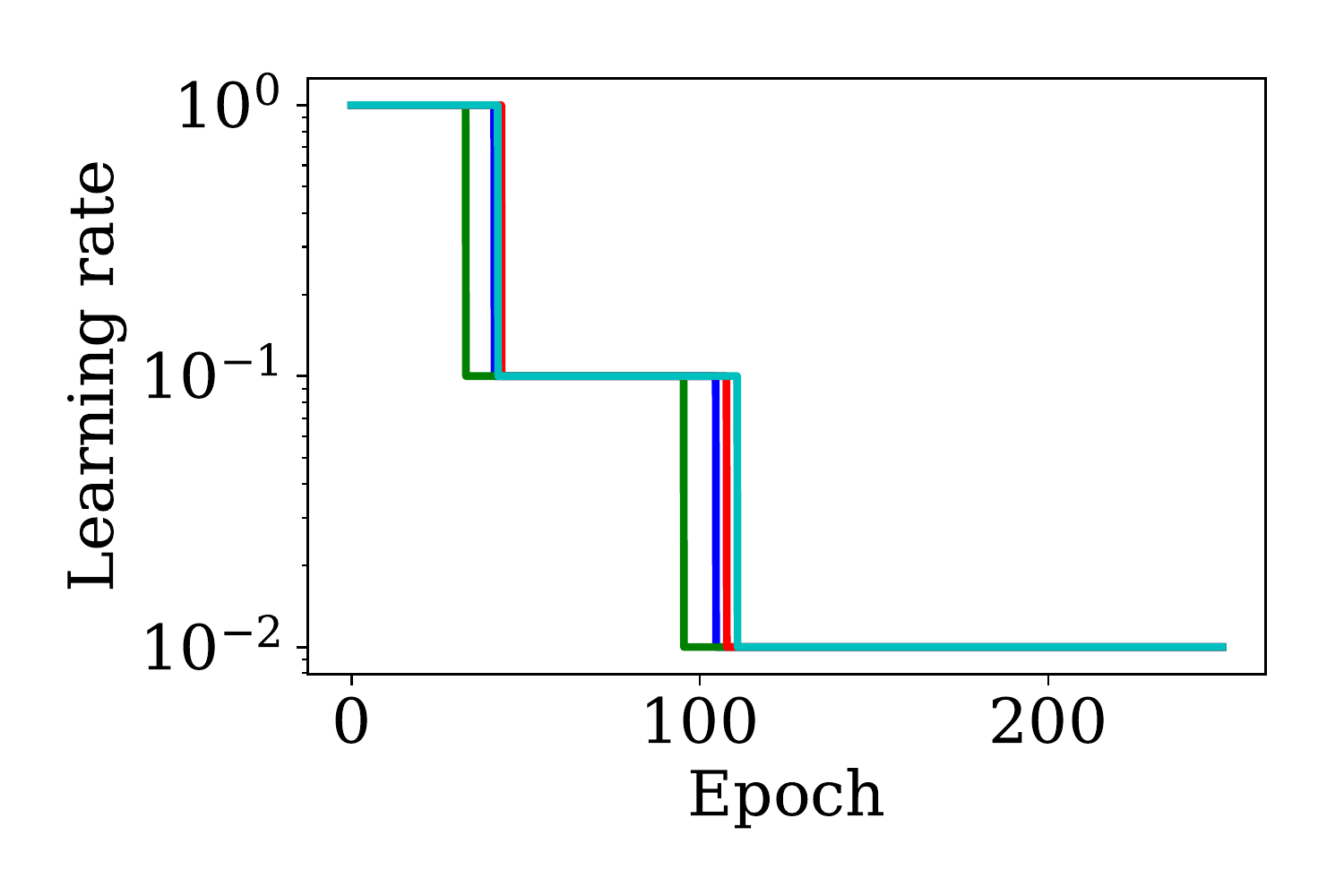} \\%[-1ex]
    \includegraphics[width=0.32\linewidth]{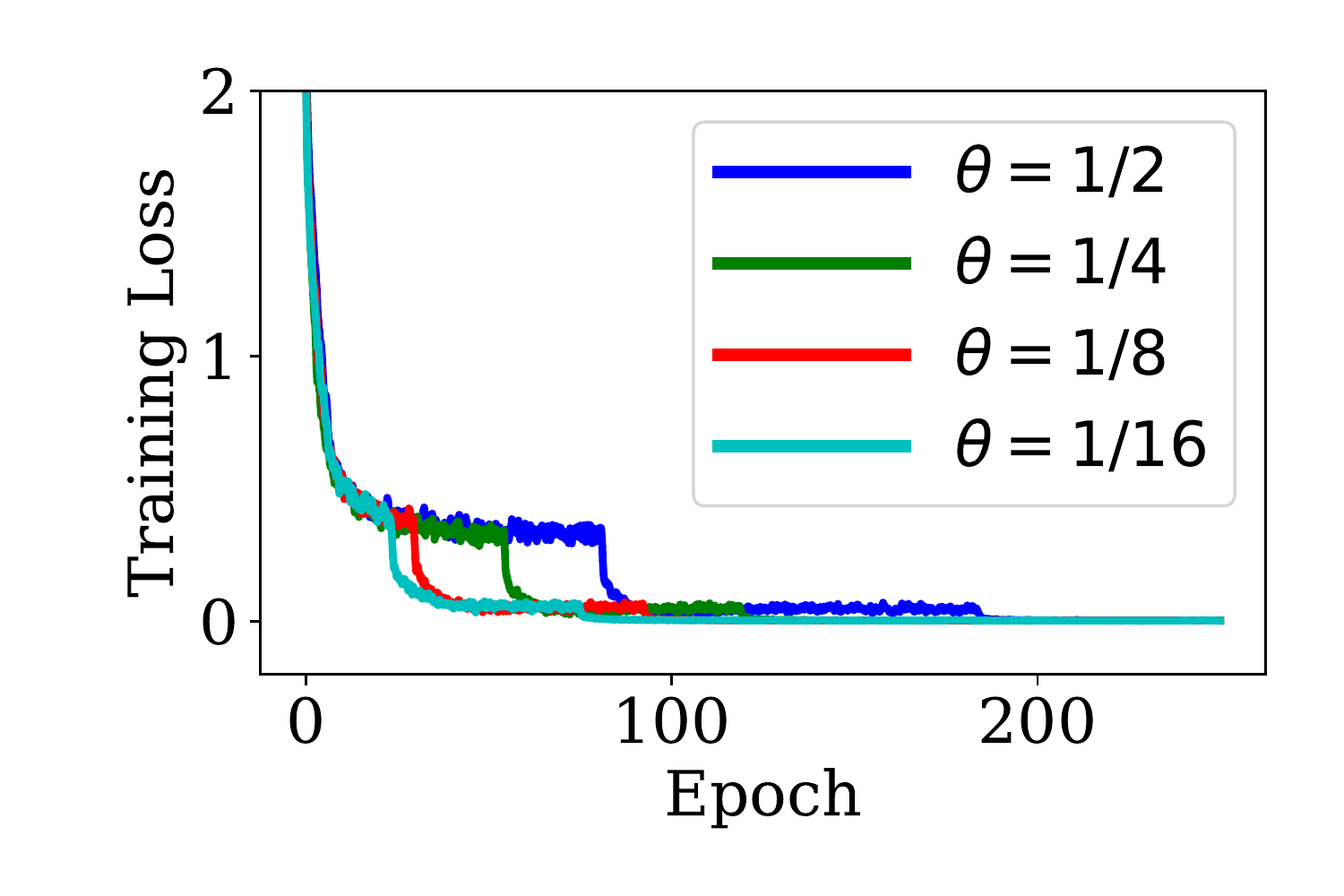}
    \includegraphics[width=0.33\linewidth]{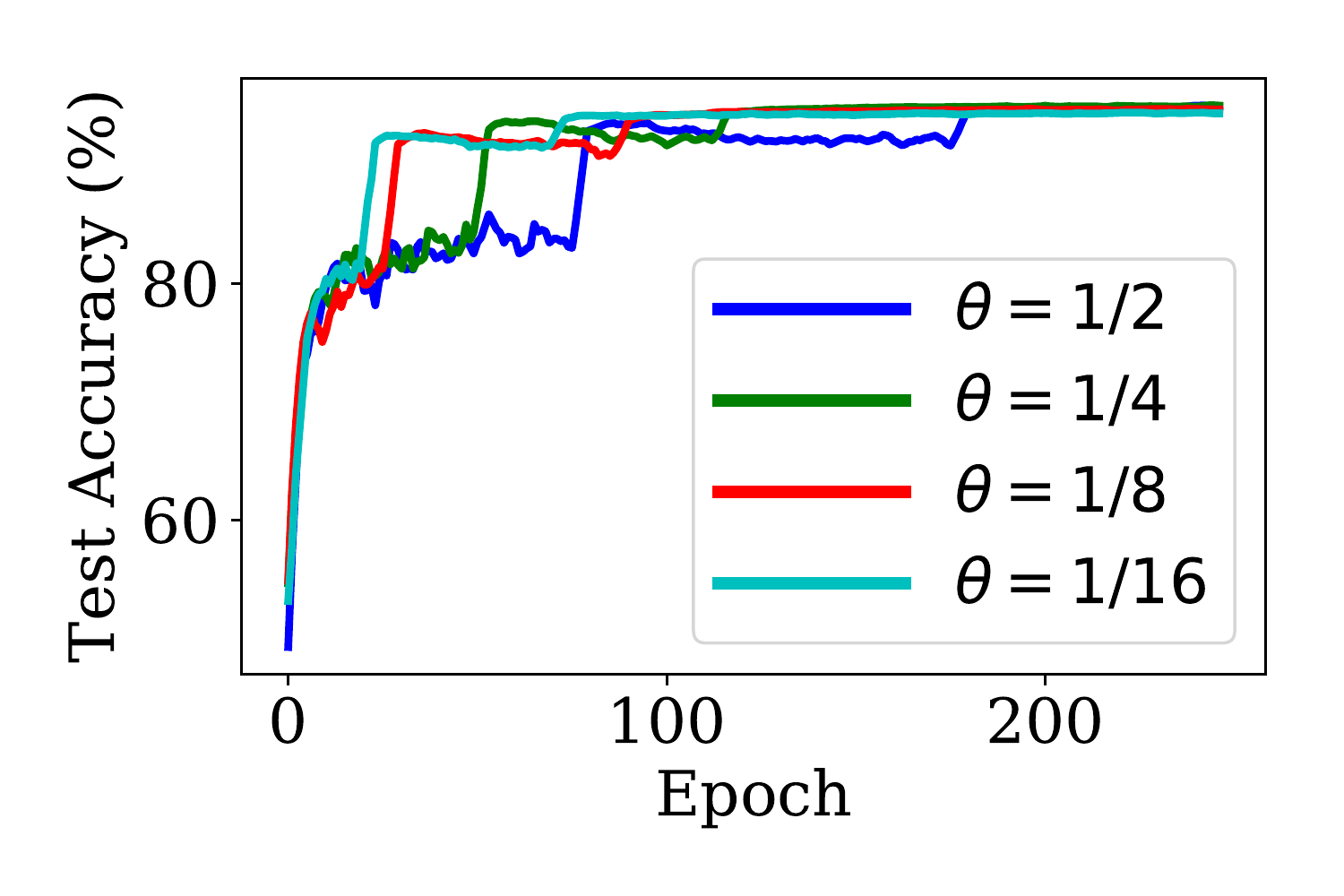}
    \includegraphics[width=0.33\linewidth]{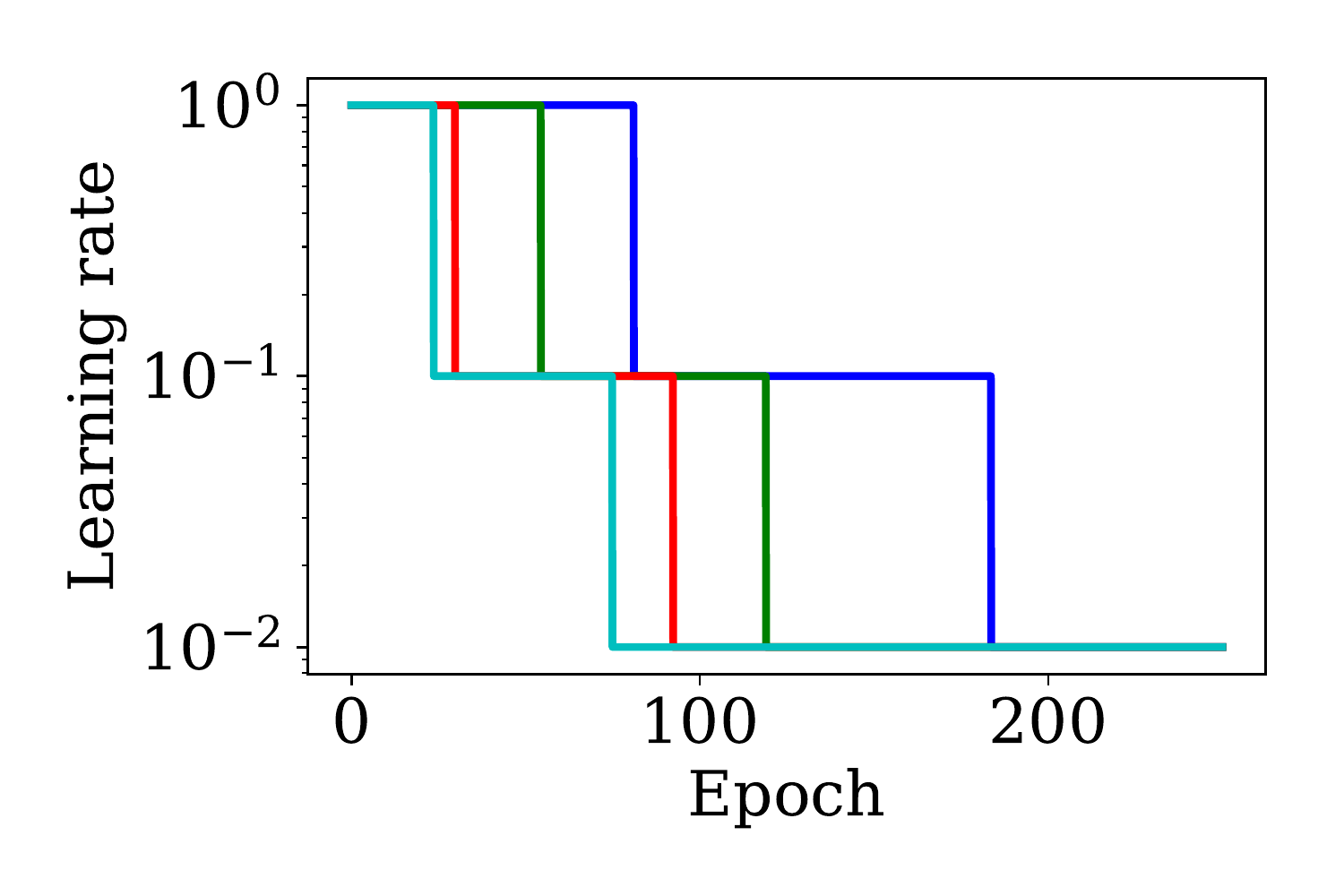} \\%[-1ex]
    \includegraphics[width=0.32\linewidth]{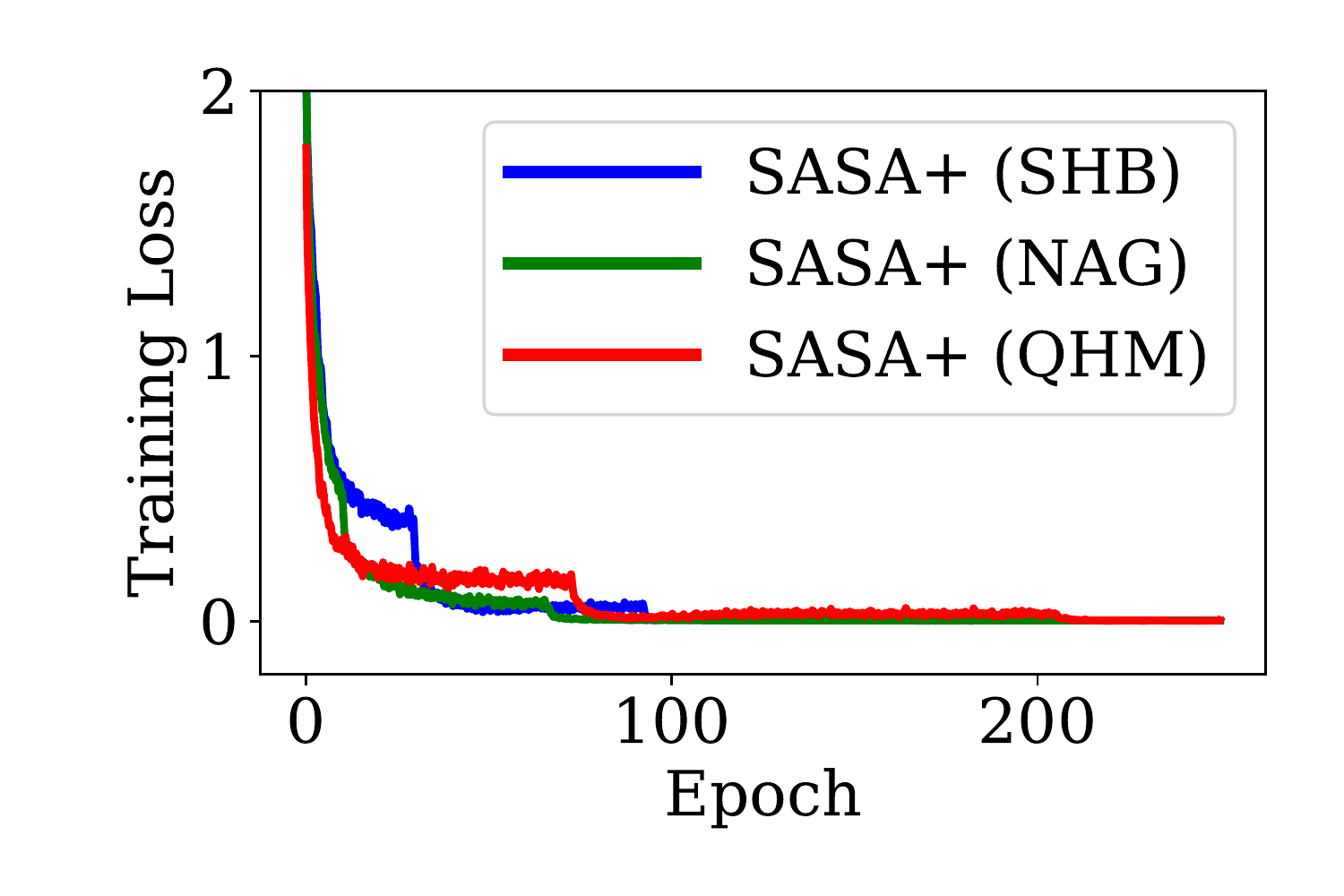}
    \includegraphics[width=0.33\linewidth]{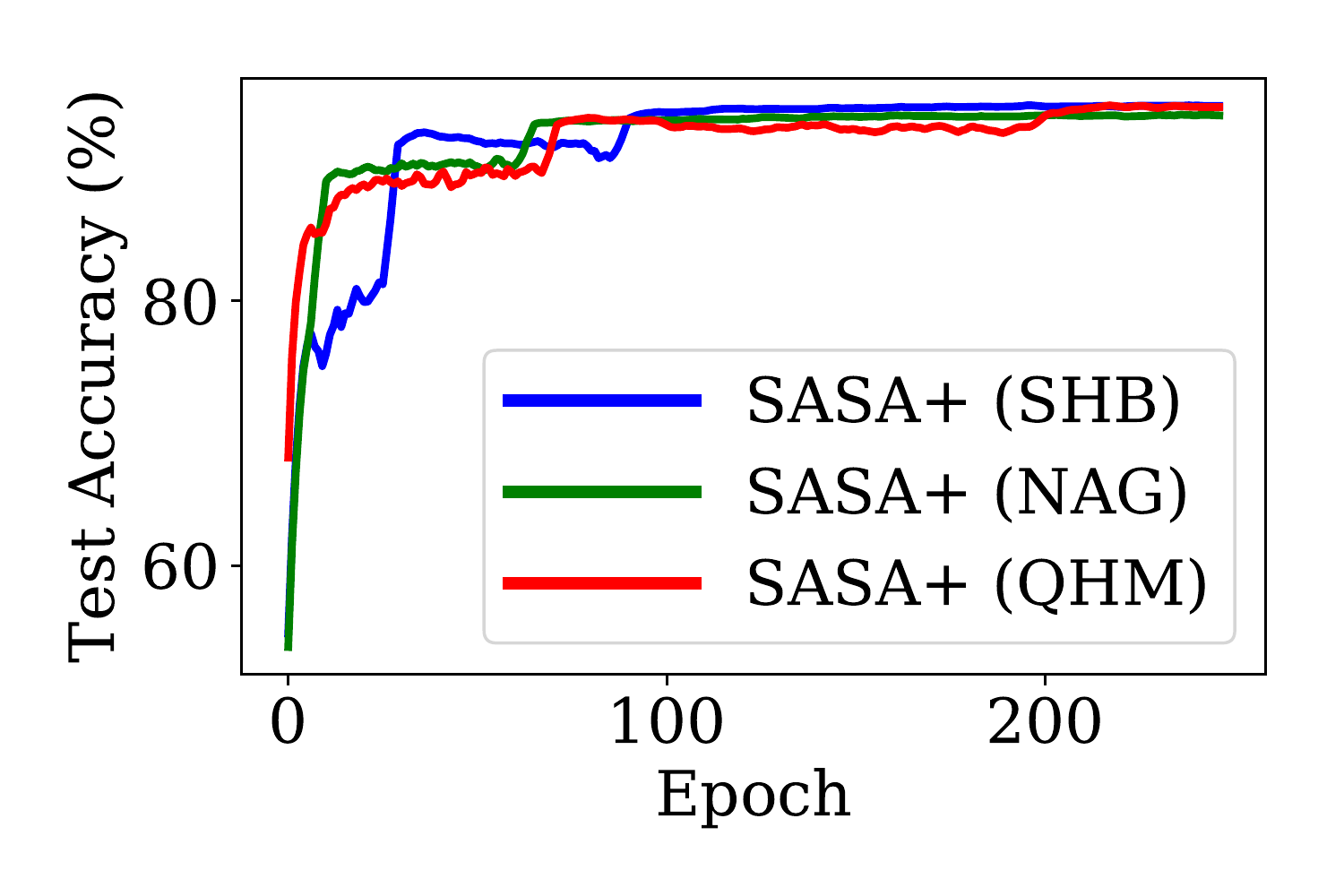}
    \includegraphics[width=0.33\linewidth]{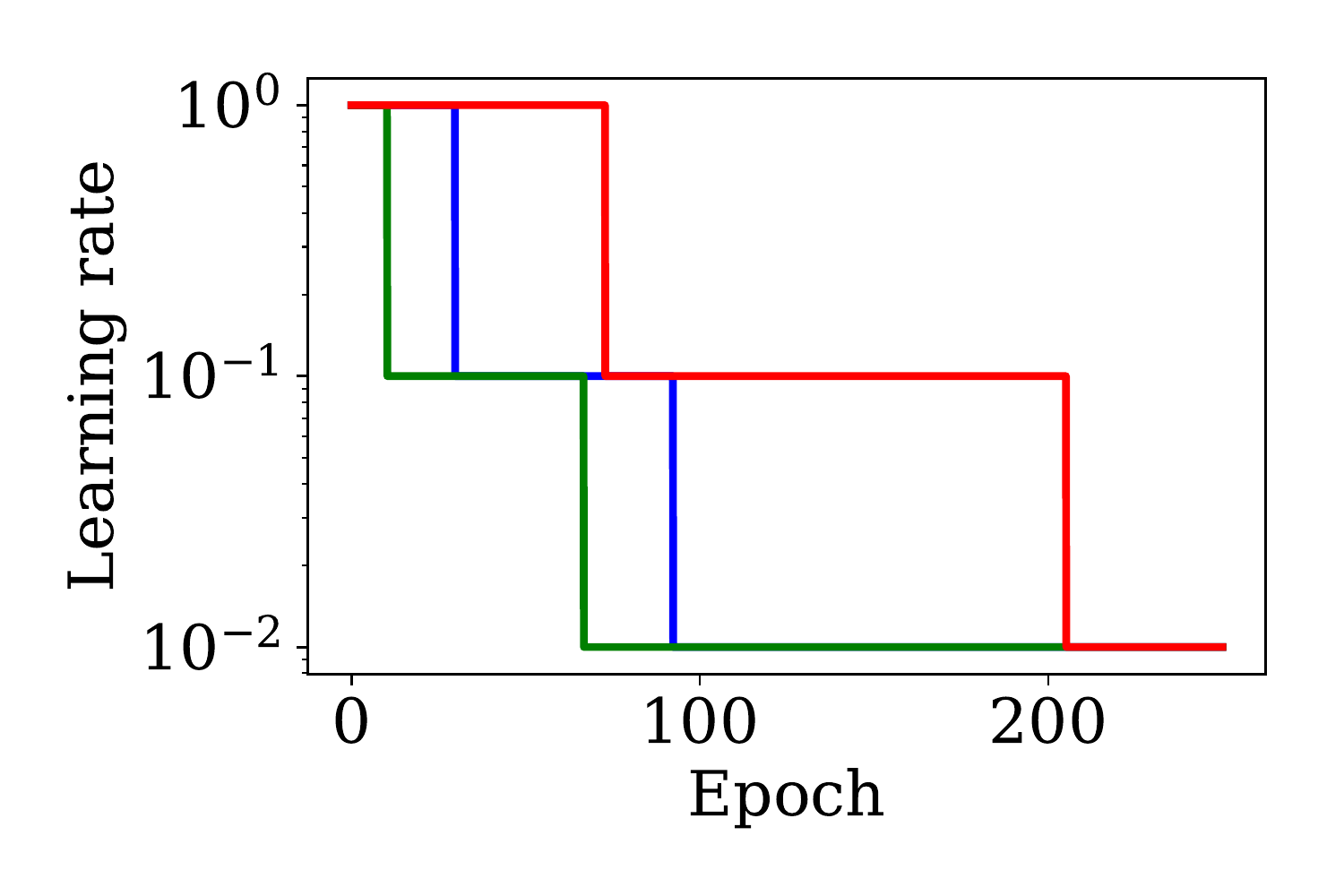}
%    \vspace{-2ex}
    \caption{Sensitivity analysis of SASA+ on CIFAR10. The test accuracy and learning rate schedule for SASA+ using different values of $\tau$ (first row), $\delta$ (second row) and $\theta$ (third row) around the default values in Table~\ref{tab:sasa-params};  SASA+ applied on different methods (i.e., SHB, NAG and QHM) (last row).}
    \label{fig:cifar_sasa_ablation}
\end{figure}

\textbf{SLOPE test for training loss.} 
Here we describe a statistical test that aims to detect if the training loss is no longer decreasing.
As the training goes on, we collect the minibatch loss $f_{\xi^k}(\xk)$ at every iteration, 
%The most intuitive criterion maybe is that the training loss stops decaying. To test this, we can
and build a linear model based on the last~$N$ mini-batches 
$\{(i, f_{\xi^i}(x^i))\}_{i=k-N+1}^k$.
We denote this linear model as $\hat{f}_{\xi^i}(x^i) = c_0 + c_1 k$,
%and test whether the slope is negative (the loss is still decreasing). 
and propose to use the following one-sided test
\begin{equation}\label{test:slope}
%	H_0: c_1 \ge 0 \quad \mbox{vs.} \quad H_1: c_1 < 0.
	\mbox{null}: c_1 \ge 0 \quad \mbox{vs.} \quad \mbox{alternative}: c_1 < 0.
\end{equation}
Like in SASA+, if we reject the null hypothesis with high confidence, 
% (say 0.95), 
%then we are very confident that the training loss is still decreasing and thus we keep the learning rate. 
then we keep the learning rate;
Otherwise, we decrease the learning rate by a constant factor. 
We use the standard t-test for linear regression \citep[e.g.,][Section~2.3]{montgomery2012introduction}. 
Specifically, we first compute the estimators 
$\hat{c}_0$ of $c_0$ and $\hat{c}_1$ of $c_1$ by solving a least-squares problem, 
then compute a statistic
$t_{\text{slope}} = \hat{c}_1/\sqrt{\hat{\sigma}_f}$, where $\hat{\sigma}_f$ can be computed with
standard formula which we provide in Appendix~\ref{apd:slopetest}.
When $f_{\xi_i}(x_i)) \!=\! c_0 \!+\! \zeta_i$ with i.i.d.\ Gaussian noises $\zeta_i$, the statistic $t_{\text{slope}}$ follows the $t$-distribution with $N\!-\!2$ degrees of freedom. 
Given a confidence level $(1\!-\!\delta$), we reject the null if $t_{\text{slope}} \!<\! t_{\delta, N-2}$ and accept it otherwise.

\begin{figure}[t]
    \centering
	\includegraphics[width=0.33\linewidth]{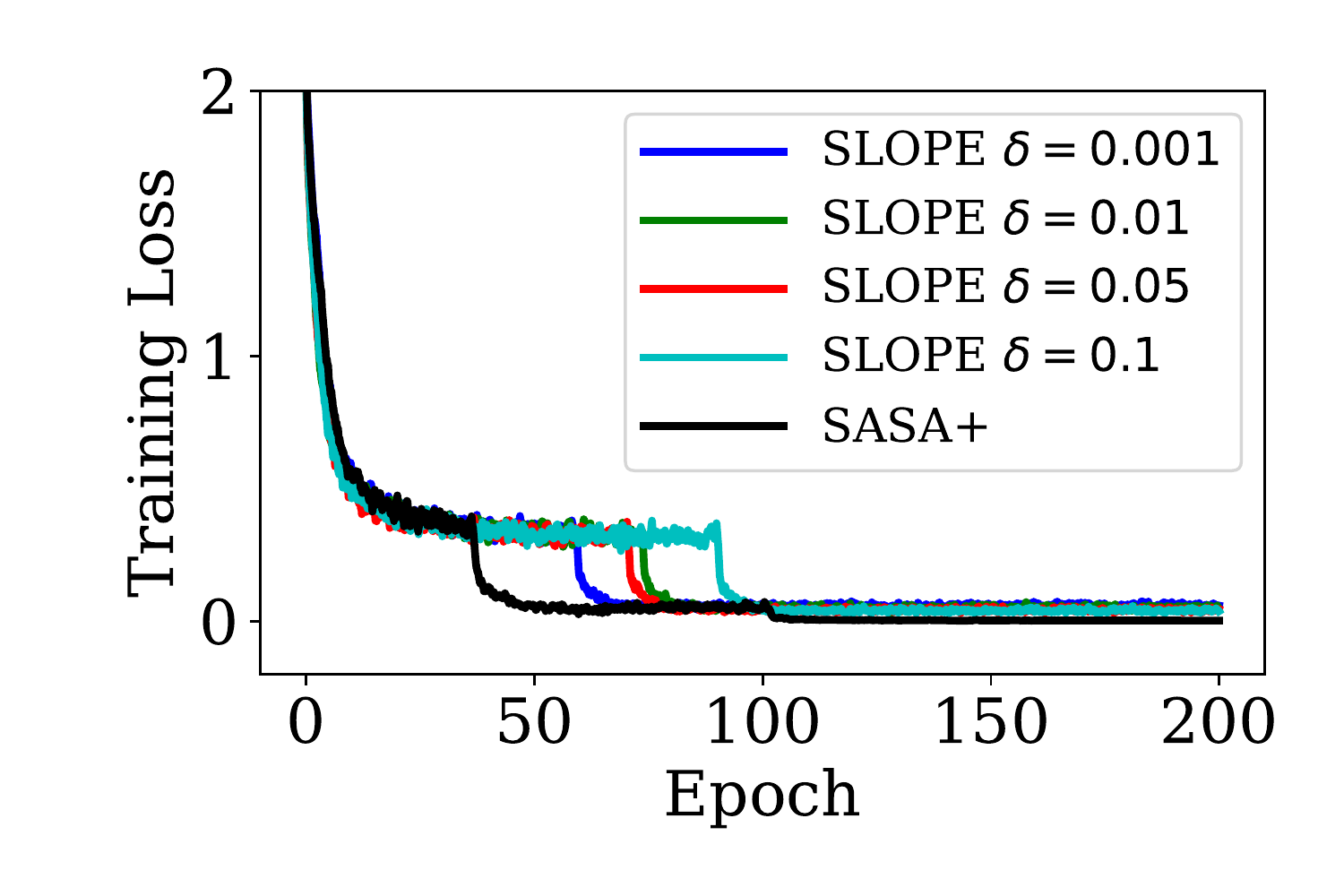}\quad
	\includegraphics[width=0.33\linewidth]{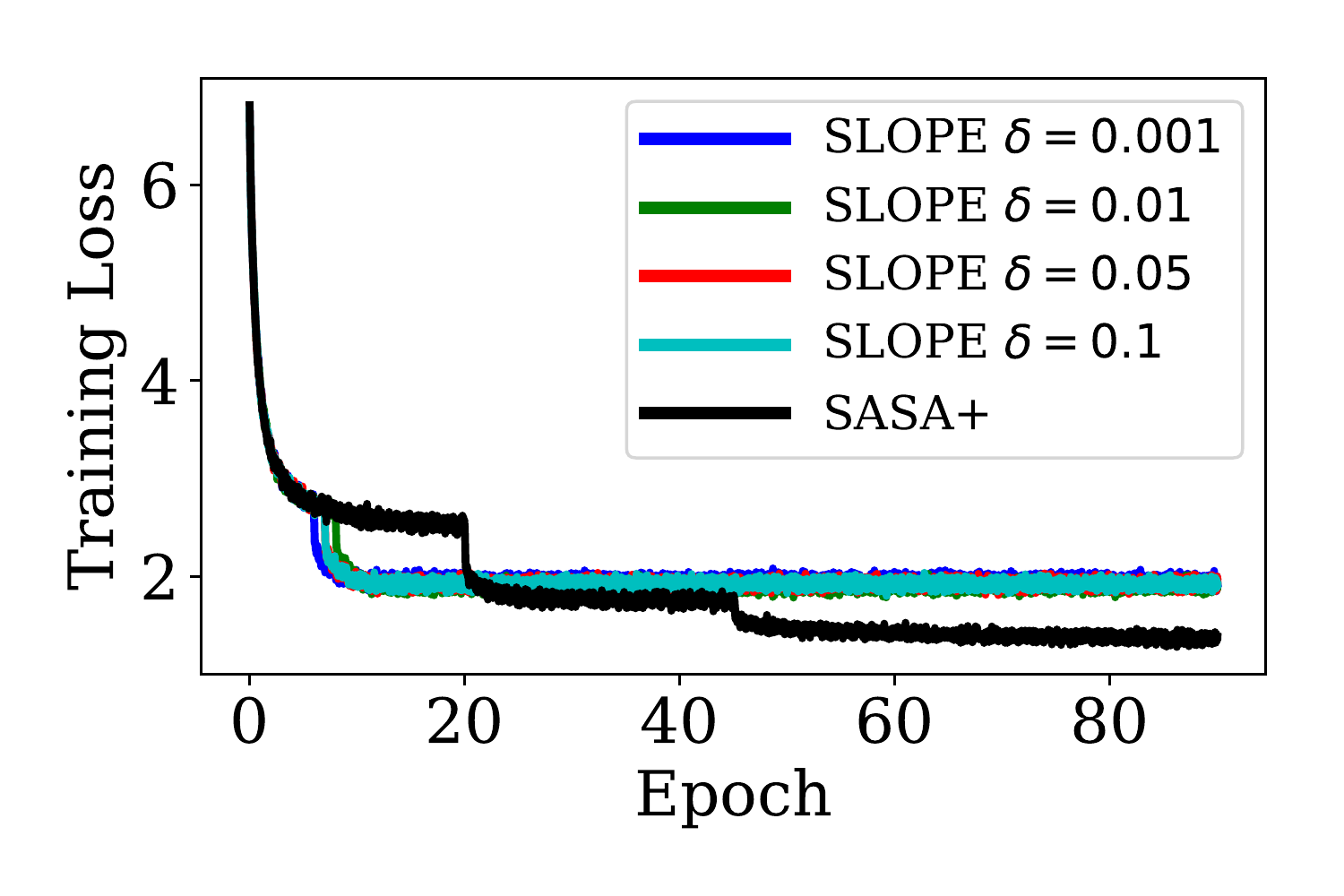}\\%[-1ex]
	\includegraphics[width=0.33\linewidth]{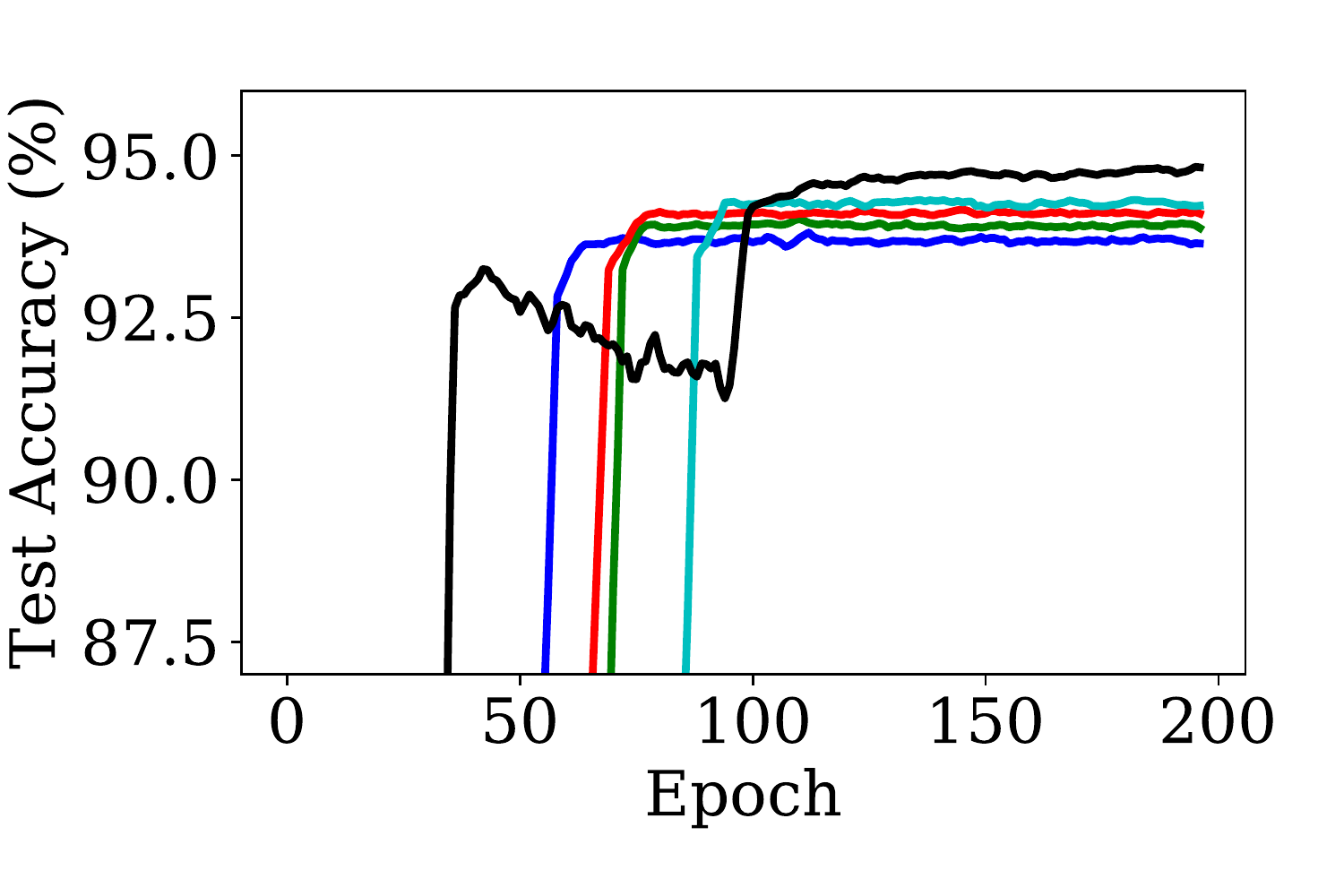}\quad
	\includegraphics[width=0.33\linewidth]{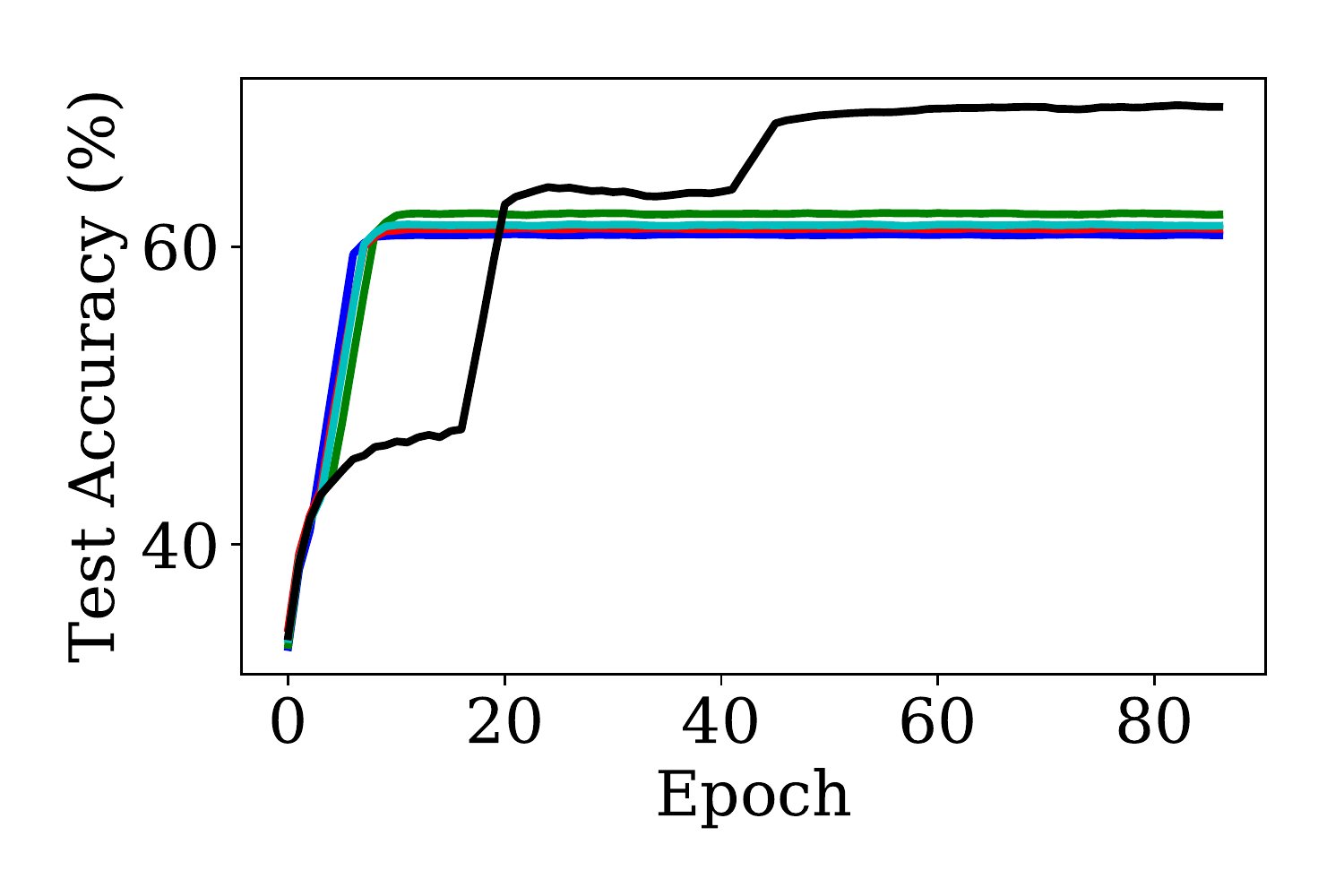}\\%[-1ex]
	\includegraphics[width=0.33\linewidth]{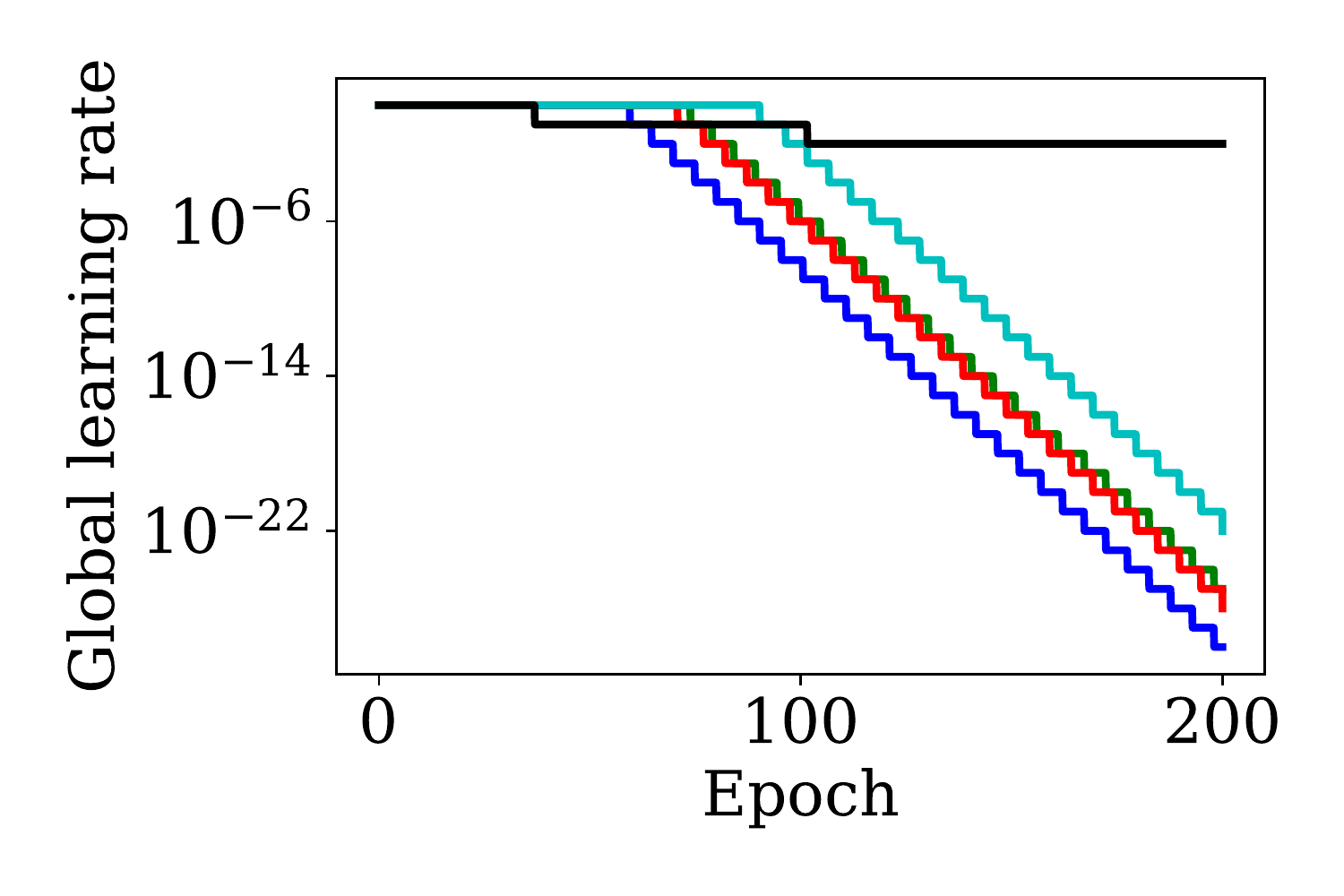}\quad
	\includegraphics[width=0.33\linewidth]{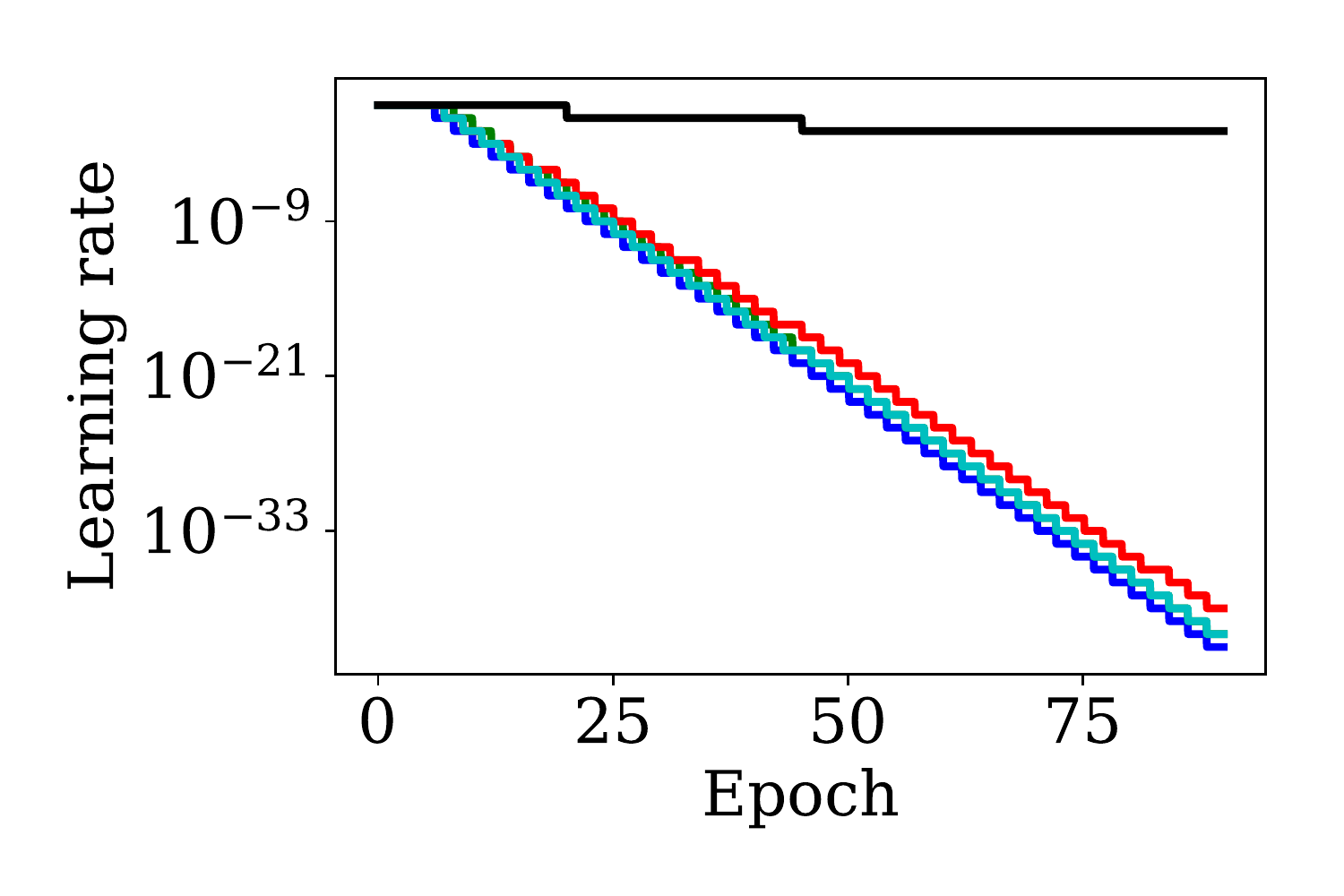}
% 	\includegraphics[width=0.32\linewidth]{archresnet18_trainloss_smooth_slope_leak8.pdf}
% 	\includegraphics[width=0.32\linewidth]{archresnet18_testacc_slope_leak8.pdf}
% 	\includegraphics[width=0.32\linewidth]{archresnet18_lrs_slope_leak8.pdf}\\
%	\vspace{-2ex}
	\caption{Comparison of SLOPE test with different confidence level $\delta$ and the default run of SASA+. First column: CIFAR-10 with ResNet18. Second column: ImageNet with ResNet18.}
	\label{fig:slopevssasa}
\end{figure}

\subsection{SASA+ Experiments}
\label{sec:sasa+experiments}
Figure \ref{fig:cifar_sasa_statistics} shows the evolution of SASA+'s different statistics during training the ResNet18 model on CIFAR-10 using the default parameters in Table~\ref{tab:sasa-params}. Between two jumps, the statistics $\Delta_k$ decays toward zero. As long as its confidence interval~\eqref{eqn:confi-interval} does {\it not} overlap with 0, we are confident that the process is not stationary yet and keep training with the current learning rate. Otherwise, we decrease the learning rate and enter another cycle of approaching stationarity. 

Figure~\ref{fig:cifar_sasa_ablation} illustrates the sensitivity of SASA+ to its hyperparameters around their default values. The first row shows that SASA+ is robust to the choice of the dropping factor $\tau$. The larger $\tau$ is, the longer the process will stay between two drops (also see Figure \ref{fig:cifar_sasa_statistics}). The second row shows SASA+ is insensitive to the confidence level $\delta$. The third row shows the effect of $\theta$, the fractions of samples to keep after reset. Smaller $\theta$ values lead to more frequent dropping, but do not impact the final performance. The fourth row shows that SASA+ can work with various algorithms, i.e., SHB~\eqref{eqn:d-shb},
%(which \cite{yaida2018fluctuation} and \cite{LangZhangXiao2019} focused on), 
NAG~\eqref{eqn:d-nes} and QHM~\eqref{eqn:d-qhm}. Results on ImageNet and MNIST in Appendix~\ref{apd:sasaplusresults} also support these findings.  

Figure~\ref{fig:slopevssasa} shows that the SLOPE test %~\eqref{test:slope} 
is too aggressive after the first learning rate drop, and thus cannot reach good training loss and testing accuracy. See similar results for logistic regression on MNIST in Appendix~\ref{apd:sloperesults}.

%% file: ssls_1c.tex
% \vspace{-4mm}
\section{Smoothed Stochastic Line Search}
\label{sec:ssls}

\iffalse
\SetAlgoHangIndent{3em}  % for lines that are too long use second line
\setlength{\algomargin}{1.2em}
\begin{algorithm}[t]
	\DontPrintSemicolon
	\caption{Smoothed Stochastic Line-Search (SSLS)}
	\label{alg:ssls}
	\textbf{input:} 
    $x^0\in\R^p$, $\alpha_{-1}>0$,
    sufficient decrease coefficient $c\in(0,1/2)$, 
    line-search factors $\rho_\mathrm{inc}\geq 1$, $\rho_\mathrm{dec}\in(0,1)$, 
    smoothing parameter $\gamma\in[0,1]$, 
    and maximum LS count~$m$ \\
	\For{$k = 0,...,T-1$}{
        Sample $\xik$, compute $\gk \gets \nabla f_{\xik}(\xk)$ 
        and $\dk$ \\ 
        $\eta_{k} \gets \rho_\mathrm{inc} \alpha_{k-1} $\\
        \While{$f_{\xik}(\xk-\eta_k\gk)>f_{\xik}(\xk)-c\cdot\eta_k\|\gk\|^2$ \textbf{and} $\eta_k>\rho_\mathrm{dec}^{m} \alpha_{k-1}$}{
            $\eta_k \gets \rho_\mathrm{dec} \eta_k$ 
        }
        $\alpha_k \gets (1-\gamma)\alpha_{k-1} + \gamma\eta_k$ \\
        $\xkp \gets \xk - \alpha_k \dk$
    }
    \textbf{output:} $x^T$ 
\end{algorithm} 
\fi

\SetAlgoHangIndent{3em}  % for lines that are too long use second line
\begin{algorithm}[t]
	\DontPrintSemicolon
	\caption{Smoothed Stochastic Line-Search (SSLS)}
	\label{alg:ssls}
	\textbf{input:} 
    $x^0$, $\alpha_0$ \\
    \textbf{parameters:} $c\in(0,1/2)$, $\rho_\mathrm{inc}\geq 1$, $\rho_\mathrm{dec}\in(0,1)$, $m>0$ \\
	\For{$k = 0,...,T-1$}{
        Sample $\xik$, compute $\gk \gets \nabla f_{\xik}(\xk)$ 
        and $\dk$ \\ 
        $\eta_{k} \gets \rho_\mathrm{inc} \alpha_{k-1} $\\
        \For{$i=1,\ldots,m$}{
        \eIf{$f_{\xik}(\xk-\eta_k\gk) < f_{\xik}(\xk)-c\cdot\eta_k\|\gk\|^2$}{
            \vspace{0.5ex}
            \textbf{break} (out of for loop)}{
            $\eta_k \gets \rho_\mathrm{dec} \eta_k$ 
        }
        }
        $\alpha_k \gets (1-\gamma)\alpha_{k-1} + \gamma\eta_k$ \\
        $\xkp \gets \xk - \alpha_k \dk$
    }
    \textbf{output:} $x^T$ 
\end{algorithm}

SASA+ can automatically decrease the learning rate to refine the last
phase of the optimization process, but it relies on an appropriate
initial learning rate to make fast progress.
The appropriate initial learning rate varies
substantially for different objective functions, machine learning
models, and training datasets. 
Setting it appropriately without expensive trial and error is a major challenge for all stochastic gradient methods, adaptive or not.

Several recent works 
\citep[e.g.,][]{SchmidtLeRouxBach2017,Vaswani2019LineSearch}
explore the use of classical line-search schemes 
\citep[e.g.,][Chapter~3]{NocedalWright2006book}
for stochastic optimization.
One of the main difficulties is that the estimated step sizes may vary a lot from step to
step and it may not capture the right step size for the average loss function.
To overcome this difficulty, we propose a smoothed stochastic
line-search (SSLS) procedure listed in Algorithm~\ref{alg:ssls}.

During each step~$k$, SSLS uses the classical Armijo line-search 
\citep[e.g.,][Chapter~3]{NocedalWright2006book}
to find a step size $\eta_k$ 
for the randomly chosen function $f_{\xik}$
(initialized by $\alpha_{k-1}$), then sets the next learning rate to be
\[
\alpha_k = (1-\gamma)\alpha_{k-1} + \gamma \eta_k,
\]
where $\gamma\in[0,1]$ is a smoothing parameter.
%When $\gamma=0$, there is no line-search and SSLS always use the initial learning rate.
When $\gamma=1$, SSLS reduces to the stochastic Armijo line-search used by
\citet{Vaswani2019LineSearch}.
For optimization over a finite dataset,
a good choice is to set $\gamma=b/n$ where $n$ is the total number of training
examples and $b$ is the mini-batch size.
Suppose $\rho_\mathrm{inc}=2$ and $\eta_k=2\alpha_{k-1}$ is accepted 
at step~$k$, then
$
    \alpha_k = (1-\gamma)\alpha_{k-1} + \gamma (2\alpha_{k-1})
    = (1+\gamma)\alpha_{k-1}.
$
If this happens at every iteration over one epoch 
(of $\lceil n/b\rceil$ iterations) and $n\gg b$, then
\begin{equation}\label{eqn:epoch-growth}
    \alpha_{k+\lceil n/b\rceil} = (1+b/n)^{\lceil n/b\rceil} \alpha_k 
    \approx e \cdot \alpha_k .
\end{equation}
Therefore, the most aggressive growth of the learning rate is by a factor 
of~$e$ over one epoch.
Such a growing factor is reasonable for line search in deterministic
optimization \citep[e.g.,][]{Nesterov2013composite}.
Setting $\gamma=\sqrt{b/n}$ leads to maximum growth of $e^2$ over one epoch.
A similar smoothing effect holds for decreasing the learning rate as well.
The smoothing scheme allows us to use standard increasing and decreasing
factors, and we use $\rho_\mathrm{inc}=2$, $\rho_\mathrm{dec}=1/2$, $c=0.05$, and $\gamma = \sqrt{b/n}$ as the default parameters.
%
%Without the smoothing scheme, \citet{Vaswani2019LineSearch} set 
%$\rho_\mathrm{inc}=2^{b/n}$ to restrict dramatic growth of $\alpha_k=\eta_k$
%(equivalent to $\gamma=1$ in SSLS).
%However, the decreasing of $\alpha_k=\eta_k$ can be excessive and premature,
%even with $\rho_\mathrm{dec}=0.9$. 
%(setting $\rho_\mathrm{dec}$ closer to~$1$ would dramatically slow down line search, losing the main point of line search).

\begin{figure}[t]
\centering
    \includegraphics[width=0.243\linewidth]{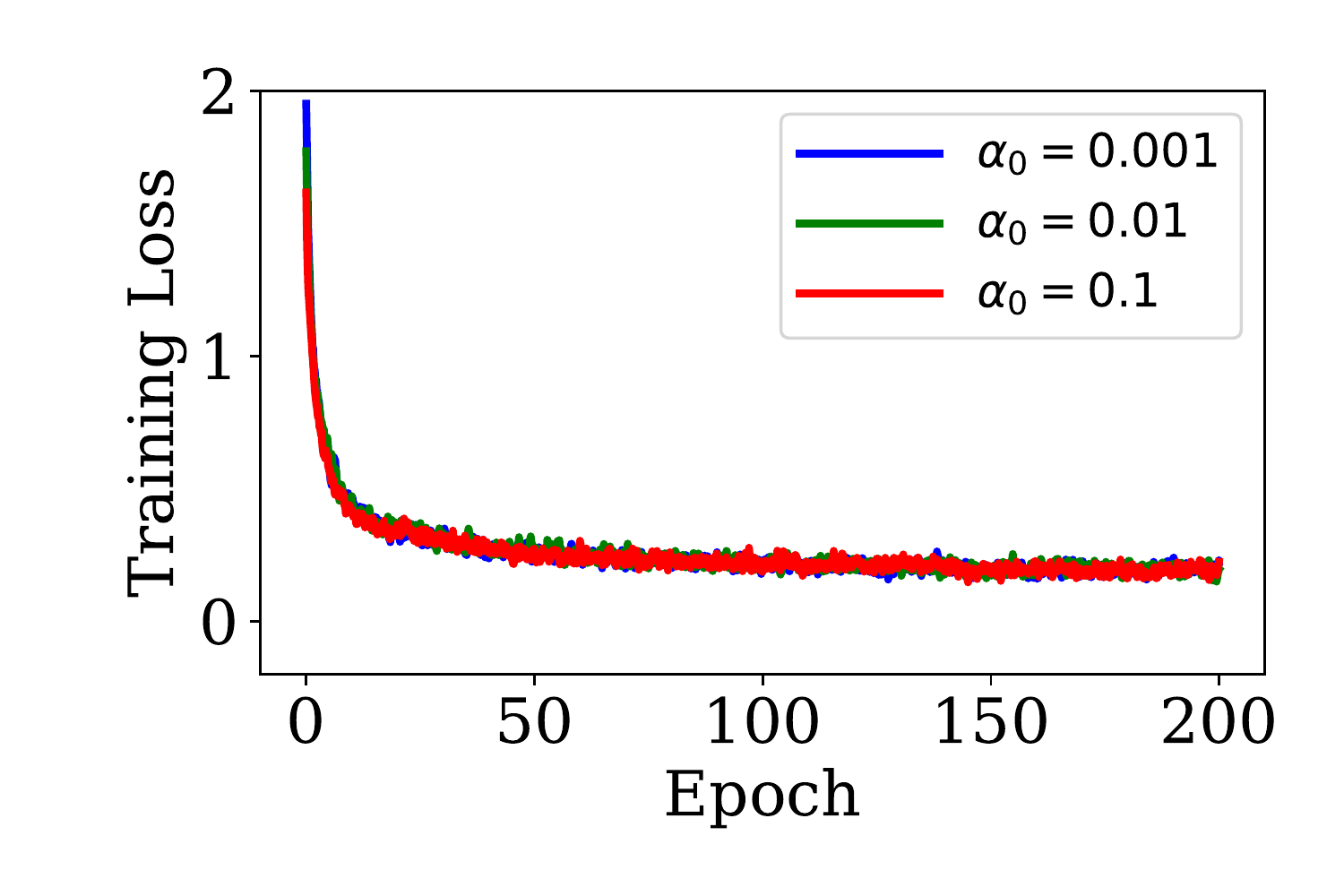}
    \includegraphics[width=0.243\linewidth]{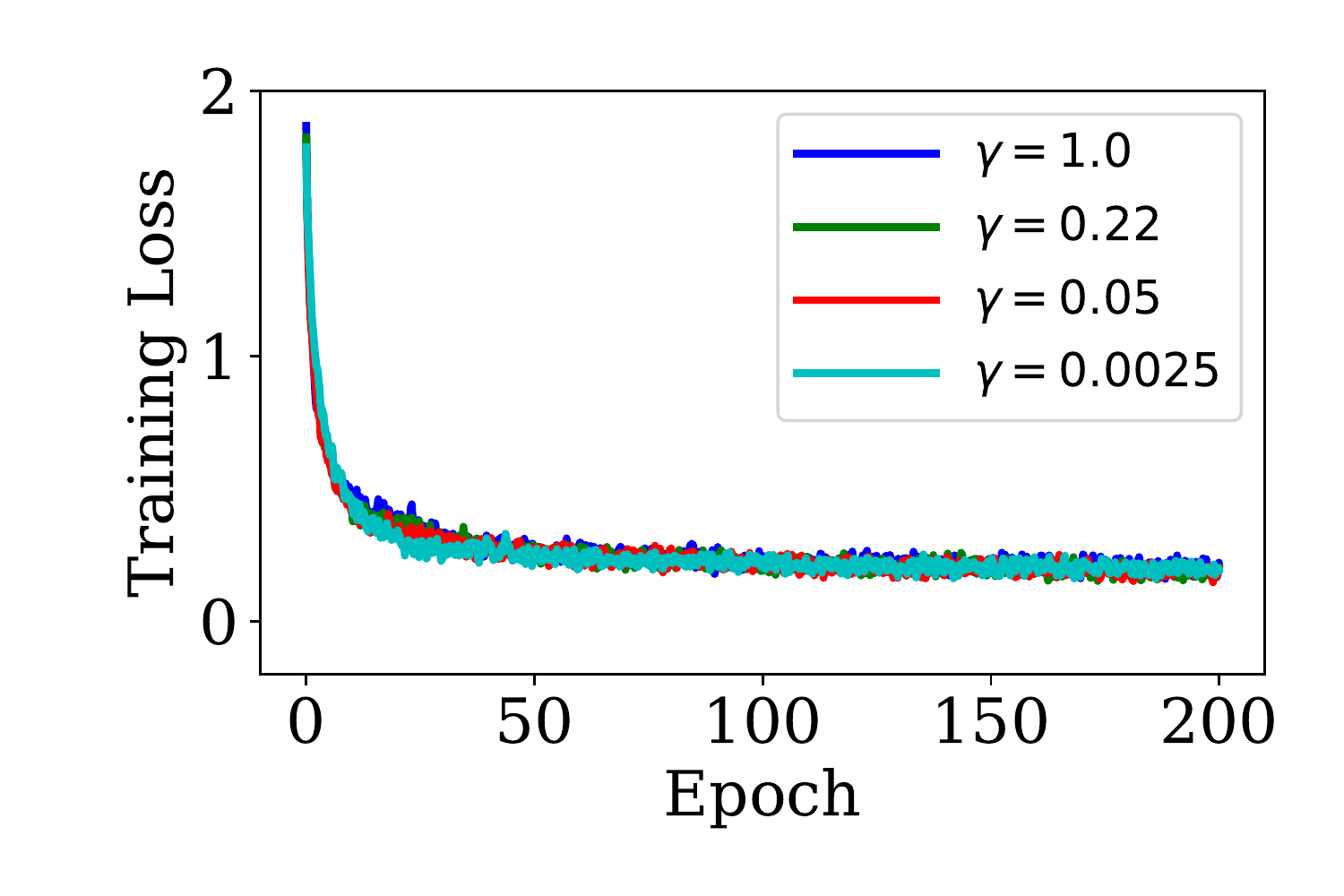}
    \includegraphics[width=0.243\linewidth]{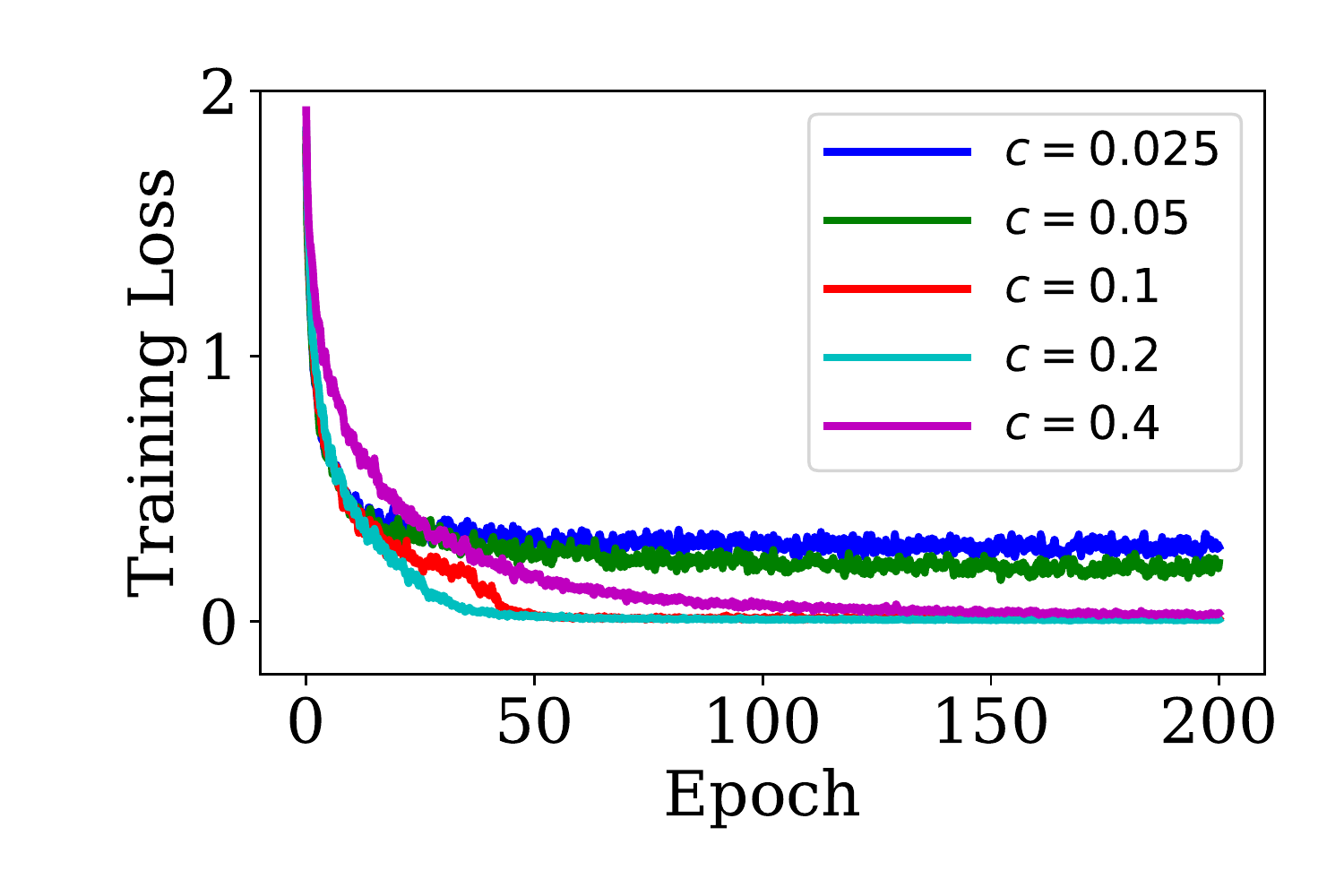}
    \includegraphics[width=0.243\linewidth]{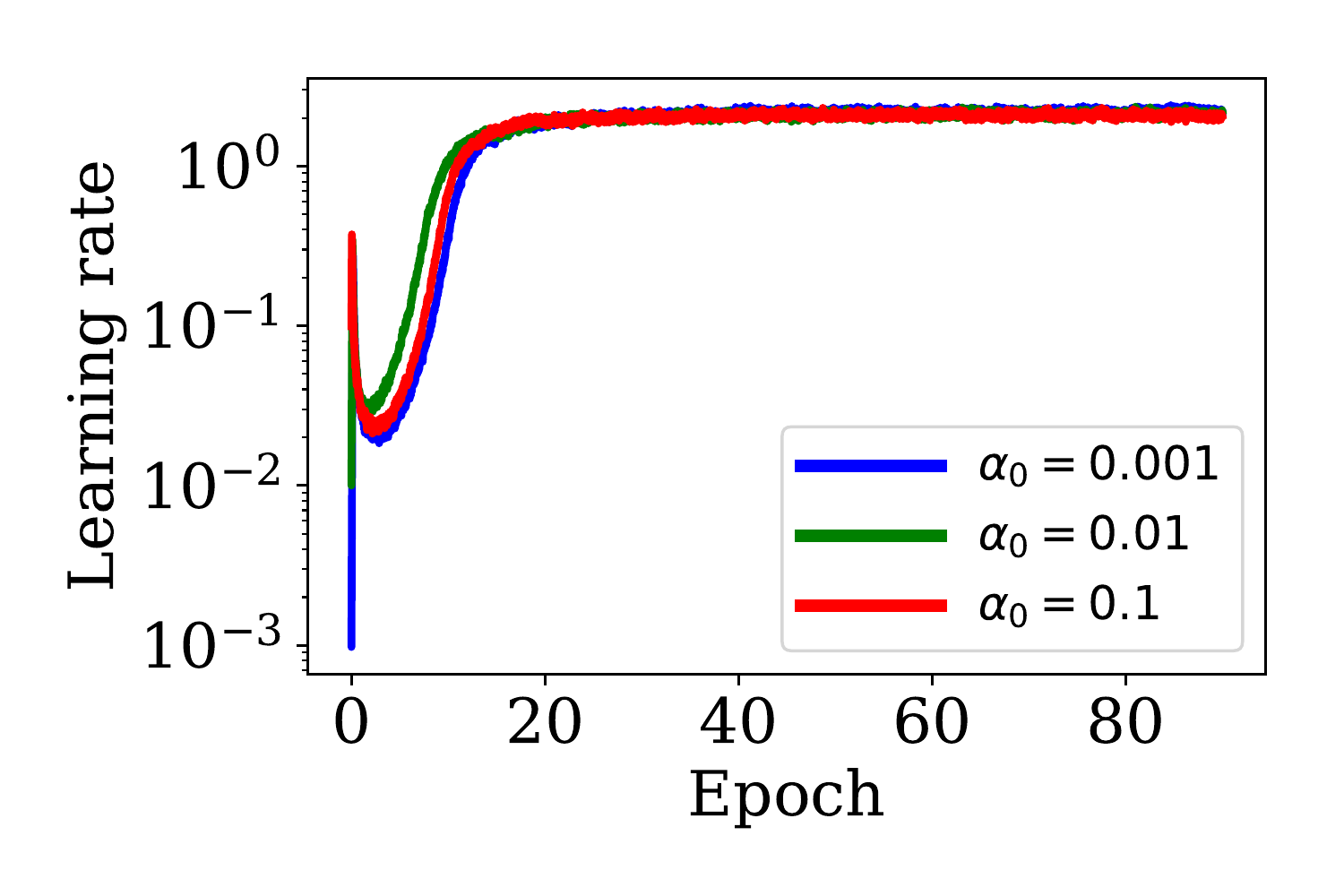} \\%[-1ex]
    \includegraphics[width=0.243\linewidth]{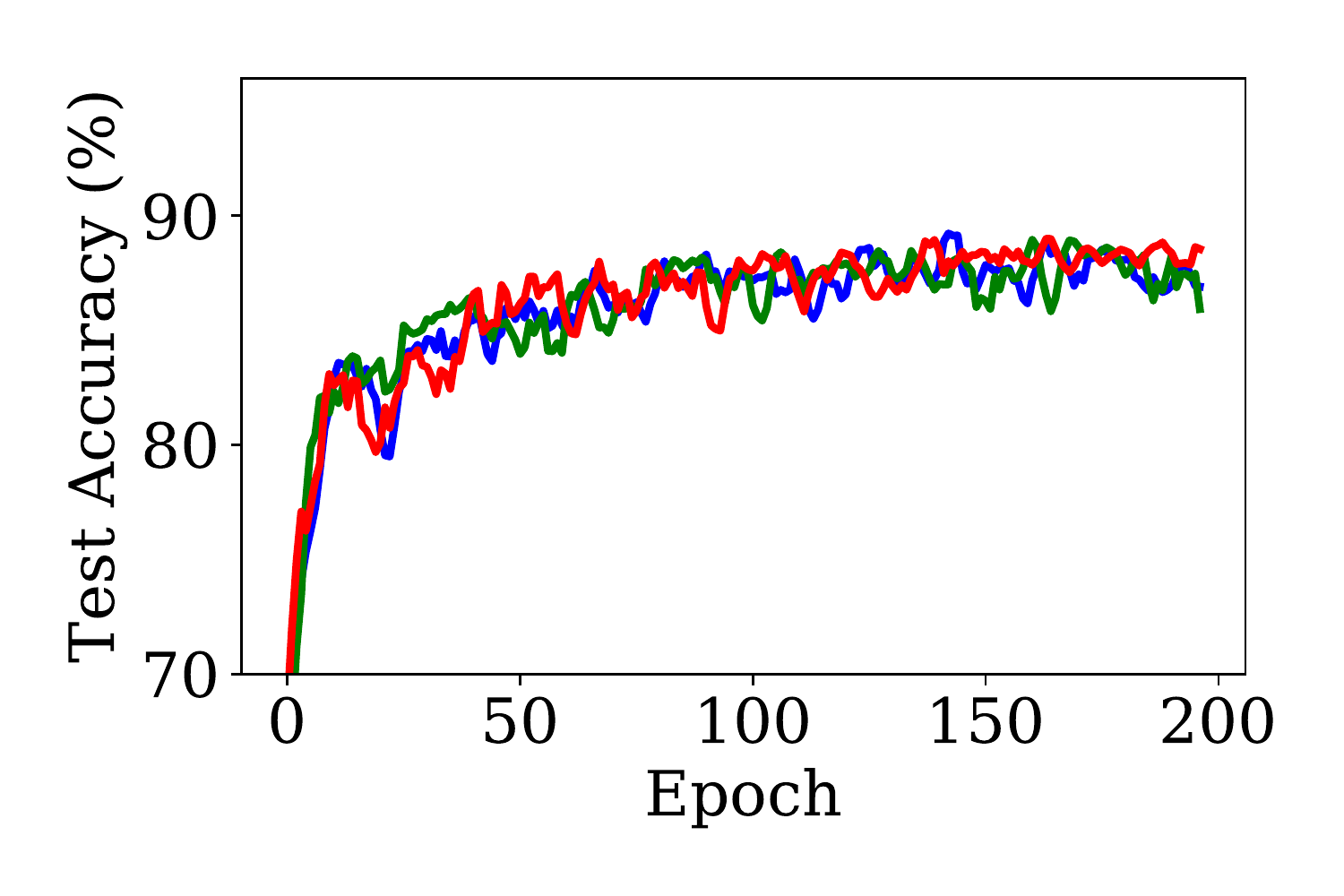}
    \includegraphics[width=0.243\linewidth]{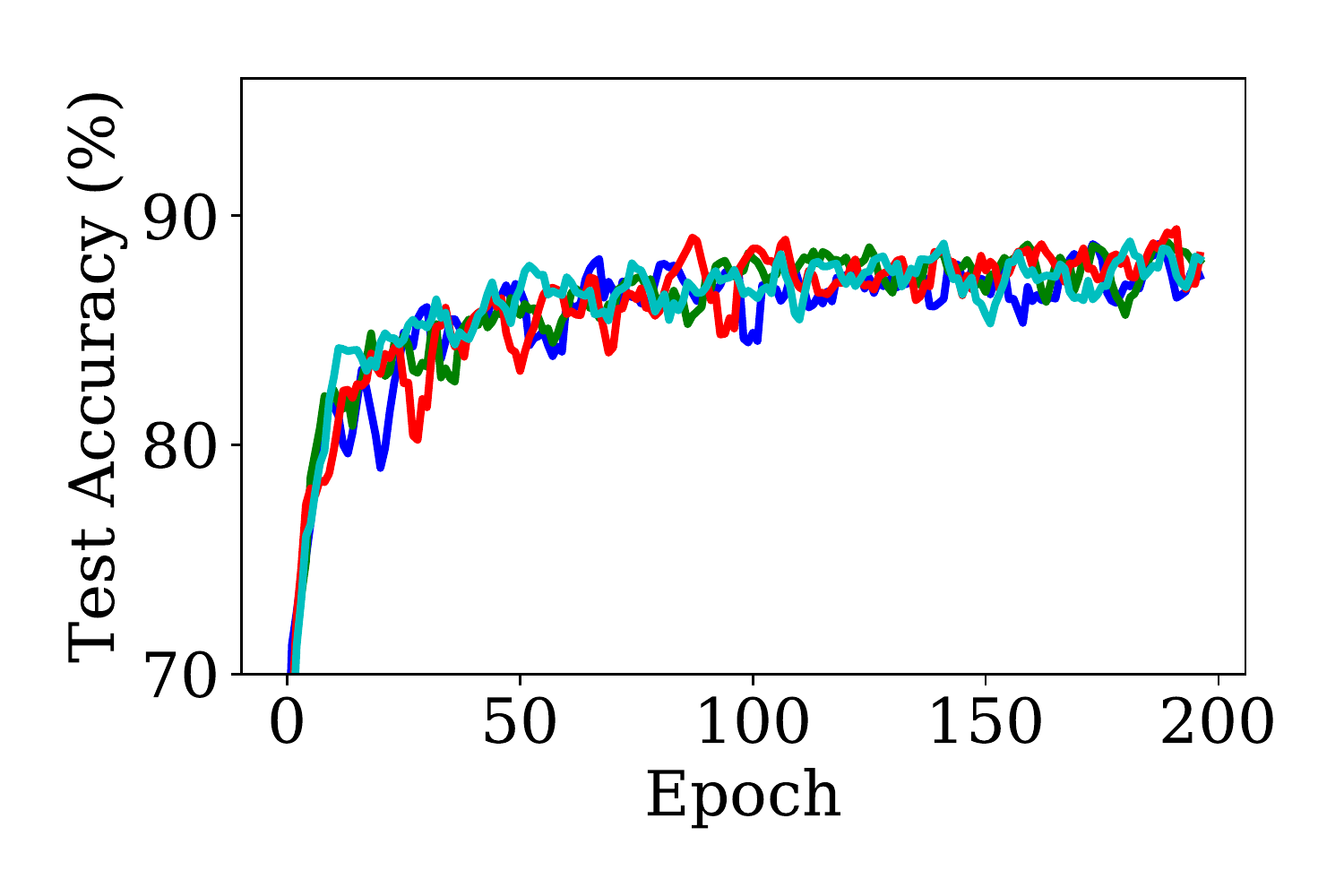}
    \includegraphics[width=0.243\linewidth]{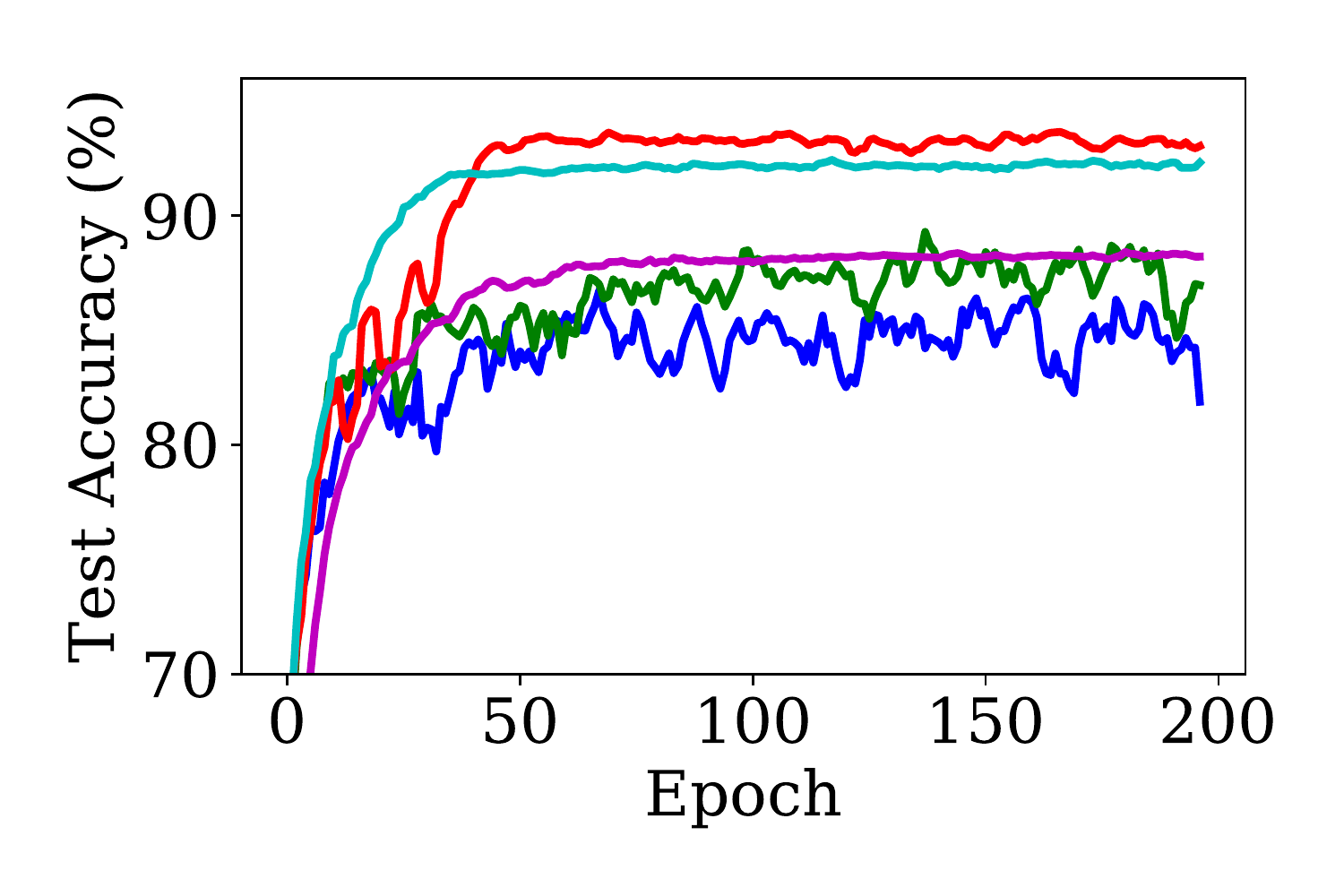} 
    \includegraphics[width=0.243\linewidth]{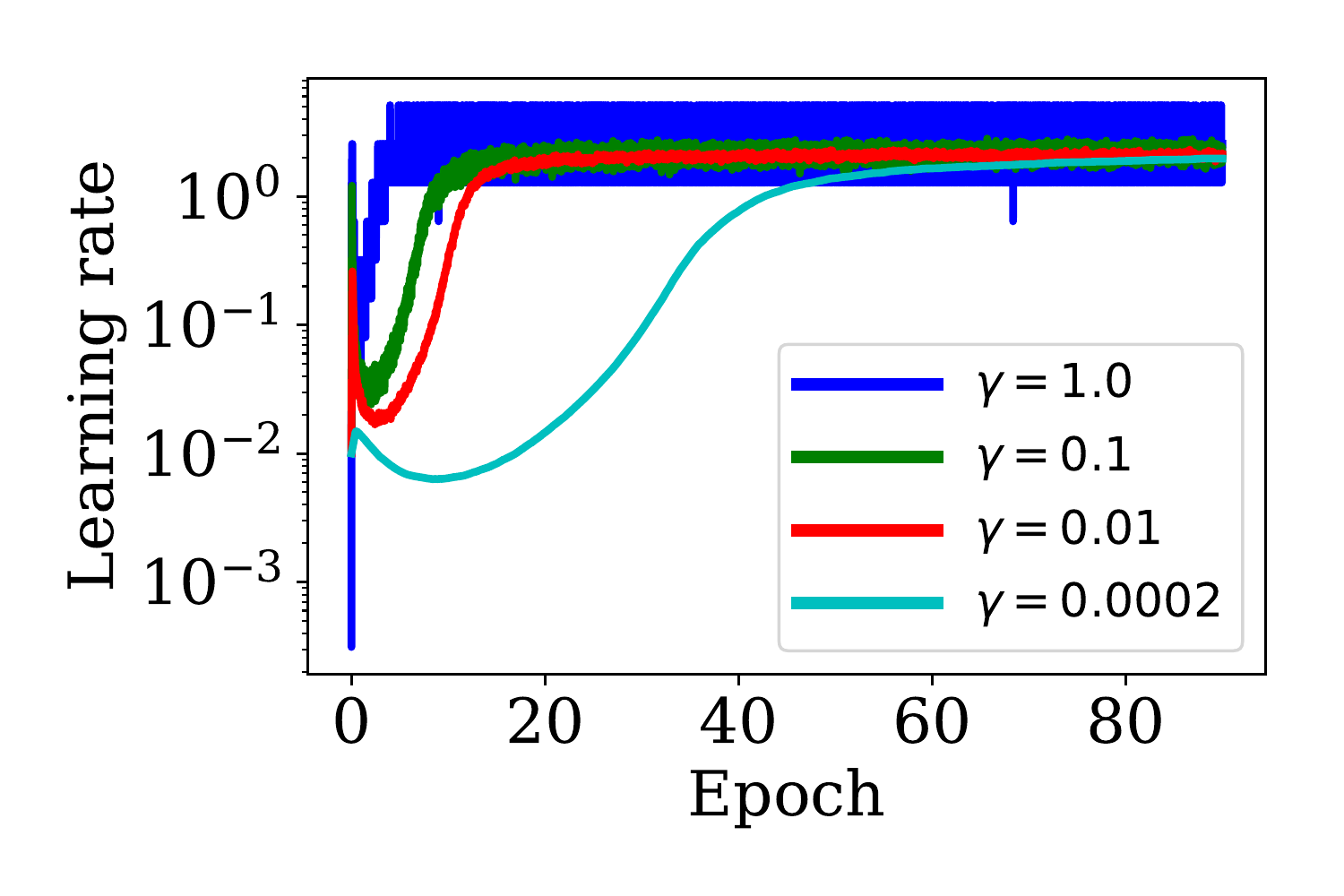}\\%[-1ex]
    \includegraphics[width=0.243\linewidth]{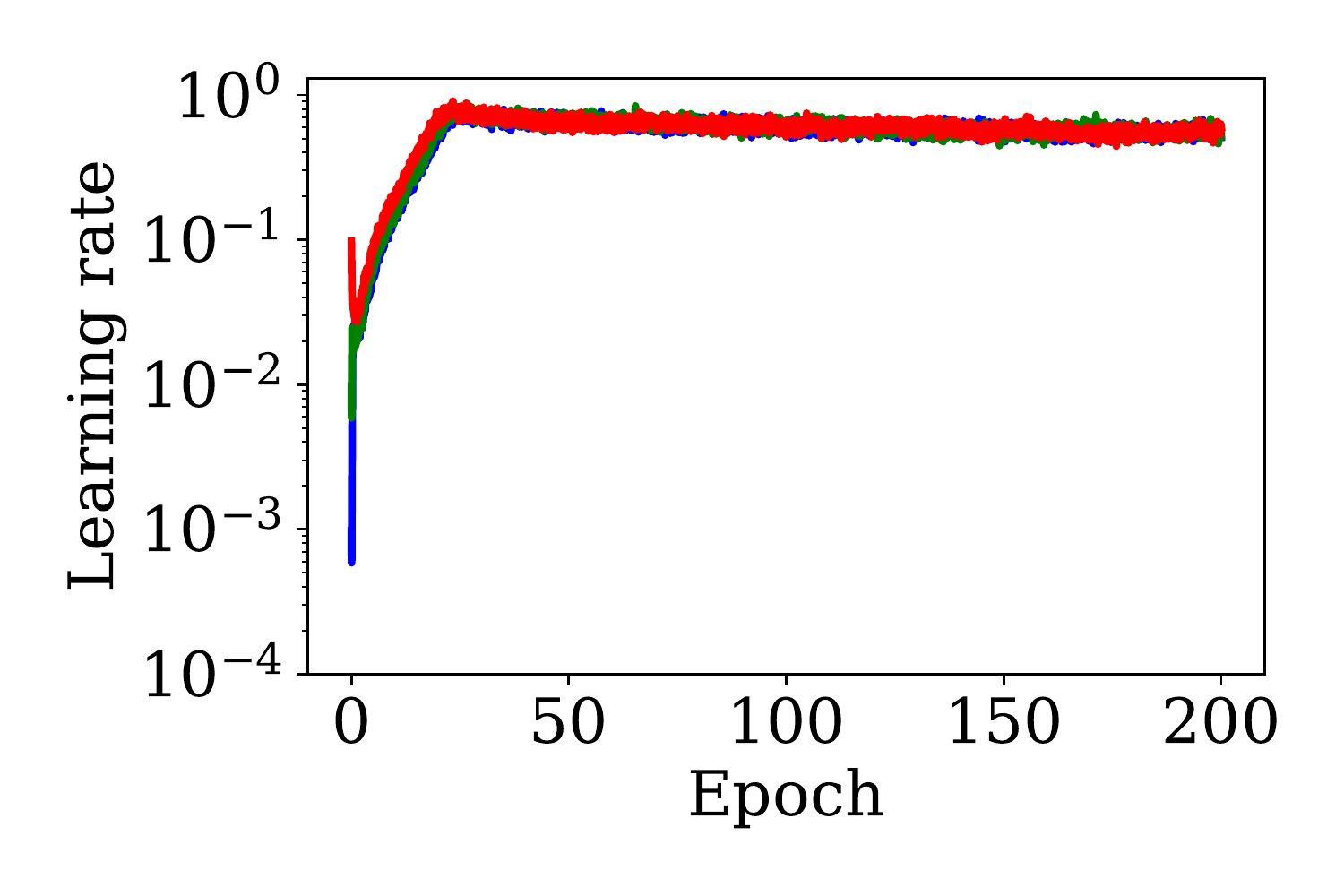} 
    \includegraphics[width=0.243\linewidth]{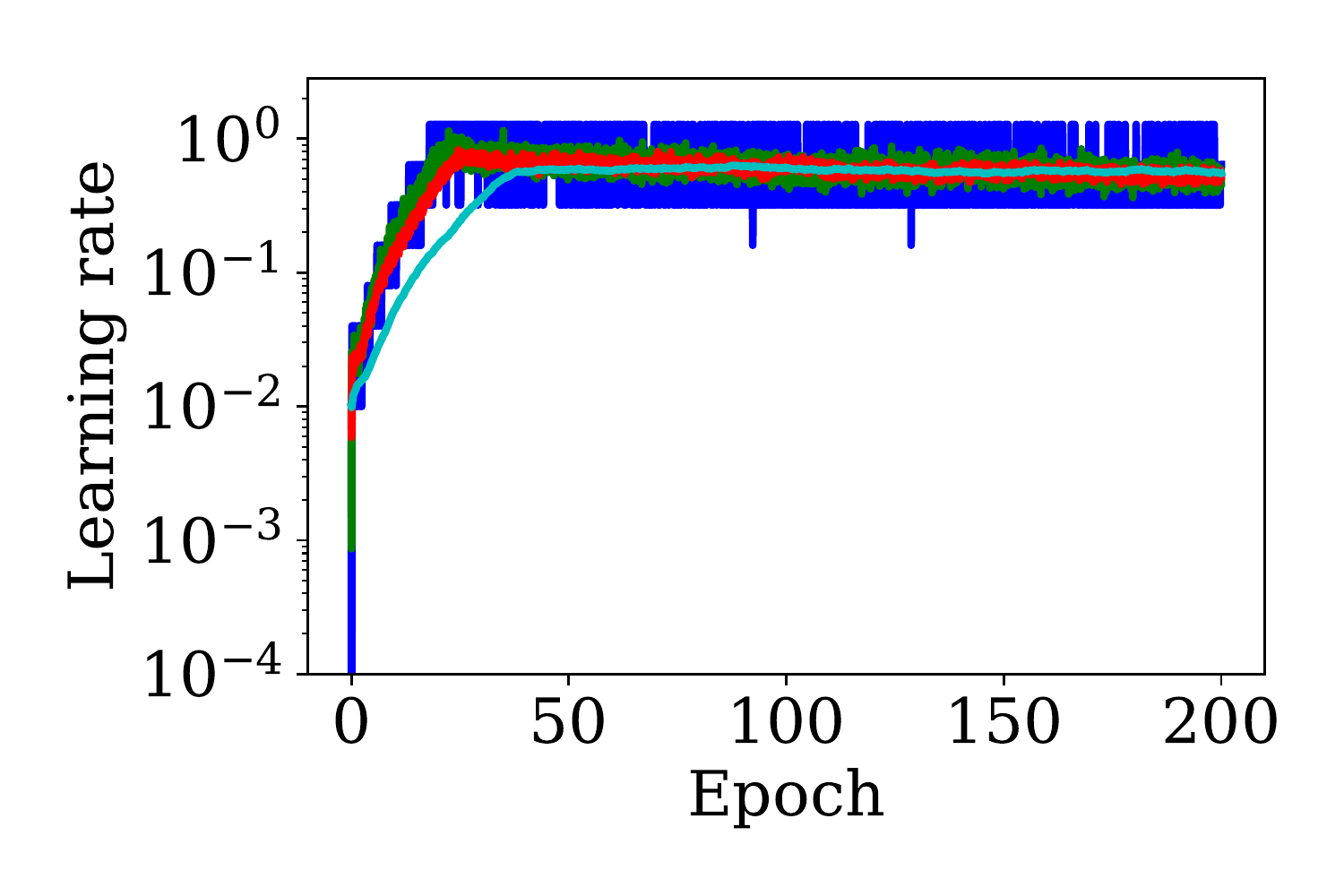} 
    \includegraphics[width=0.243\linewidth]{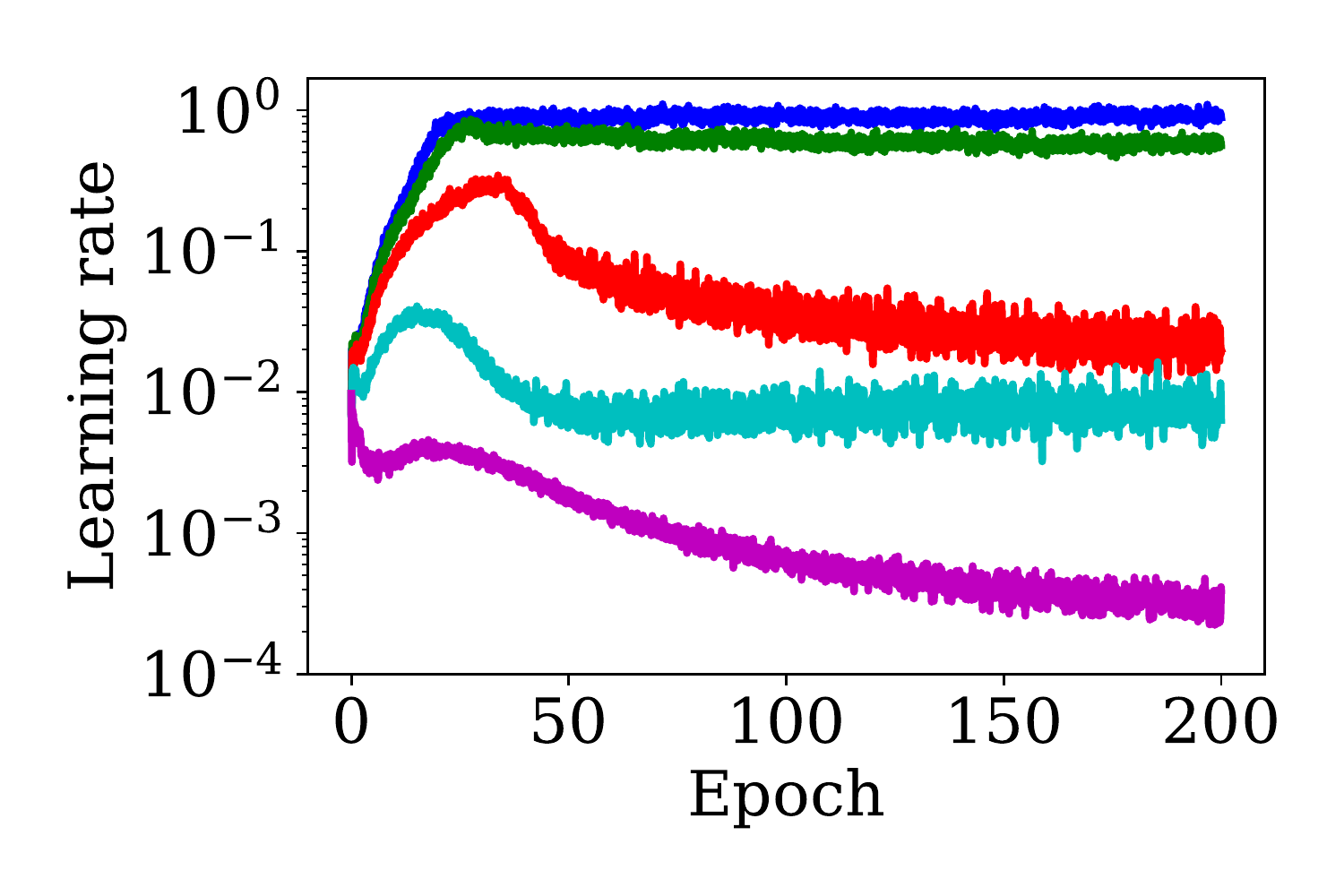}
    \includegraphics[width=0.243\linewidth]{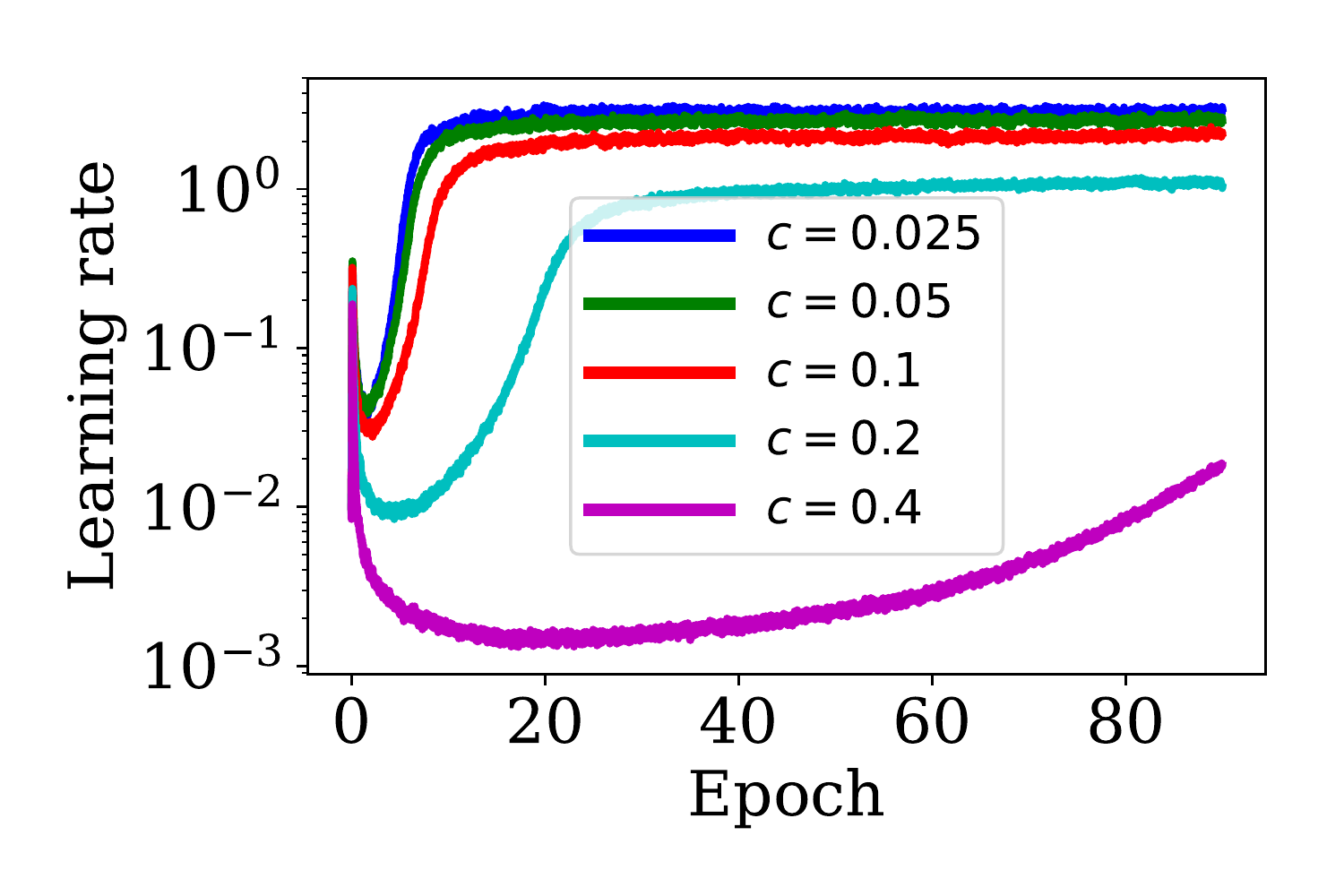}\\%[-1ex]
 	\centerline{\small (a) \hspace{3.5cm} (b) \hspace{3.5cm} (c) \hspace{3.5cm} (d)}
% 	\vspace{-1ex}
    \caption{Sensitivity of SSLS around its default values $c=0.05$ and $\gamma = \sqrt{b/n}$ (0.05 for CIFAR-10 and 0.01 for ImageNet). (a) Initial learning rate $\alpha_0$, (b) The smoothing factor $\gamma$ and (c) and the sufficient decrease constant $c$ on CIFAR-10 with ResNet18. The column (d) shows the learning rate schedules of SSLS when applied to ImageNet with ResNet18 with different hyperparameters. Best viewed in color.}
    \label{fig:sgdsls}
\end{figure}

Following the analysis of Armijo line-search in \citet{Vaswani2019LineSearch},
it is possible to show that SSLS has similar convergence properties 
under the smooth and interpolation assumptions,
which we leave as a future research project.
In this paper, we are mainly interested in its performance as a practical heuristic.
In particular, we use it on deep neural network models with ReLU activations. 
Here the loss functions are non-smooth, and classical theoretical 
analysis for line-search do not carry over.
Nevertheless, we found SSLS to have robust performance in 
all of our experiments. 
In order to handle the case of potential non-descent directions in the 
non-smooth case, we always exit the line search after a maximum of $m$ tries.
By default, we set $m=2$ (there is no significant difference from $m=2$ to $m=10$).

\textbf{SSLS Experiments.} Figure~\ref{fig:sgdsls} shows our sensitivity study on SSLS. Column (a) and (b) indicate that for a wide range of the initial learning rate $\alpha_0$ and the smoothing parameter~$\gamma$, SSLS always settles down to a stable learning rate ($0.53$ for CIFAR-10 and $2.1$ for ImageNet, which is comparable with the best hand-tuned value of 1.0).
However, large~$\gamma$ causes oscillation around a stable learning rate. Column~(c) shows that the larger the sufficient descent constant $c$ is, the smaller the stable learning rate is and the slower SSLS reaches it. This is intuitive because larger $c$ requires steeper descent, forcing SSLS to settle on a smaller learning rate on average. Note that both training loss and validation accuracy obtained by SSLS is worse than the results obtained by hand-tuned optimizers.
%however, we can switch to SASA+ to gradually reducing the leraning rates.
Column (d) shows the resulting learning rates of SSLS on ImageNet, with more details in Appendix~\ref{apd:sslsresults}.

% \begin{figure*}[t]
%     \includegraphics[width=0.245\linewidth]{archmyresnet18_trainloss_smooth_sslslrs.pdf}
%     \includegraphics[width=0.245\linewidth]{archmyresnet18_mom09_ssls_gammas_trainloss_smooth.pdf}
%     \includegraphics[width=0.245\linewidth]{archmyresnet18_trainloss_smooth_ssls_sdcs.pdf}
%     \includegraphics[width=0.245\linewidth]{archresnet18_dirg_lrs_sslslrs.pdf} \\
%     \includegraphics[width=0.245\linewidth]{archmyresnet18_testacc_sslslrs.pdf}
%     \includegraphics[width=0.245\linewidth]{archmyresnet18_mom09_ssls_gammas_testacc.pdf}
%     \includegraphics[width=0.245\linewidth]{archmyresnet18_testacc_ssls_sdcs.pdf} 
%     \includegraphics[width=0.245\linewidth]{archresnet18_dirg_lrs_sslsgamma.pdf} \\
%     \includegraphics[width=0.245\linewidth]{archmyresnet18_lrs_sslslrs.pdf} 
%     \includegraphics[width=0.245\linewidth]{archmyresnet18_mom09_ssls_gammas_lrs.pdf} 
%     \includegraphics[width=0.245\linewidth]{archmyresnet18_lrs_ssls_sdcs.pdf}
%     \includegraphics[width=0.245\linewidth]{archresnet18_dirg_lrs_sslssdcs.pdf}\\[-1ex]
%  	\centerline{\small (a) \hspace{4cm} (b) \hspace{4cm} (c) \hspace{4cm} (d)}
%  	\vspace{-4ex}
%     \caption{Sensitivity of SSLS. (a) Initial learning rate $\alpha_0$, (b) The smoothing factor $\gamma$ and (c) and the sufficient decrease constant $c$ on CIFAR-10 with ResNet18. The column (d) shows the learning rate schedules of SSLS when applied to ImageNet with ResNet18 with different hyperparameter settings.}
%     \label{fig:sgdsls}
% \end{figure*}

%% file: salsa_experiments_1c.tex
%\clearpage
%\section{SALSA and Experiments}
\section{SALSA}
\label{sec:salsa}

\begin{algorithm}[t]
	\DontPrintSemicolon
	\caption{SALSA: SASA+ with warmup by SSLS}
	\label{alg:salsa}
    \textbf{input:} $x^0\in\R^p$, $\alpha_0>0$, $switched$=False\\
    % $switched\gets$ False \\
    \For{$k = 0,...,T$}{
        \eIf{\emph{not} $switched$}{
            Run one step of SSLS (Algorithm~\ref{alg:ssls}) \\
%            $x\_stationary\gets$ stationary\_test (SASA+)\\
%            $f\_stationary\gets$ slope\_test (SLOPE) \\
            $x\_stationary\gets$ SASA+ test in Algorithm~\ref{alg:sasa} \\
            $f\_stationary\gets$ SLOPE test~\eqref{test:slope} \\
            $switched\gets$  $x\_stationary$ or $f\_stationary$
        }{
            Run one step of SASA+ (Algorithm~\ref{alg:sasa}) 
        }
    }
    \textbf{output:} $x^{T}$ 
\end{algorithm}

% \begin{figure*}[t]
%     \includegraphics[width=0.245\linewidth]{archmyresnet18_trainloss_smooth_sslslrs.pdf}
%     \includegraphics[width=0.245\linewidth]{archmyresnet18_mom09_ssls_gammas_trainloss_smooth.pdf}
%     \includegraphics[width=0.245\linewidth]{archmyresnet18_trainloss_smooth_ssls_sdcs.pdf}
%     \includegraphics[width=0.245\linewidth]{archresnet18_dirg_lrs_sslslrs.pdf} \\
%     \includegraphics[width=0.245\linewidth]{archmyresnet18_testacc_sslslrs.pdf}
%     \includegraphics[width=0.245\linewidth]{archmyresnet18_mom09_ssls_gammas_testacc.pdf}
%     \includegraphics[width=0.245\linewidth]{archmyresnet18_testacc_ssls_sdcs.pdf} 
%     \includegraphics[width=0.245\linewidth]{archresnet18_dirg_lrs_sslsgamma.pdf} \\
%     \includegraphics[width=0.245\linewidth]{archmyresnet18_lrs_sslslrs.pdf} 
%     \includegraphics[width=0.245\linewidth]{archmyresnet18_mom09_ssls_gammas_lrs.pdf} 
%     \includegraphics[width=0.245\linewidth]{archmyresnet18_lrs_ssls_sdcs.pdf}
%     \includegraphics[width=0.245\linewidth]{archresnet18_dirg_lrs_sslssdcs.pdf}\\[-1ex]
%  	\centerline{\small (a) \hspace{4cm} (b) \hspace{4cm} (c) \hspace{4cm} (d)}
%  	\vspace{-4ex}
%     \caption{Sensitivity of SSLS. (a) Initial learning rate $\alpha_0$, (b) The smoothing factor $\gamma$ and (c) and the sufficient decrease constant $c$ on CIFAR-10 with ResNet18. The column (d) shows the learning rate schedules of SSLS when applied to ImageNet with ResNet18 with different hyperparameter settings.}
%     \label{fig:sgdsls}
% \end{figure*}

Finally, we combine SSLS with SASA+ to form Algorithm~\ref{alg:salsa},
which we call SALSA (Stochastic Approximation with Line-search and
Statistical Adaption). Without prior knowledge of the loss function
and training dataset, we start with a very small learning rate and use
SSLS to gradually increase it to be around a stationary value that is
(automatically) customized to the problem, as shown in Figure~\ref{fig:sgdsls}. At every $K_\text{test}$ iterations, SALSA performs the stationary test in SASA+ (Algorithm~\ref{alg:sasa}) and the SLOPE test~\eqref{test:slope}, to determine whether the dynamics become stationary and whether the training loss is still decreasing, respectively. 
%If either the dynamics become stationary or the training loss is {\it not} decreasing, 
If either form of stationarity is detected,
SALSA switches from SSLS to SASA+. After the switch, SASA+ takes over the learning rate scheduling and finishes the training. 
The SLOPE test proves to be very effective in detecting whether the training loss is still decreasing, which prevents the learning rate growing too large 
(but not as effective in reducing the learning rate afterwards, as shown in Section~\ref{sec:sasa+experiments}) 
%However, the SLOPE test is not as a replacement for SASA+. Figure \ref{fig:slopevssasa} shows that SASA+ outperforms the SLOPE test when the latter is used as a trigger for dropping the learning rate. 

%. use of the SLOPE test for this switch is necessary because the stationary learning rate reached by SSLS may be too large, resulting in non-decreasing training loss. The SLOPE test is effective in preventing from those large stationary learning rate. See the experiment $\gamma=2$ in Figure~\eqref{fig:imagenet_ssls_ablation} Row2 for an example. 

\textbf{Computational overhead of SALSA.} When we fix the number of training epochs, the wall-clock time of SSLS is {\it at most} 1.5 times of that of without using line search. This 0.5x overhead is due to the {\it at most 3} extra function evaluations in each line search step (with $m=2)$. Notice that line search is performed on the same minibatch and only needs the function value (not the gradient). In SALSA, SSLS typically switches to SASA+ in less than one-third of the total training epochs. Therefore, the total overhead of using SSLS is typically only 0.15 that of the total time without SSLS. The overhead of SASA+ is the same as that of SASA \citep{LangZhangXiao2019}, which is negligible in practice.

% \twocolumn[
\begin{figure}[t]
\centering
	\includegraphics[width=0.243\textwidth]{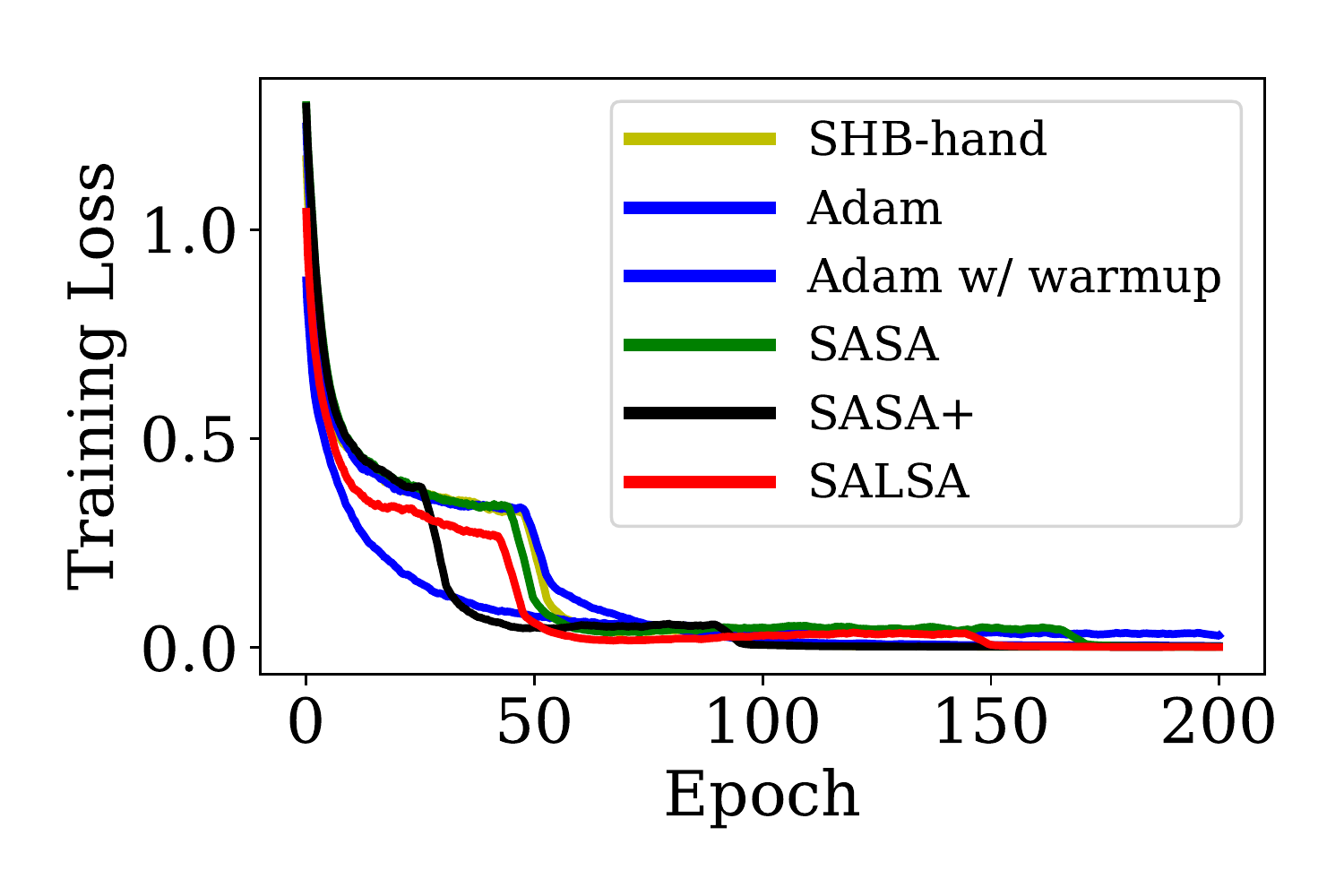}
	\includegraphics[width=0.243\textwidth]{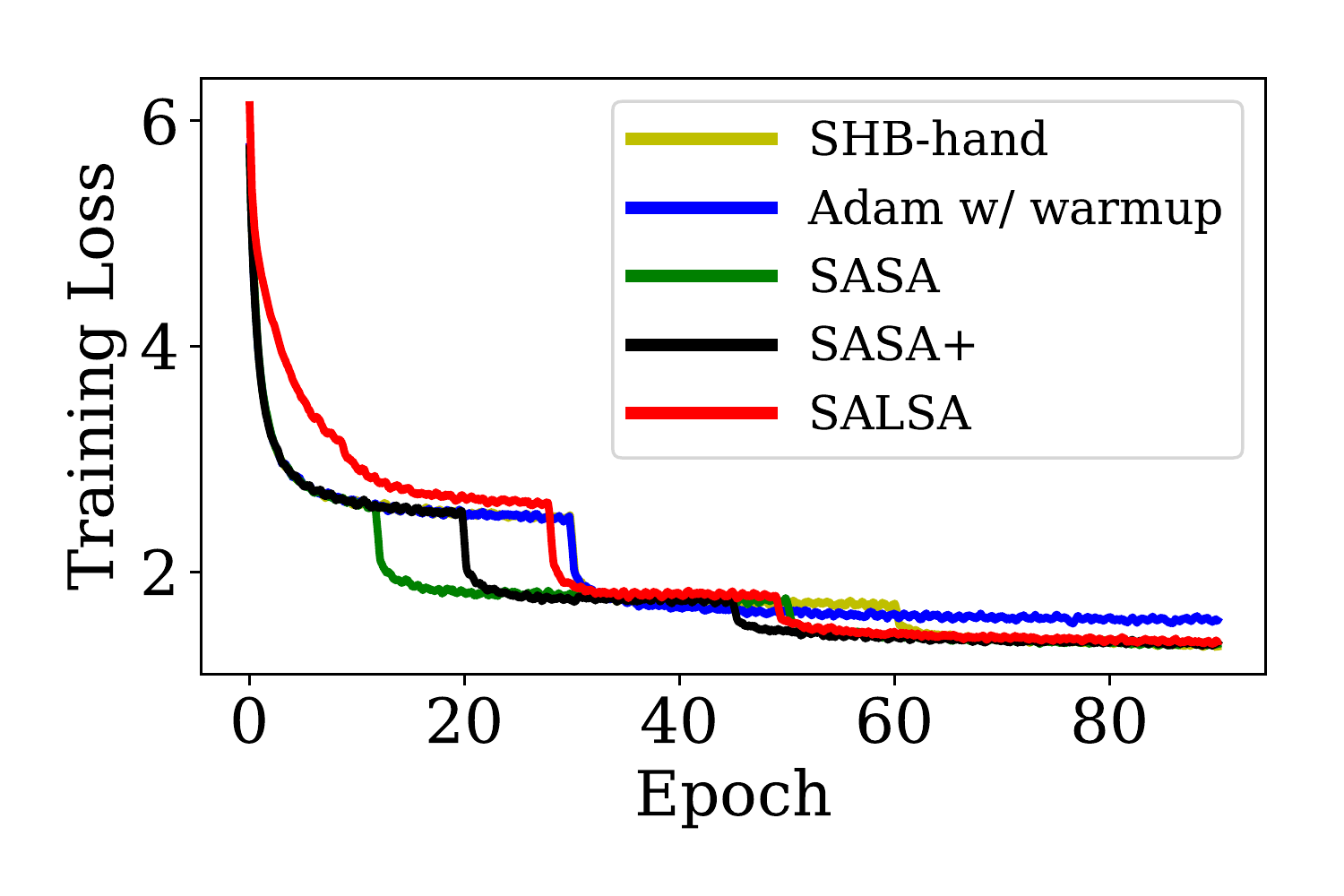}
	\includegraphics[width=0.243\textwidth]{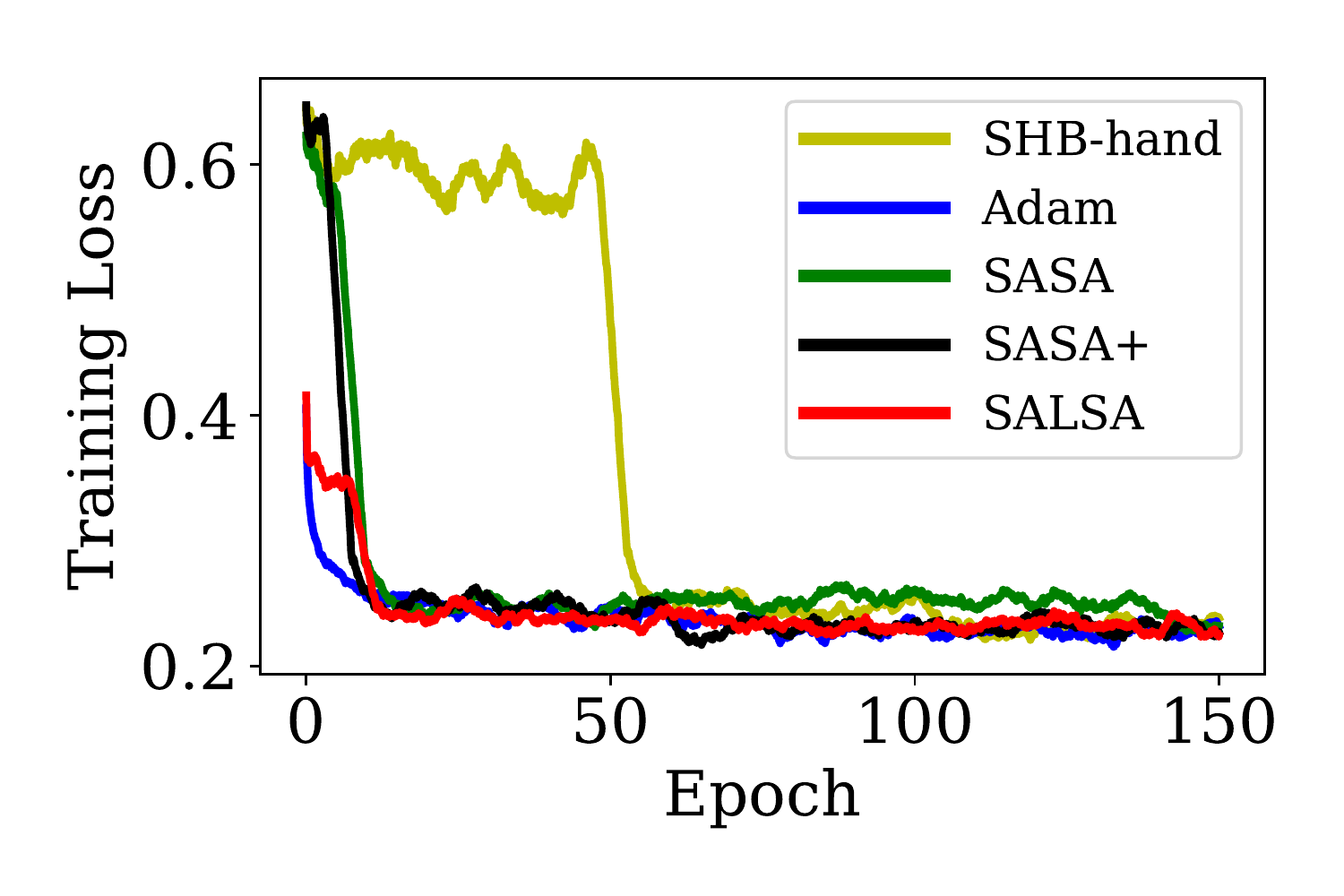}
	\includegraphics[width=0.243\textwidth]{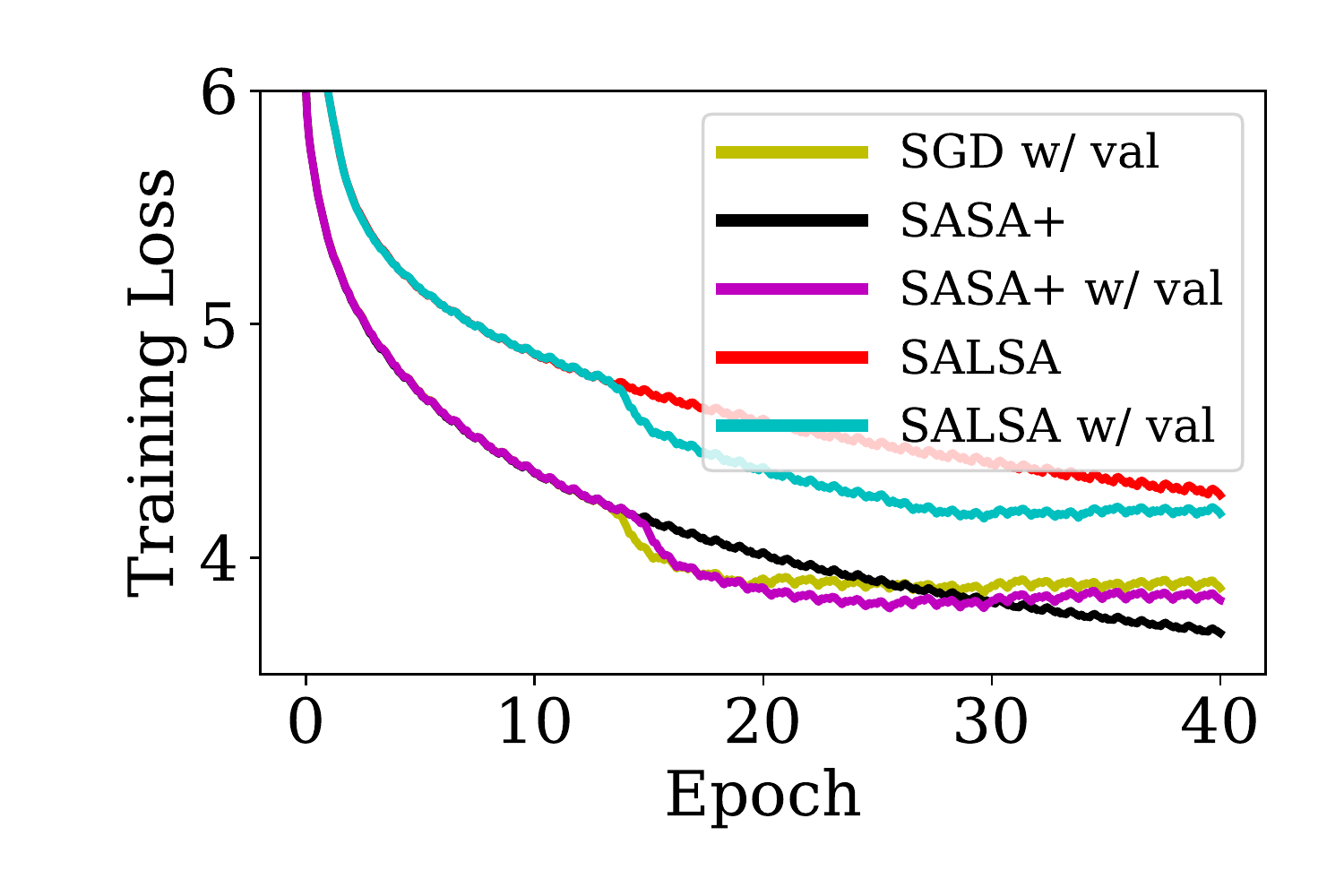}
%	\vspace{-1ex}
	\\
	\includegraphics[width=0.243\textwidth]{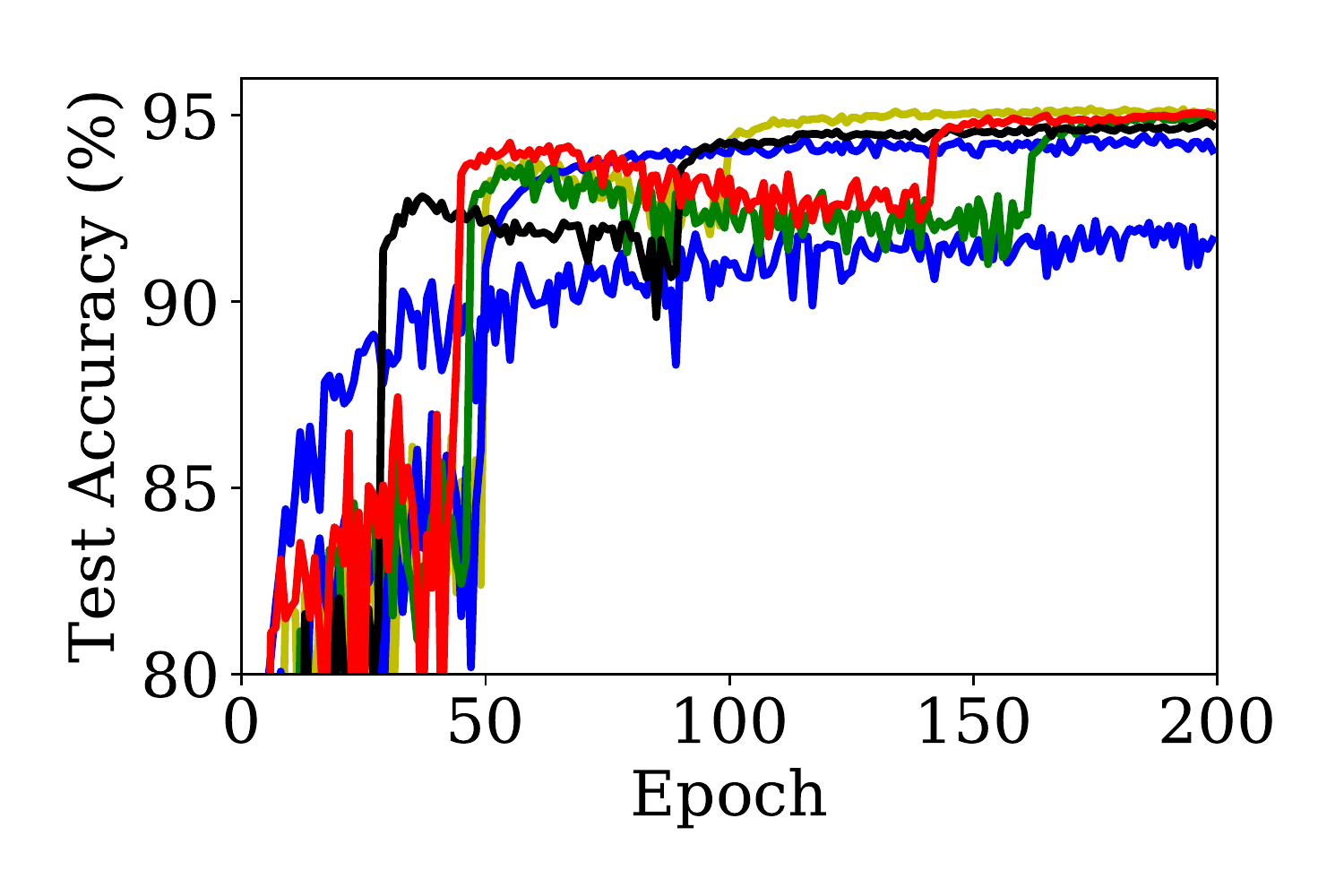}
	\includegraphics[width=0.243\textwidth]{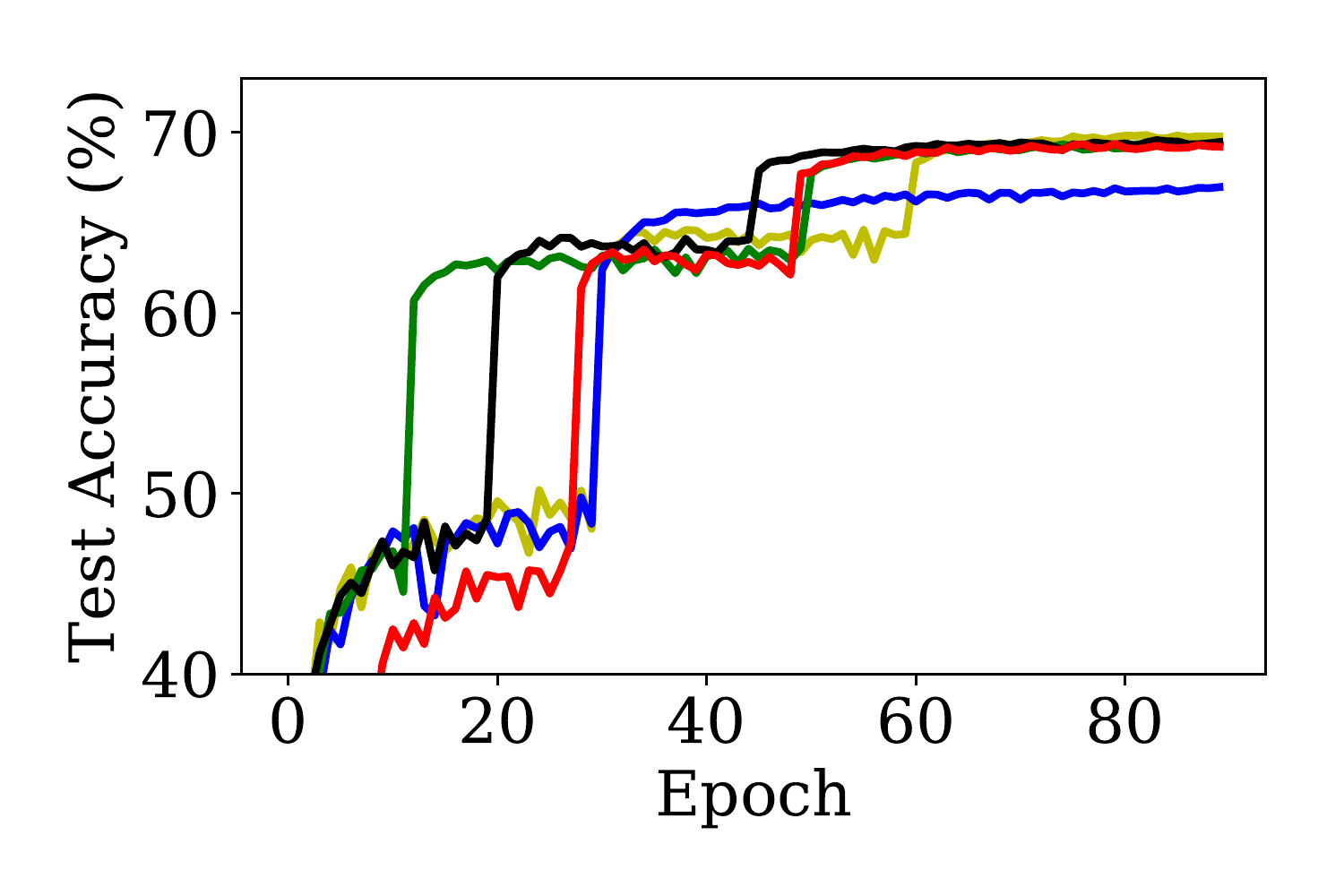}
	\includegraphics[width=0.243\textwidth]{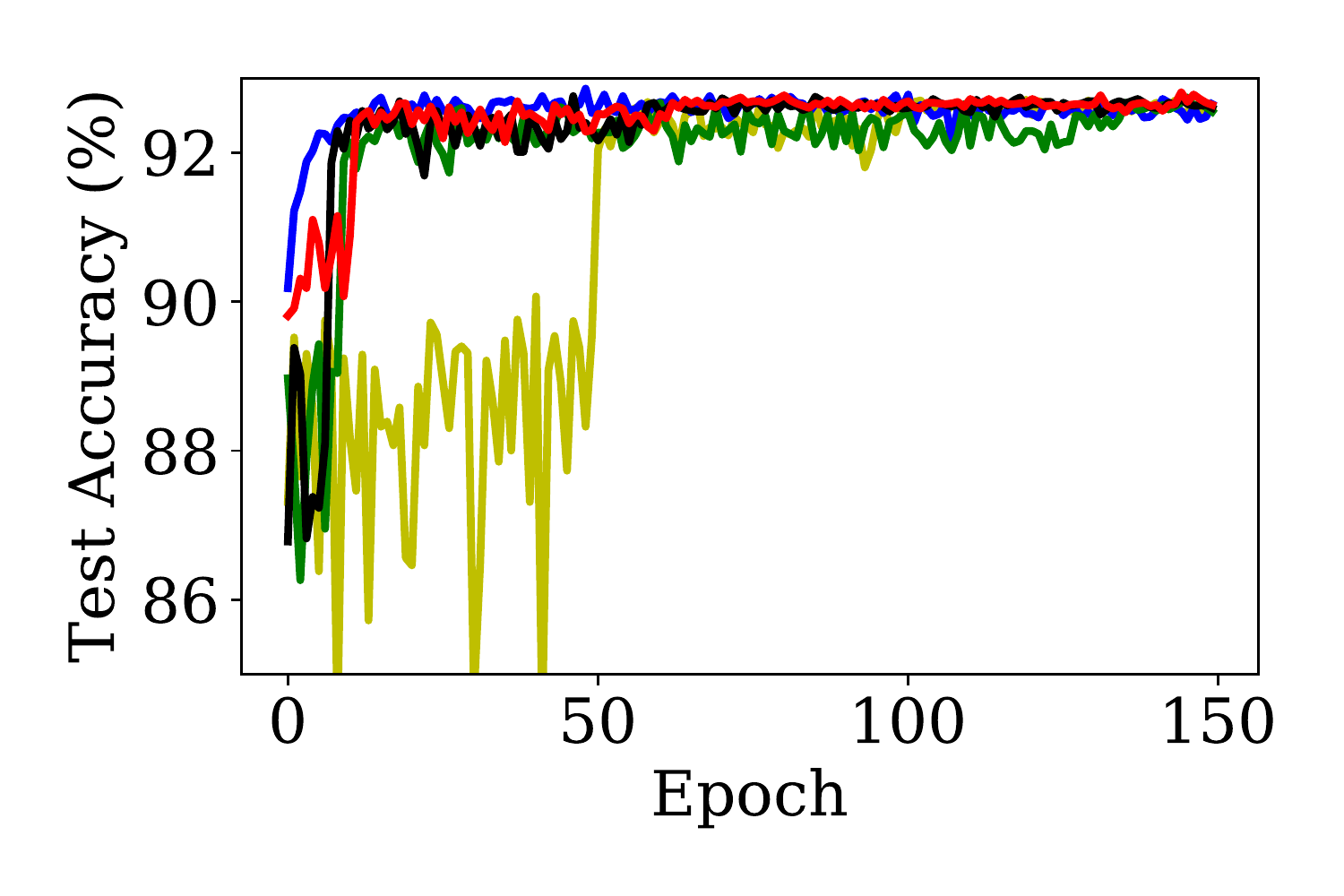}
	\includegraphics[width=0.243\textwidth]{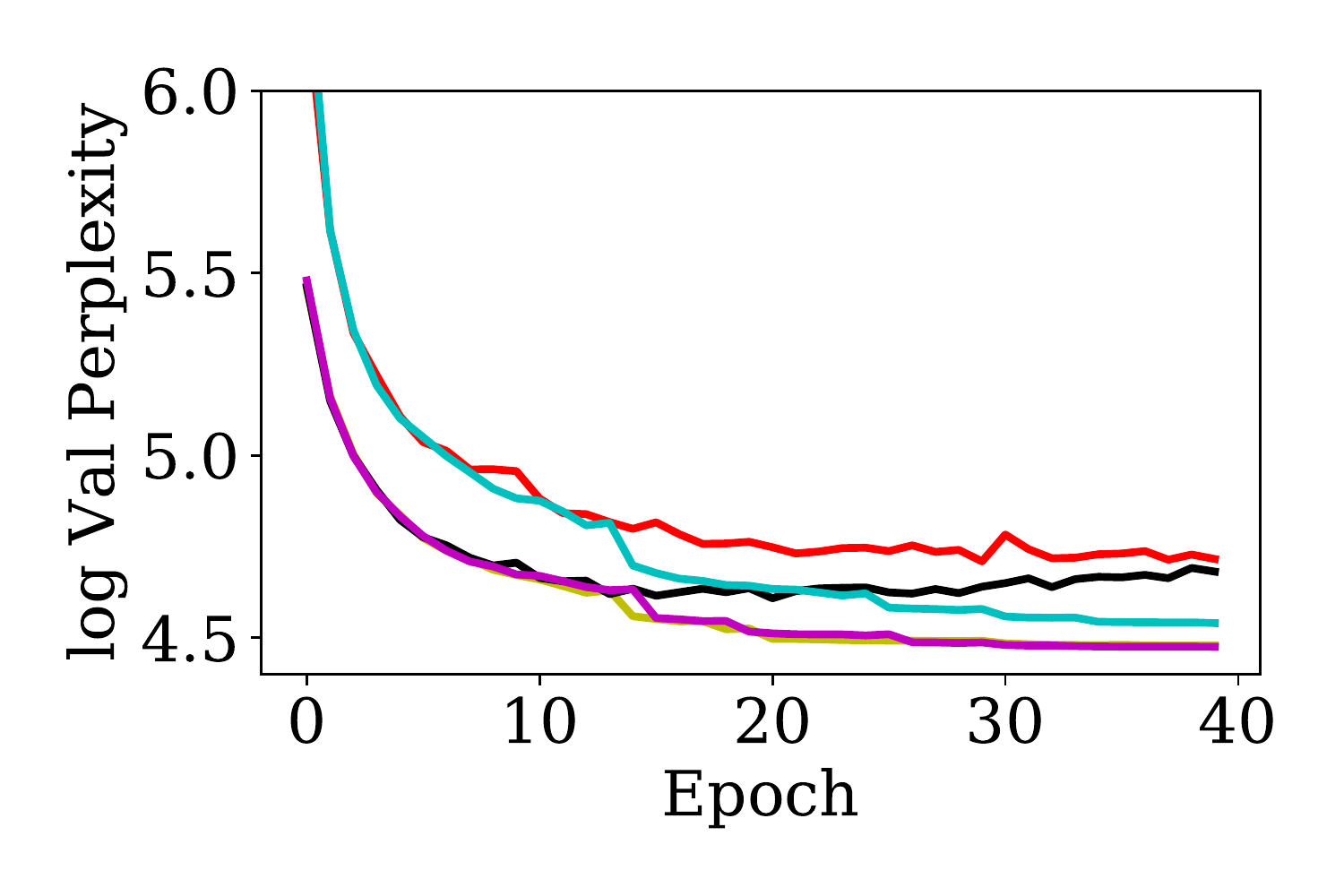}
%	\vspace{-1ex}
	\\
	\includegraphics[width=0.243\textwidth]{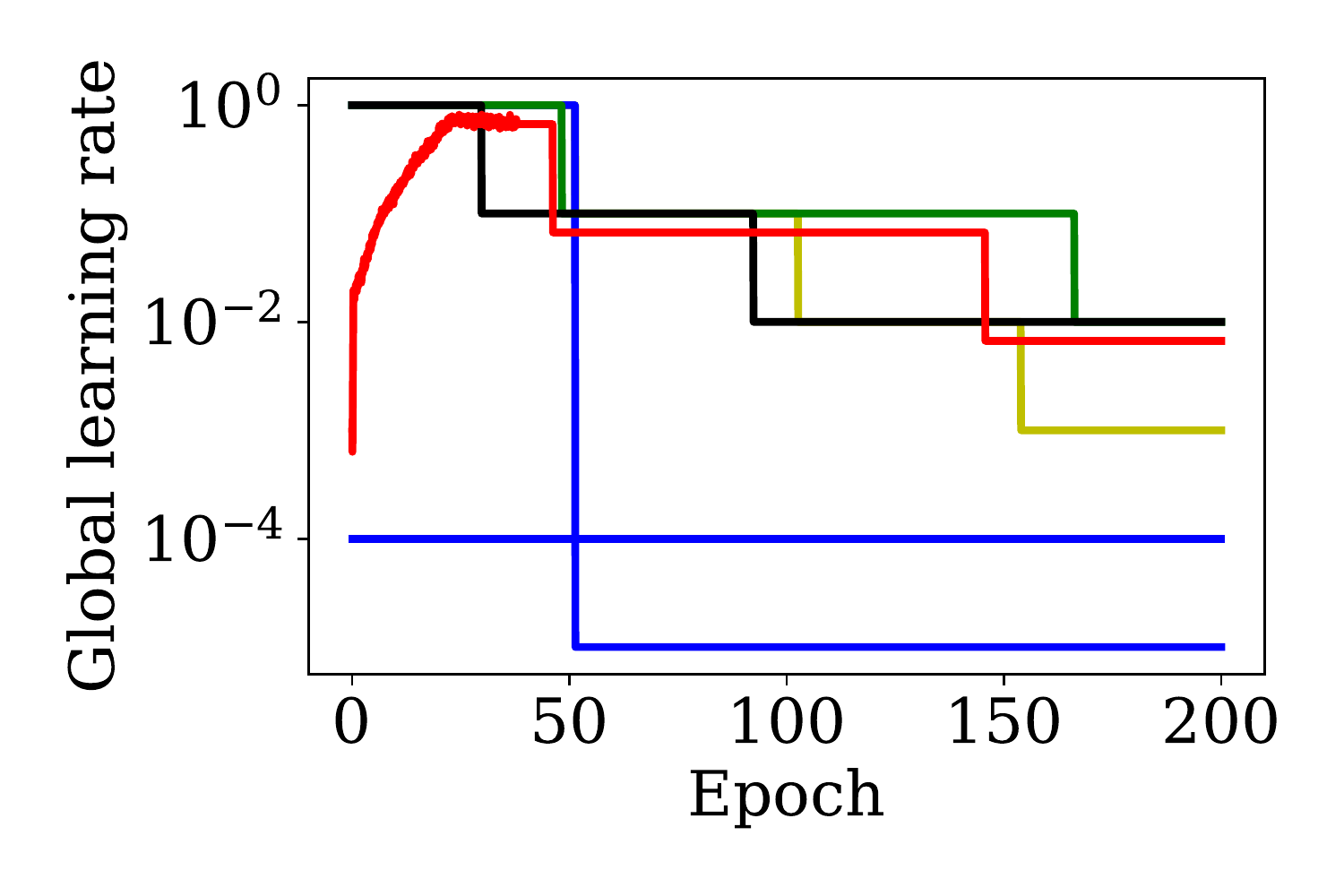}
	\includegraphics[width=0.243\textwidth]{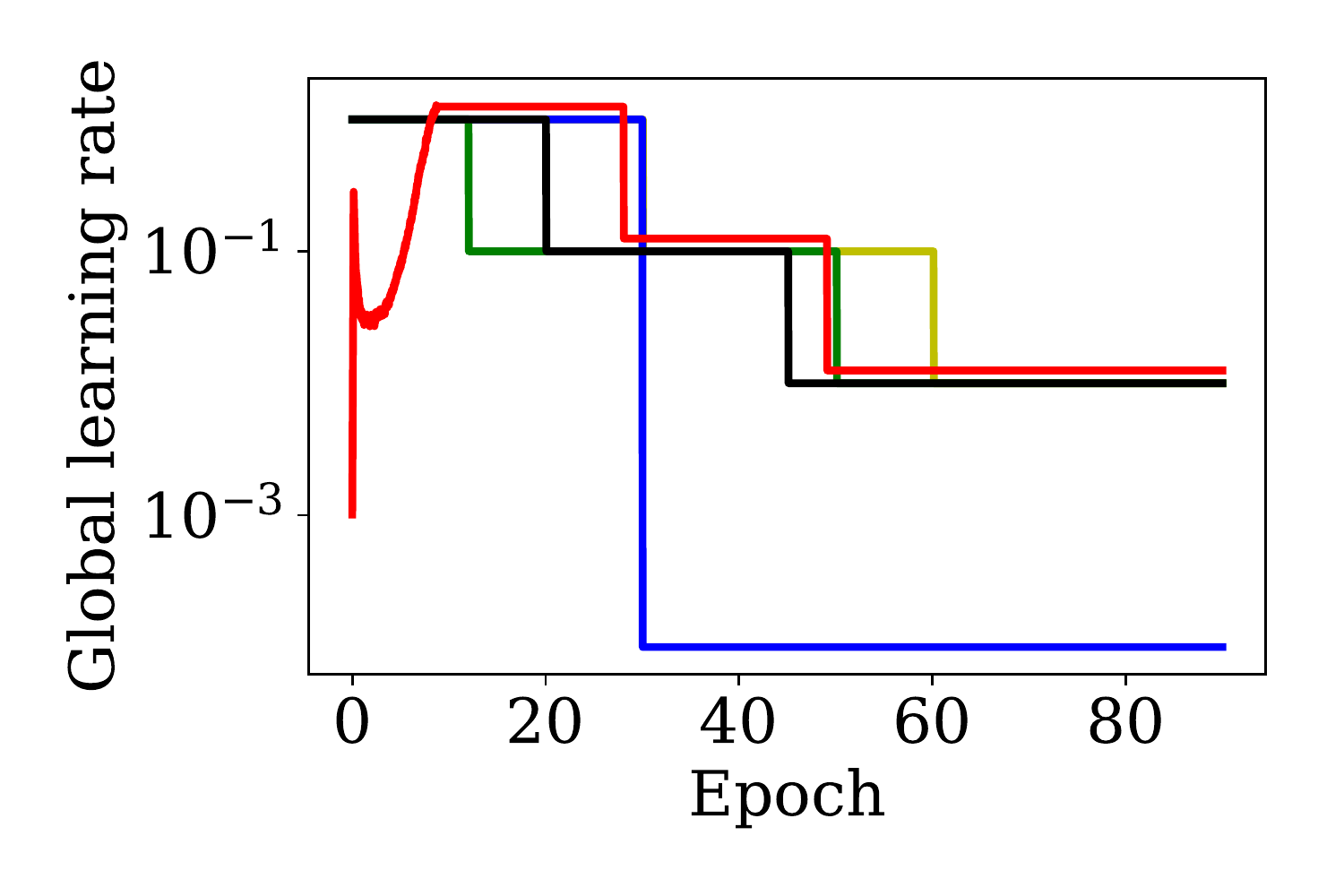}
	\includegraphics[width=0.243\textwidth]{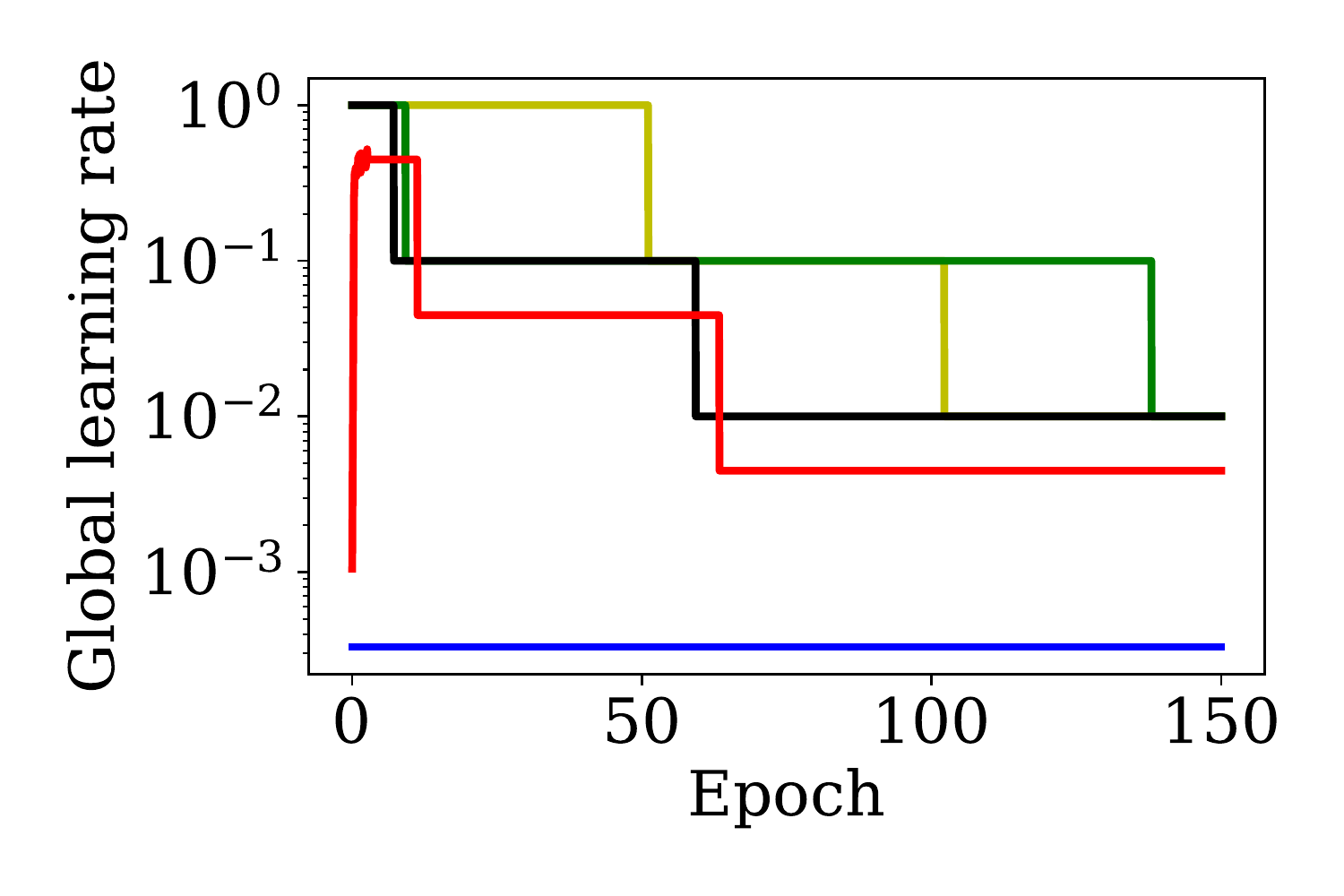}
	\includegraphics[width=0.243\textwidth]{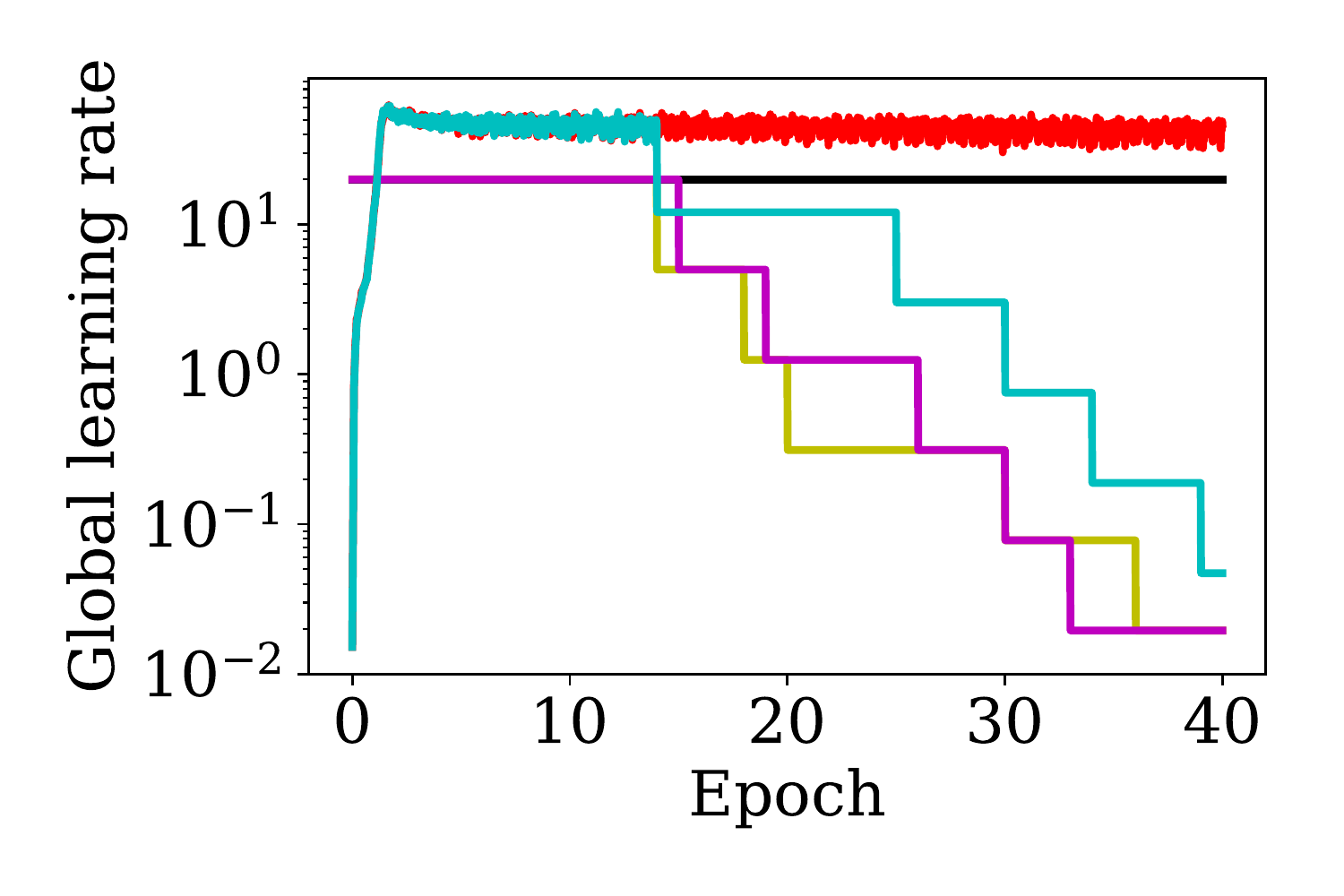}\\[-1.5ex]
% 	\centerline{\small (a) \hspace{4cm} (b) \hspace{4cm} (c) \hspace{4cm} (d)}
% 	\vspace{-1ex}
	\caption{Comparison of different optimizers on CIFAR-10 with ResNet18 (first column), ImageNet with ResNet18 (second column), MNIST with a linear model (third column) and Wikitext-2 with LSTM (fourth column). SALSA starts from an small, arbitrary learning rate while other methods start with a hand-tuned initial learning rate. In Column (d), ``w/ val'' means decreasing the learning rate when the validation loss stops improving (or the statistical test fires, whichever comes first). Best viewed in color.}
	\label{fig:cifar_results}
%	\vspace{-2ex}
\end{figure}
% ]

% \subsection{Experiments}

\textbf{SALSA Experiments.} In Figure~\ref{fig:cifar_results}, we evaluate the performance of SALSA with experiments on four popular datasets. We compare SALSA to the following baselines:
SHB with a \emph{hand-tuned} constant-and-cut learning rate schedule (SHB-hand), Adam~\cite{KingmaBa2014adam} with a tuned warmup phase (Adam w/ warmup) \citep[e.g.,][]{wilson2017marginal}, SASA from \citet{LangZhangXiao2019}, and our SASA+. SALSA starts from a small, arbitrary learning rate while all other methods start with a hand-tuned initial learning rate. 
We showed in Figure~\ref{fig:salsa_cifar10_imagenet} that for both CIFAR-10 and ImageNet, SALSA can start with an arbitrary small initial learning rate and obtain training and testing performance that are on par with that of a hand-tuned optimizer. 
Figure~\ref{fig:salsa_mnist_rnn} in Appendix~\ref{apd:salsaresults} shows that this is also true for the MNIST and Wikitext-2 datasets. 
Unless otherwise stated, we use the \emph{default values} in Table~\ref{tab:sasa-params} for hyperparameters in SASA+ and SALSA.

\textbf{CIFAR-10.}
With random cropping and random horizontal flipping for data augmentation, one can train a modified ResNet18 model (with weight decay $0.0005$) using a hand-tuned optimizer (SHB-hand)
to achieve testing accuracy of 95\% \citep{pytorchcifar}.
%which is higher than the standard ResNet18 result on CIFAR-10 \citep{krizhevsky2009learning}. 
Using a tuned $\alpha_0 = 1e^{-4}$ (by grid search), Adam is only able to reach around 92.5\% for the same model. With an additional ``warmup'' phase of 50 epochs, ``Adam w/ warmup'' can achieve 94.5\% test accuracy.
In contrast, both SASA+ and SALSA are able to reach test accuracy similar to SHB-hand. While other methods starts from a hand-tuned initial learning rate of $1.0$, SALSA starts from a small initial learning rate of~$0.01$.

\textbf{ImageNet.}
On the large scale ImageNet dataset \citep{imagenet_cvpr09}, we use the ResNet18 architecture, random cropping and random horizontal flipping for data augmentation, and weight decay $0.0001$. 
Column (b) in Figure \ref{fig:cifar_results} compares the performance of the different optimizers.
Even with a hand-tuned $\alpha_0 = 1e^{-4}$ and allowed to have a long warmup phase (30 epochs), ``Adam w/ warmup'' fails to match the testing accuracy of the hand-tuned SHB. On the other hand, both SASA+ and SALSA are able to match the best performance.

\textbf{MNIST.}
We train a linear model (logistic regression) on the MNIST dataset. Column (c) in Figure \ref{fig:cifar_results} shows that, for this simple convex optimization problem, all methods finally achieve similar performance.

\textbf{Wikitext-2.}
We train the PyTorch word-level language model example
\citep{pytorchmodel} on the Wikitext-2 dataset
\citep{merity2016pointer}. We use 1500-dimensional embeddings, 1500 hidden units, tied weights, and dropout 0.65, and also gradient clipping with threshold 0.25. We compare against SGD with a learning rate tuned using a validation set. This baseline starts with a hand-tuned initial learning rate of $\alpha_0=20$ and drops the learning rate by a factor of 4 when the validation loss stops improving. Column (d) of Figure~\ref{fig:cifar_results} shows that SASA \citep{LangZhangXiao2019}, SASA+ and SALSA all overfit to the training loss in this case. 
%To fix the reliance of these statistical methods (which cannot detect overfitting) on the training dynamics, we can combine them with the validation heuristic, so we drop the learning rate if the statistical test fires \emph{or} the validation loss stops decreasing, whichever comes first. 
We tried to combine them with the validation heuristic, so we drop the learning rate if the statistical test fires \emph{or} the validation loss stops decreasing (yet another statistical test!), whichever comes first. 
This combination (shown as ``SASA+ w/val'' and ``SALSA w/val'') largely closes the gap in testing performance between SALSA and ``SGD w/val.''

%Column (d) of Figure~\ref{fig:cifar_results} shows that \emph{with using the validation set and without knowing the oracle maximal learning rate}, SALSA is competitive with the baseline in terms of the final perplexity score on the test dataset. Compared with hand-tuned maximal learning rate 20, the large stationary learning rate (around 40) found by SSLS results in a small gap between SGD and SALSA.

%Notice that neither SASA+ and SASA decay the learning rate without the validation set! The top figure in Column (d) of Figure~\ref{fig:cifar_results} shows that their training loss (black and red lines) are rapidly decreasing and the training dynamics are far away from stationary, while the middle figure in Column (d) suggests that the model is already overfitting the data. The original SASA also suffers from this overfitting and never decreases the learning rate, see more results in Appendix~\ref{apd:sasaplusresults}. It is impossible to avoid this overfiting by purely examing the training dynamics, as all SASA, SASA+ and SALSA currently do. However, it is possible to do statistical test on the validation set to schedule the learning rate more robustly, compared with the simple rule of decaying the learning rate when the validation loss stops improving.

%% file: conclusions_1c.tex
\section{Conclusions}
\label{sec:conclusion}

We presented SASA+, a simpler, yet more powerful variant of the 
statistical adaptive stochastic approximation (SASA) method
proposed by \citet{LangZhangXiao2019}.
SASA+ uses a single condition for (non-)stationarity that works 
without modification for a broad family of stochastic optimization methods.
This greatly simplifies its implementation and deployment in software packages. 
%The statistical test used for this general condition is simpler and more rigorous than the one used in SASA.
While SASA+ focuses on how to automatically reduce the learning rate 
to obtain better asymptotic convergence, 
we also propose a smoothed stochastic line-search (SSLS) method to
warm up the optimization process, thus removing the burden of 
expensive trial and error for setting a good initial learning rate.
The combined algorithm, SALSA, is highly autonomous and robust to 
different models and datasets.
%changes in its hyperparameters. 
Using the same default settings, SALSA obtained
state-of-the-art performance on several common deep learning models 
that is competitive with the best hand-tuned optimizers. 

In general, we believe that statistical tests are powerful tools that should be exploited further in stochastic optimization, especially for making the training of large-scale machine learning models more autonomous and reliable.

%XXX Should say this in the conclusion: Heuristics sometimes are unavoidable in pushing automatic training of machine learning models, but many can be distilled and captured by powerful statistical tools. 

\iffalse
% This should appear at the end of a talk! Not in the paper!
\begin{quote}
    The speed of convergence questions are closely related to on-line rules
    for determining step coefficients in SA algorithms. In the authors'
    opinion, the use of statistical tests, similar to those described in
    Section~IV, is a promising direction of further research. 
    Of interest is especially the choice of gains near the equilibrium.
    The known results on asymptotic behavior of SA algorithms with constant
    gains (see \dots) could be very useful for constructing appropriate rules.
\end{quote}
\fi

%% file: morestatistics_1c.tex
\section{More details on statistical tests}
\subsection{From the master equation in SASA+ to that in \cite{yaida2018fluctuation} and \cite{LangZhangXiao2019}}
\begin{prop}
When the dynamics of $(\xk, \gk, \dk)$ is given by the stochastic heavy ball method, i.e., \eqref{eqn:sgm-general} and \eqref{eqn:d-shb}, the master equation in SASA+, i.e., 
\begin{equation}\label{eqn:master-condition-app}
\E_{(\xk, \gk, \dk)\sim \pi} \left[\bigl\langle x^t,d^t\bigr\rangle
-\textstyle\frac{\alpha}{2}\|d^t\|^2\right] = 0 \quad \forall t \ge k
\end{equation}
and the stationarity of $\|\dk\|^2$ together lead to the master equation in \cite{yaida2018fluctuation}, i.e.,
\begin{equation}\label{eqn:yaida-hb-test-app}
\E_{(\xk, \gk, \dk)\sim \pi} \left[\bigl\langle x^t,g^t\bigr\rangle \textstyle
-\frac{\alpha}{2}\frac{1+\beta}{1-\beta}\|d^t\|^2\right] = 0 \quad \forall t\ge k+1.
\end{equation}
Here, $(\xk, \gk, \dk)\sim \pi$ means that the dynamics of $(\xk, \gk, \dk)$ reaches its stationary distribution~$\pi$ at time~$k$.
\end{prop}
\begin{proof}
To avoid cumbersome subscripts, we simply write $\E_{(\xk, \gk, \dk)\sim \pi} $ as $\E$ in the proof.
For any $t \ge k+1$, from SASA+'s master equation, we have
\begin{equation*}
\begin{aligned}
\E\left[\bigl\langle x^t,d^t\bigr\rangle
-\textstyle\frac{\alpha}{2}\|d^t\|^2\right] &= 0, \\
\E\left[\bigl\langle x^t,d^{t-1}\bigr\rangle
+\textstyle\frac{\alpha}{2}\|d^{t-1}\|^2\right] &= 0.
\end{aligned}
\end{equation*}
Then we have
\begin{equation*}
    \begin{aligned}
    &\frac{\alpha}{2} \E[\|d^t\|^2] = \E\left[\bigl\langle x^t,(1-\beta)g^t + \beta d^{t-1}\bigr\rangle\right] \\
    &= (1-\beta) \E\left[\bigl\langle x^t,g^t\bigr\rangle\right] + \beta \E\left[\bigl\langle x^t,d^{t-1}\bigr\rangle\right] \\
    &= (1-\beta) \E\left[\bigl\langle x^t,g^t\bigr\rangle\right] - \beta \frac{\alpha}{2}\E\left[\|d^{t-1}\|^2\right]
    \end{aligned}
\end{equation*}
Thanks to the stationarity of $\|d^t\|^2$ (i.e., $\E[\|d^t\|^2] = \E[\|d^{t+1}\|^2]$ for $t \ge k$), we obtain Equation~\eqref{eqn:yaida-hb-test-app}, i.e., the master equation in \cite{yaida2018fluctuation}.
\end{proof}

In fact, when the dynamics of $(\xk, \gk, \dk)$ is given by the stochastic heavy ball method, i.e., \eqref{eqn:sgm-general} and \eqref{eqn:d-shb}, one can verify the following identity for any $k \ge 1$:
\begin{equation}\label{eqn:shb-identity}
\begin{aligned}
    &\bigl\langle x^k,g^k\bigr\rangle \textstyle
-\frac{\alpha}{2}\frac{1+\beta}{1-\beta}\|d^k\|^2 \\
= &\frac{1}{2\alpha (1-\beta)} \big((\beta\|x^k\|^2 - \|x^{k+1}\|^2 - \alpha^2 \beta \|d^k\|^2) \\
&- (\beta\|x^{k-1}\|^2 - \|x^{k}\|^2 - \alpha^2 \beta \|d^{k-1}\|^2)\big).
\end{aligned}
\end{equation}
Therefore, the master equation in \cite{yaida2018fluctuation} is taking a test function 
\begin{equation}\label{eqn:shb-yaida-testphi}
    \phi(\{\xk, \gk, \dk\}) = \beta\|x^k\|^2 - \|x^{k+1}\|^2 - \alpha^2 \beta \|d^k\|^2.
\end{equation}
The tests in \cite{yaida2018fluctuation} and \cite{LangZhangXiao2019} are essentially testing whether $\phi(\{\xk, \gk, \dk\})$ in \eqref{eqn:shb-yaida-testphi} reaches stationarity or not.

\subsection{MCMC variance estimators}
\label{apd:batch-means}
Several estimators for the asymptotic variance of the history-average
estimator of a Markov chain have appeared in work on Markov Chain
Monte Carlo (MCMC). \citet{jones2006fixed} gives a nice example of
such results. Here we simply list two common variance estimators, and
we refer the reader to that work and \citet{LangZhangXiao2019} (from
which we borrow notation) for appropriate context and formality.

\paragraph{Batch Means (BM) variance estimator.}
Let $\bar{z}_N$ be the history average estimator with $N$ samples,
that is, given samples $\{z_i\}$ from a Markov chain, $\bar{z}_N =
\frac{1}{N}\sum_{i=i_0}^{i_0 + N} z_i$. Now form $p$ batches from the
$N$ samples, each of size $q$. Compute the ``batch means'' $\bar{z}^j
= \frac{1}{q}\sum_{i=jq}^{(j+1)q - 1}z_i$ for each batch $j$. Then
compute the batch means estimator using:
\begin{equation}
\label{eqn:bm}
\hat{\sigma}_N^2 = \frac{q}{p-1}\sum_{j=0}^{p-1}(\bar{z}^j - \bar{z}_N)^2.
\end{equation}
The estimator is just the variance of the batch means around the
history average $\bar{z}_N$. This statistic has $p-1$ degrees of
freedom.  As in \citet{jones2006fixed} and \citet{LangZhangXiao2019},
we take $p = q = \sqrt{N}$ when using this estimator.

\paragraph{Overlapping batch means variance estimator.}
The \emph{overlapping batch means} (OLBM) estimator
\citet{flegal2010batch} has better asymptotic variance than the batch
means estimator. The OLBM estimator is conceptually the same, but it
uses $N - p + 1$ overlapping batches of size $p$ (rather than disjoint
batches) and has $N - p$ degrees of freedom. It can be computed as:
\begin{equation}
\label{eqn:olbm}
\hat{\sigma}_N^2 = \frac{Np}{(N-p)(N-p+1)}\sum_{j=0}^{N-p}(\bar{z}_N - \frac{1}{p}\sum_{i=1}^{p}z_{j+i})^2.
\end{equation}
In practice, there is not much difference between using the BM or the OLBM estimators in SASA+.

\subsection{Slope test for linear regression}
\label{apd:slopetest}
% \linx{Maybe put the slope test details below to appendix, together with the formulas for BM adn OLBM.}

For the SLOPE test presented in Section~\ref{sec:stats-tests}, the formulas for 
the statistic are given by

%\begin{equation}\label{test:slope_stat}
\begin{equation}
\begin{aligned}
t_{\text{slope}} &= \hat{c}_1/\sqrt{\hat{\sigma}_f}, \\
%t_{\text{slope}} = \frac{\hat{c}_1}{\sqrt{\hat{\sigma}_f^2}}, \quad
\hat{\sigma}_f^2 &= \frac{\textstyle\sum_{i=k-N+1}^k (f_{\xi_i}(x_i) - \hat{f}_{\xi_i}(x_i))^2}{\textstyle \sum_{i=k-N+1}^k (i - \overline{i})^2}
%	t_{\text{slope}} = \frac{\hat{c}_1}{\sqrt{\sum\limits_{i=k-N+1}^k (f_{\xi_i}(x_i) - \hat{f}_{\xi_i}(x_i))^2 / \sum\limits_{i=k-N+1}^k (i - \overline{i})^2}},
\end{aligned}
\end{equation}
%\end{equation}
where $\overline{i} = \frac{1}{N}\sum_{i=k-N+1}^k i$ is the center of the regressor~$i$. 
See, for example, \citet[][Section~2.3]{montgomery2012introduction} for further details.
%Note that we can also use log-scale for the iteration, i.e., $\hat{f}_{\xi_i}(x_i) = \hat{c}_1 \log k + \hat{c}_0$. We empirically found that this linear/log scale choice does not make a difference.

%% file: sasavssasaplus_1c.tex
\section{The statistical tests in SASA and SASA+}
\label{apd:sasavssasaplus}
As mentioned in our contributions, there are two main differences between SASA from \cite{LangZhangXiao2019} and SASA+. First, SASA+ has a much more general ``master stationary condition'' \eqref{eqn:master-condition}, so it can be applied to any stochastic optimization method with constant hyperparameters (e.g., SGD, stochastic heavy-ball, NAG, QHM, etc.), while the stationary condition in Yaida (2018) and SASA proposed in \citet{LangZhangXiao2019} (see Equation~\eqref{eqn:yaida-hb-test}) only applies to the stochastic heavy-ball method. Second, the statistical tests used in SASA and SASA+ are quite different, as we elaborate below. 

The difference starts from a conceptual change from SASA to SASA+. In
SASA, one wants to \textbf{confidently} detect stationarity. If
stationarity is detected, one decreases the learning rate, and otherwise
keeps it the same. In SASA+, one wants to \textbf{confidently} detect
non-stationarity. If non-stationarity is detected, one keeps the learning
rate the same, and otherwise decreases. This conceptual change leads to
a simpler and more rigorous statistical test in SASA+.

\subsection{Equivalence test in SASA}
\label{apd:sasatest}
To confidently detect stationarity, SASA has to set non-stationarity as null hypothesis and stationarity as the alternative hypothesis. If one confidently rejects the null (non-stationarity), then one can be confident that the process is stationarity. Instead of detecting stationarity, SASA simplifies to only detect the \textbf{necessary but not sufficient} stationarity condition \eqref{eqn:yaida-hb-test}. Formally, the test in SASA is
\begin{equation}\label{test:sasa0}
	H_0: \E[\Delta] \neq 0 \qquad \mbox{ vs. } \qquad H_1: \E[\Delta] = 0,
\end{equation}
where samples of $\Delta$, i.e., $\Delta_k \triangleq \bigl\langle \xk,\dk\bigr\rangle-\frac{\alpha}{2}\|\dk\|^2$, are collected along the training process. This kind of test is called an equivalence test in statistics, see, e.g., \cite{streiner2003unicorns}. There is no power\footnote{In statistical hypothesis testing, power is the ability to reject the null hypothesis when it is false.} in the equivalence test \eqref{test:sasa0}, i.e., one cannot confidently reject the null hypothesis and prove stationarity at all! Intuitively, even when the process is stationary, with only a finite number of (noisy) samples $\{\Delta_k\}_{k=0}^{N_1}$, the sample mean $\bar{\Delta}_N \neq 0$ (with probability one) is always more likely to be the true mean than the \textbf{singleton} $0$. In other words, one can not deny that the process is probably infinitely close to stationary but still non-stationary.

To gain power in the equivalence test, one needs to use domain knowledge to define an \emph{equivalence interval}. Formally, the true test in SASA is
\begin{equation}\label{test:sasa}
H_0: |\E[\Delta]| > \zeta\nu \qquad \mbox{ vs. }\qquad H_1: \E[\Delta] \in [-\zeta \nu, \zeta \nu],
\end{equation}
where $\zeta \nu$ is the equivalence interval. In English, instead of the usual null hypothesis of not-equal-to-zero in \eqref{test:sasa0}, now the null hypothesis is not-equal-to-zero by a margin $\zeta \nu$. In this case, when $\E[\Delta]$'s confidence interval is contained in the equivalence interval, i.e.,
\begin{equation}\label{eqn:sasatest}
	\biggl[\bar{\Delta}_N - t_{1-\delta/2}^*\frac{\hat{\sigma}_N}{\sqrt{N}}, ~\bar{\Delta}_N + t_{1-\delta/2}^*\frac{\hat{\sigma}_N}{\sqrt{N}}\biggr] \subset [-\zeta \bar{\nu}_N, \zeta \bar{\nu}_N].
\end{equation}
we are confident to reject the null hypothesis and to prove/accept $H_1$ (the stationary condition is met within an error tolerance). This is exactly the test in SASA~\cite{LangZhangXiao2019}, see its Equation (10). However, this equivalence test requires estimation of $\bar{\nu}_N$ (that estimates the magnitude of $\Delta$) and an additional hyperparameter $\zeta$ that controls the equivalence interval width. In SASA's notation, the equivalence width $\zeta$ is denoted by $\delta$.

Moreover, SASA makes the unjustified assumption that under their null
hypothesis ($H_0: |\E[\Delta]| > \zeta\nu$), the Markov central limit
theorem holds. Under their $H_0$, the process is non-stationary, while
the construction of $\E[\Delta]$'s confidence interval in
Eqn.~\eqref{eqn:sasatest} relies on the (asymptotic) stationarity of
the Markov process. Therefore, despite the empirical success of SASA,
its statistical test is intrinsically flawed because of this testing setup.

\subsection{Standard test in SASA+}
\label{apd:sasaplustest}
In SASA+, we want to \textbf{confidently} detect non-stationarity. If
non-stationarity is detected, one keeps the learning rate the same, and
otherwise one decreases it. This conceptual change naturally removes the
complication of the equivalence test and the intrinsic flaw in SASA.

To confidently detect non-stationarity, SASA+ sets stationarity as the
null hypothesis and non-stationarity as the alternative:
\begin{equation}\label{test:sasaplus}
\begin{aligned}
&H_0: \mbox{The process is staionary} \quad \mbox{ vs. } \quad H_1: \mbox{The process is not staionary}.
\end{aligned}
\end{equation}
The master stationary condition $\E[\Delta] = 0$ is a necessary
condition for stationary of the process, and thus confidently
rejecting $\E[\Delta] = 0$ is sufficient to confidently reject the
null (stationarity) hypothesis. Moreover, under this null hypothesis,
the Central Limit Theorem of Markov Processes exactly holds true
(because the process is stationary), validating the construction of
$\E[\Delta]$'s confidence interval. Therefore, the intrinsic flaw in
SASA is naturally solved in SASA+.

Now we describe the test in SASA+. When the confidence interval of $\E[\Delta]$ does not contain $0$, i.e.,
\begin{equation}\label{eqn:sasaplustest}
0 \not\in \biggl[\bar{\Delta}_N - t_{1-\delta/2}^*\frac{\hat{\sigma}_N}{\sqrt{N}}, ~\bar{\Delta}_N + t_{1-\delta/2}^*\frac{\hat{\sigma}_N}{\sqrt{N}}\biggr],
\end{equation}
then we reject the null (stationary) hypothesis and keep the learning
rate the same. Otherwise, we decrease the learning rate. Notice that
there is no additional hyperparameter $\zeta$.

\begin{figure}[t]
	\begin{center}
	\includegraphics[width=0.7\linewidth]{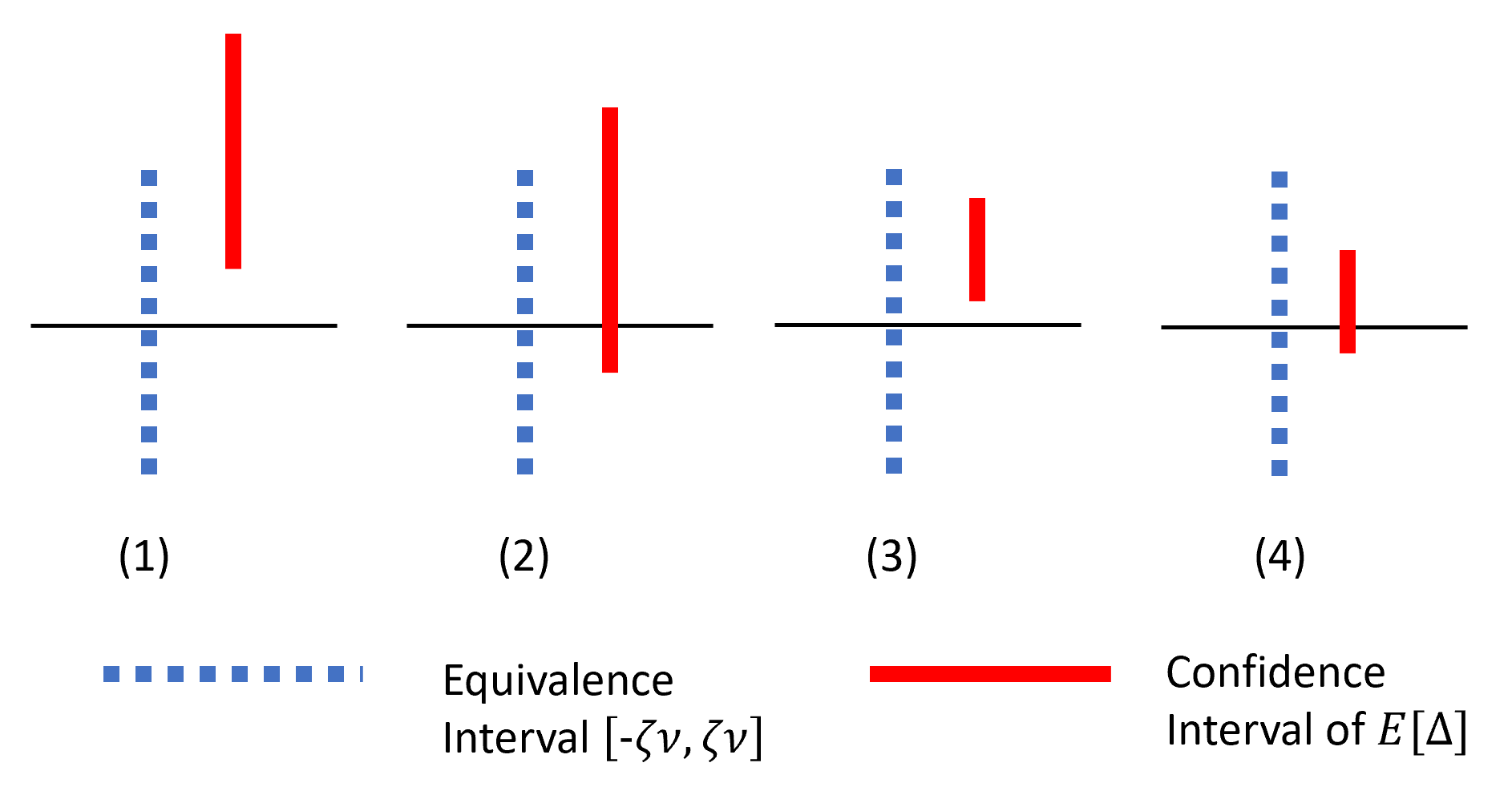}
	\end{center}
	\caption{Statistical tests in SASA and SASA+. Case (1): both SASA and SASA+ keep the learning rate. Case (4): both SASA and SASA+ decrease the learning rate. Case (2): SASA keeps the learning rate while SASA+ decreases. Case (3): SASA decreases the learning rate while SASA+ keeps.}
	\label{fig:sasavssasaplus}
\end{figure}

\subsection{The difference in practice}
\label{apd:sasavssasapluspractice}

Although SASA and SASA+ are conceptually quite different, they both use
the same confidence interval (see \eqref{eqn:sasatest} and
\eqref{eqn:sasaplustest}) for $\E[\Delta]$. In practice, their
difference is illustrated in Figure~\ref{fig:sasavssasaplus}. Their
difference lies in Case (2) and (3). In Case (3), the process has not
reached stationarity yet with high probability, according to the
confidence interval of $\E[\Delta]$ (Red). SASA+'s test is confident
to reject its (stationary) null hypothesis, so it keeps the learning
rate the same. However, SASA decreases its learning rate because it is
confident that the stationary condition holds true within its error
tolerance (equivalence interval). In this case, SASA makes an error
due to its relatively large equivalence interval. In Case (2), SASA+'s
test is not confident to reject its (stationary) null hypothesis and
so it decreases its learning rate. On the contrary, SASA's test is not
confident to claim that the stationary condition holds true within its
error tolerance and thus keeps its learning rate. In this sense, SASA+
is more aggressive than SASA in decreasing learning rate.

In numerical experiments on the CIFAR10 and ImageNet datasets, the (width
of) the equivalence interval $\zeta \nu$ is typically much smaller than
the (width of) the confidence interval of $\E[\Delta]$, see Figure 5(a),
7(a) and 9(a) in \cite{LangZhangXiao2019}. Notice that in those
figures, the yellow curve is $\nu$ instead of $\zeta \nu$, and thus
the equivalence interval is even smaller. Therefore, Case (3) happens
very rarely in those experiments, and this explains the reason why
SASA does not make obvious mistakes in decreasing the learning
rate. In practice, Case (2) sometimes happens, and thus we can see
that SASA+ seems to be slightly more aggressive at decreasing the
learning rate than SASA. For example in Figure
\ref{fig:cifar_results}, the black curve (SASA+) is faster at
decreasing the learning rate than the green curve (SASA).

%% file: more_experiments_1c.tex
\section{More experimental results}
\subsection{More results for SALSA}
\label{apd:salsaresults}
In Figure~\ref{fig:salsa_cifar10_imagenet}, we showed that for both CIFAR-10 and ImageNet, SALSA can start with a small initial learning rate and the training dynamics and final performance are the same as when it starts with a hand-tuned initial learning rate. Here, in Figure~\ref{fig:salsa_mnist_rnn}, we show that the same holds true for SALSA on MNIST and Wikitext-2 datasets.
\begin{figure*}[t]
    \includegraphics[width=0.33\linewidth]{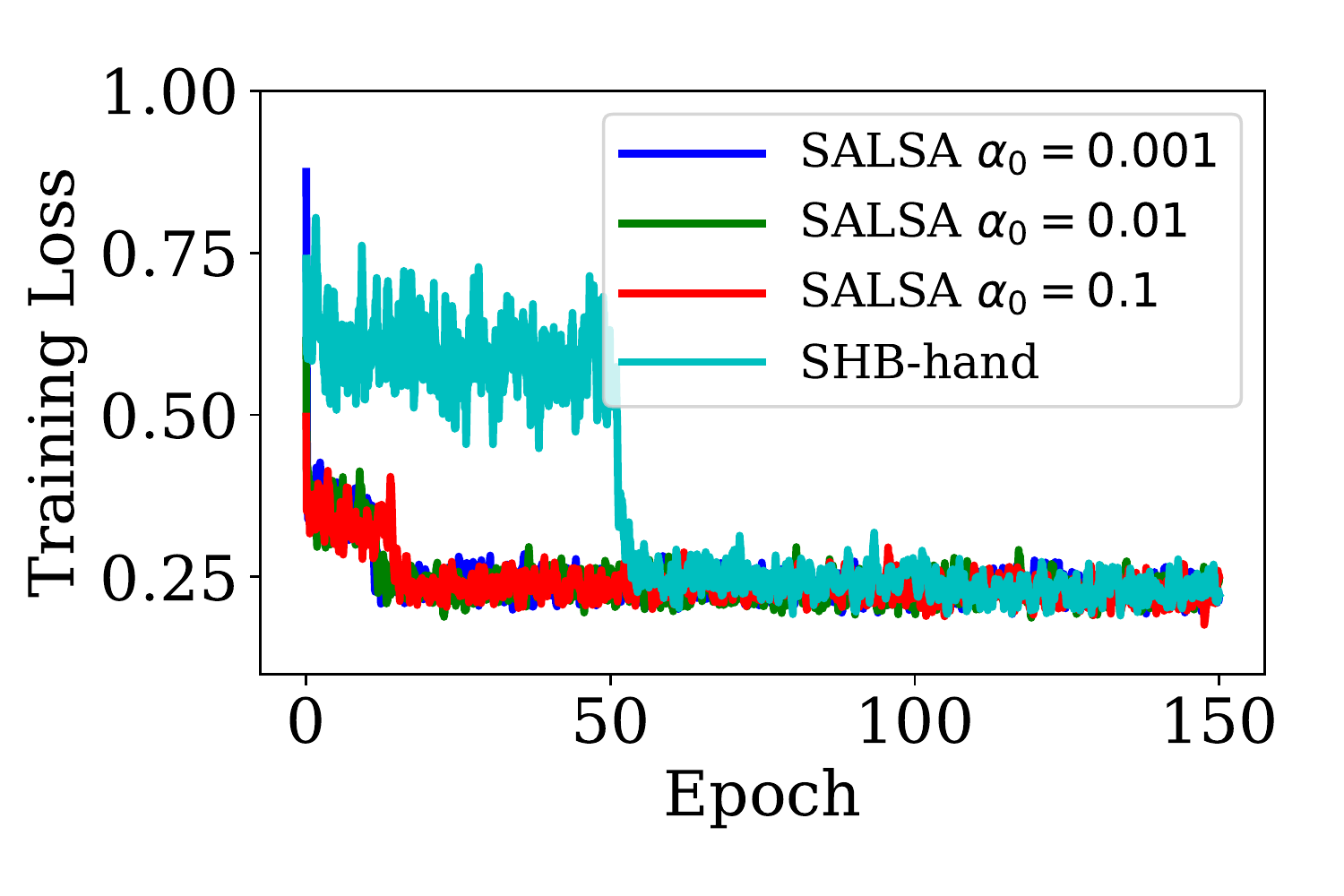}
    \includegraphics[width=0.33\linewidth]{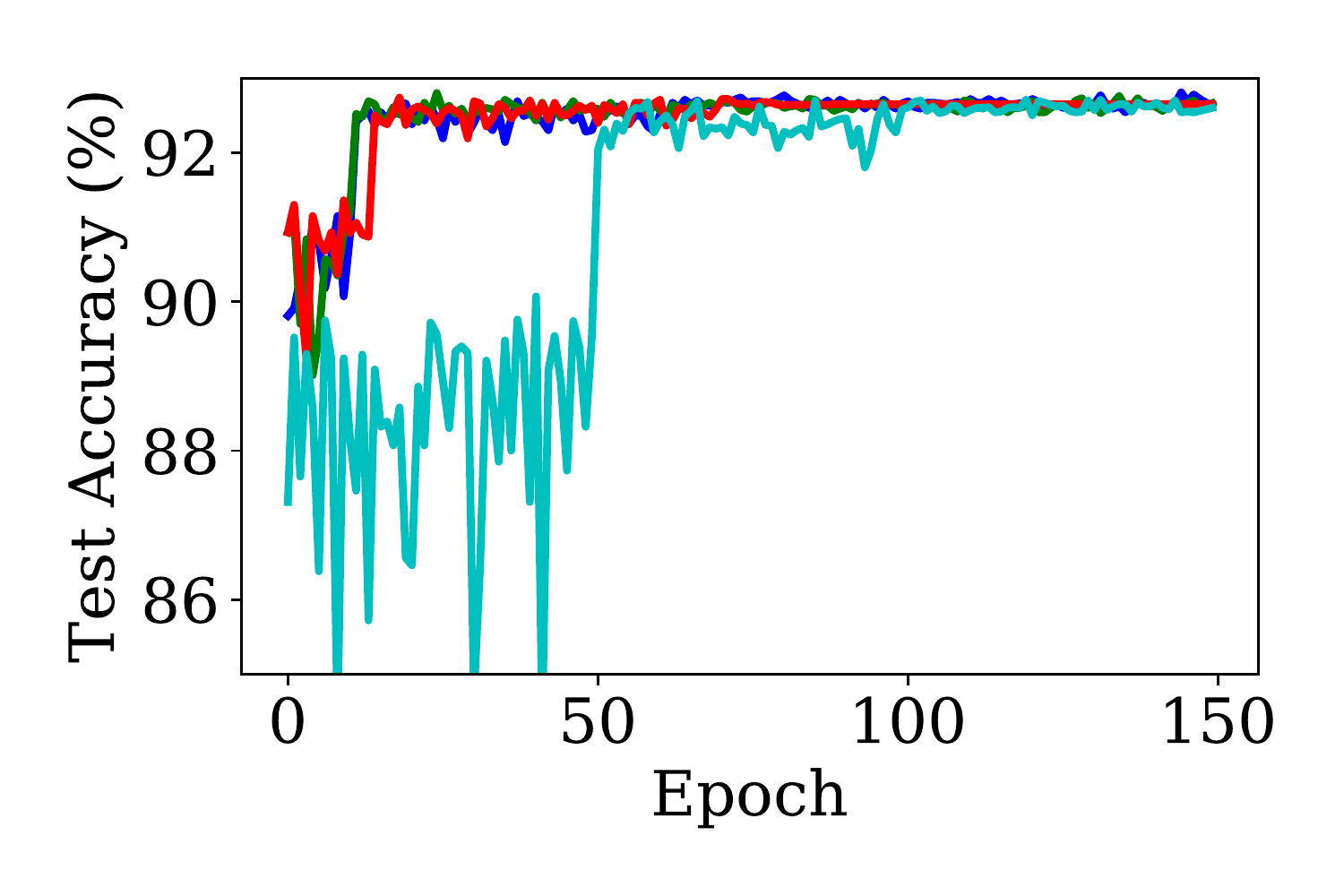}
    \includegraphics[width=0.33\linewidth]{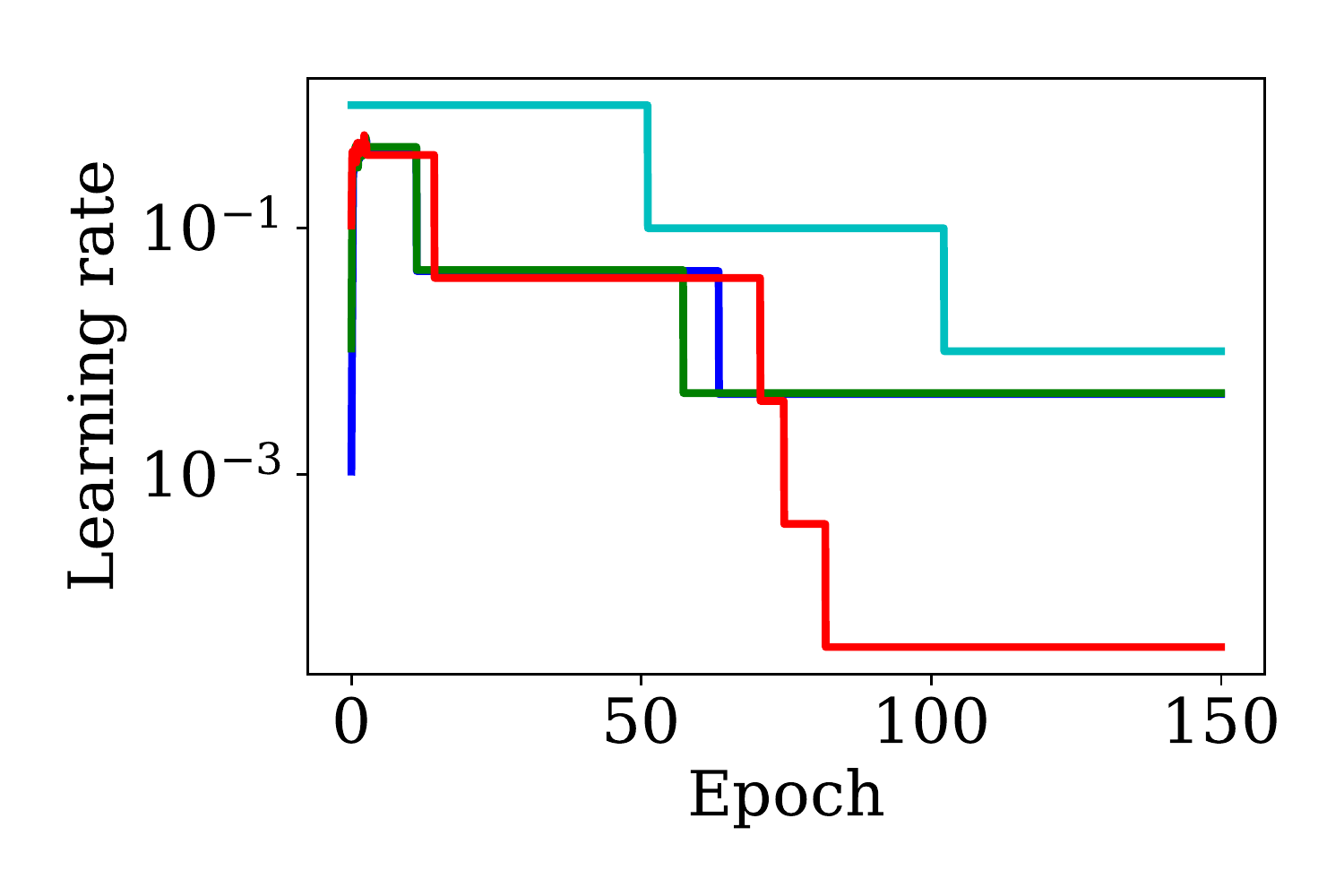}
	\\
	\includegraphics[width=0.33\linewidth]{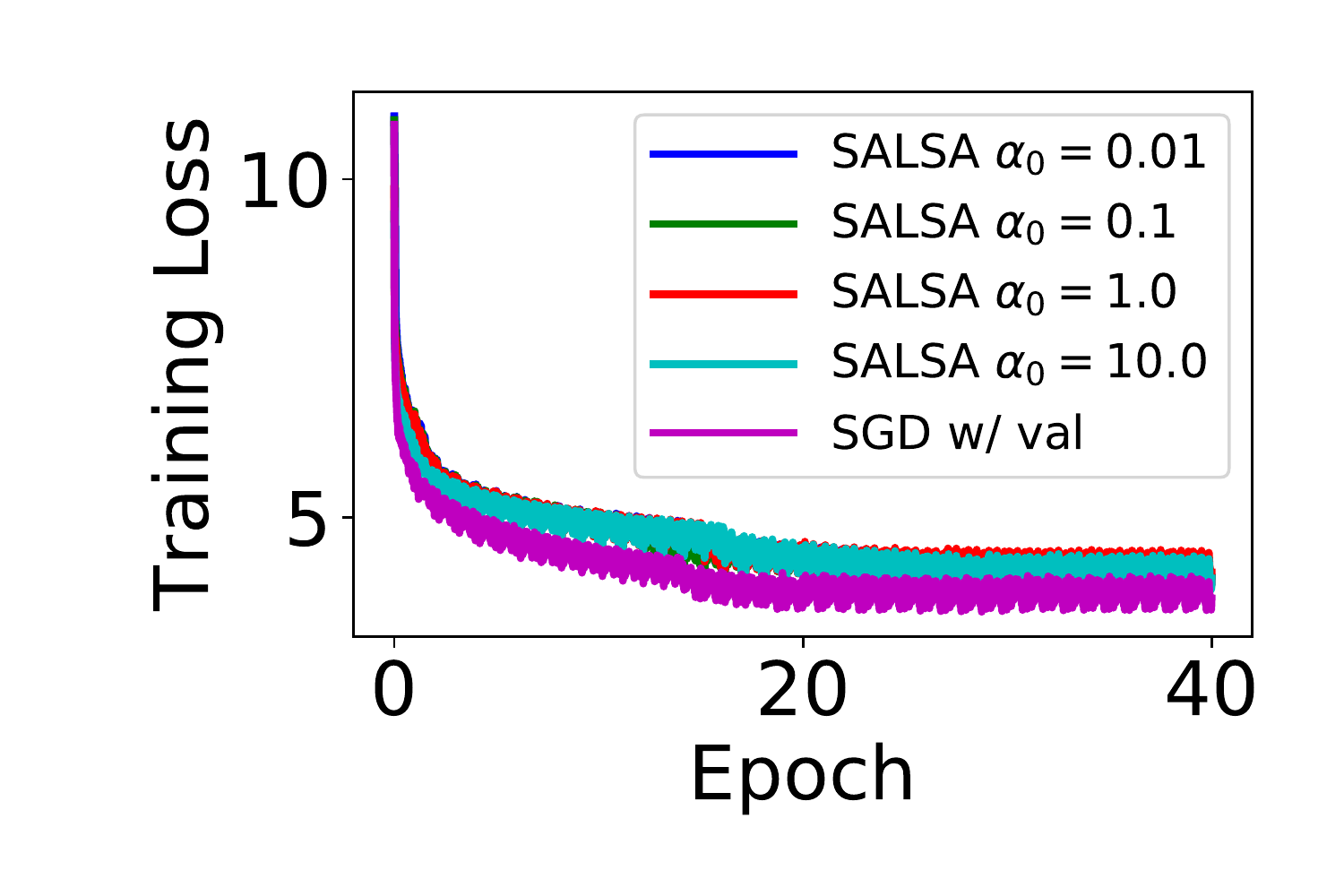}
    \includegraphics[width=0.33\linewidth]{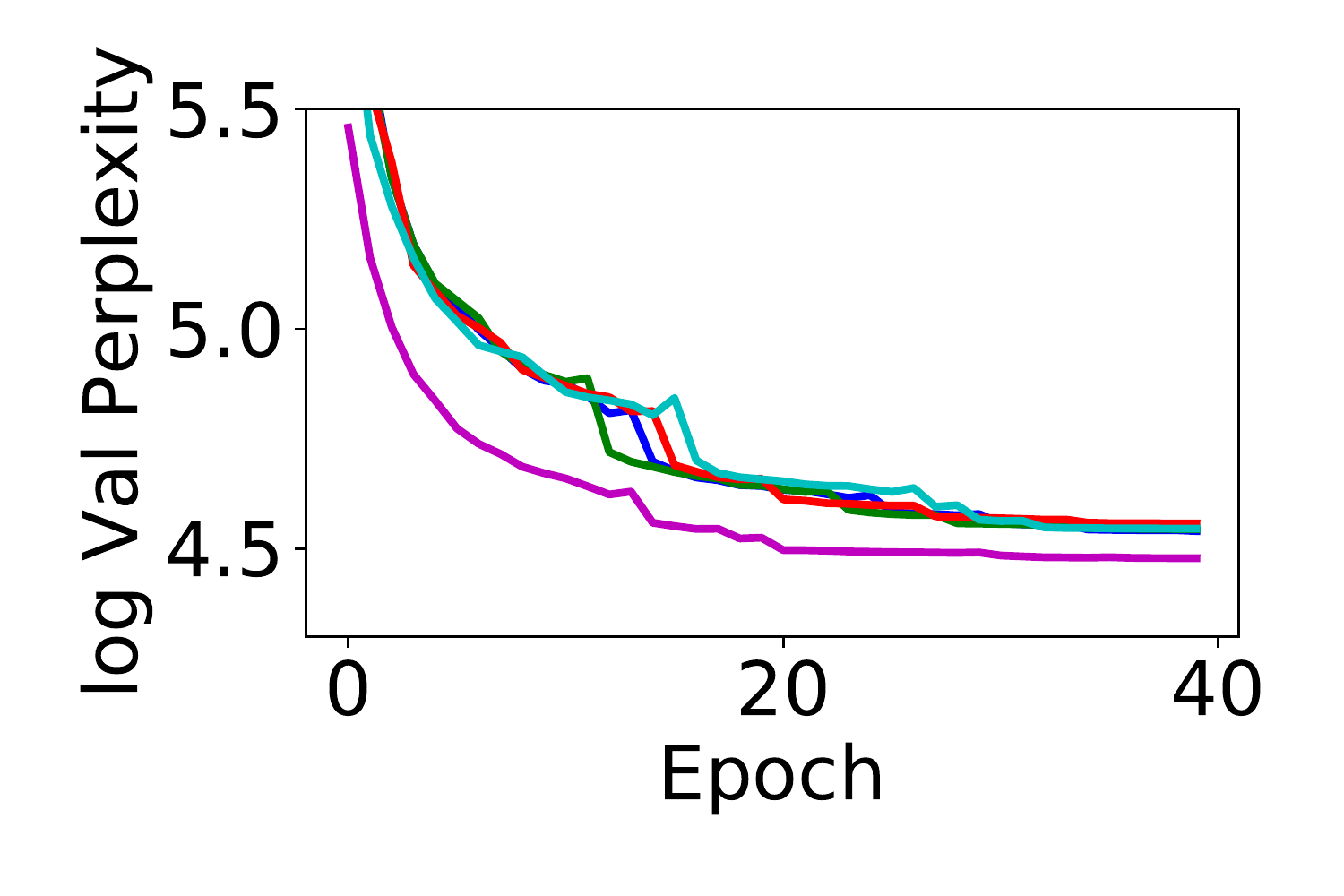}
    \includegraphics[width=0.33\linewidth]{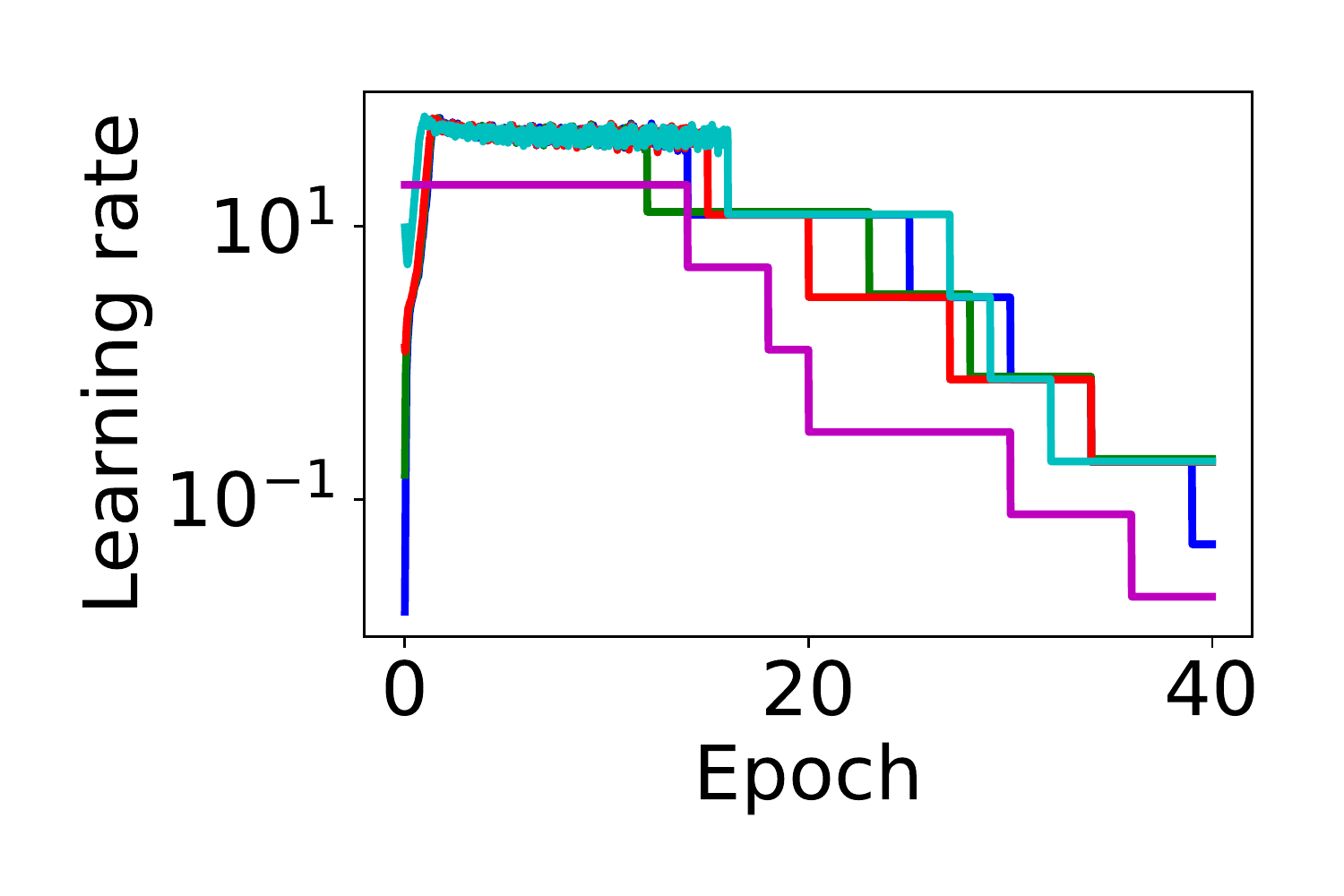}
    \vspace{-3ex}
    \caption{``SALSA'' on MNIST (top row) and ``SALSA w/ val'' Wikitext-2 (bottom row) with default parameters, starting from three different initial learning rates. The performance of SALSA is competitive with the state-of-the-art results reported in the literature.}
    \label{fig:salsa_mnist_rnn}
    \vspace{3ex}
\end{figure*}

When SALSA with the default parameter is applied to the task of LSTM on Wikitext-2 (Figure~\ref{fig:salsa_mnist_rnn}, bottom), the stable learning rate obtained by the SSLS, i.e., around 45, is larger than the hand-tuned initial learning rate 20. This larger initial learning rate results in the small gap between ``SALSA w/ val'' and ``SGD w/ val'' in terms of the final perplexity on the validation dataset. To verify this, we show ``SALSA w/ val'' with different $c$'s ($c = 0.025, 0.05, 0.1, 0.2$) in Figure~\ref{fig:rnn_salsa}. The larger the $c$ is, the smaller the stable learning rate SSLS reaches. For $c=0.1$ and $0.2$ (larger than the default $0.05$), the stable learning rates obtained by SSLS are nearly the same with the hand-tuned initial learning rate, and there is no performance gap between ``SALSA w/ val'' and ``SGD w/ val'' in terms of the final validation perplexity.

\begin{figure*}[t]
    \includegraphics[width=0.33\linewidth]{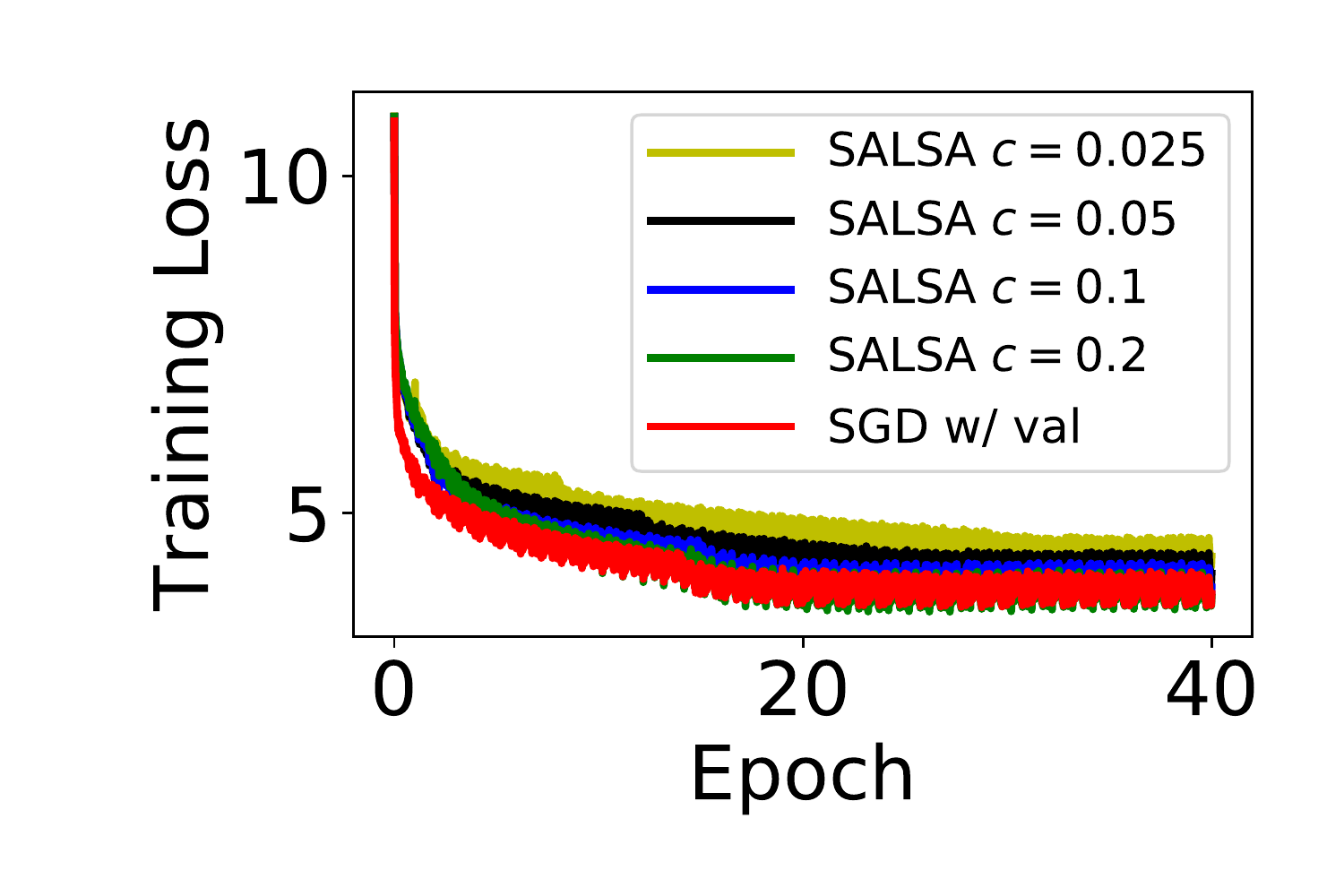}
    \includegraphics[width=0.33\linewidth]{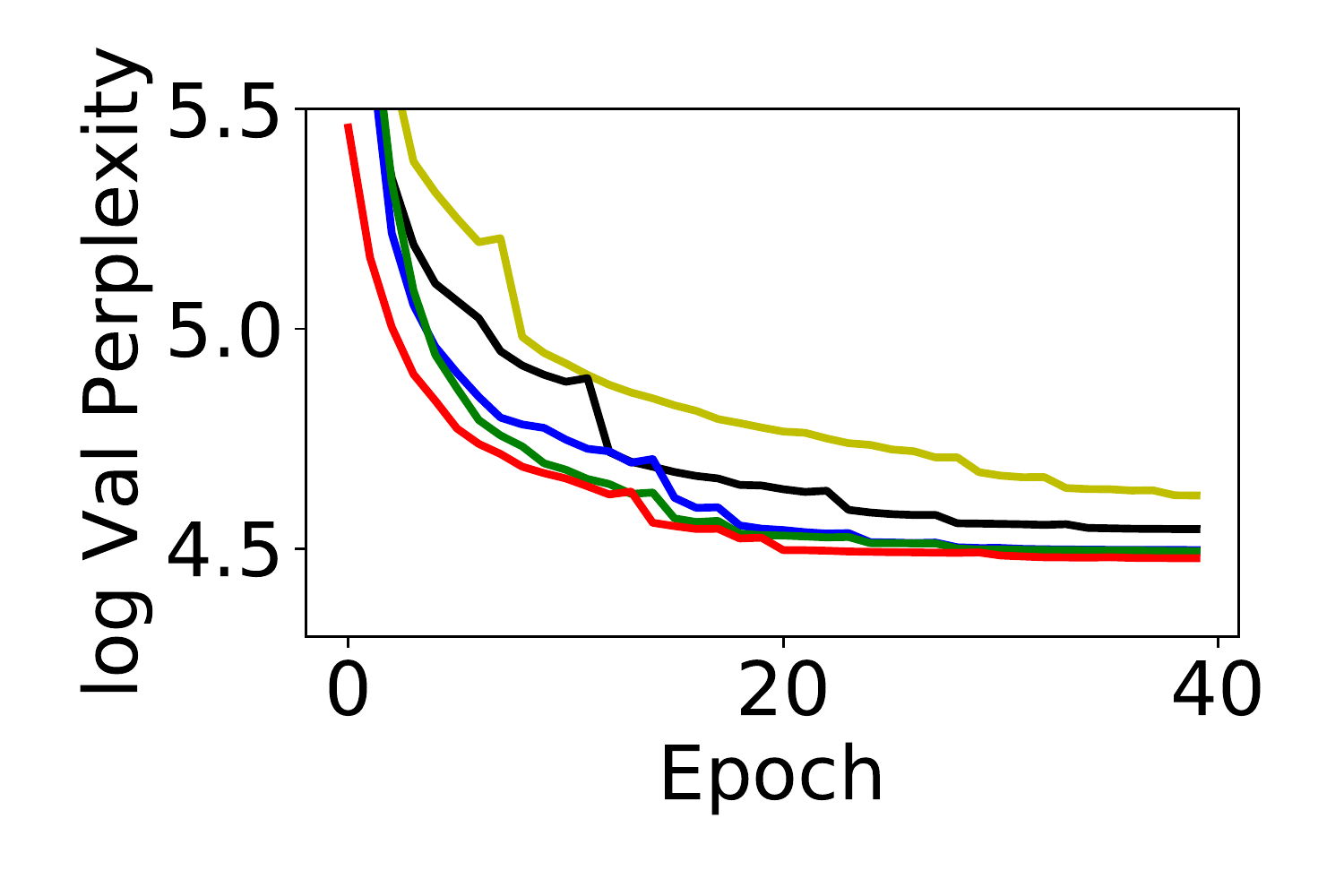}
    \includegraphics[width=0.33\linewidth]{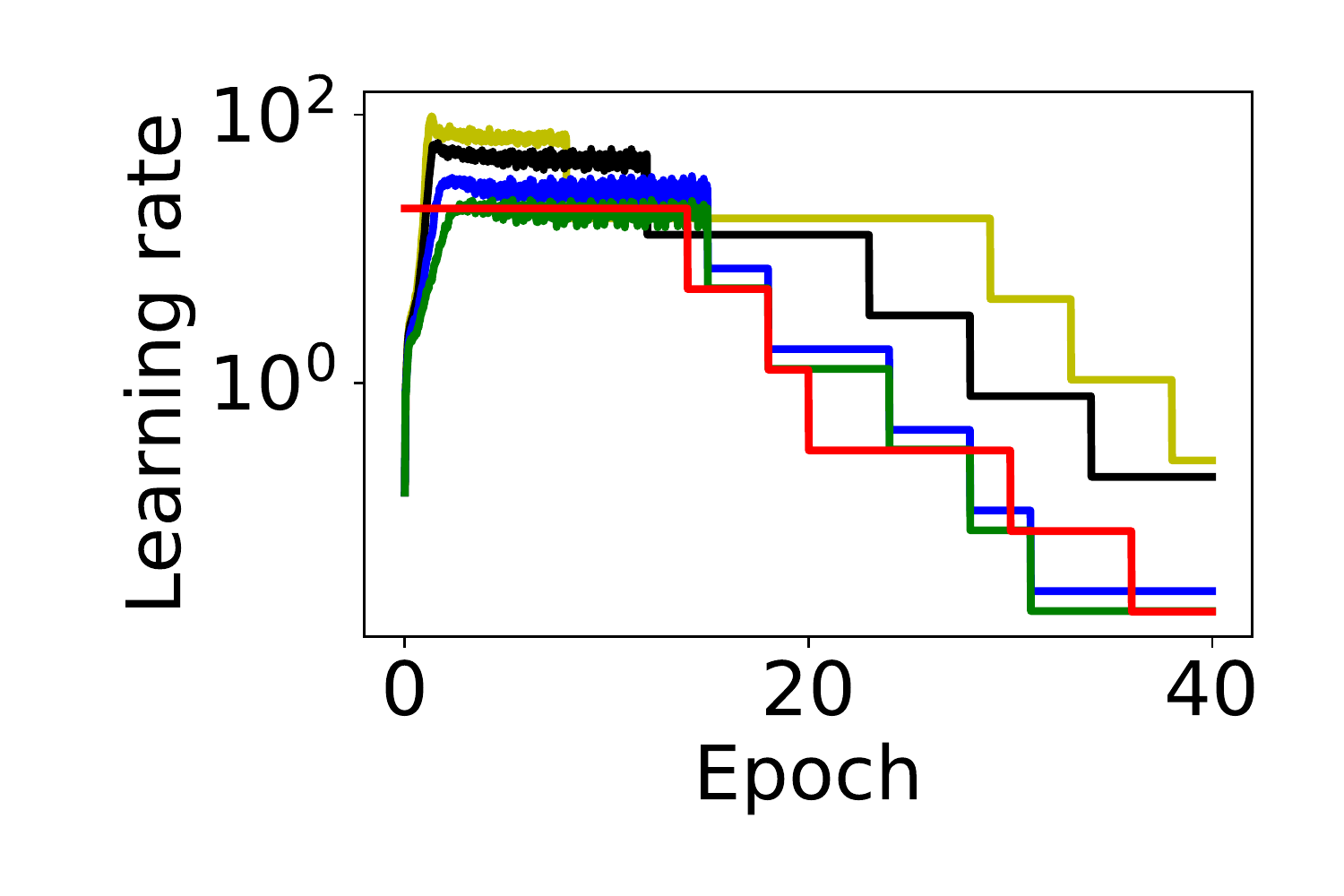}
    \vspace{-3ex}
    \caption{RNN: ``SALSA w/ val'' with different sufficient descent constant $c$.}
    \label{fig:rnn_salsa}
\end{figure*}

\subsection{More results for SASA+}\label{apd:sasaplusresults}
In Figure~\ref{fig:cifar_sasa_ablation}, we show the sensitivity analysis of SASA+ on the CIFAR-10 dataset, by showing the testing accuracy and learning rate schedule. Here, in Figure~\ref{fig:cifar_sasa_ablation_app}, we provide the full sensitivity analysis, including the drop factor $\tau$ (first row), the test confidence parameter $\delta$ (second row),  the ratio of recent samples to keep $\theta$ (third row), the test frequency $K_\textrm{test}$ (fourth row), and combined with different methods (last row). Notice that with higher testing frequency $K_\textrm{test}=100$, the test is fired earlier. However, with different test frequency $K_\textrm{test}$, the final results are nearly the same. This indicates that our test does not suffer from the multiple-test problem in statistics. Since the statistics in our tests are changing very smoothly, see the middle row in Figure~\ref{fig:cifar_sasa_statistics}, this robustness against multiple tests seems reasonable.
\begin{figure*}[t]
    \includegraphics[width=0.33\linewidth]{testsperepoch1_archmyresnet18_trainloss_smooth_df.pdf}
    \includegraphics[width=0.33\linewidth]{testsperepoch1_archmyresnet18_testacc_df.pdf}
    \includegraphics[width=0.33\linewidth]{testsperepoch1_archmyresnet18_lrs_df.pdf} \\
    \includegraphics[width=0.33\linewidth]{testsperepoch1_archmyresnet18_trainloss_smooth_sigma.pdf}
    \includegraphics[width=0.33\linewidth]{testsperepoch1_archmyresnet18_testacc_sigma.pdf}
    \includegraphics[width=0.33\linewidth]{testsperepoch1_archmyresnet18_lrs_sigma.pdf} \\
    \includegraphics[width=0.32\linewidth]{testsperepoch1_archmyresnet18_trainloss_smooth.pdf}
    \includegraphics[width=0.33\linewidth]{testsperepoch1_archmyresnet18_testacc.pdf}
    \includegraphics[width=0.33\linewidth]{testsperepoch1_archmyresnet18_lrs.pdf} \\
    \includegraphics[width=0.33\linewidth]{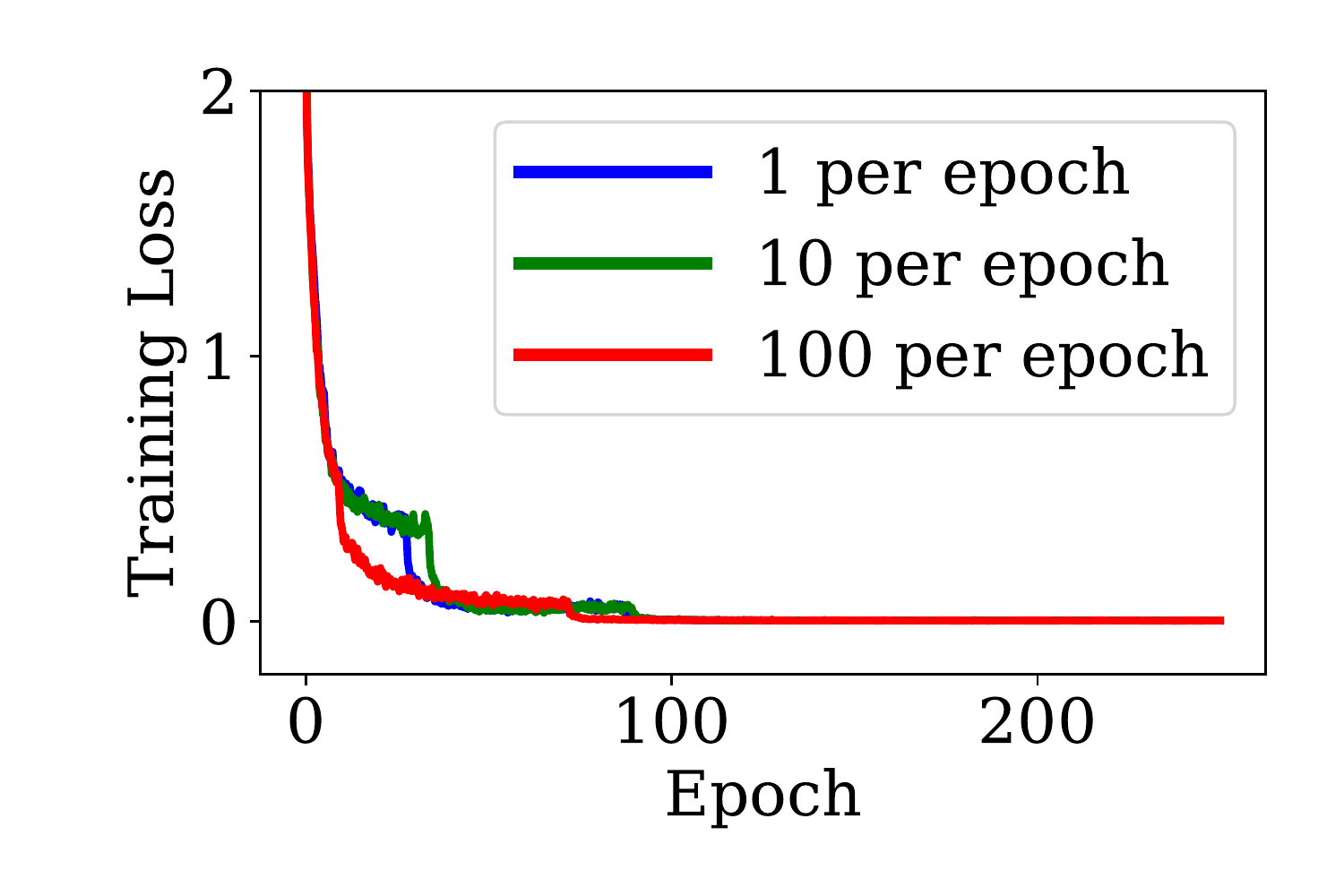}
    \includegraphics[width=0.33\linewidth]{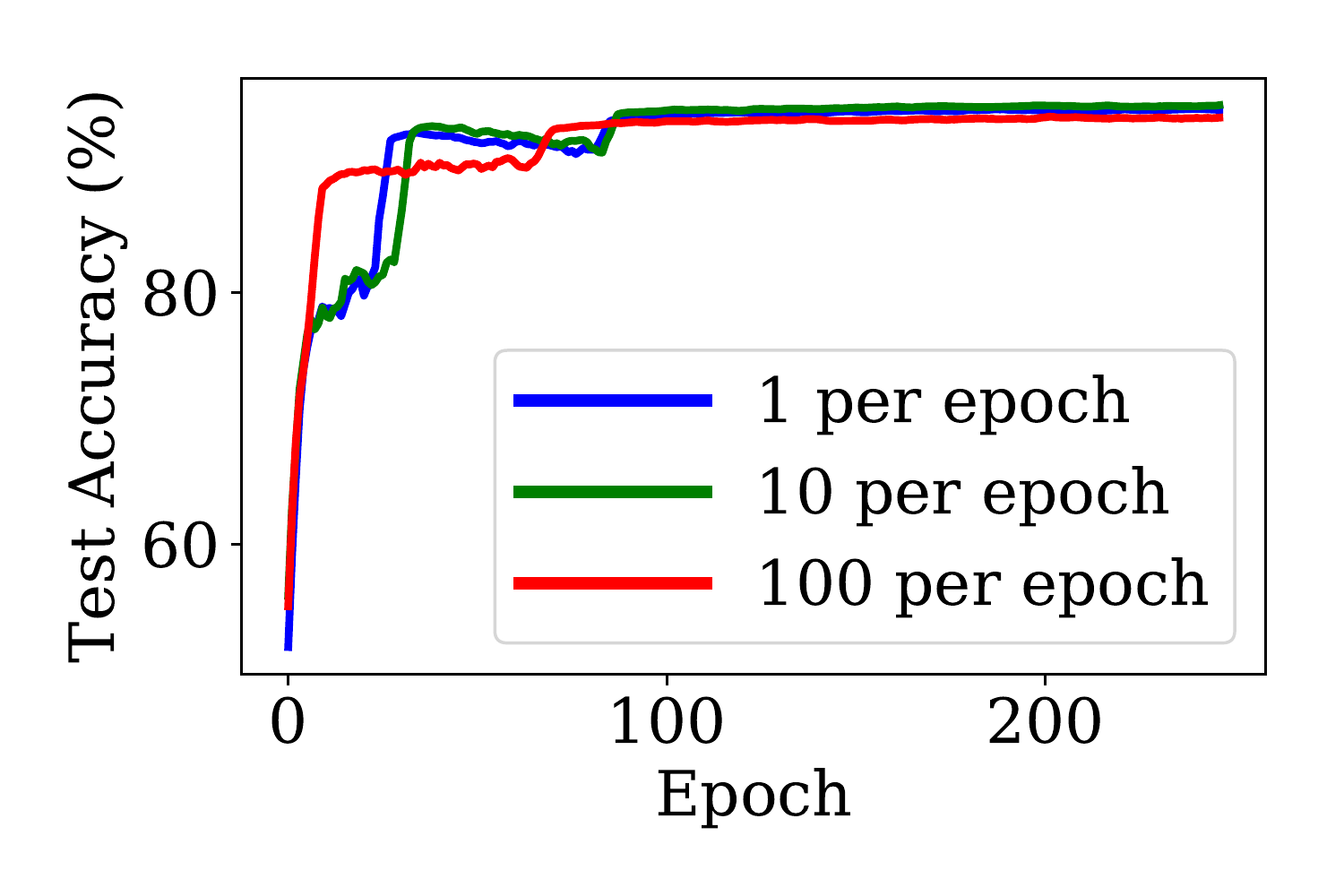}
    \includegraphics[width=0.33\linewidth]{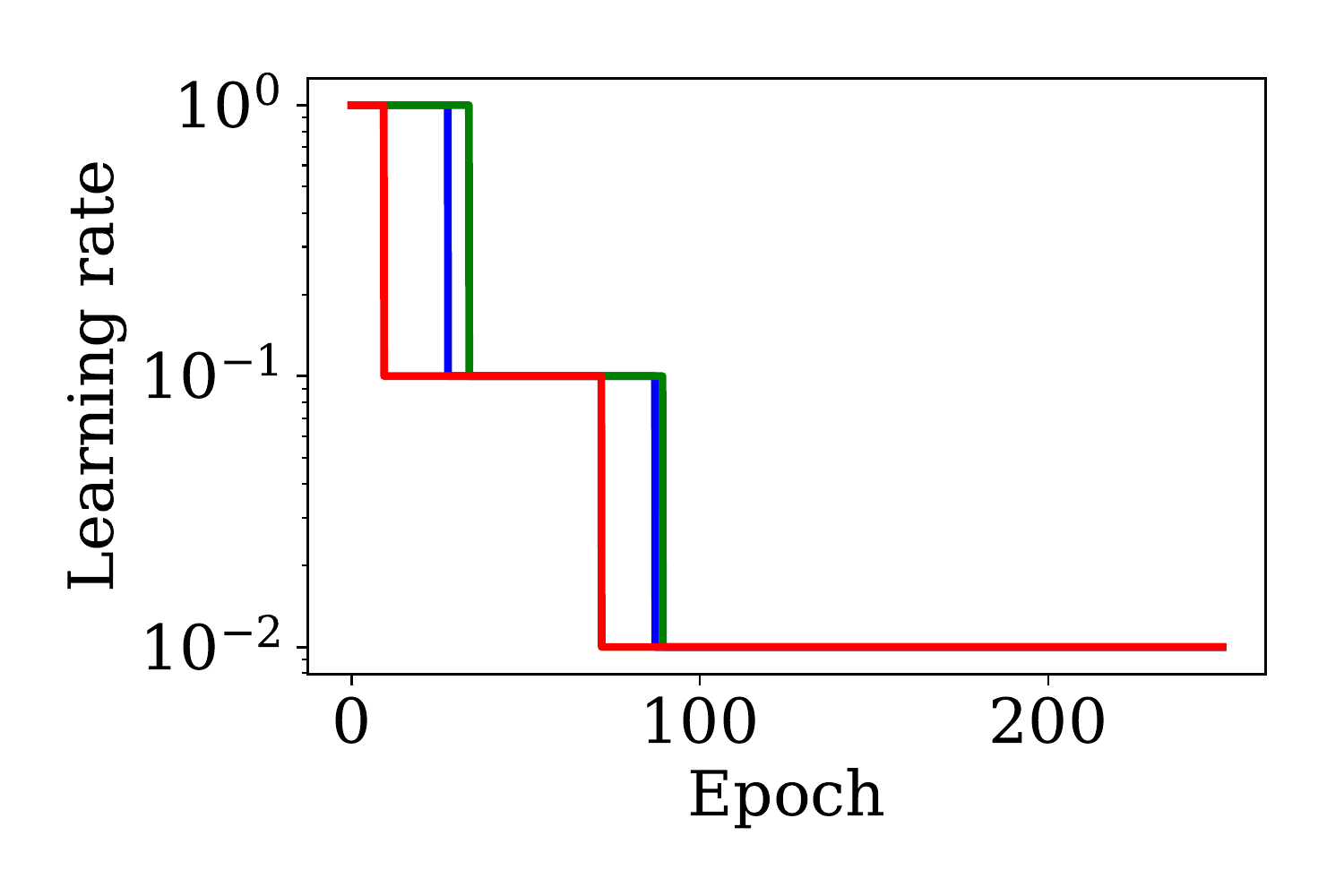} \\
    \includegraphics[width=0.33\linewidth]{qhm_archmyresnet18_trainloss_smooth_method.pdf}
    \includegraphics[width=0.33\linewidth]{qhm_archmyresnet18_testacc_method.pdf}
    \includegraphics[width=0.33\linewidth]{qhm_archmyresnet18_lrs_method.pdf}
    \caption{Sensitivity analysis of SASA+ on CIFAR10, using $\beta=0.9$ and $\nu=1$. The training loss, test accuracy, and learning rate schedule for SASA+ using different values of $\tau$ (first row), $\delta$ (second row),  $\theta$ (third row), $K_\textrm{test}$ (fourth row on number of tests in an epoch) around the default values are shown, and the SASA+ applied on different methods (last row, i.e., SGD with momenton, NAG and QHM).}
    \label{fig:cifar_sasa_ablation_app}
\end{figure*}

In the last row of Figure~\ref{fig:cifar_sasa_ablation_app}, we show that SASA+ can work with different optimization methods (i.e., SHB, NAG and QHM) on the CIFAR-10 dataset. 
Similar results for ImageNet are shown in Figure~\ref{fig:imagenet_sasaplus}.
We also show a sensitivity test of SASA+ on ImageNet with respect to $\tau$ (first row), $\delta$ (second row) and  $\theta$ (third row) in Figure~\ref{fig:imagenet_sasa_ablation_app}.
\begin{figure*}[t]
	\includegraphics[width=0.33\linewidth]{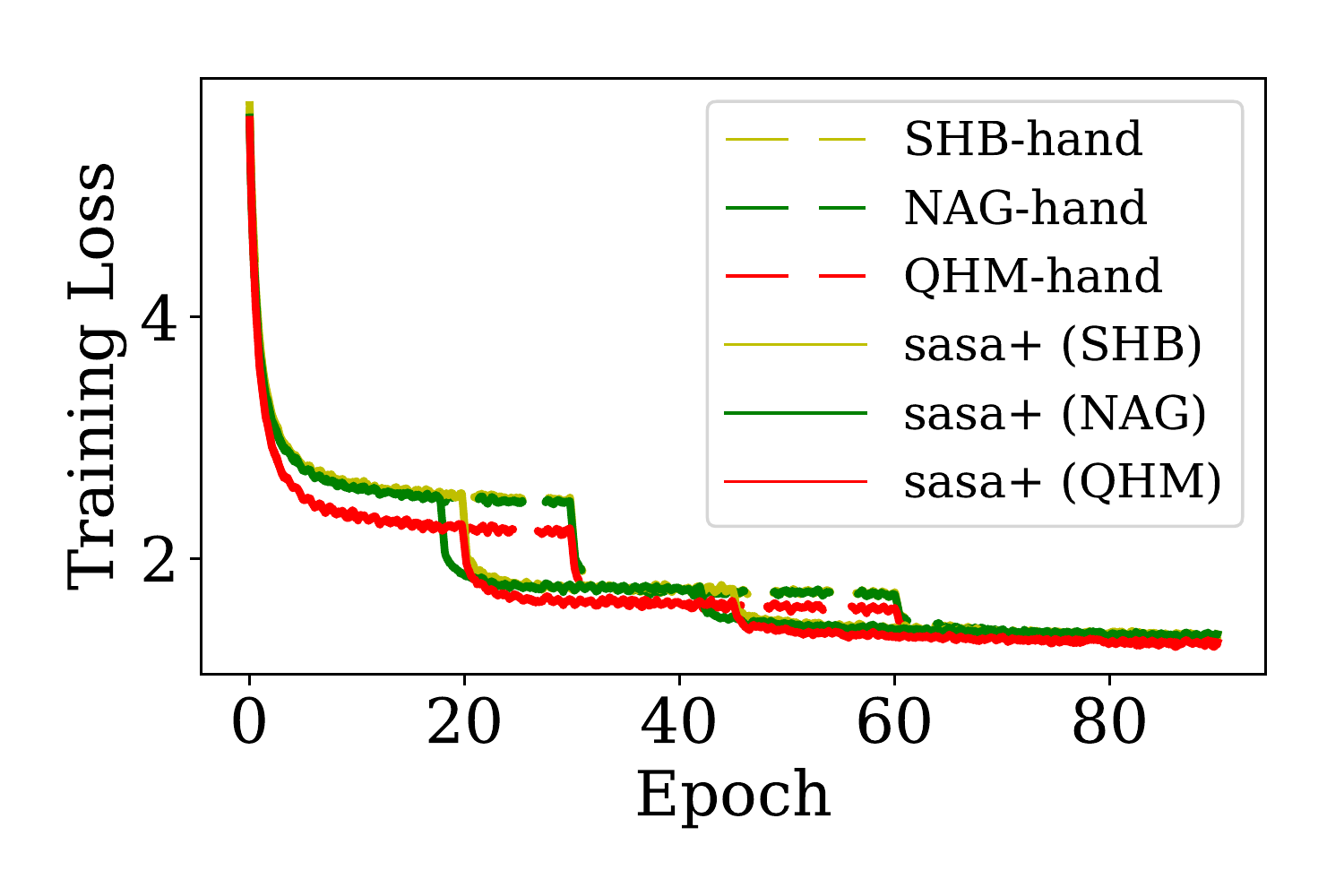}
	\includegraphics[width=0.33\linewidth]{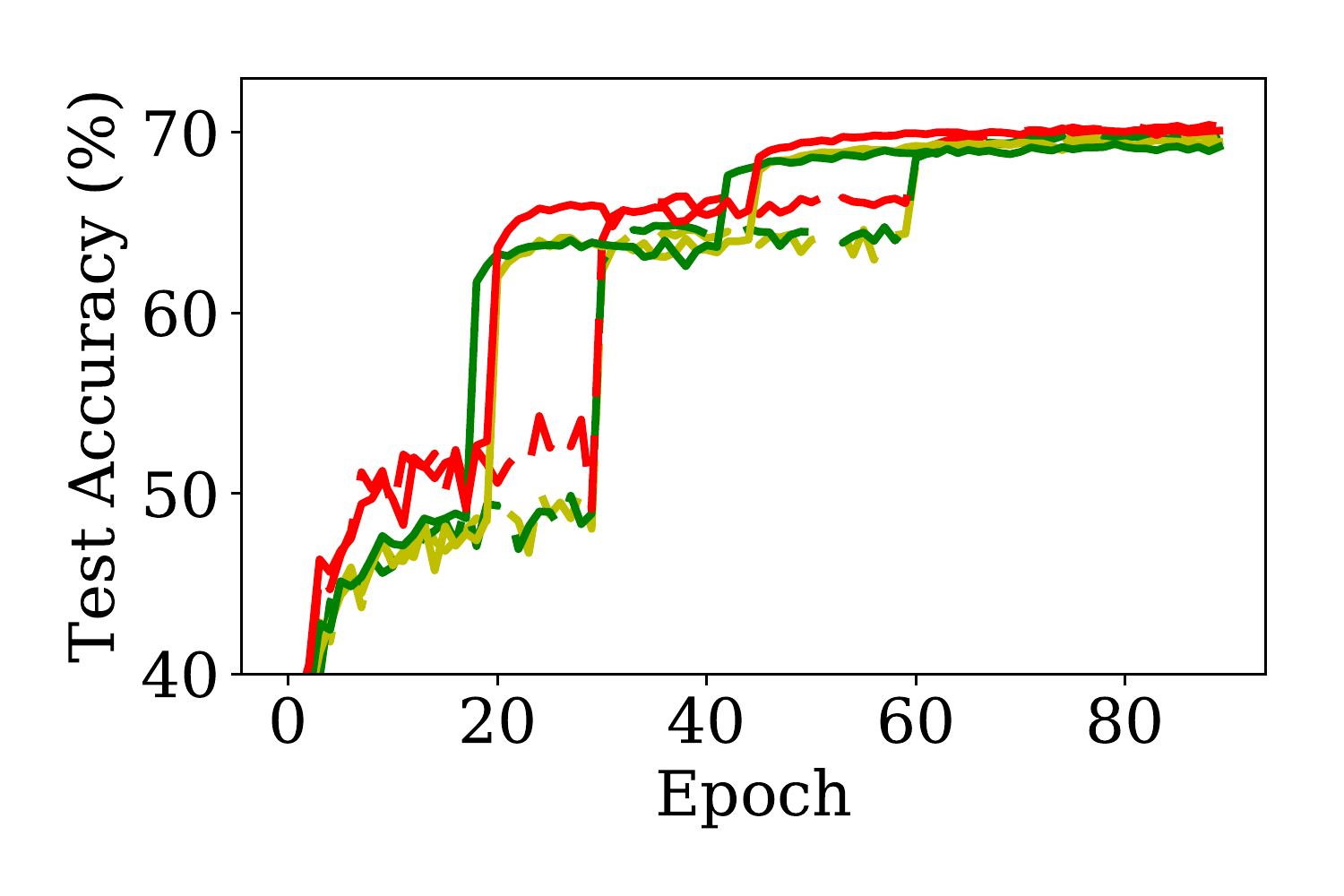}
	\includegraphics[width=0.33\linewidth]{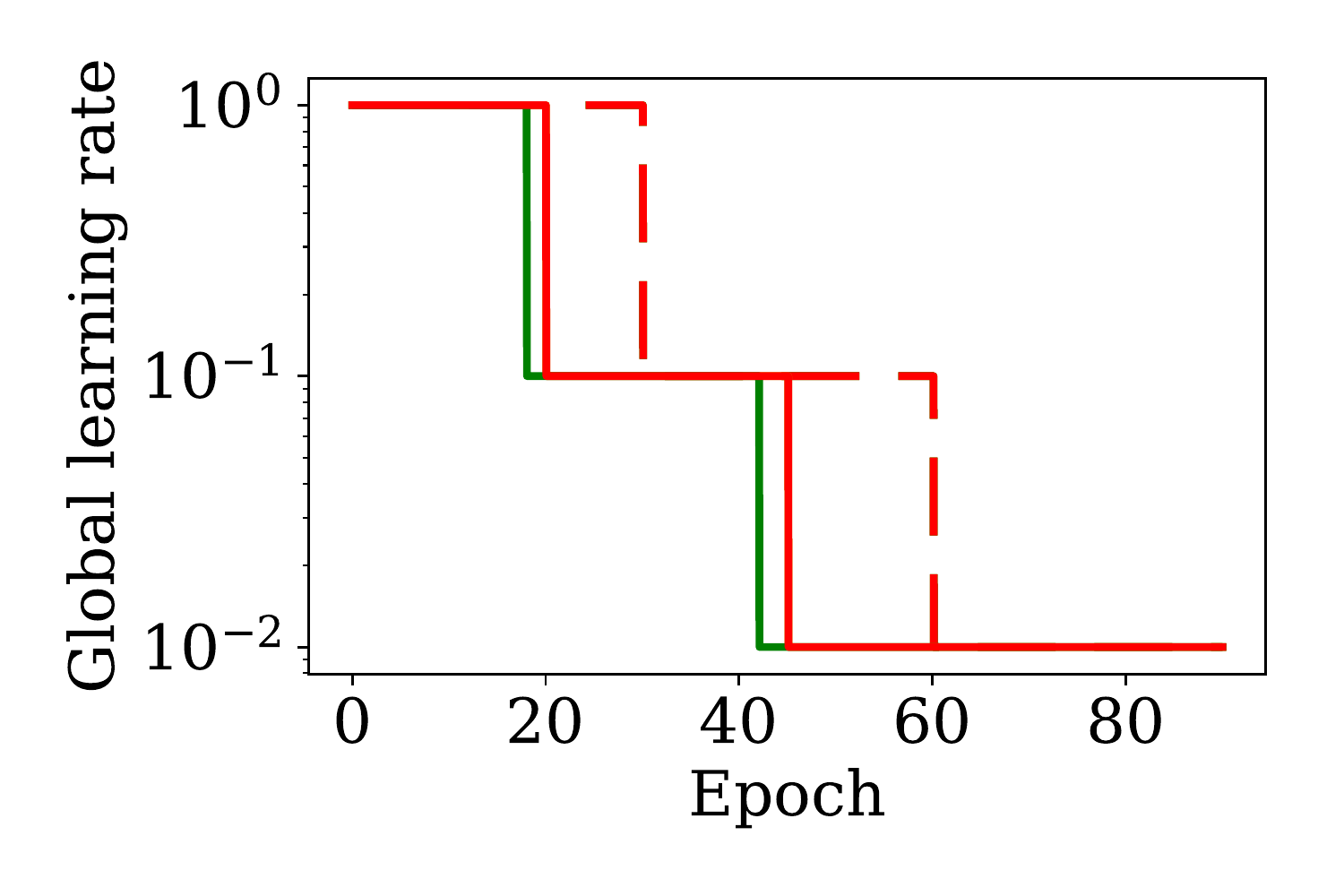} 
	\caption{The dotted curves are stochastic optimization algorithms with hand-tuning learning rate, i.e., decrease every 30 epochs, as shown in the lower right panel. The solid curves are SASA+ combined with these 3 algorithms. SASA+ automatically adapts their learning rate (see the right panel) and achieves comparable and even slightly higher testing accuracy (see the middle panel).}
	\label{fig:imagenet_sasaplus}
\end{figure*}

\begin{figure*}[t]
    \includegraphics[width=0.33\linewidth]{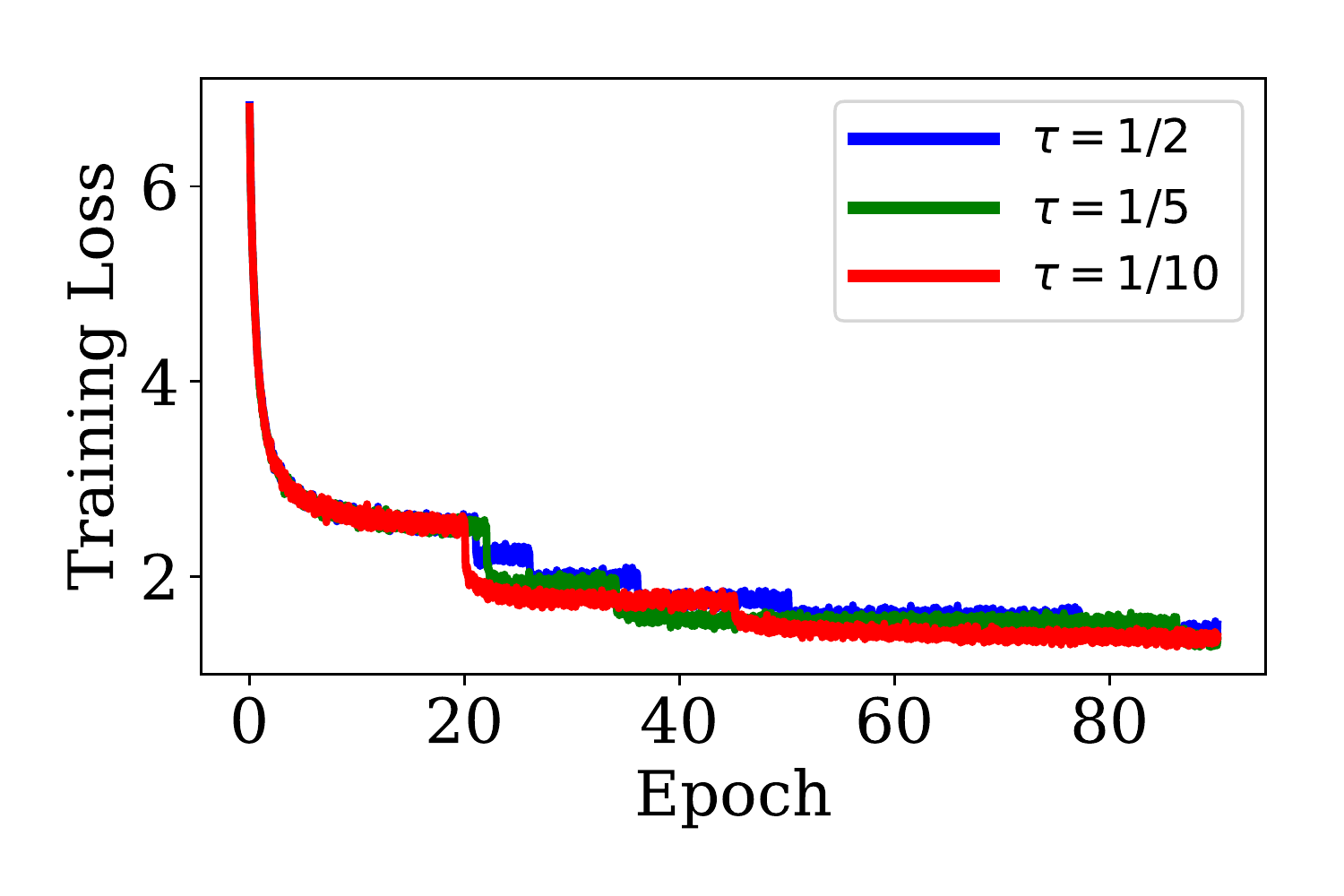}
    \includegraphics[width=0.33\linewidth]{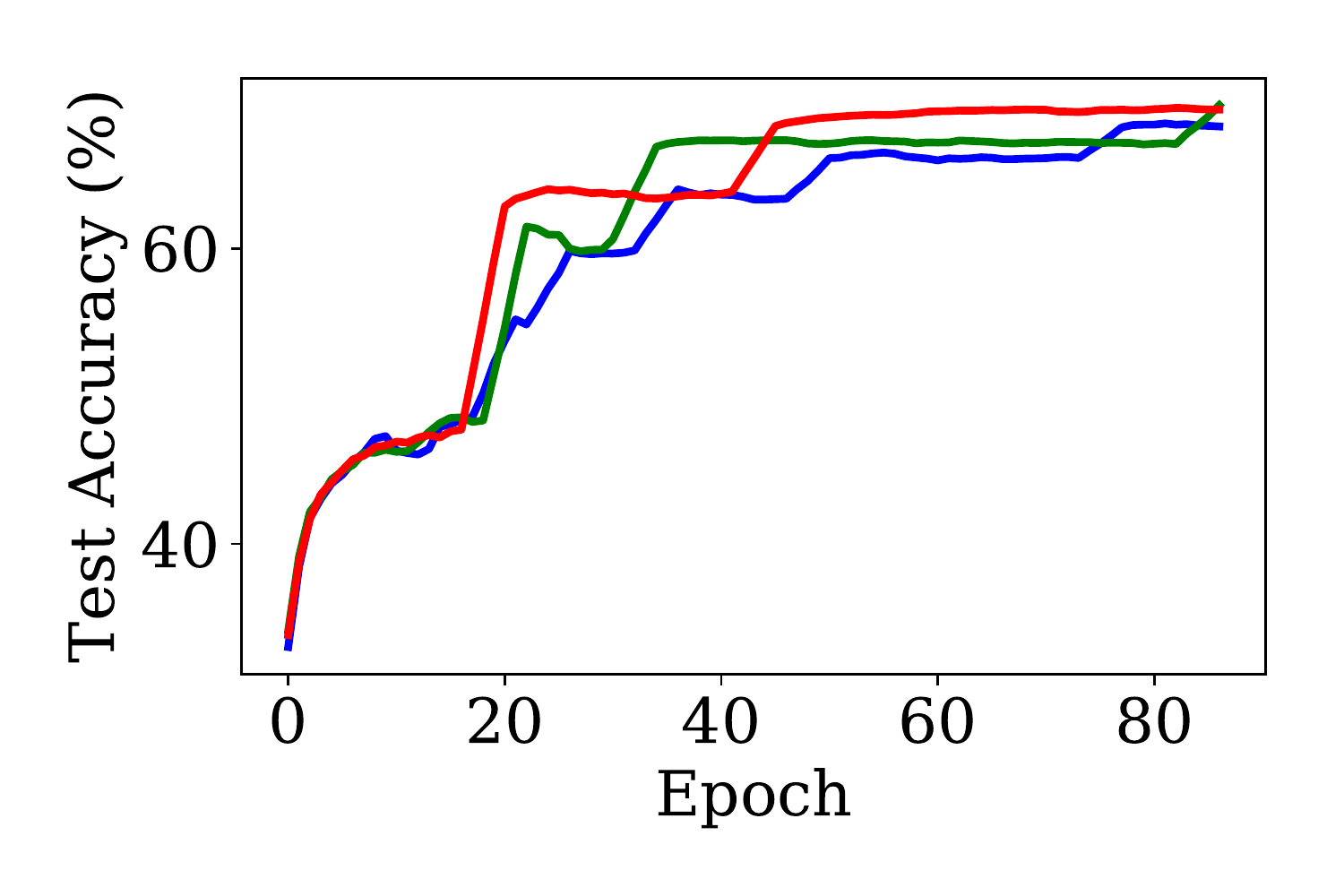}
    \includegraphics[width=0.33\linewidth]{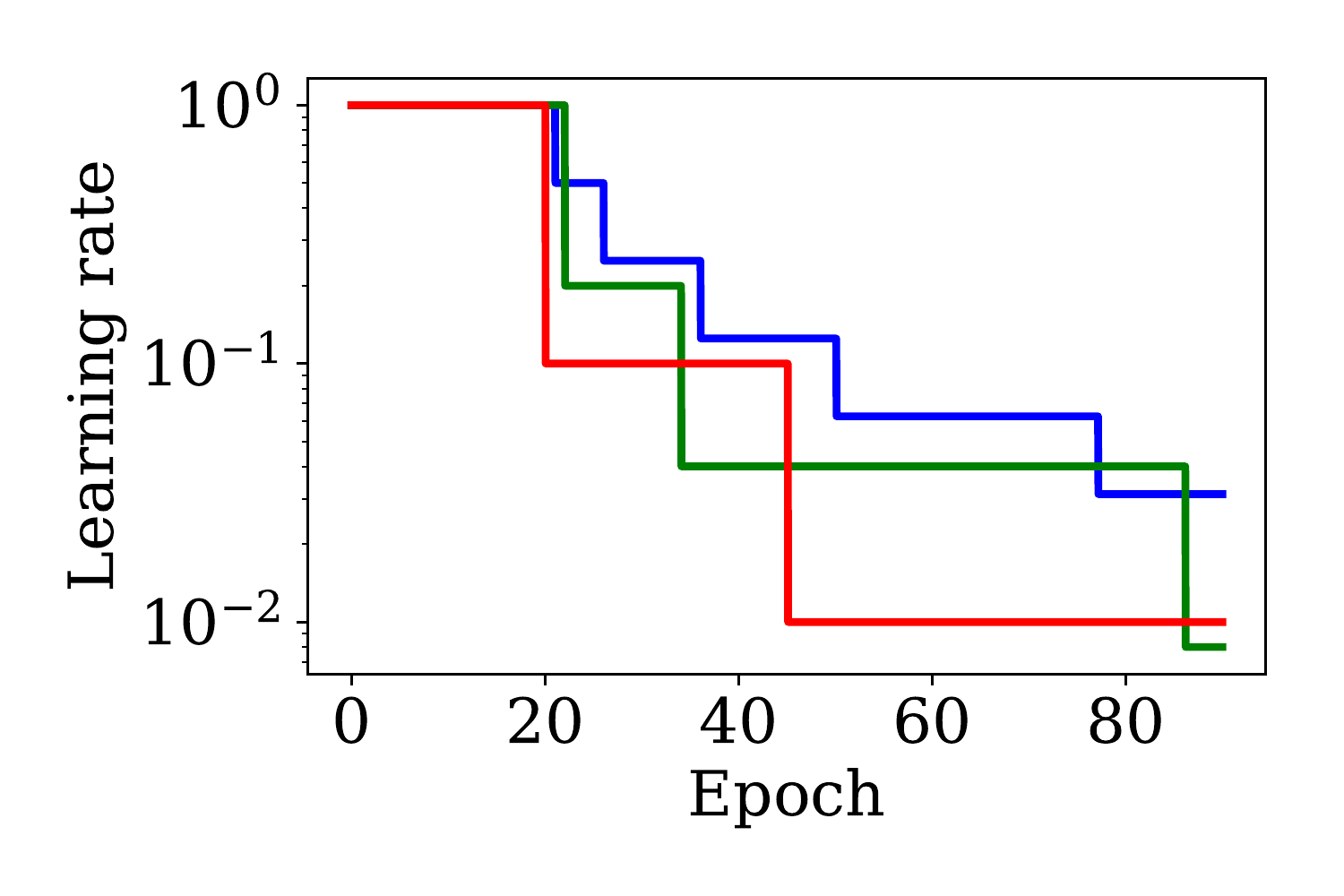} \\
    \includegraphics[width=0.33\linewidth]{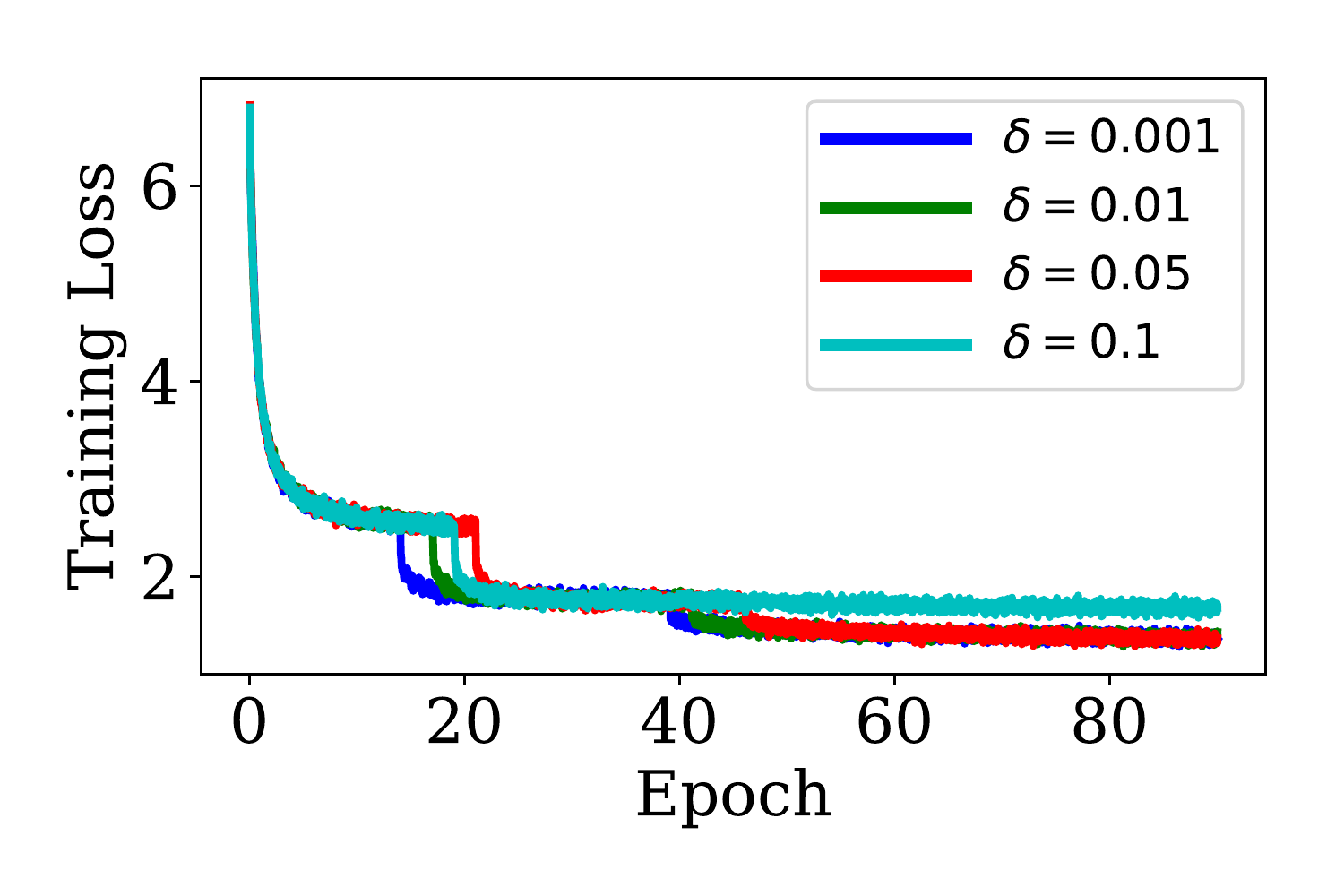}
    \includegraphics[width=0.33\linewidth]{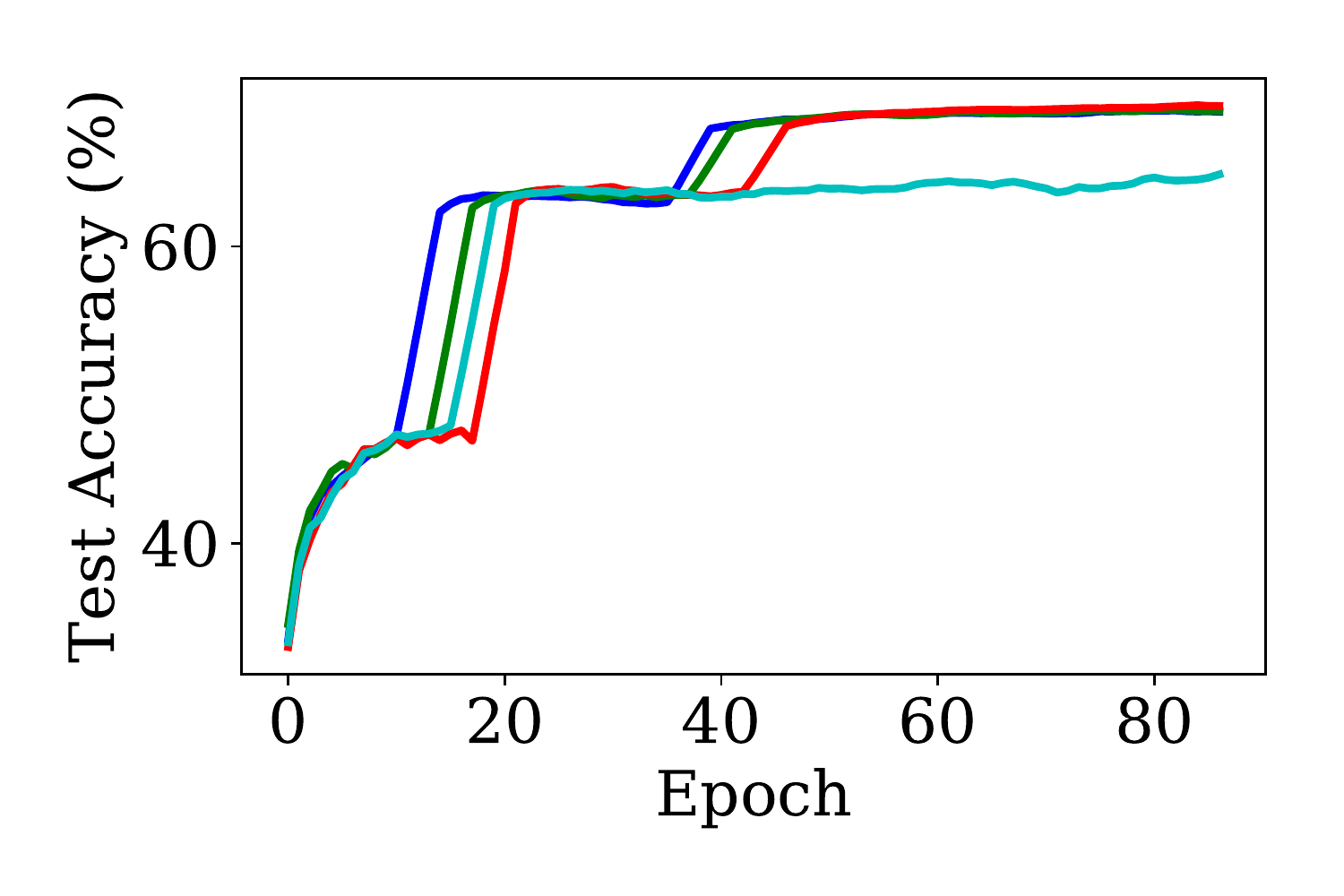}
    \includegraphics[width=0.33\linewidth]{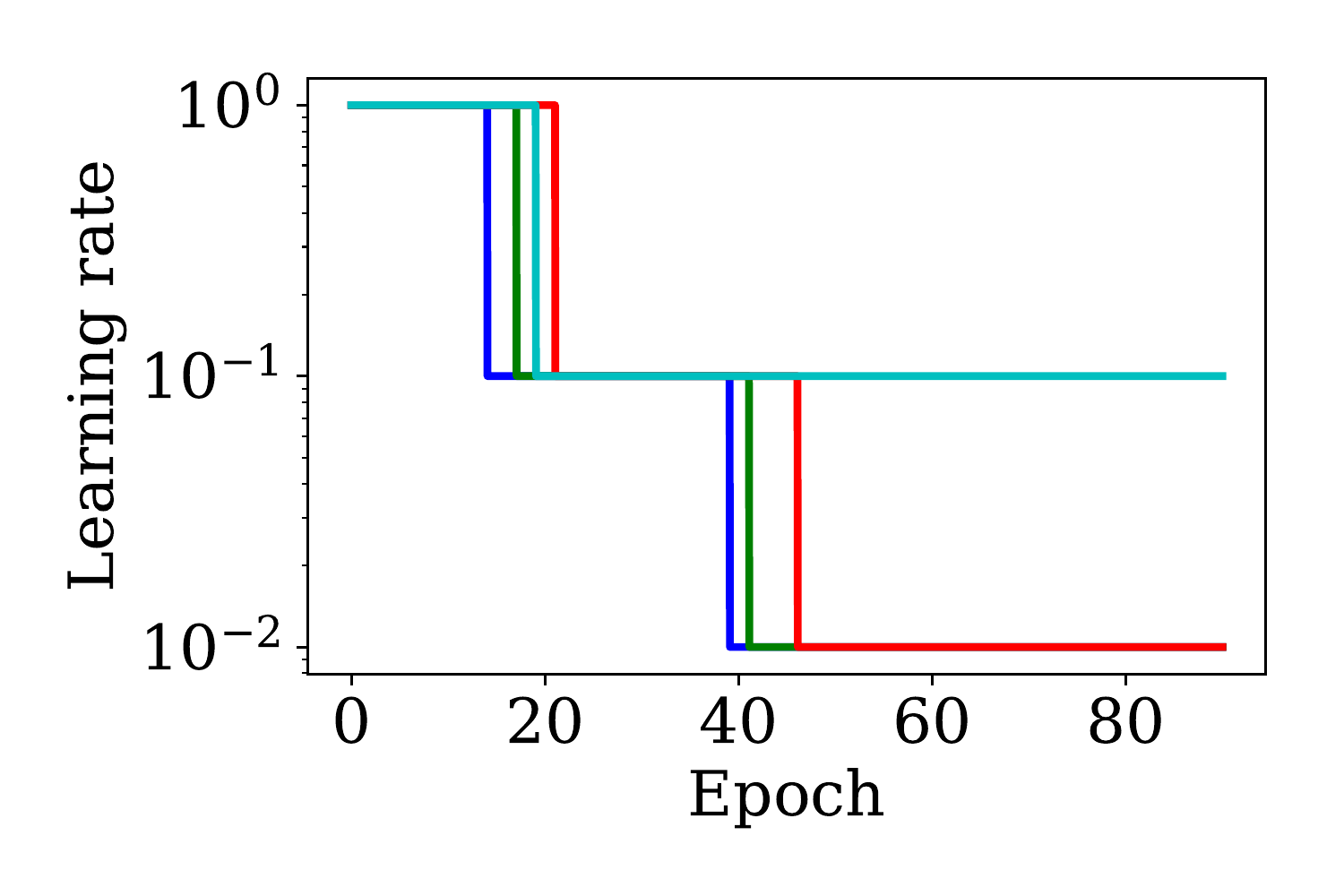} \\
    \includegraphics[width=0.33\linewidth]{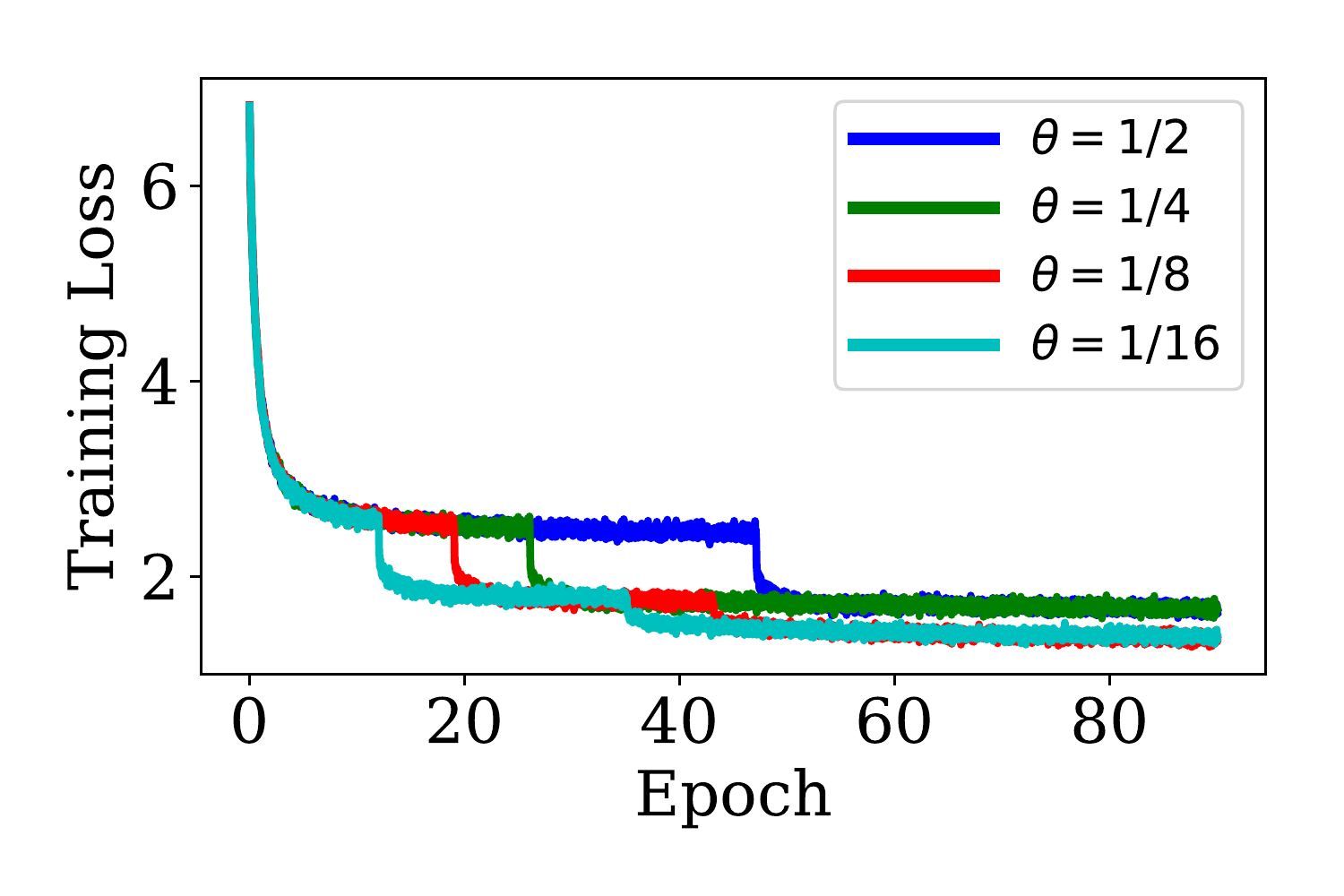}
    \includegraphics[width=0.33\linewidth]{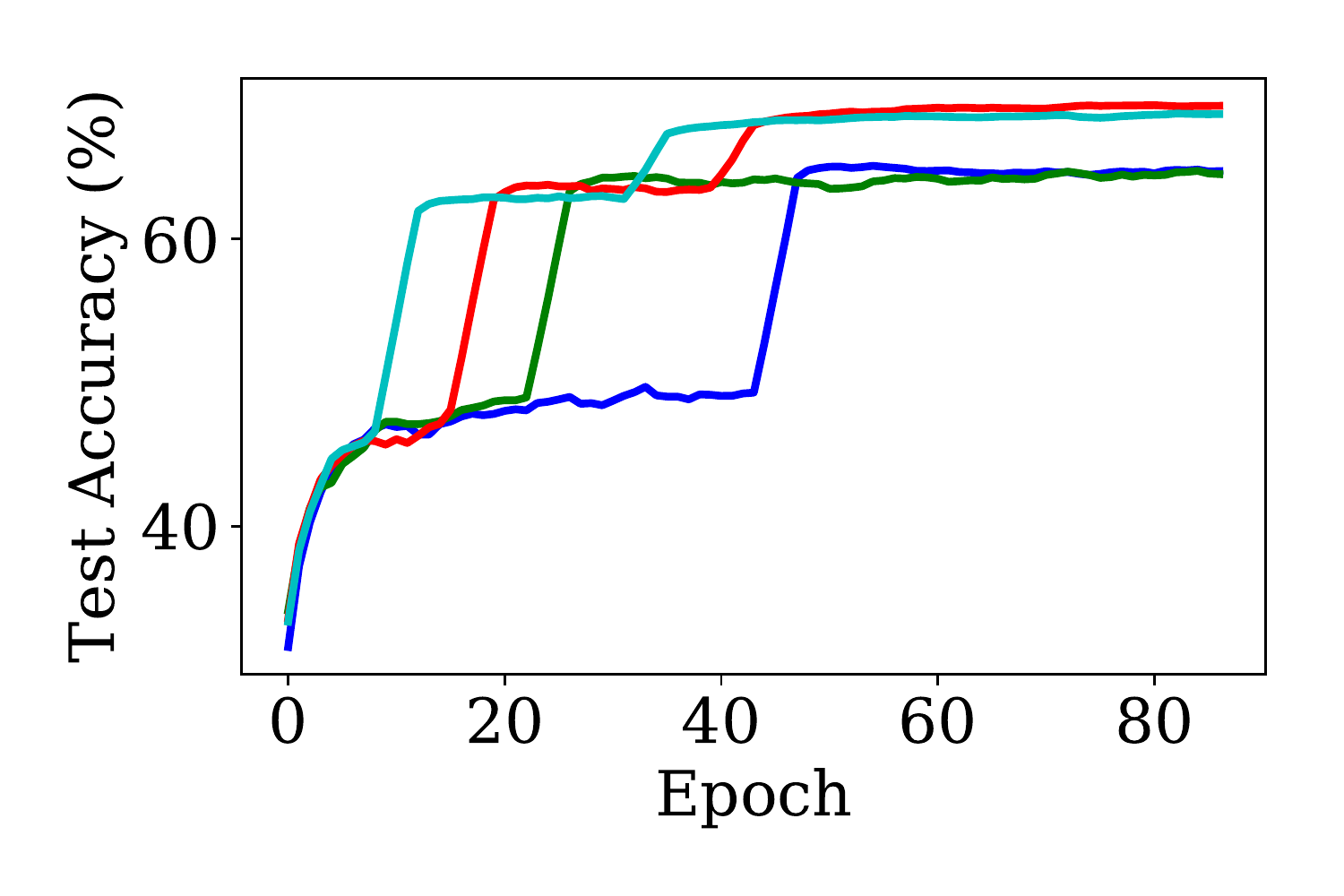}
    \includegraphics[width=0.33\linewidth]{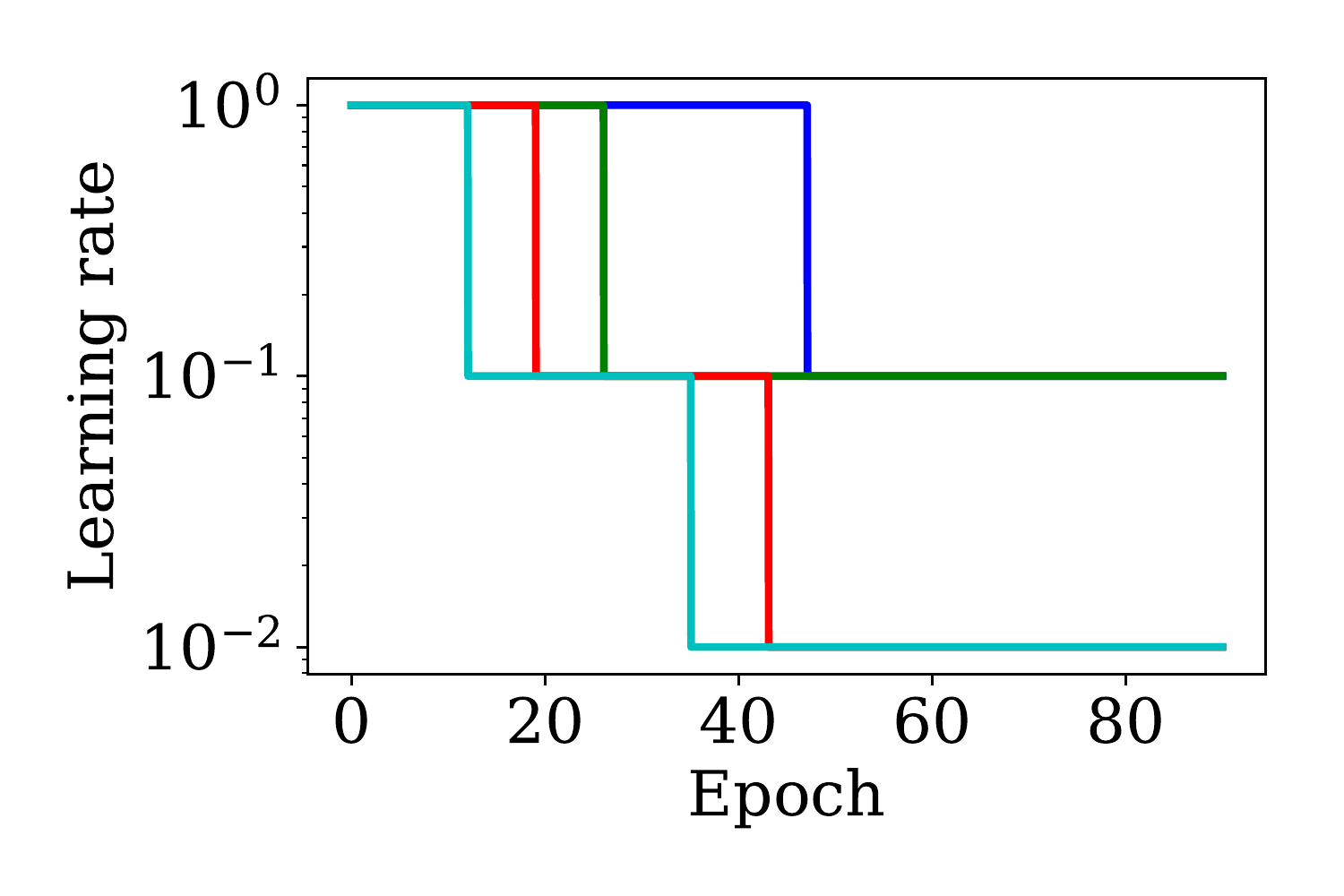} 
    \caption{Sensitivity analysis of SASA+ on ImageNet, using $\beta=0.9$ and $\nu=1$. The training loss, test accuracy, and learning rate schedule for SASA+ using different values of $\tau$ (first row), $\delta$ (second row) and  $\theta$ (third row) around the default values are shown.}
    \label{fig:imagenet_sasa_ablation_app}
\end{figure*}

In Figure~\ref{fig:rnn_sasa}, we show that the large LSTM model (as described in Section~\ref{sec:salsa}) quickly overfits the Wikitest-2 dataset, which can be seen from the quickly decreasing training loss but increasing validation perplexity for SASA and SASA+. Adding the validation dataset as another learning rate drop criterion avoids overfitting during training and to obtain good final performance on the validation/testing set. 
\begin{figure*}[t]
    \includegraphics[width=0.33\linewidth]{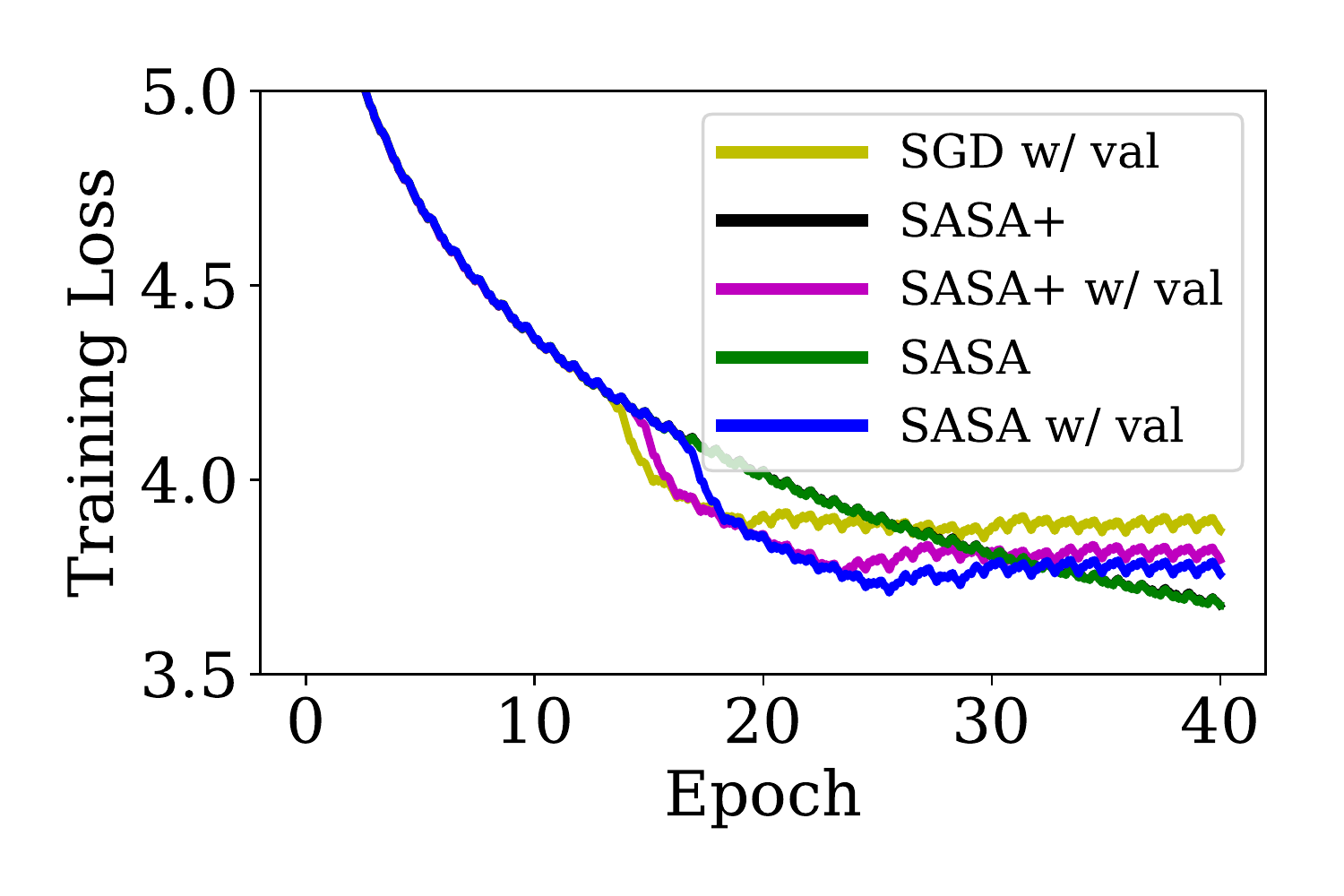}
    \includegraphics[width=0.33\linewidth]{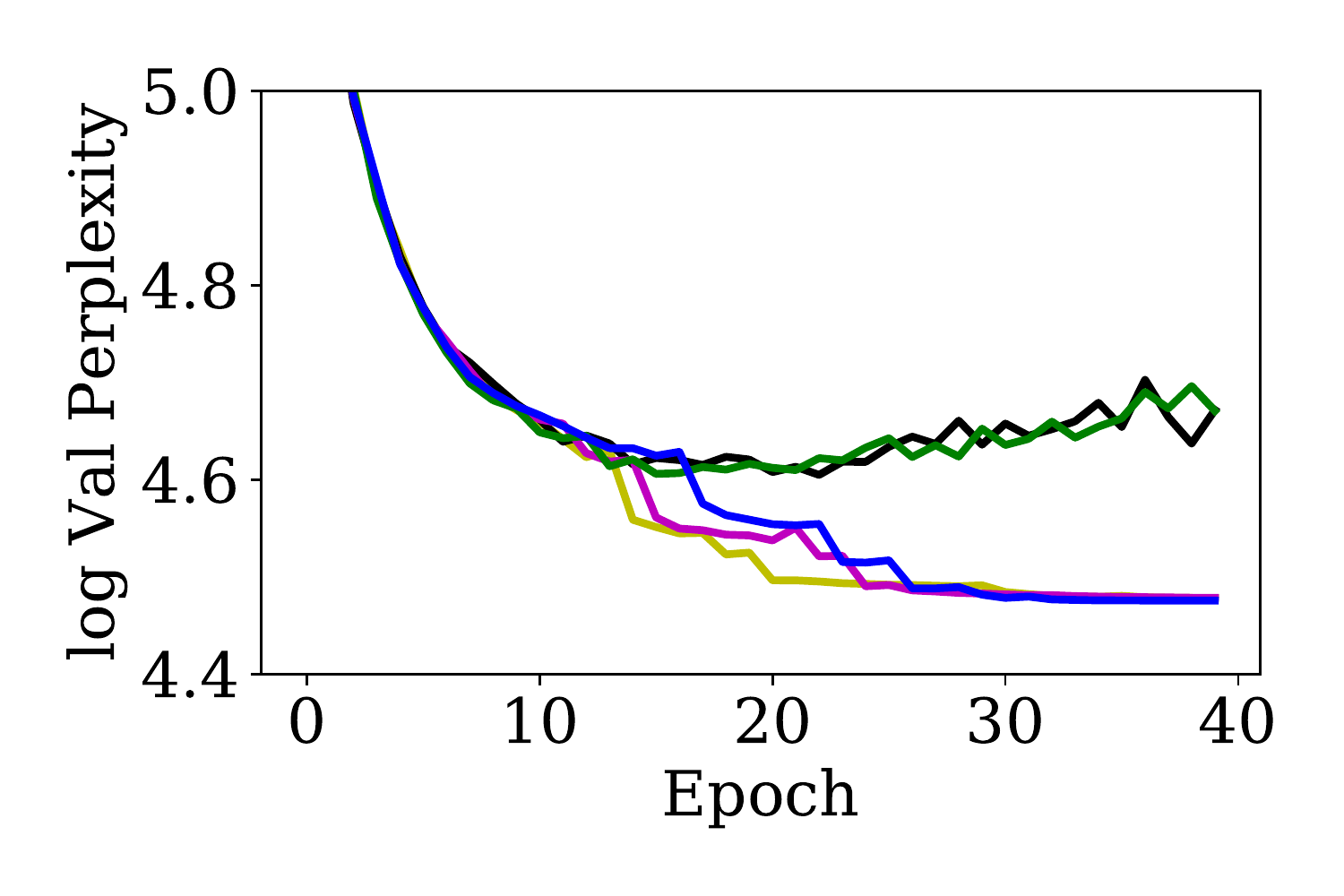}
    \includegraphics[width=0.33\linewidth]{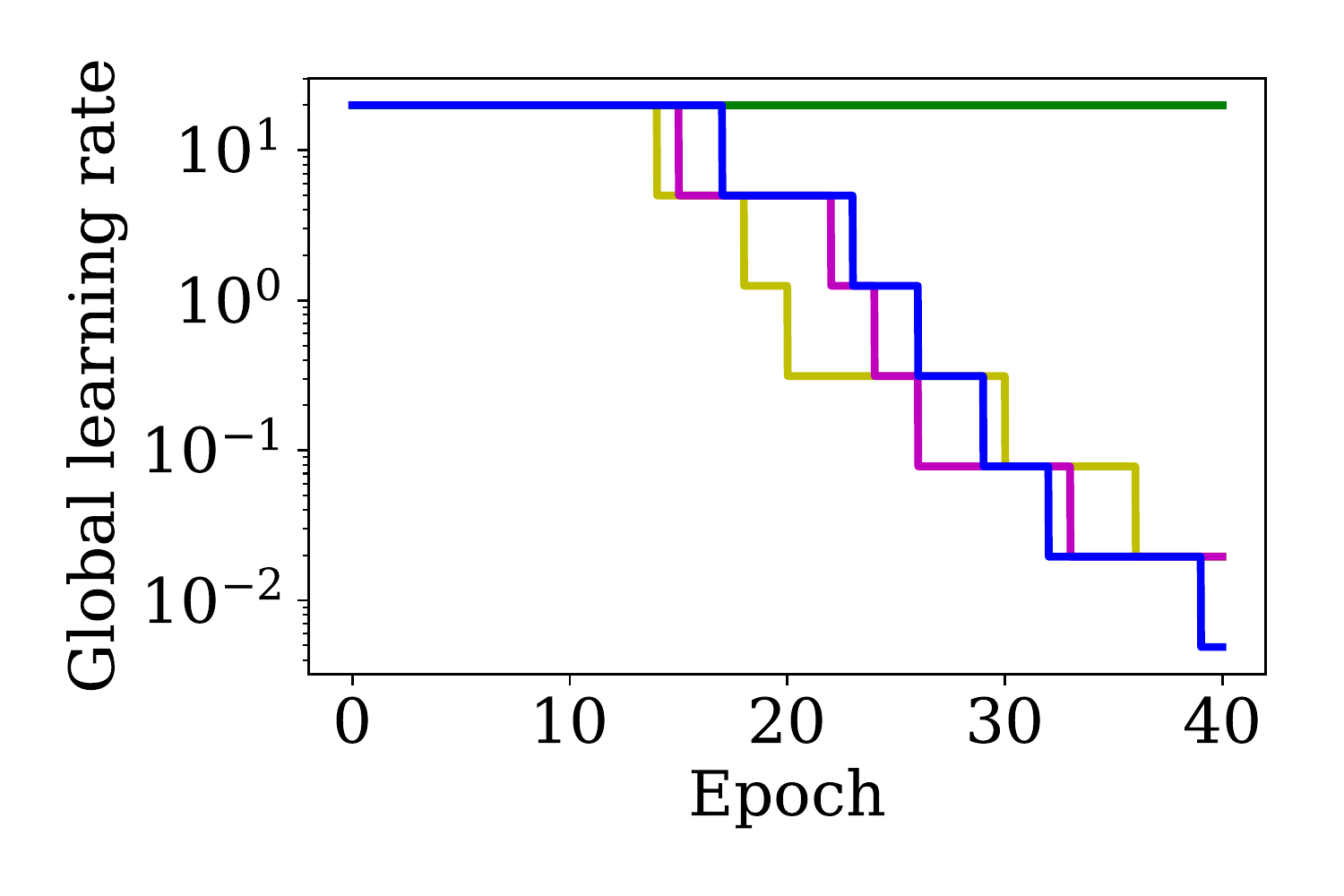}
    \caption{RNN: comparison of the baseline (SGD with validation dataset), SASA+ and SASA. The large LSTM model quickly overfits the Wikitest-2 dataset, which can be seen from the quickly decreasing training loss but increasing validation perplexity for SASA and SASA+. Adding the validation dataset as another learning rate drop criterion is necessary to avoid overfitting during training.}
    \label{fig:rnn_sasa}
\end{figure*}

\subsection{More results for SLOPE test}
\label{apd:sloperesults}
In addition to Figure~\ref{fig:slopevssasa}, Figure~\ref{fig:slopevssasa_app} contains more results to comparing SASA+ and the SLOPE test. The learning rate schedules from SLOPE on CIFAR-10, ImageNet and MNIST are similar: it takes a few epochs for the first drop, and then the learning rate drops exponentially because the SLOPE test always fires (we do the test once every epoch). Due to the aggressive learning rate drop, the final testing performance is not comparable with SASA+ and hand-tuned SHB.  

\begin{figure}[t]
    \centering
	\includegraphics[width=0.32\linewidth]{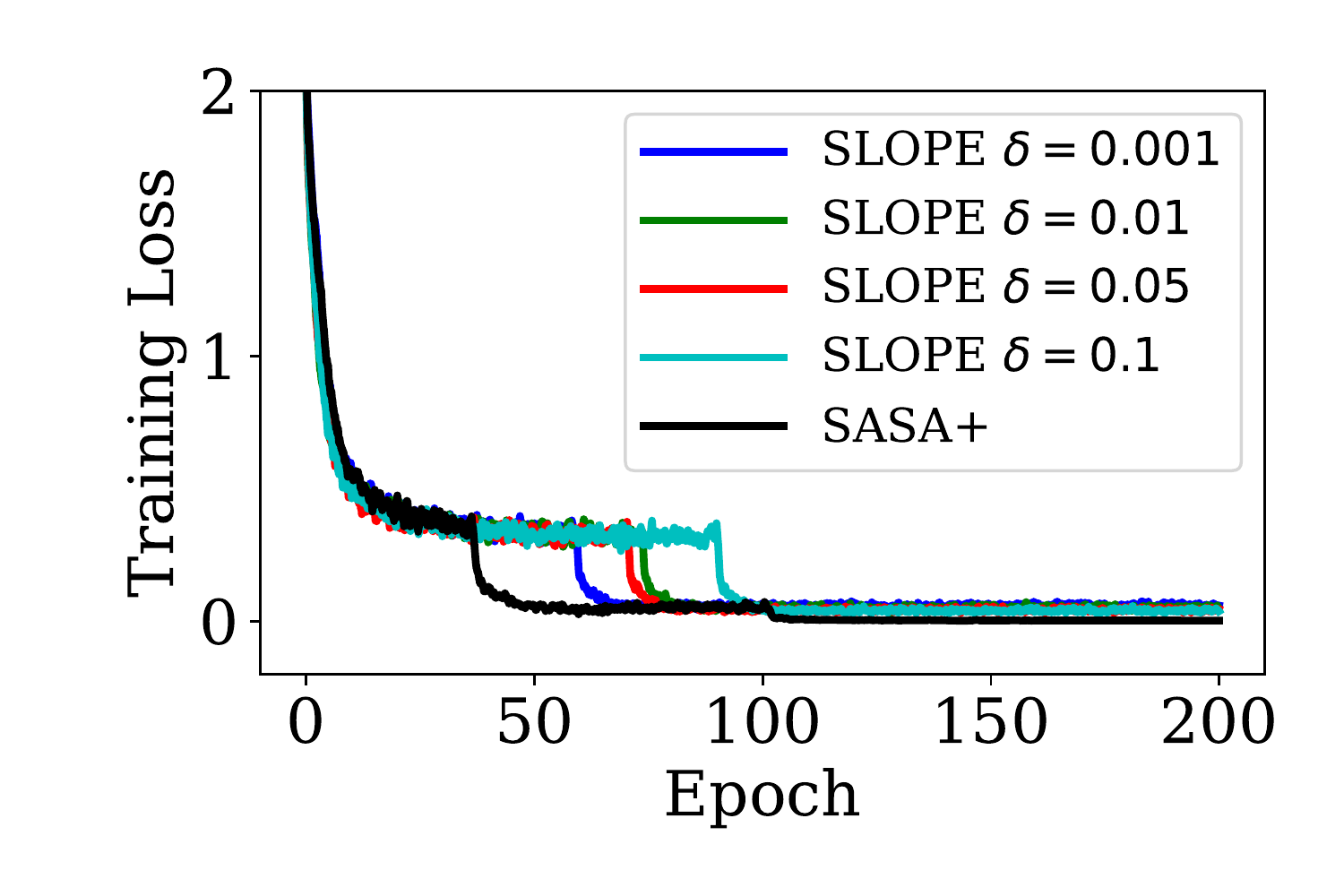}
	\includegraphics[width=0.32\linewidth]{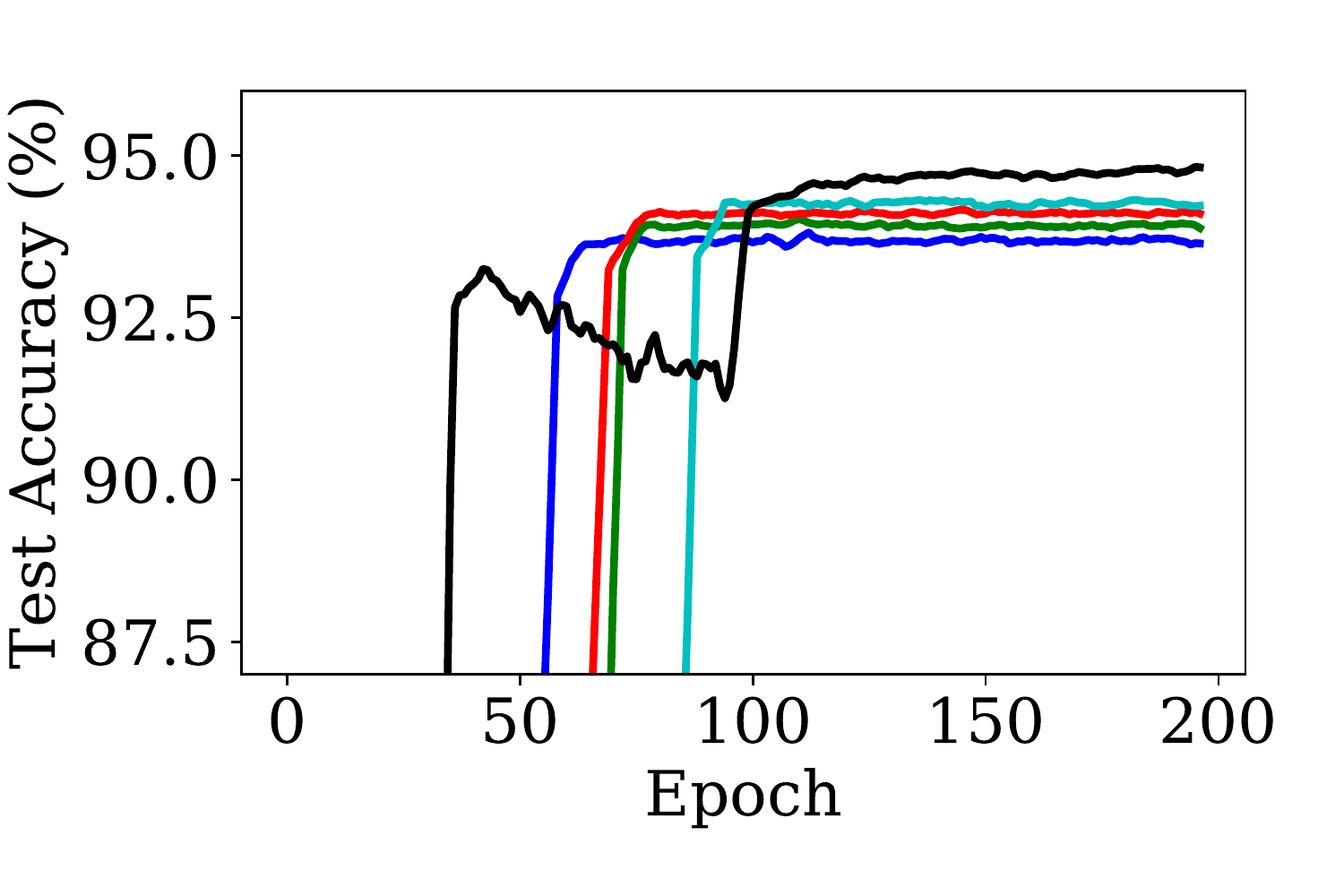}
	\includegraphics[width=0.32\linewidth]{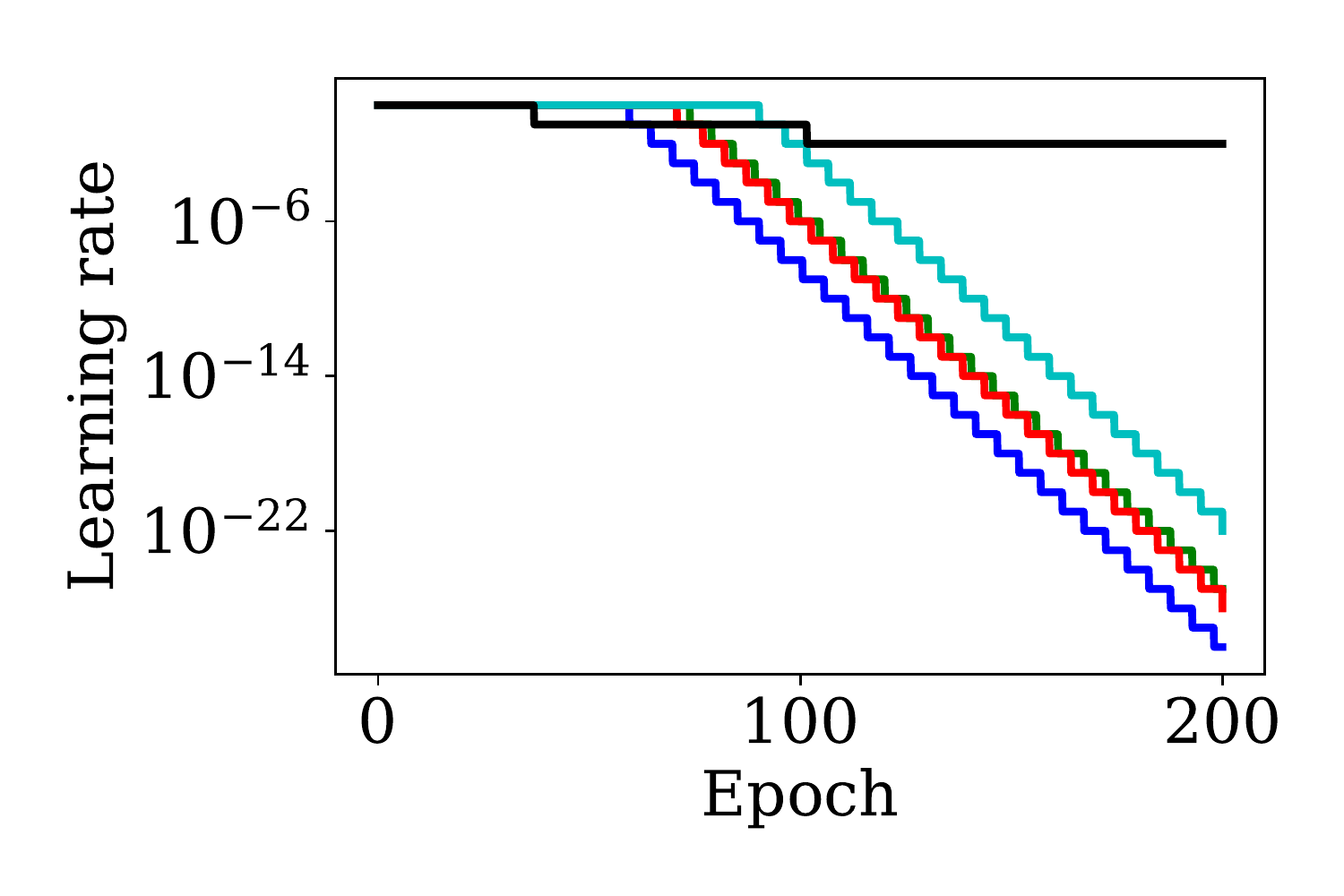}\\
	\includegraphics[width=0.32\linewidth]{archresnet18_trainloss_smooth_slope_leak8.pdf}
	\includegraphics[width=0.32\linewidth]{archresnet18_testacc_slope_leak8.pdf}
	\includegraphics[width=0.32\linewidth]{archresnet18_lrs_slope_leak8.pdf}\\
	\includegraphics[width=0.32\linewidth]{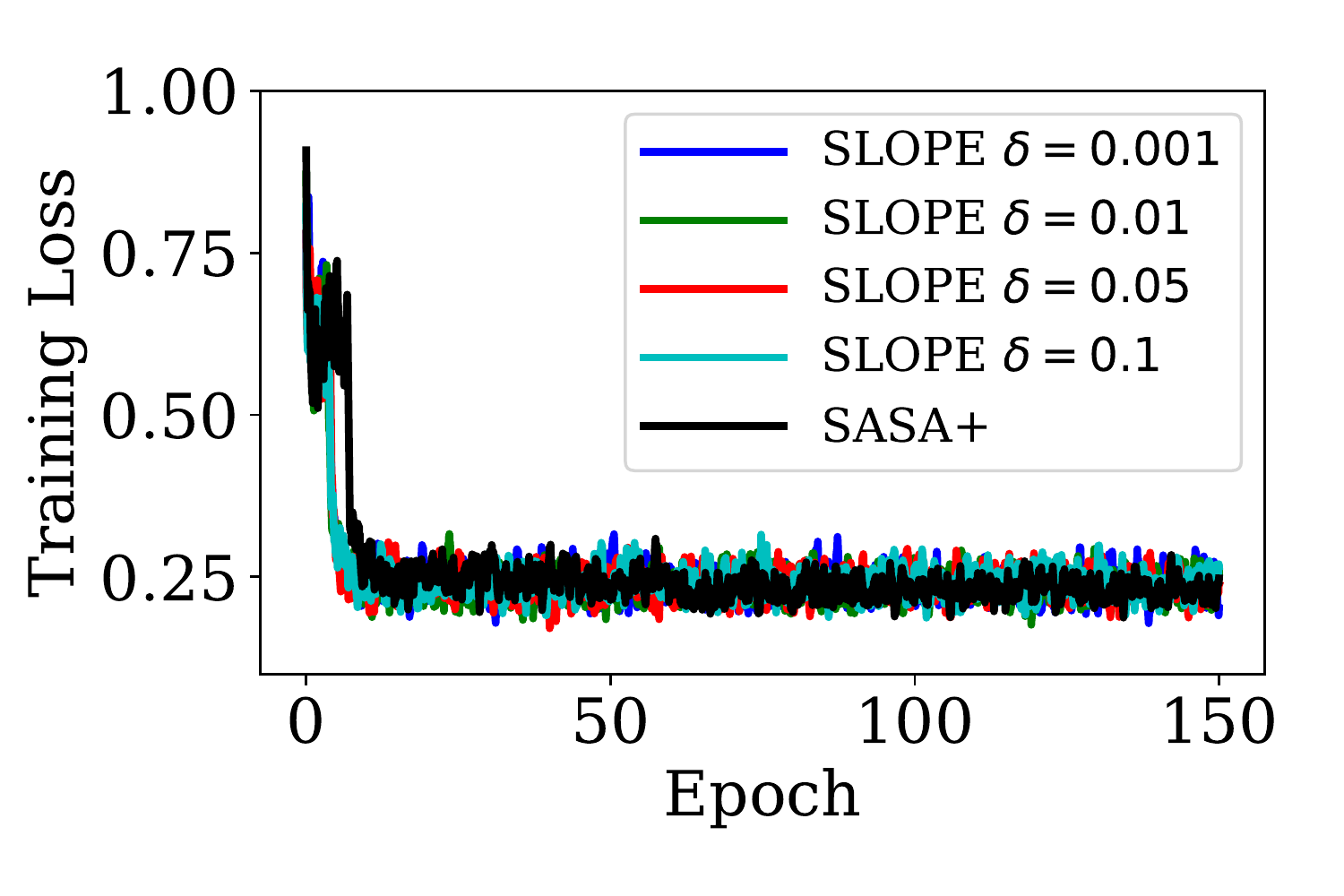}
	\includegraphics[width=0.32\linewidth]{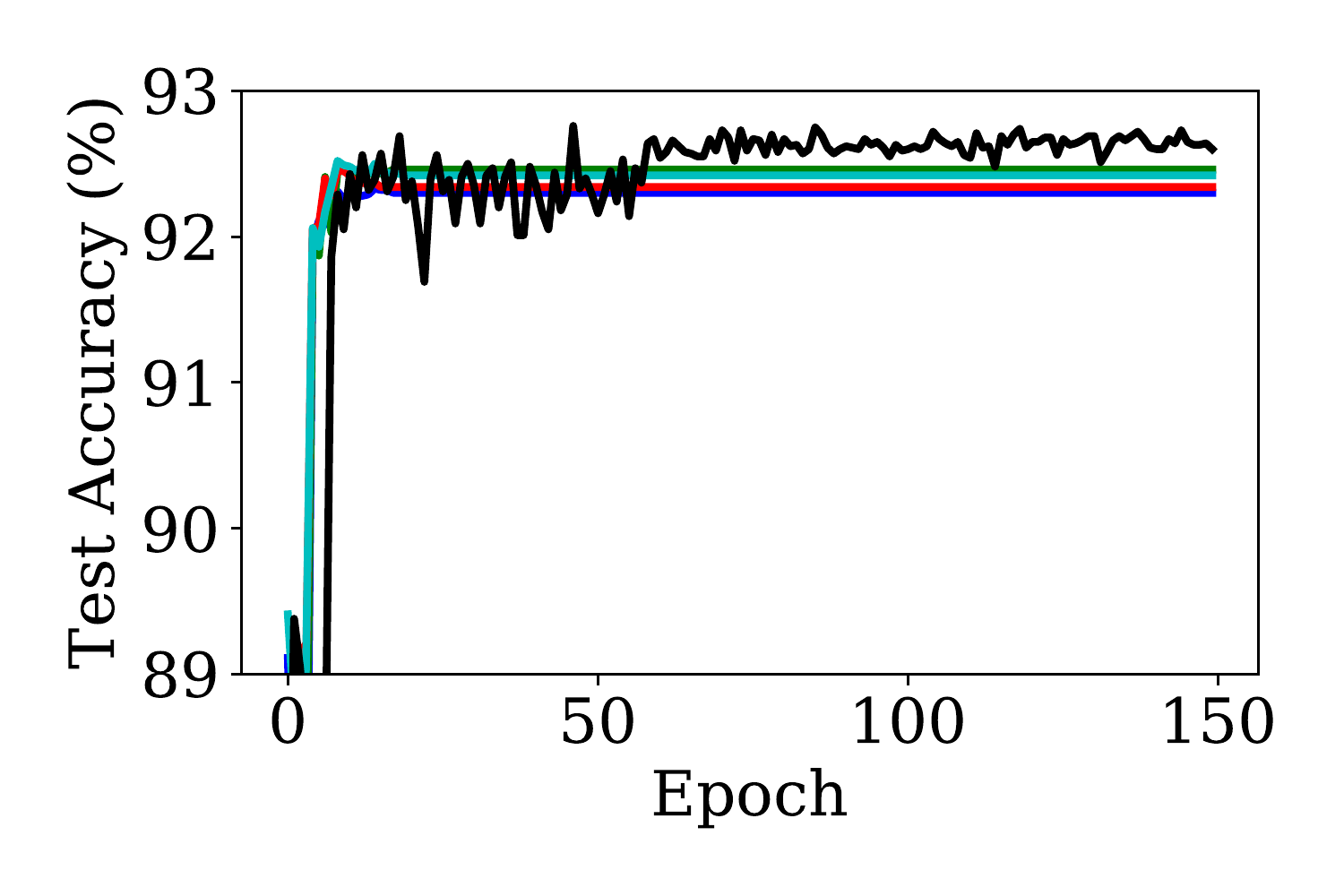}
	\includegraphics[width=0.32\linewidth]{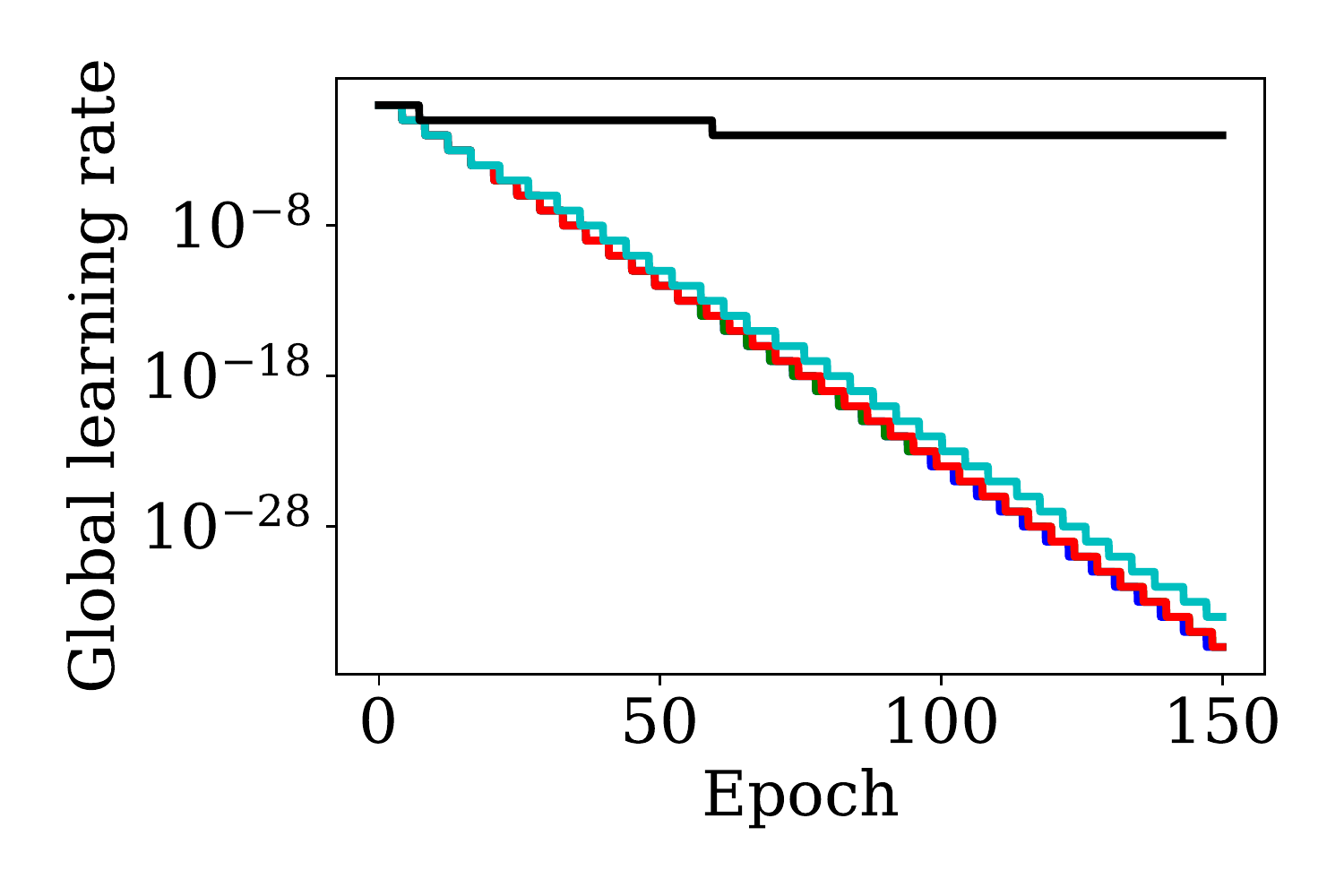}%
	\caption{Comparison of SLOPE test and SASA+. SLOPE test is shown with the same default parameters as SASA+ and with different confidence levels $\delta$. First column: CIFAR-10 with ResNet18. Second column: ImageNet with ResNet18. Third column: MNIST with the linear model.}
	\label{fig:slopevssasa_app}
\end{figure}

\subsection{More results for SSLS}
\label{apd:sslsresults}
In Figure~\ref{fig:sgdsls}, we present the sensitivity analysis of SSLS on CIFAR-10. In Figure~\ref{fig:imagenet_ssls_ablation}, \ref{fig:mnist_ssls_ablation} and \ref{fig:rnn_ssls_ablation}, we show similar results on ImageNet, MNIST and Wikitest-2, respectively. These results confirm that: 
\begin{itemize}
    \item The final stable learning rate obtained by SSLS is robust to the initial learning rate $\alpha_0$. The dynamics starting from different initial learning rates are nearly the same. 
    \item The final stable learning rate obtained by SSLS is robust to the smoothing factor $\gamma$. The smaller $\gamma$ is, the smoother the learning rate schedule is and the slower the process reaches the stable learning rate. Empirically, our recommended default $\gamma = \sqrt{b/n}$ achieves a good trade-off between smoothness of the learning rate schedule and the speed to reach the stable learning rate.
    \item The larger the sufficient decrease constant $c$ is, the smaller the final stable learning rate is. The final stable learning rates from different $c$'s are still at the same order, and are also at the same order of the hand-tuned initial learning rate. 
\end{itemize}
\begin{figure*}[t]
    \includegraphics[width=0.33\linewidth]{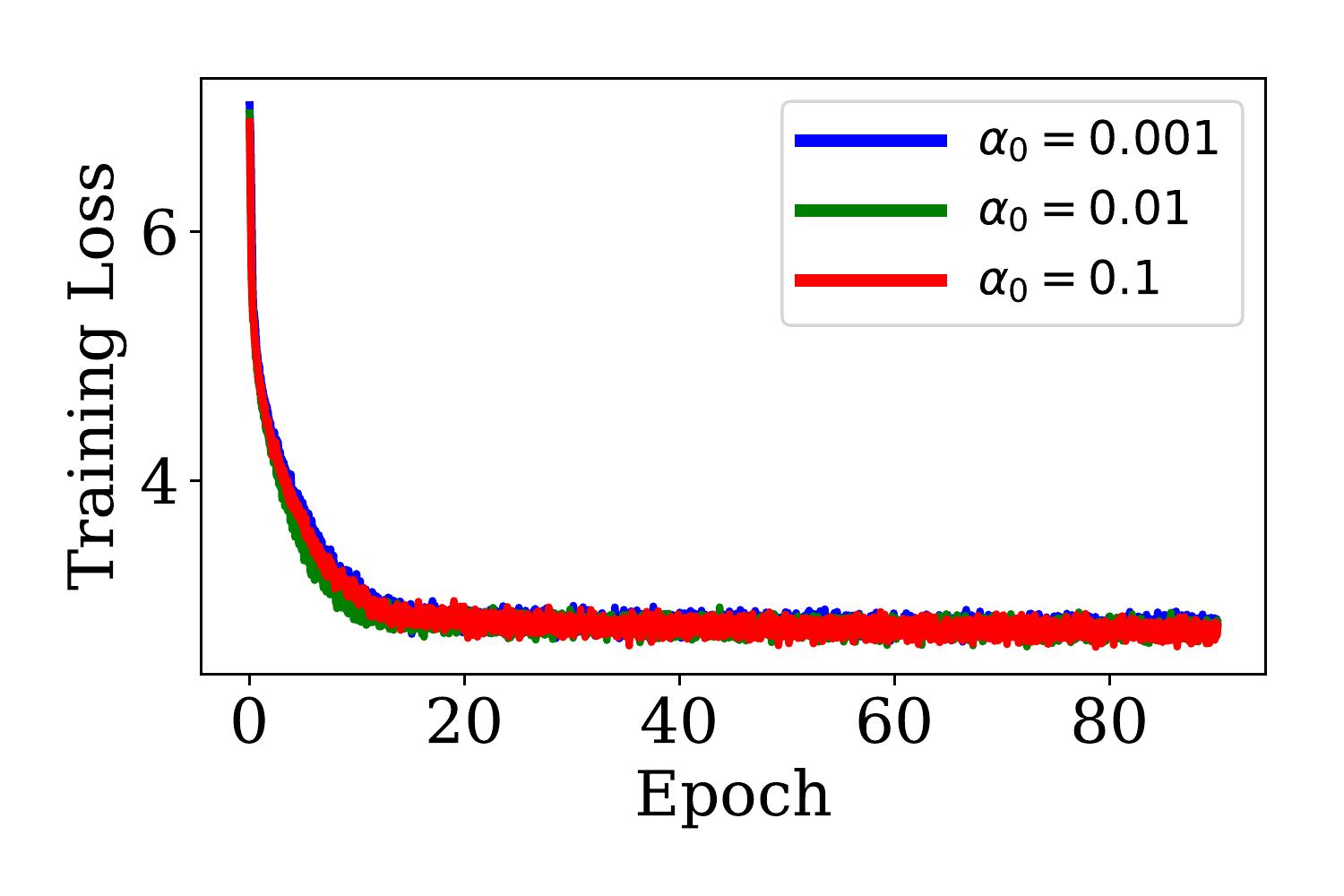}
    \includegraphics[width=0.33\linewidth]{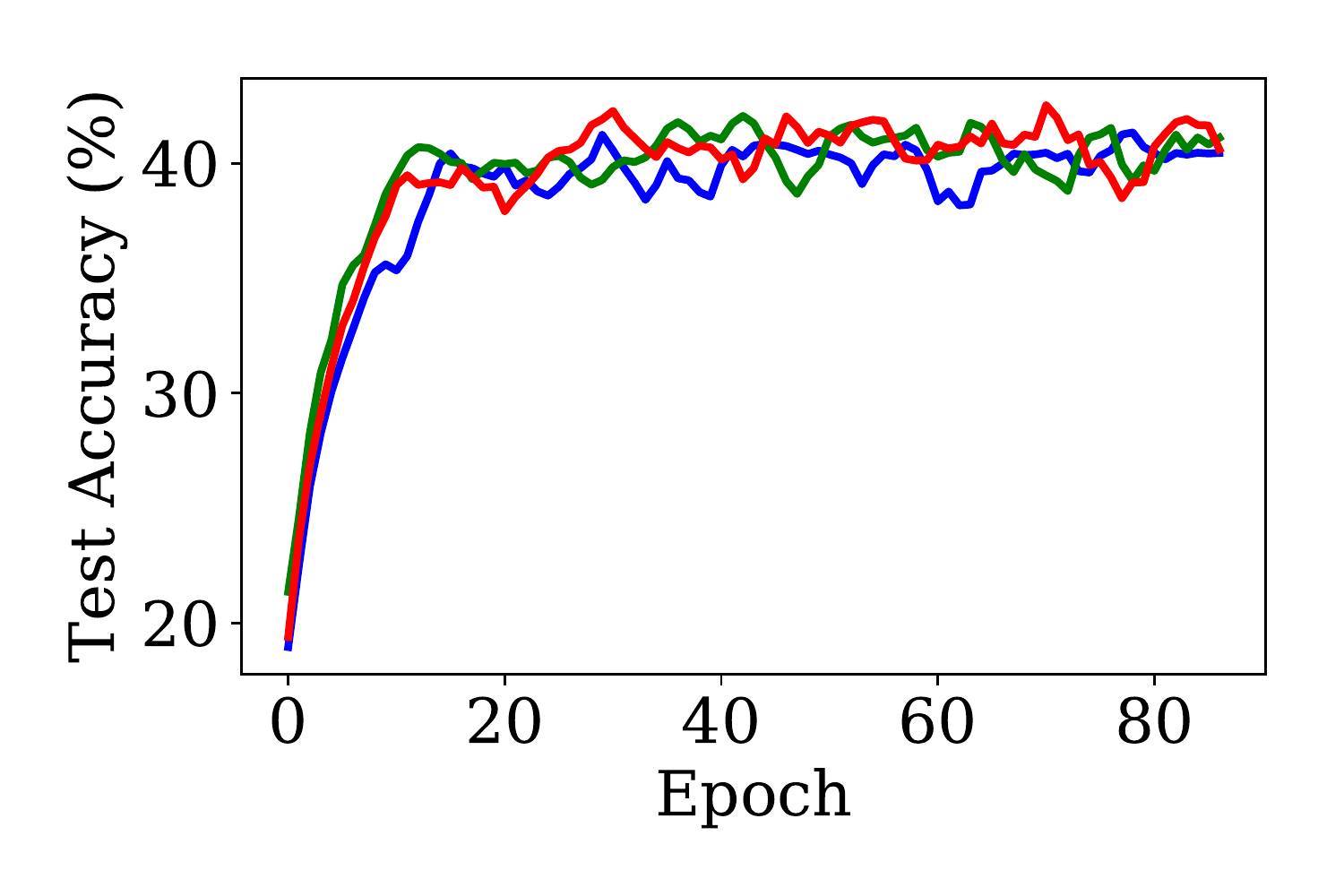}
    \includegraphics[width=0.33\linewidth]{archresnet18_dirg_lrs_sslslrs.pdf} \\
    \includegraphics[width=0.33\linewidth]{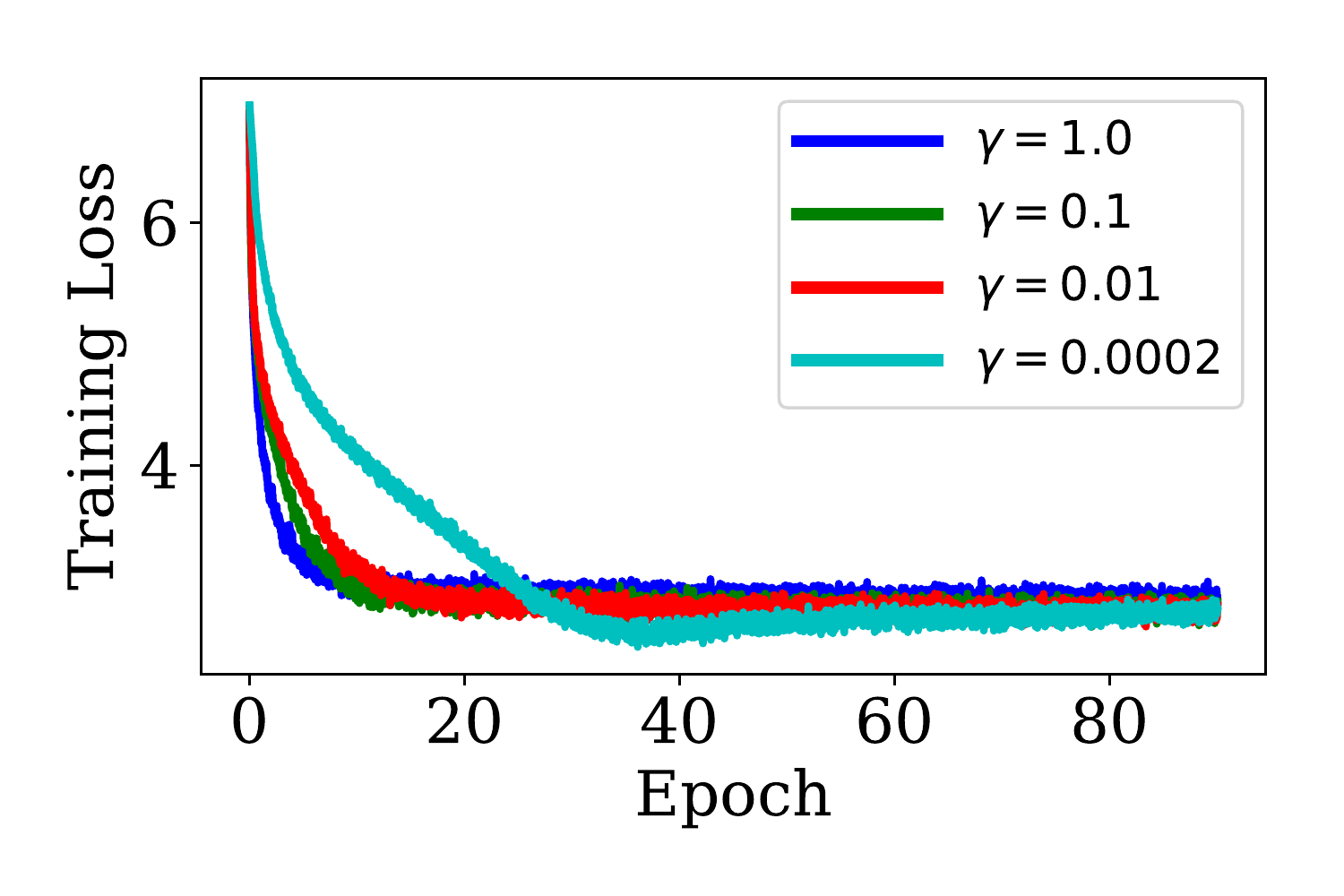}
    \includegraphics[width=0.33\linewidth]{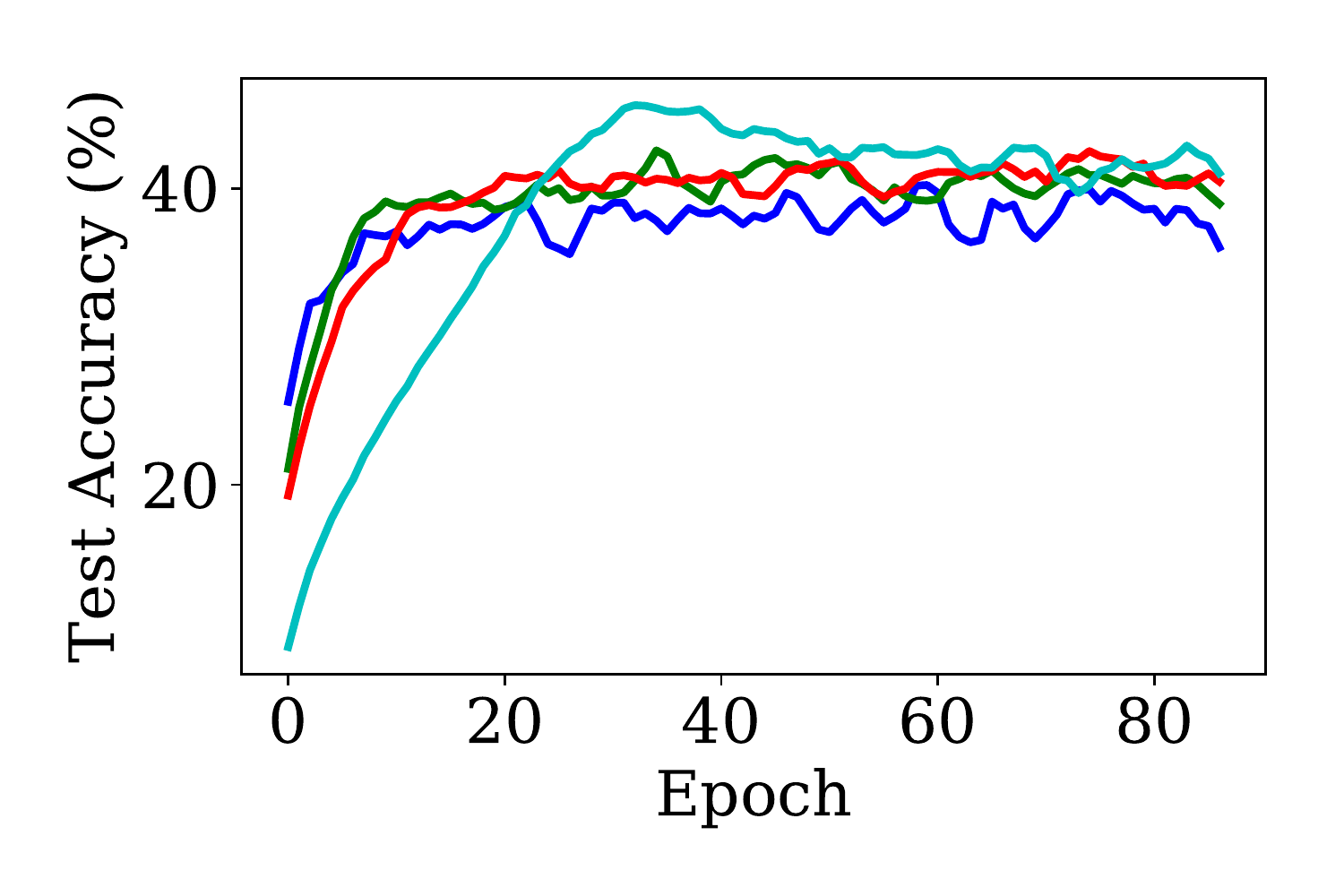}
    \includegraphics[width=0.33\linewidth]{archresnet18_dirg_lrs_sslsgamma.pdf} \\
    \includegraphics[width=0.33\linewidth]{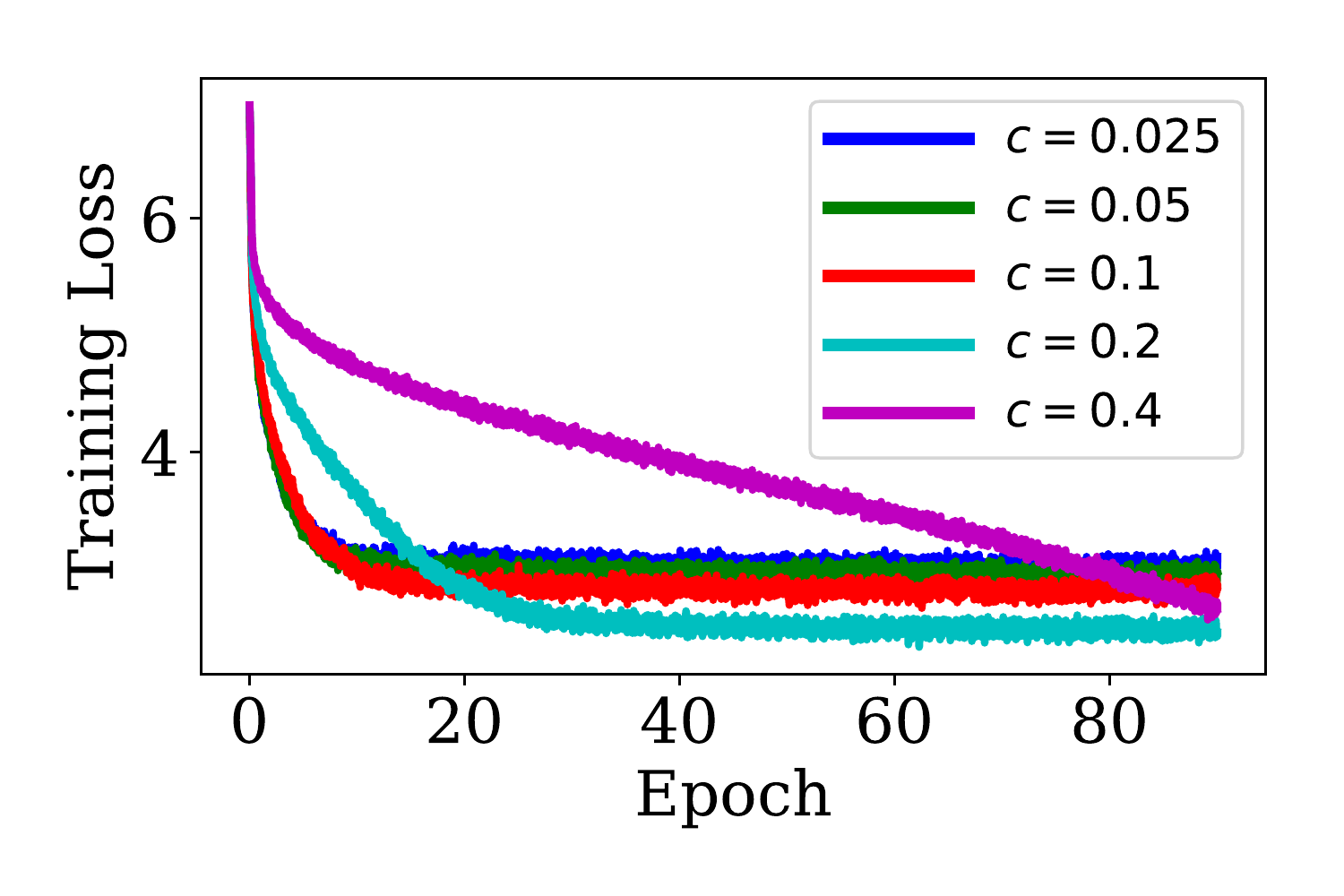}
    \includegraphics[width=0.33\linewidth]{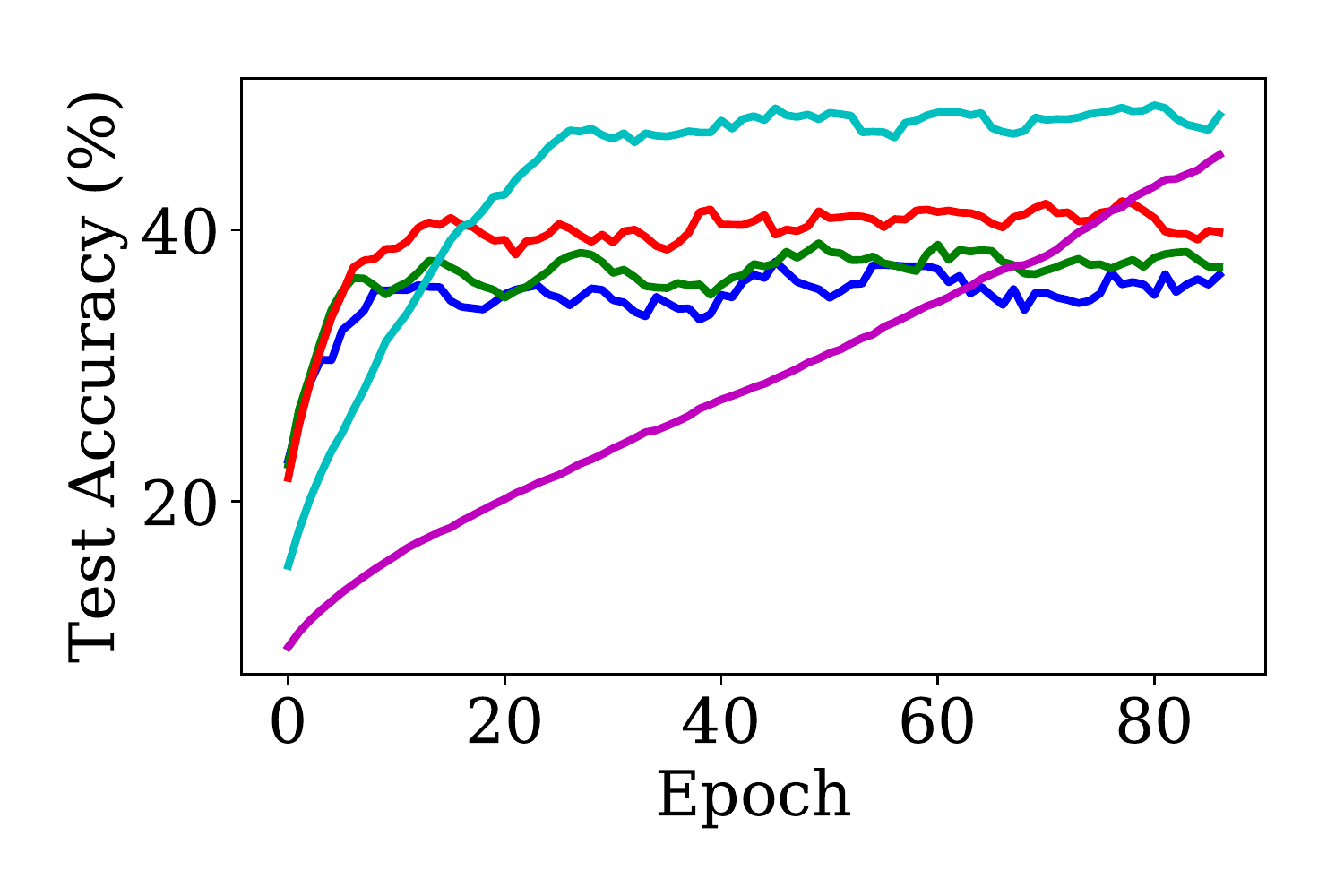}
    \includegraphics[width=0.33\linewidth]{archresnet18_dirg_lrs_sslssdcs.pdf}
    \caption{ImageNet: Sensitivity of SSLS to the initial learning rate $\alpha_0$ (first row), the smoothing factor $\gamma$ (second row) and the sufficient decrease constant $c$ (third row).}
    \label{fig:imagenet_ssls_ablation}
\end{figure*}

\begin{figure*}[t]
    \includegraphics[width=0.33\linewidth]{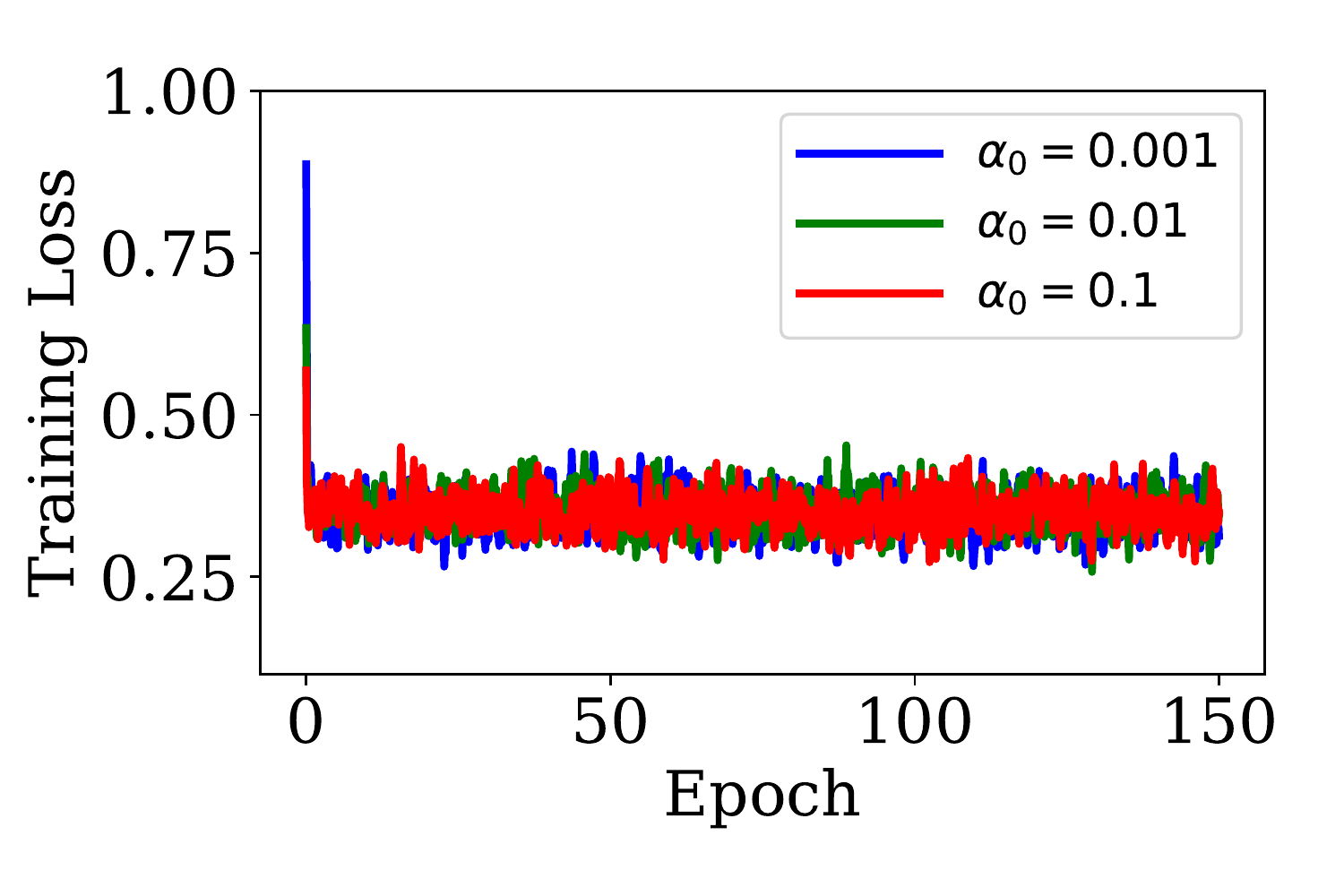}
    \includegraphics[width=0.33\linewidth]{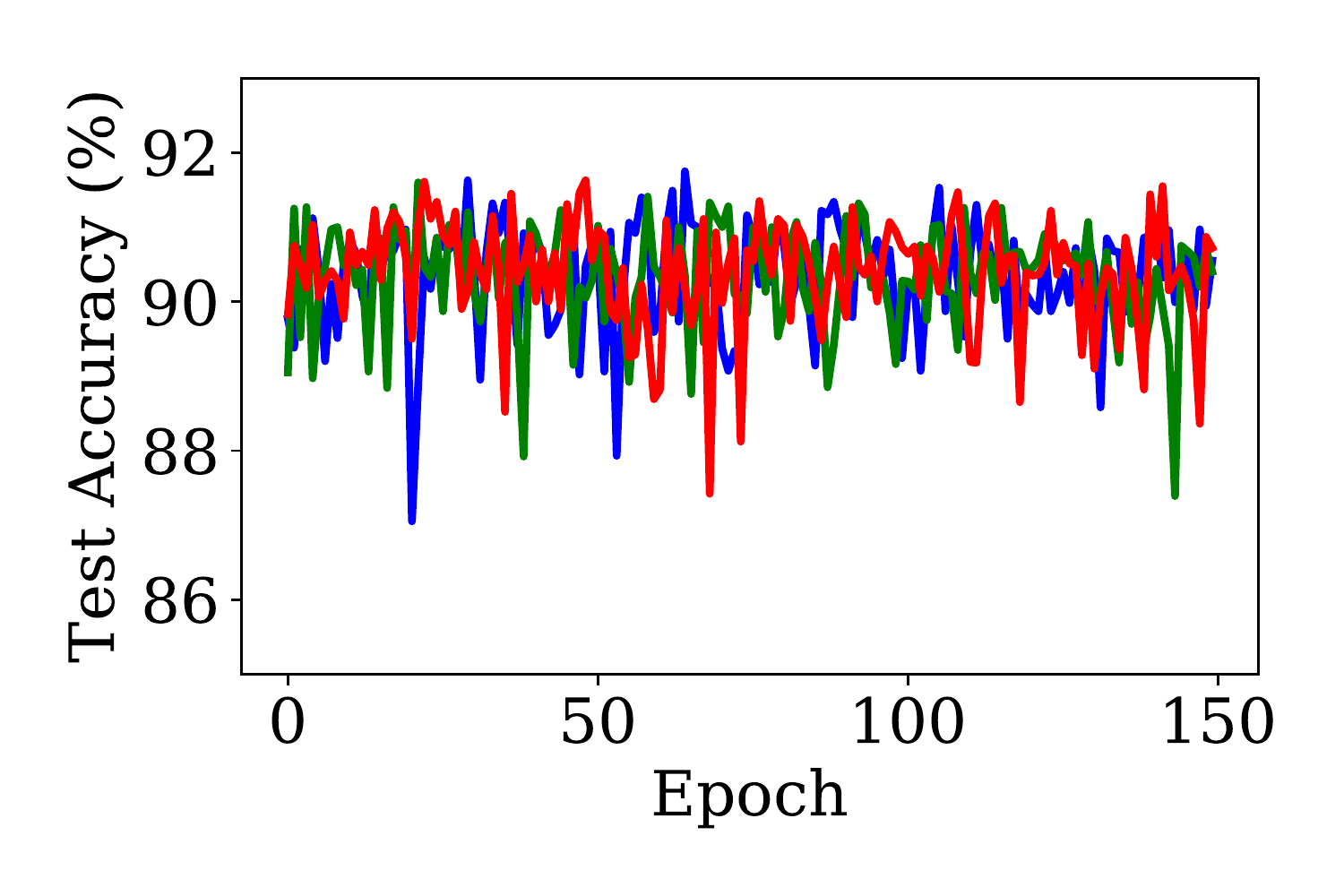}
    \includegraphics[width=0.33\linewidth]{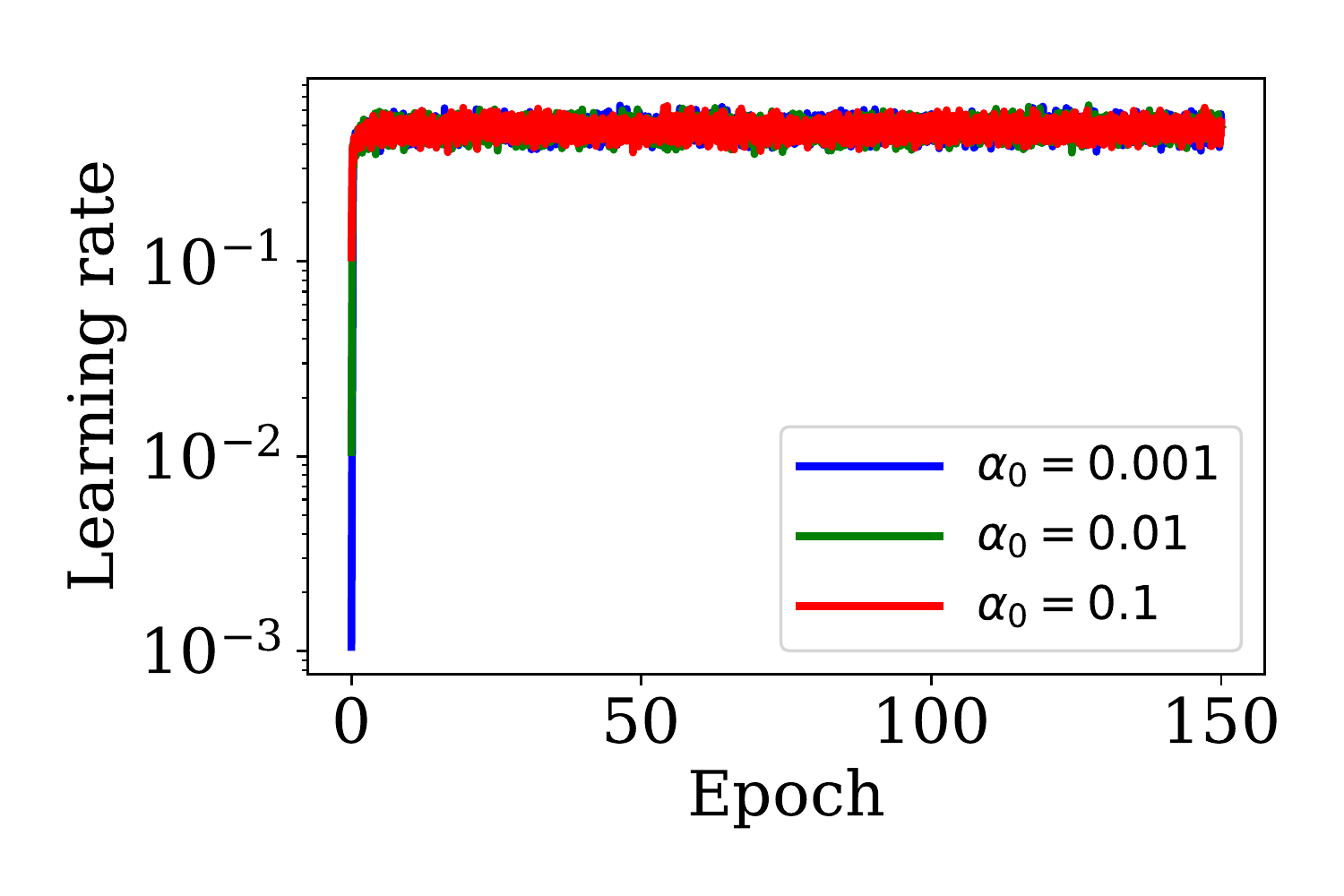} \\
    \includegraphics[width=0.33\linewidth]{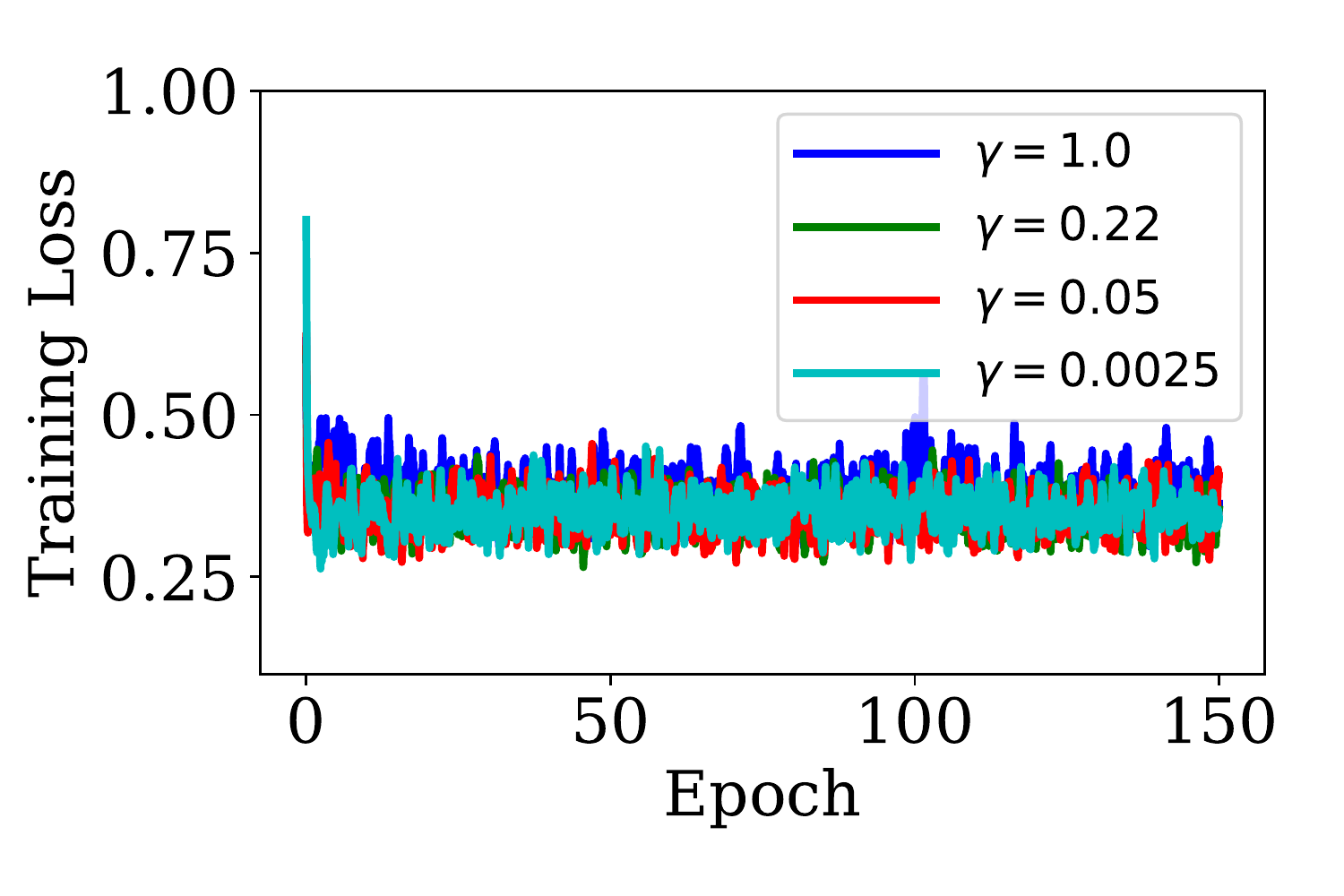}
    \includegraphics[width=0.33\linewidth]{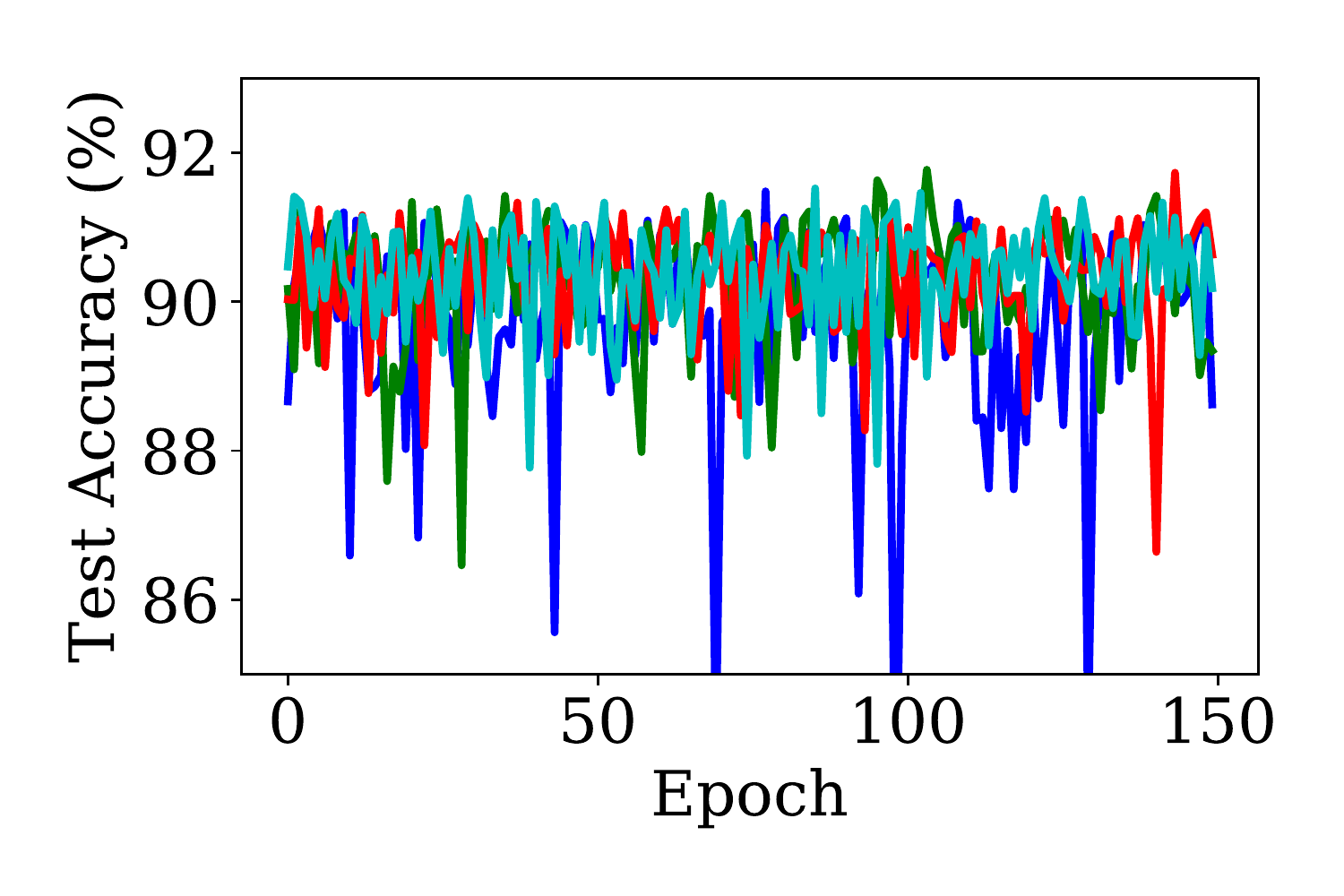}
    \includegraphics[width=0.33\linewidth]{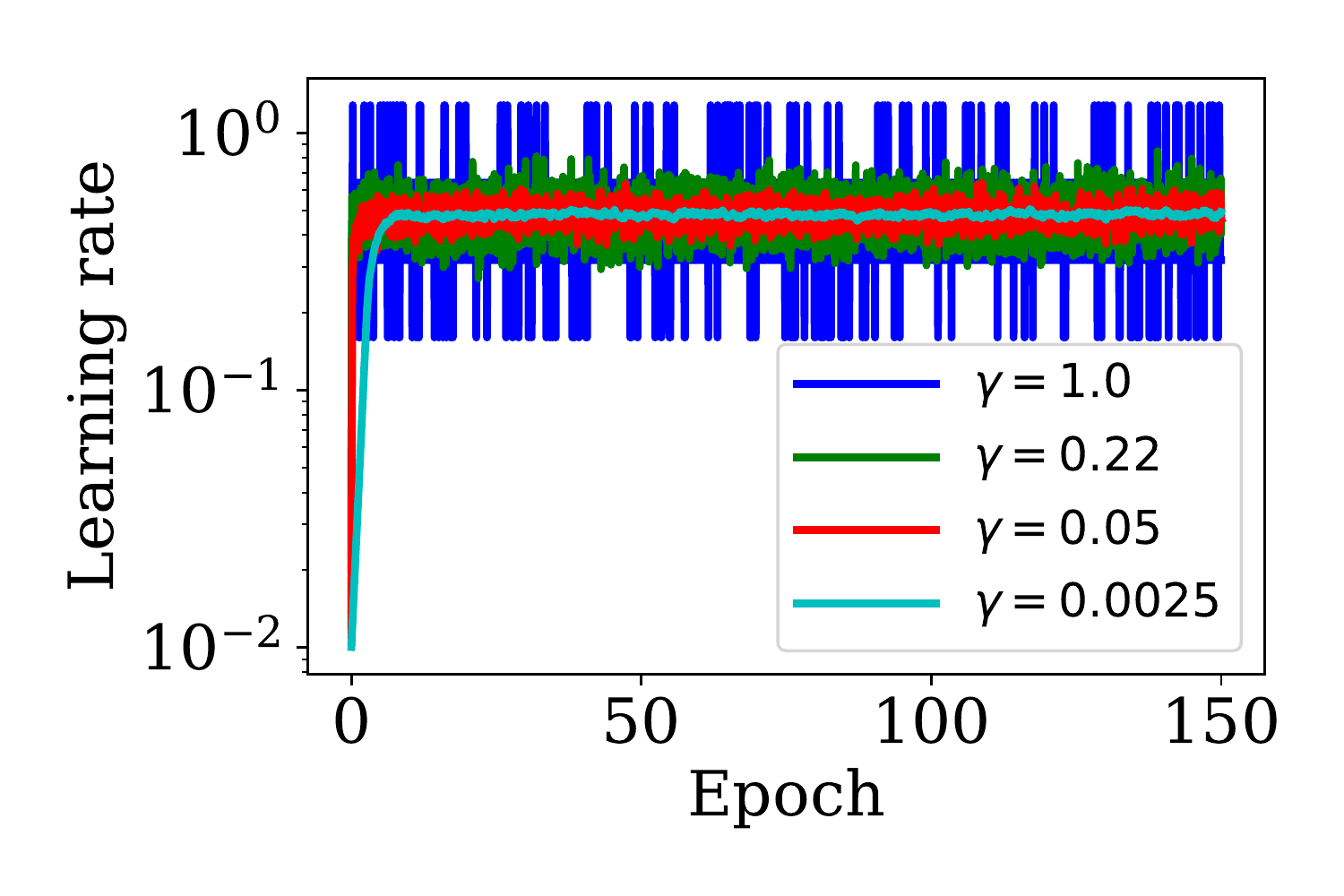} \\
    \includegraphics[width=0.33\linewidth]{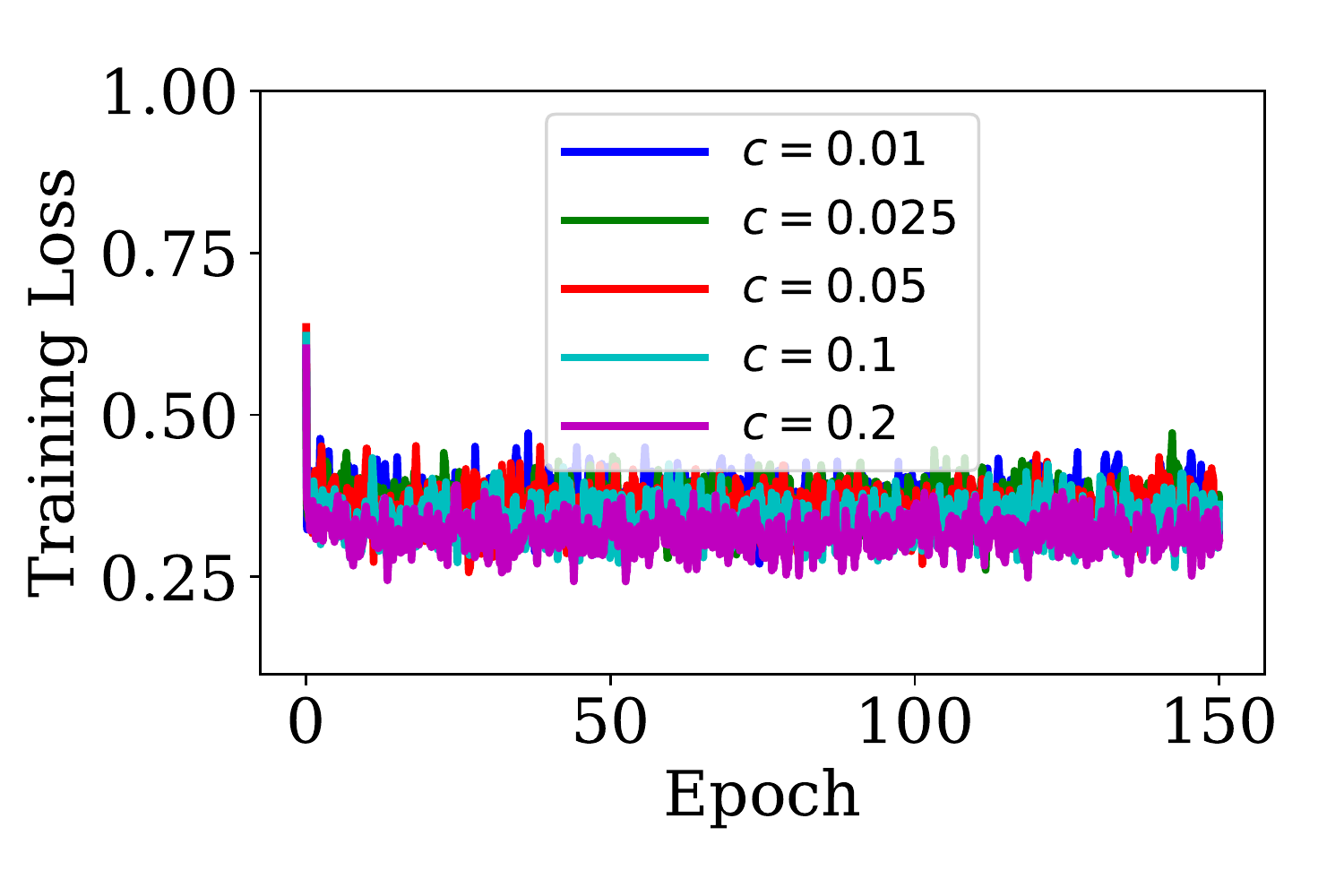}
    \includegraphics[width=0.33\linewidth]{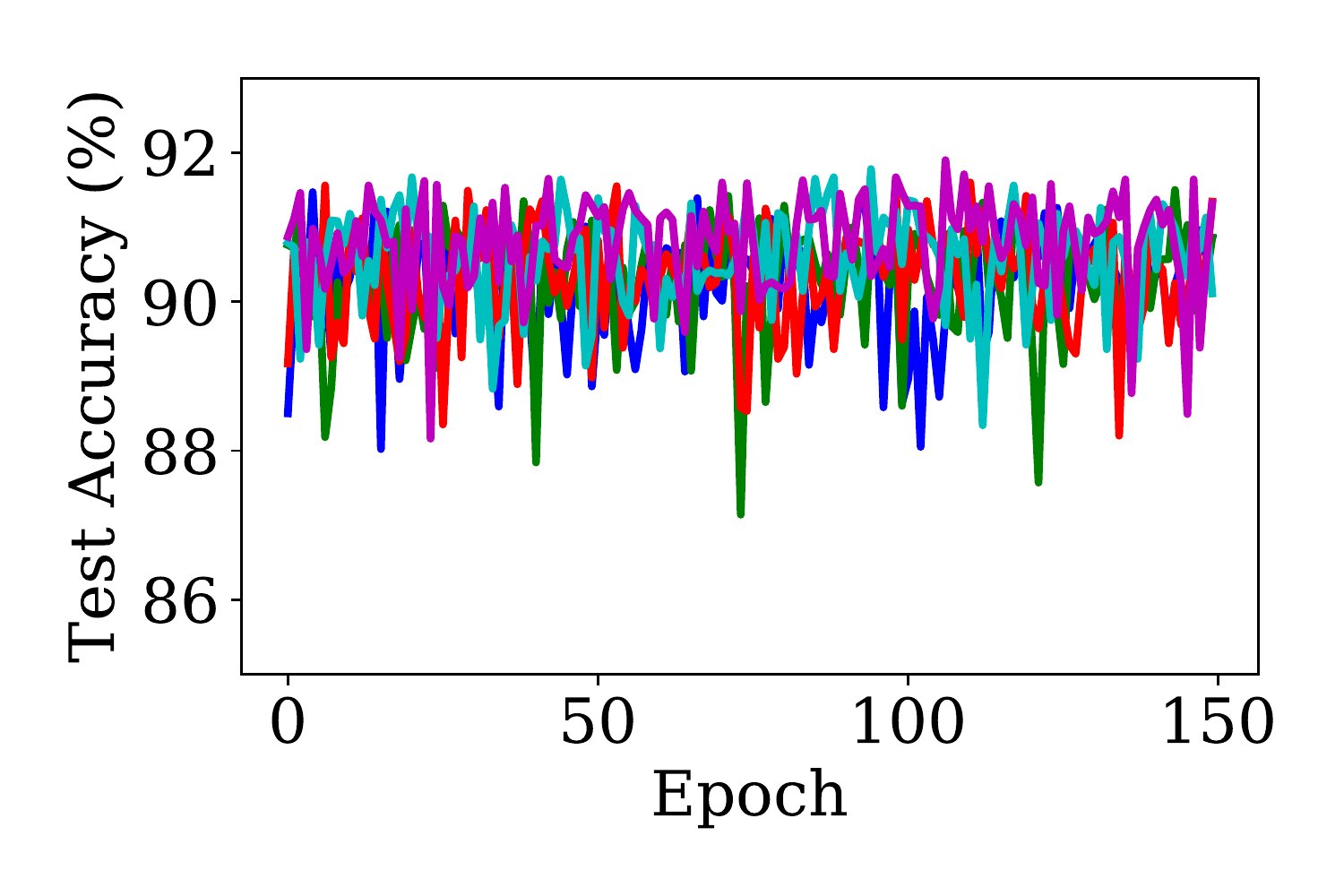}
    \includegraphics[width=0.33\linewidth]{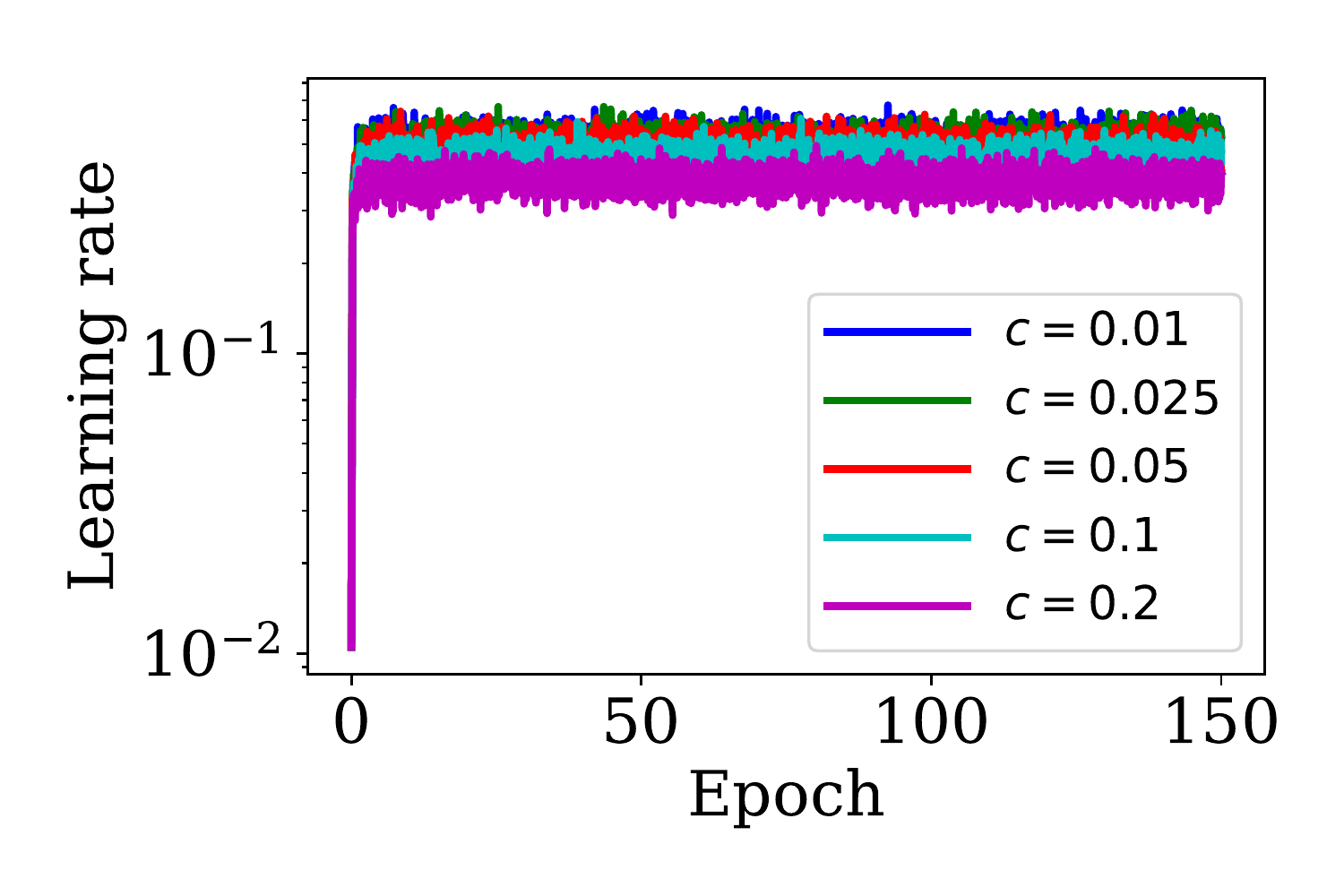}
    \caption{MNIST: Sensitivity of SSLS to the initial learning rate $\alpha_0$ (first row), the smoothing factor $\gamma$ (second row) and the sufficient decrease constant $c$ (third row).}
    \label{fig:mnist_ssls_ablation}
\end{figure*}

\begin{figure*}[t]
    \includegraphics[width=0.33\linewidth]{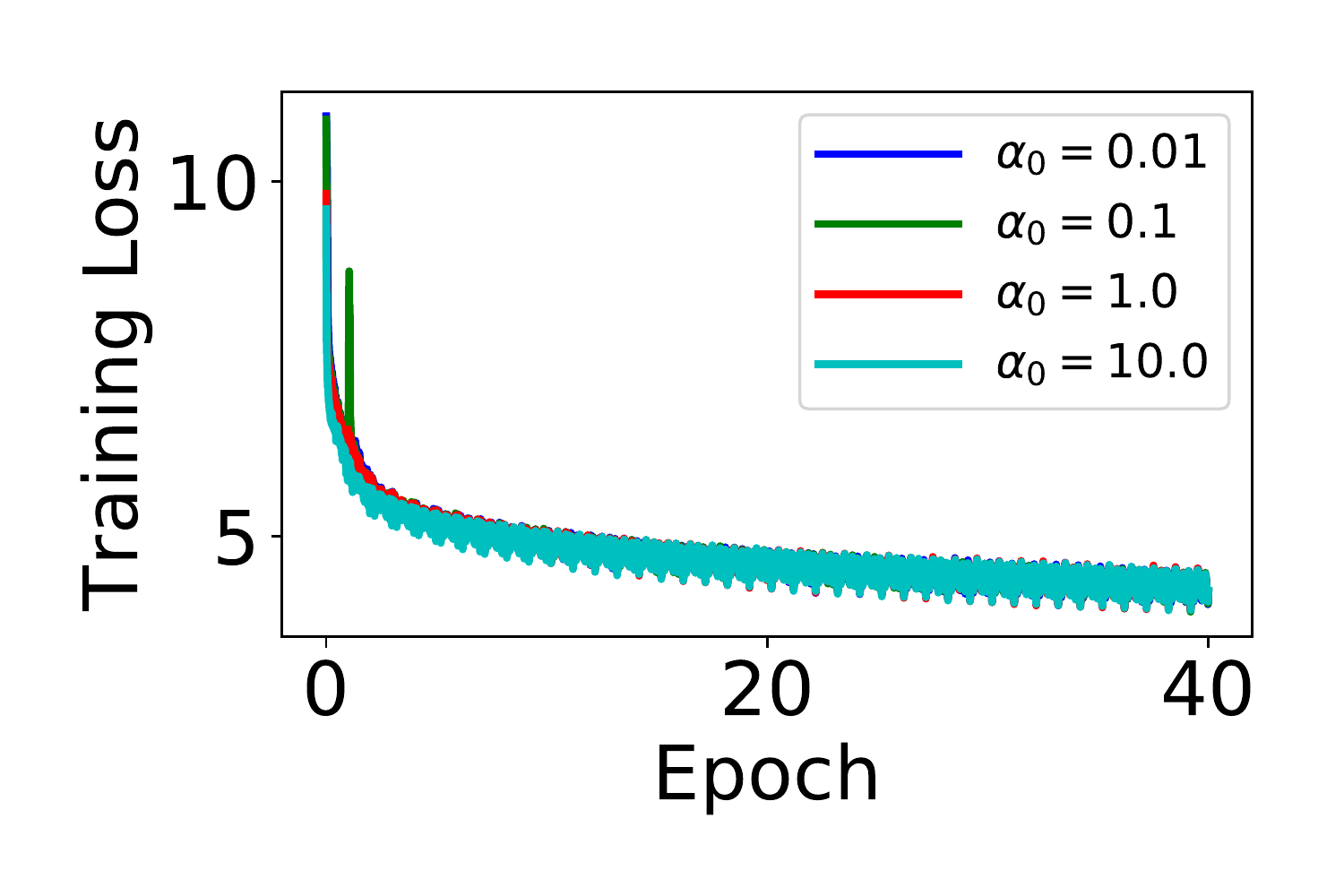}
    \includegraphics[width=0.33\linewidth]{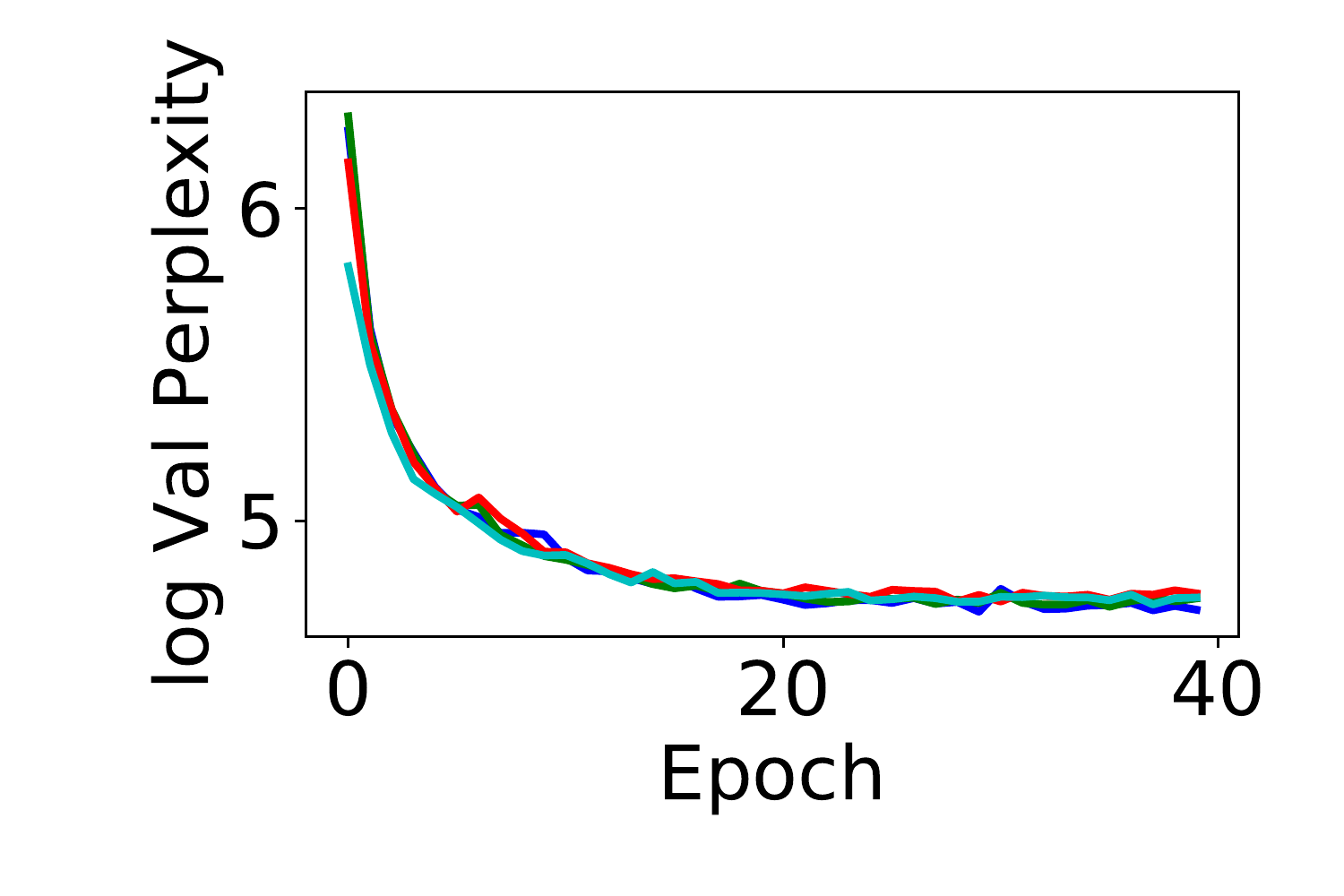}
    \includegraphics[width=0.33\linewidth]{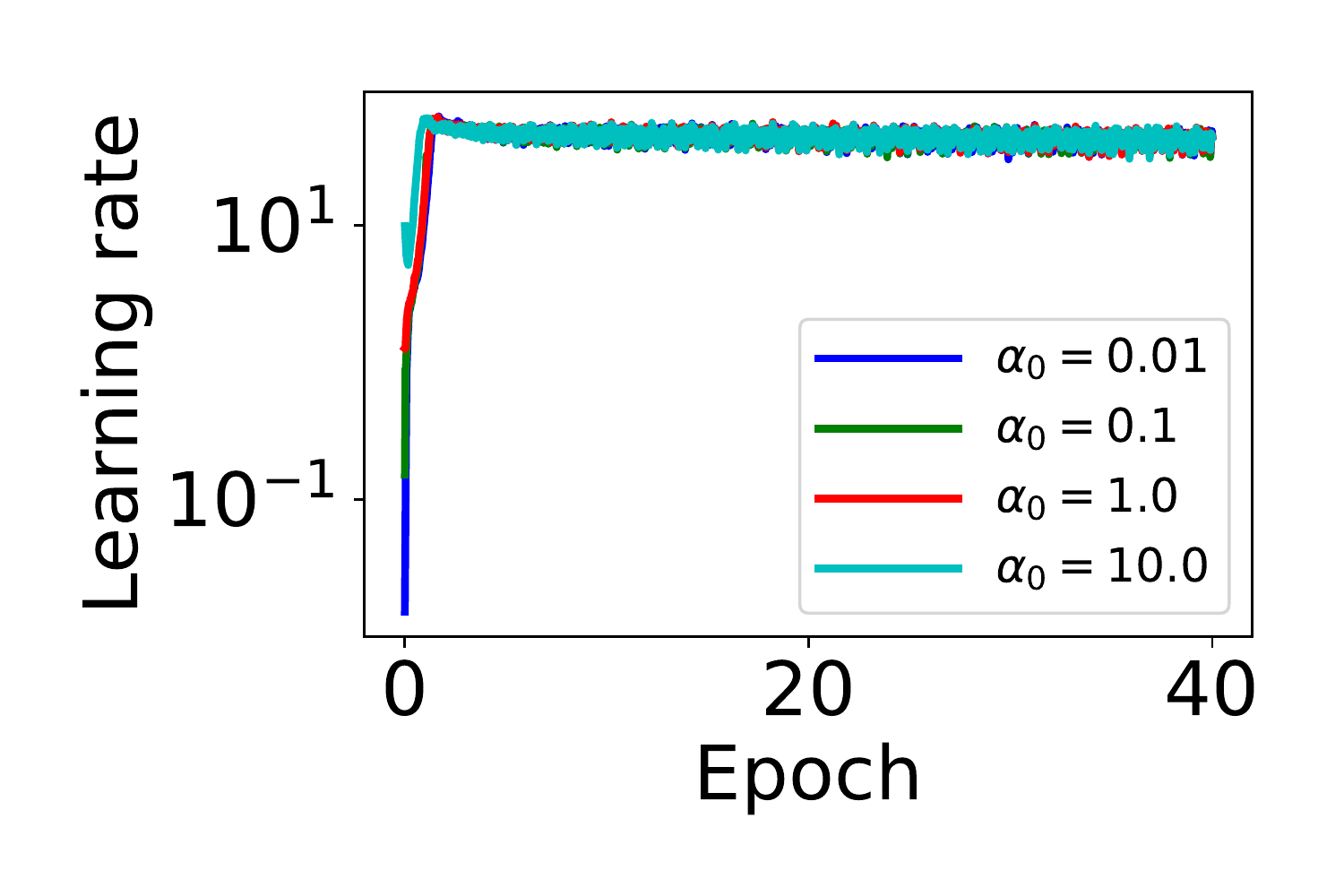} \\
    \includegraphics[width=0.33\linewidth]{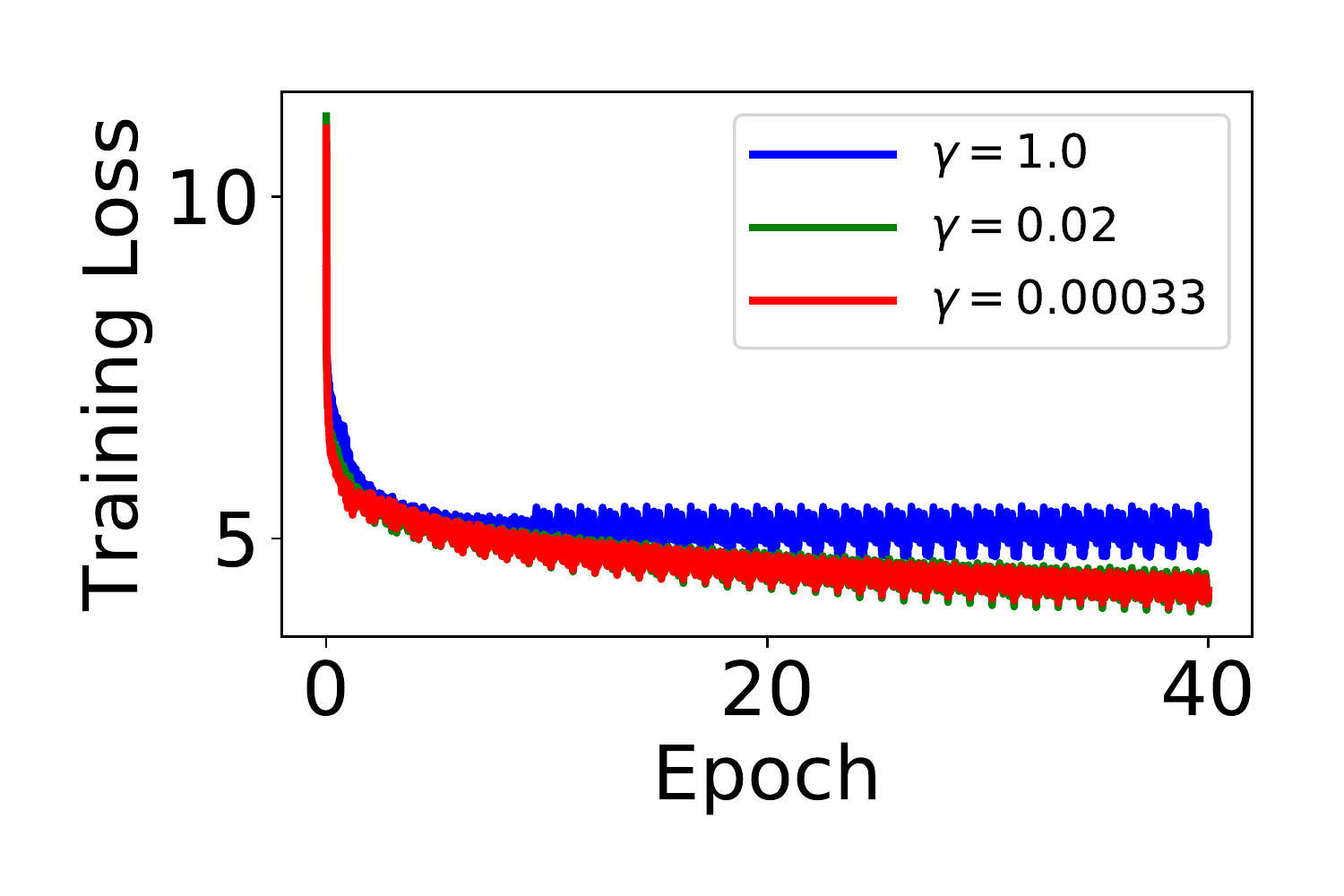}
    \includegraphics[width=0.33\linewidth]{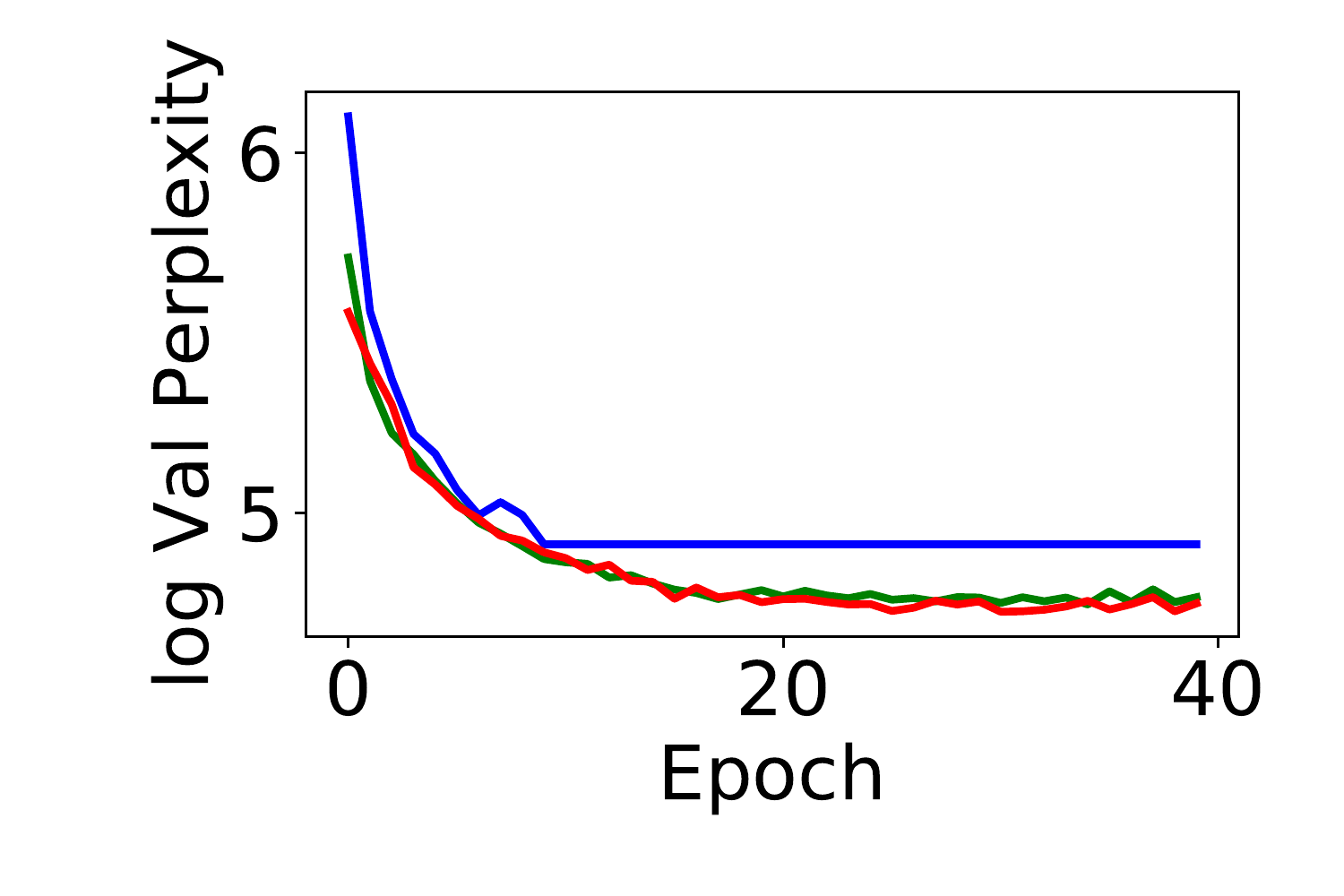}
    \includegraphics[width=0.33\linewidth]{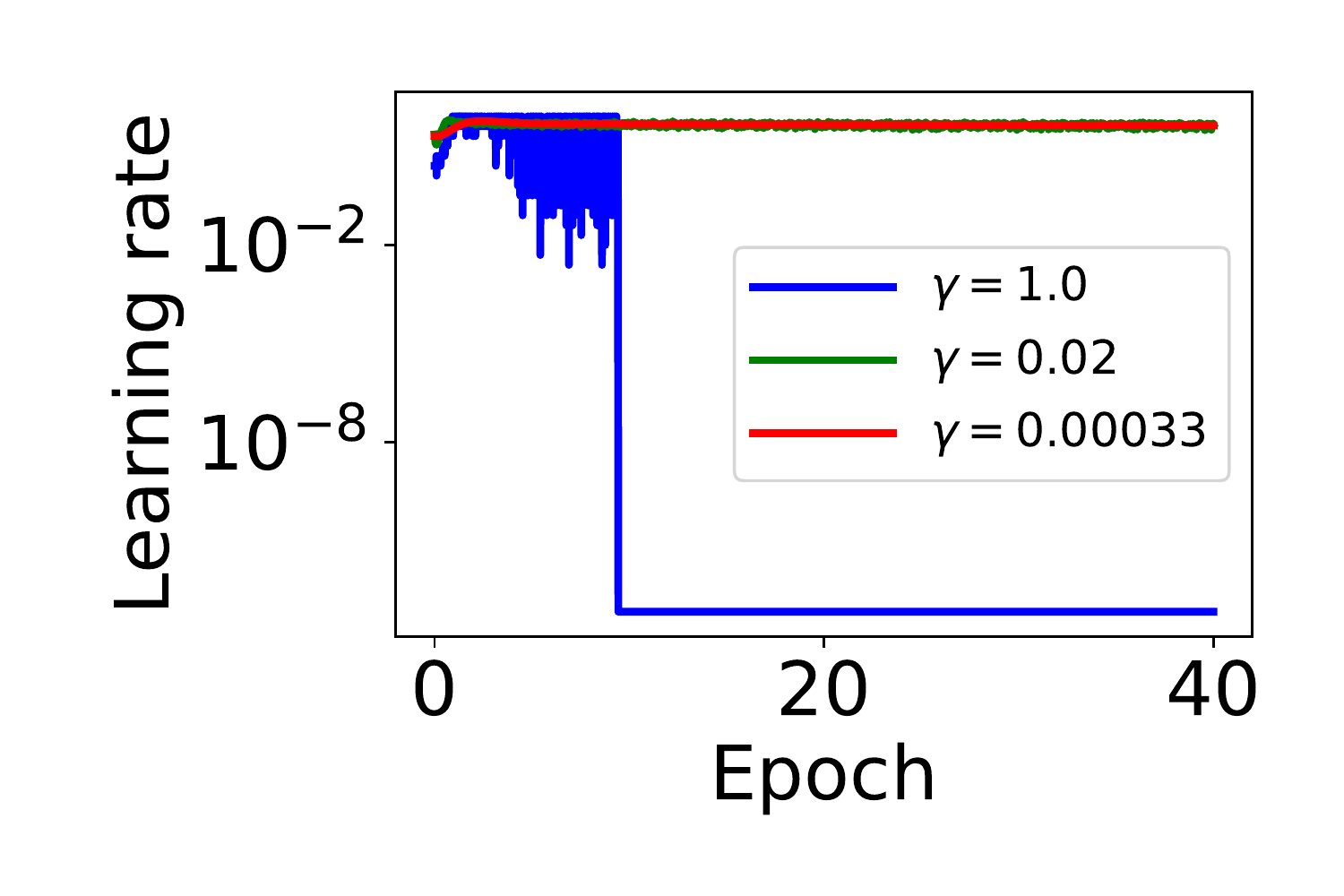} \\
    \includegraphics[width=0.33\linewidth]{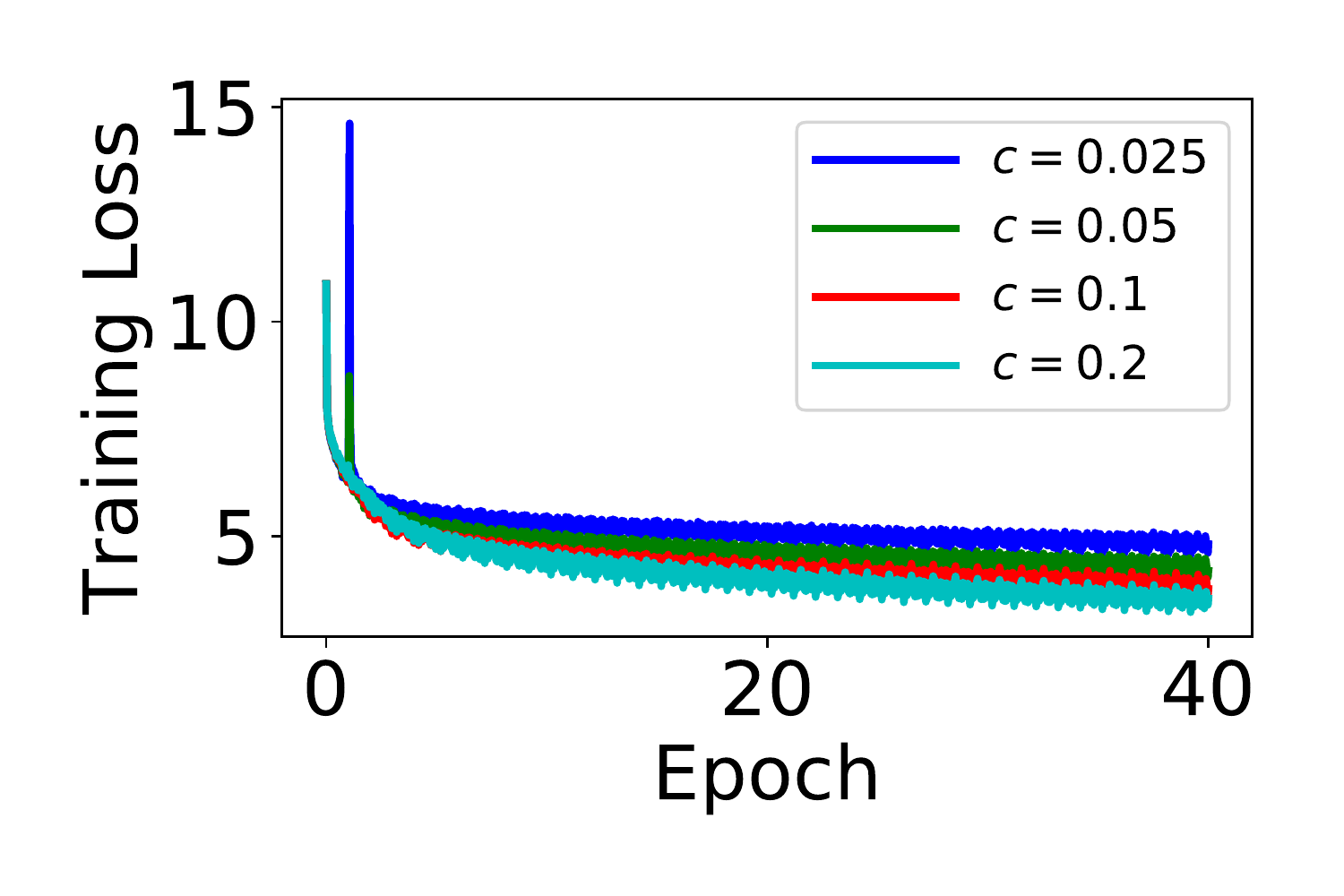}
    \includegraphics[width=0.33\linewidth]{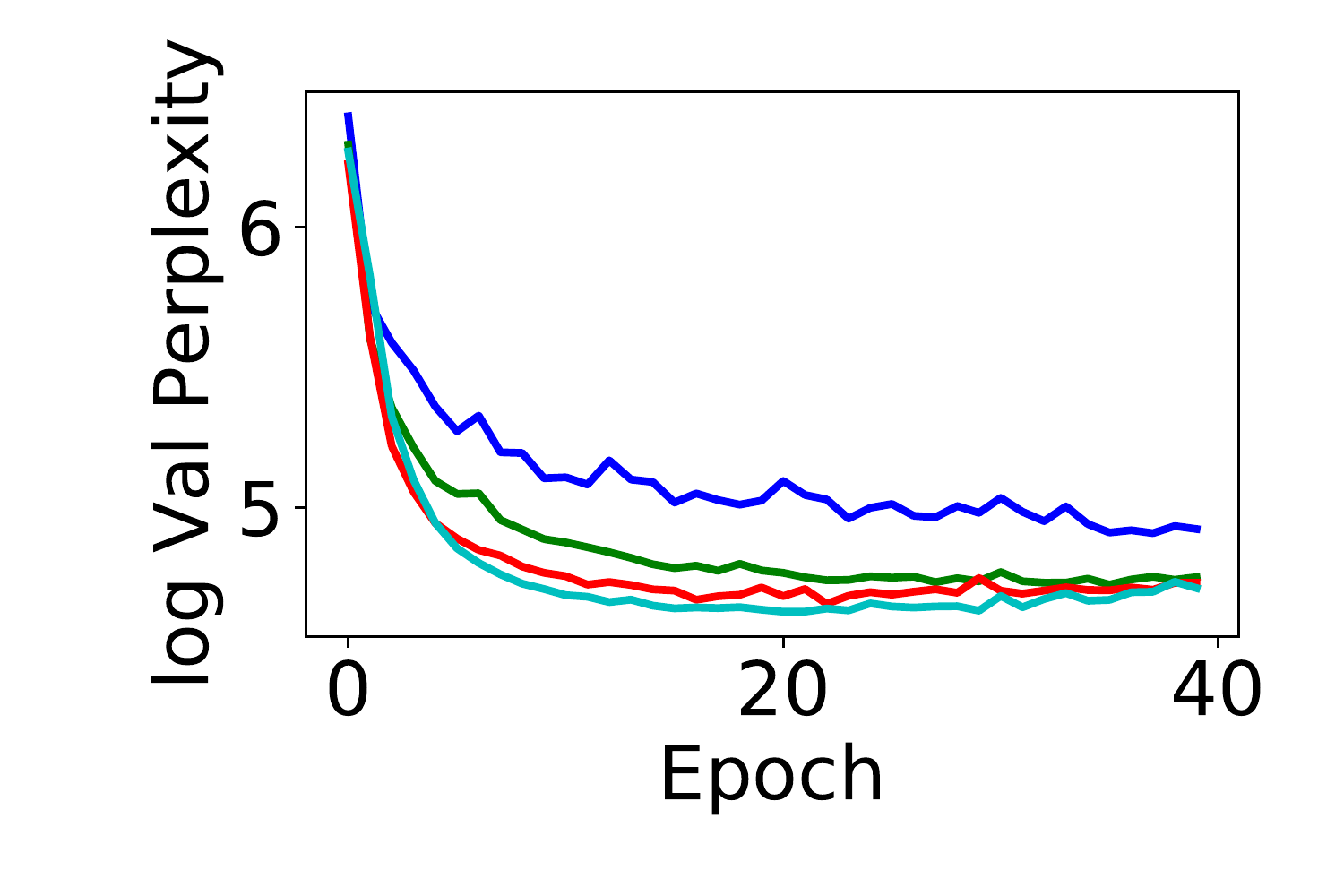}
    \includegraphics[width=0.33\linewidth]{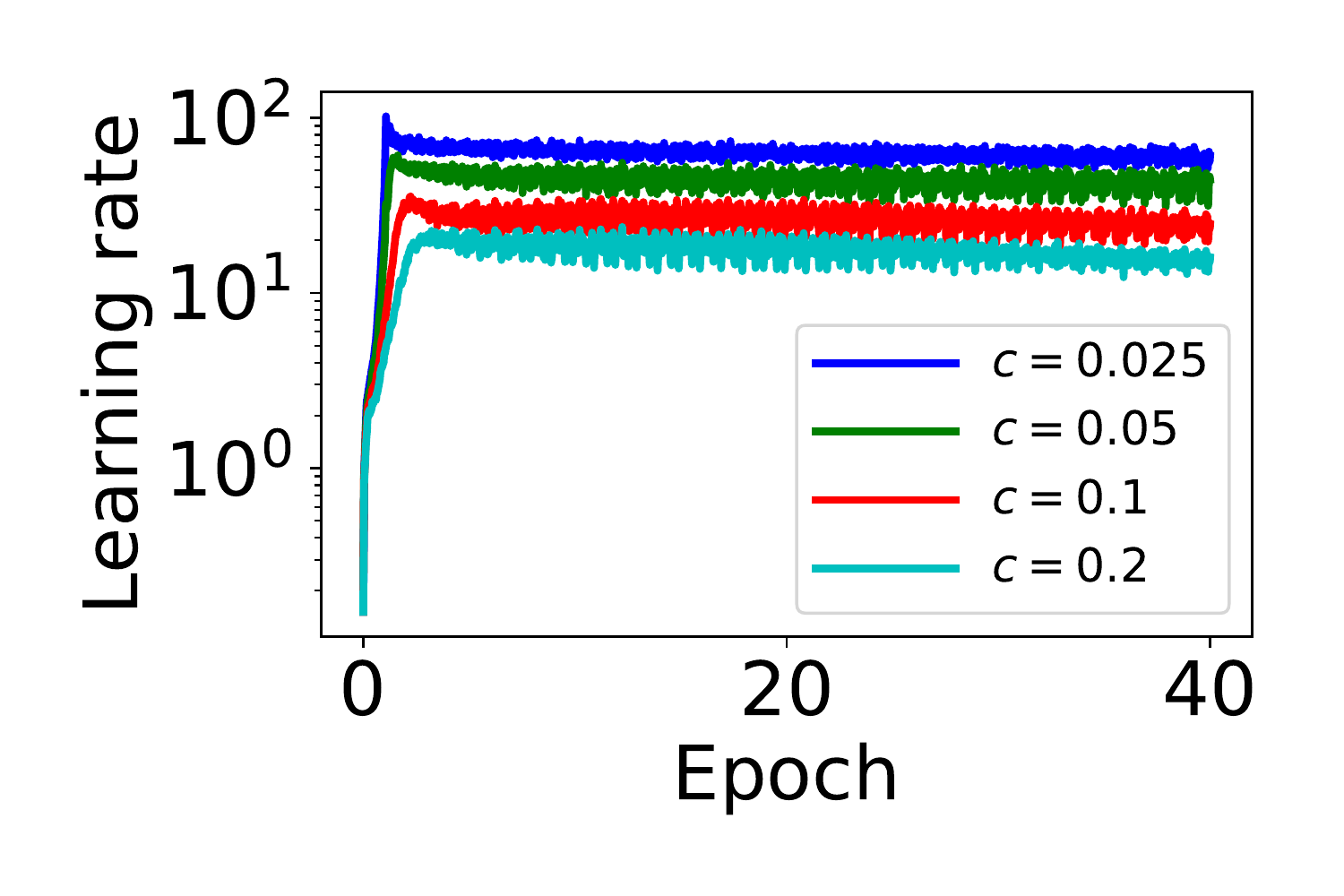}
    \caption{Wikitext-2: Sensitivity of SSLS to the initial learning rate $\alpha_0$ (first row), the smoothing factor $\gamma$ (second row) and the sufficient decrease constant $c$ (third row).}
    \label{fig:rnn_ssls_ablation}
\end{figure*}